\documentclass{article}

\usepackage{arxiv-a4}
\usepackage[utf8]{inputenc} % allow utf-8 input
\usepackage[T1]{fontenc}    % use 8-bit T1 fonts
\usepackage[hyperindex,breaklinks]{hyperref}
\hypersetup{
  breaklinks=true,   % splits links across lines
  colorlinks=false,   % displays links as colored text instead of blocks
  pdfusetitle=true,  % \title and \author values into pdf metadata
}
\usepackage{url}            % simple URL typesetting
\usepackage{booktabs}       % professional-quality tables
\usepackage{amsfonts}       % blackboard math symbols
\usepackage{nicefrac}       % compact symbols for 1/2, etc.
\usepackage{microtype}      % microtypography
\def\nev{\sc NEV} % \mathrm
\def\nes{\sc NES} % \mathrm

\author{
  Aurélien Delage\\
  CITI\\
  INSA Lyon \\
  Villeurbanne\\
  \texttt{aurelien.delage@insa-lyon.fr} \\
   \And
  Olivier Buffet \\
  INRIA - CNRS \\
  Université de Lorraine \\
  Villers-lès-Nancy\\
  \texttt{olivier.buffet@inria.fr} \\
  \And
  Jilles S. Dibangoye \\
  CITI \\
  INSA Lyon \\
  Villeurbanne\\
  \texttt{jilles-steeve.dibangoye@inria.fr} \\
  \And
  Abdallah Saffidine \\
  University of New South Wales \\
  Sidney \\
  \texttt{abdallahs@cse.unsw.edu.au}
}
\newcommand{\ifextended}[2]{\ifdefined\extended{#1}\else{#2}\fi}

\newcommand{\ifArxiv}[2]{\ifdefined\Arxiv{#1}\else{#2}\fi}
\def\Arxiv{1}

\newcommand{\labelT}[1]{\label{#1}} % fake label command introduced for theorems, lemmas, ...

\ifdefined\extended
\newcommand{\extref}[1]{\ref{#1}}

\newcommand{\extCshref}[1]{\Cshref{#1}}

\else
\usepackage{xr}
\newcommand{\extref}[1]{\ref{ext-#1}}
\makeatletter

\newcommand{\extCshref}[1]{\Cshref{ext-#1}}

\makeatother
\fi

\usepackage{booktabs}
\usepackage{bookmark}

\newcommand\Crefpage[1]{\Cref{#1}, p.\,\pageref{#1}}

\newcommand\cshrefpage[1]{\cshref{#1}, p.\,\pageref{#1}}

\usepackage{pgfplots}
\pgfplotsset{compat=1.7}
\usepgfplotslibrary{fillbetween}
\usetikzlibrary{patterns}

\usepackage[most]{tcolorbox}

\usepackage{booktabs}
\usepackage{multirow}

\usepackage{siunitx}
\DeclareSIUnit{\minute}{\text{m}}

\usepackage{xspace}
\usepackage{graphicx}
\graphicspath{{./figures/}{./}{../ICML-21/figures/}{./logs2}} %{../JFPDA-18/figures/}{./JFPDA-18/}}

\usepackage[inline]{enumitem}
\setdescription{leftmargin=\widthof{[$\beta^i_\depth$]~}} %\parindent}

\usepackage{tabularx}
\usepackage{rotating}
\usepackage{pifont}
\usepackage[normalem]{ulem}
\usepackage{verbatim}
\usepackage[most]{tcolorbox}

\usepackage{scalefnt}
\usepackage{amsmath} % already included
\usepackage{amsfonts} % doesn't provide \cal, ...
\usepackage{etoolbox}
\newcommand{\mybfseries}{\fontseries{b}\selectfont}
\newrobustcmd{\BS}{\mybfseries}

\makeatletter
\DeclareRobustCommand*\cal{\@fontswitch\relax\mathcal}
\makeatother
\usepackage{amsthm}

\usepackage{amssymb}
\usepackage{bbold} % for the ``bold'' 1
\usepackage{empheq}
\usepackage{cancel}
\newsavebox\CBox
\newcommand\hcancel[2][.5pt]{%
  \ifmmode\sbox\CBox{$#2$}\else\sbox\CBox{#2}\fi%
  \makebox[0pt][l]{\usebox\CBox}%  
  \rule[0.2em-#1/2]{\wd\CBox}{#1}}

\usepackage{xpatch}
\usepackage{thmtools}
\usepackage{apptools}
\makeatletter
\xpatchcmd{\thmt@restatable}% Edit \thmt@restatable
{\csname #2\@xa\endcsname\ifx\@nx#1\@nx\else[{#1}]\fi}% Replace this code
{\IfAppendix{\csname #2\@xa\endcsname}{\csname #2\@xa\endcsname\ifx\@nx#1\@nx\else[{#1}]\fi}}% with this code
{}{} % execute code for success/failure instances
\makeatother

\ifdefined\extended
\def\nonumber{}
\else
\fi
\usepackage{xfp}

\allowdisplaybreaks
\usepackage{centernot}
\DeclareMathOperator*{\argmin}{arg\,min}
\DeclareMathOperator*{\argmax}{arg\,max}

\DeclareMathOperator{\E}{\mathbb{E}}
\DeclarePairedDelimiter\ceil{\lceil}{\rceil}

\newcommand{\twodots}{\mathinner {\ldotp \ldotp}}
\usepackage{xspace}
\def\ie{{\em i.e.}\@\xspace}
\def\eg{{\em e.g.}\@\xspace}
\def\cf{{\em cf.}\@\xspace}

\def\reals{{\mathbb R}}
\def\hmm{\text{\sc hmm}\xspace}
\def\cS{{\cal S}}
\def\cA{{\cal A}}
\def\cB{{\cal B}}

\def\cZ{{\cal Z}}
\def\occ{\sigma}

\def\os{{\sc os}\xspace}
\def\aoh{{\sc aoh}\xspace}

\def\Occ{{\cal O}^\sigma}
\newcommand{\PP}[4]{P_{#2}^{#4}(#3|#1)}

\def\cI{{\cal I}}
\def\cJ{{\cal J}}

\def\l{\lambda}
\providecommand{\lt}{}
\renewcommand\lt[1]{\lambda_{#1}} %{\lambda_{H-#1}}
\def\p{1} % p value of p-norm
\def\WUL{W} %\norm{U^{(0)}-L^{(0)}}_\infty}
\def\va{{\boldsymbol{a}}}

\def\vz{{\boldsymbol{z}}}
\def\vth{{\boldsymbol{\theta}}}
\def\vTh{{\boldsymbol{\Theta}}}
\def\vpi{{\boldsymbol{\pi}}}
\def\vbeta{{\boldsymbol{\beta}}}

\def\dr{{\sc dr}\xspace}
\def\vmu{{\boldsymbol{\mu}}}
\def\cM{{\cal M}}
\def\vx{{\boldsymbol{x}}}
\def\vy{{\boldsymbol{y}}}

\def\vv{{\boldsymbol{v}}}

\def\xx{x}
\def\yy{y}
\def\zz{z}
\def\nev{NEV} % \mathrm
\def\nes{NES} % \mathrm
\newcommand{\lp}[1]{\textsc{lp}#1}
\newcommand{\dlp}[1]{\textsc{dlp}#1}

\def\width{\mathit{width}}

\def\depth{\tau}
\def\supp{\mathit{Supp}}
\def\radius{\rho}
\def\thr{thr}
\def\indep{\perp \!\!\! \perp}

\def\l{\lambda}
\def\thr{thr}

\newcommand{\Tm}[1]{T^{#1}_m}
\newcommand{\Tc}[1]{T^{#1}_c}
\def\vpi{\boldsymbol \pi}

\def\vr{r}

\def\Lin{Lin}
\def\zsomg{zs-oMG\xspace}

\newcommand{\upbW}[2]{\overline{W}_{#1}}
\newcommand{\lobW}[2]{\underline{W}_{#1}}

\def\omgHSVIlccc{{\sc omg}-HSVI}
\usepackage[ruled,vlined,linesnumbered]{algorithm2e}
\SetKwProg{Fct}{Fct}{}{}

\usepackage{float}

\newtheorem{lemma}{Lemma}

\newcommand{\eqdef}     {\stackrel{{\textrm{\rm\tiny def}}}{=}}

\newcommand{\poubelle}[1]{}
\newcommand{\temporallyHidden}[1]{}
\newcommand\upb[1]{\overline{#1}} % upperbound (majorant)
\newcommand\lob[1]{\underline{#1}} % lowerbound (minorant)
\newcommand{\h}[3]{h_{#2}} % {h_{#2}^{#1,#3}}
\ifdefined\noproofs
\usepackage{environ}
\NewEnviron{killcontents}{}

\else{}\fi

\makeatletter % ~ \hfill
\newcommand{\pushright}[1]{\ifmeasuring@#1\else\omit\hfill$\displaystyle#1$\fi\ignorespaces}
\newcommand{\pushleft}[1]{\ifmeasuring@#1\else\omit$\displaystyle#1$\hfill\fi\ignorespaces}
\newcommand{\specialcell}[1]{\ifmeasuring@#1\else\omit$\displaystyle#1$\ignorespaces\fi}
\makeatother

\DeclarePairedDelimiter{\abs}{\lvert}{\rvert}%
\DeclarePairedDelimiter{\norm}{\lVert}{\rVert}
\usepackage{esvect}
\newcommand\vabs[1]{\vv{\abs{#1}}}
\newcommand\vnorm[1]{\vv{\norm{#1}}}
\def\vleq{\ \vec\leq\ }

\def\player{}

\def\vmu{\boldsymbol \mu}

\usepackage[nottoc,notlot,notlof]{tocbibind} % add References to ToC

\usepackage{pgf}%[version=0.96]{pgf}
\usepackage{tikz}

\usetikzlibrary{trees}
\usetikzlibrary{matrix,arrows.meta}
\usepackage{tikz-cd}
\usetikzlibrary{arrows}
\usetikzlibrary{arrows,decorations,shapes,automata,backgrounds,petri,calc,shadows,shadings,shapes.symbols,shapes.geometric}

\tikzset{
    bus/.style={draw, circle, minimum size=2em, inner sep=0pt},
    pienode/.style n args={5}{
    #3, minimum size=#1, 
    draw=red, text=red,
    inner sep=0pt,
    path picture={
        \fill[#4] (path picture bounding box.center) -- ++ (0:#1) arc[start angle=0, end angle=3.6*#2, radius=#1]
        --cycle;
        \fill[#5] (path picture bounding box.center) -- ++ (3.6*#2:#1) arc[start angle=3.6*#2, end angle=360, radius=#1]
        --cycle;}}}

\tikzset{
    double color fill third/.code 2 args={
        \pgfdeclareverticalshading[%
            tikz@axis@top,tikz@axis@middle,tikz@axis@bottom%
        ]{raimbow}{100bp}{%
            color(0bp)=(tikz@axis@bottom);
            color(50bp)=(tikz@axis@bottom);
            color(0.01mm)=(tikz@axis@middle);
            color(50bp)=(tikz@axis@top);
            color(100bp)=(tikz@axis@top)
        }
        \tikzset{shade, left color=#1, right color=#2, shading=raimbow}
    }
}

\tikzset{
    double color fill second/.code 2 args={
        \pgfdeclareverticalshading[%
            tikz@axis@top,tikz@axis@middle,tikz@axis@bottom%
        ]{raimbow}{100bp}{%
            color(0bp)=(tikz@axis@bottom);
            color(33 bp)=(tikz@axis@bottom);
            color(0.01mm)=(tikz@axis@middle);
            color(33bp)=(tikz@axis@top);
            color(100bp)=(tikz@axis@top)
        }
        \tikzset{shade, left color=#1, right color=#2, shading=raimbow}
    }
}
\tikzset{
    double color fill first/.code 2 args={
        \pgfdeclareverticalshading[%
            tikz@axis@top,tikz@axis@middle,tikz@axis@bottom%
        ]{raimbow}{100bp}{%
            color(0bp)=(tikz@axis@bottom);
            color(100bp)=(tikz@axis@bottom)
        }
        \tikzset{shade, left color=#1, right color=#2, shading=raimbow}
    }
}

\tikzset{
    bus/.style={draw, circle, minimum size=2em, inner sep=0pt},
    pienode/.style n args={5}{
    #3, minimum size=#1, 
    draw=red, text=red,
    inner sep=0pt,
    path picture={
        \fill[#4] (path picture bounding box.center) -- ++ (0:#1) arc[start angle=0, end angle=3.6*#2, radius=#1]
        --cycle;
        \fill[#5] (path picture bounding box.center) -- ++ (3.6*#2:#1) arc[start angle=3.6*#2, end angle=360, radius=#1]
        --cycle;}}}

\tikzset{
    pics/circle vertically split/.style 2 args = {
       code = {
         \node[inner sep=3pt,left] (-left) {#1};
         \node[inner sep=3pt,right] (-right) {#2};
         \node[inner sep=3pt,right] (-top) {$e$};
         \path let
              \p1 = ($(-left.north west) - (-left.east)$),
              \p2 = ($(-right.west) - (-right.south east)$),
              \p3 = ($(-top.west)-(-top.east)$),
              \n1 = {max(veclen(\p1), veclen(\p2))*2}
           in node[minimum size=\n1, circle, draw] (-shape) at (0,0) {};
         \draw (-shape.west) -- (-shape.east);
         \draw (-shape.center) -- (-shape.north);

       }
    }
}
\usetikzlibrary{trees}
\usetikzlibrary{positioning}
\usepackage{adjustbox}

\usepackage[color]{changebar}

\usepackage{blkarray}

\definecolor{pink}{rgb}{0.858, 0.188, 0.478}

\newcommand{\persComment}[3]{
  \ifmmode
  \text{\textcolor{#3}{[#2] #1}}
  \else
  \textcolor{#3}{[#2] \em #1}
  \fi
}
\newcommand{\Olivier}[1]{\persComment{#1}{ob}{teal}} % OliveGreen}} %DELETE
\newcommand{\Aurelien}[1]{\persComment{#1}{ad}{pink}} %DELETE
\newcommand{\olivier}[1]{\Olivier{#1}} %DELETE
\newcommand{\aurelien}[1]{\Aurelien{#1}} %DELETE
\usepackage{comment}

\usepackage[noabbrev]{cleveref}
\DeclareRobustCommand{\abbrevcrefs}{%
\Crefname{appendix}{App.}{Apps.}%
\Crefname{section}{Sec.}{Secs.}%
\Crefname{equation}{Eq.}{Eqs.}%
\Crefname{figure}{Fig.}{Figs.}%
\Crefname{algorithm}{Alg.}{Algs.}%
\Crefname{tabular}{Tab.}{Tabs.}%
\Crefname{table}{Tab.}{Tabs.}%
\Crefname{lemma}{Lem.}{Lems.}%
\Crefname{corollary}{Cor.}{Cors.}%
\Crefname{theorem}{Thm.}{Thms.}%
\Crefname{proposition}{Prop.}{Props.}%
\Crefname{line}{L.}{Ls.}%
\crefname{appendix}{app.}{apps.}%
\crefname{section}{sec.}{secs.}%
\crefname{equation}{eq.}{eqs.}%
\crefname{figure}{fig.}{figs.}%
\crefname{algorithm}{alg.}{algs.}%
\crefname{tabular}{tab.}{tabs.}%
\crefname{table}{tab.}{tabs.}%
\crefname{lemma}{lem.}{lems.}%
\crefname{corollary}{cor.}{cors.}%
\crefname{theorem}{thm.}{thms.}%
\crefname{proposition}{prop.}{props.}%
\crefname{line}{l.}{ls.}%
}
\DeclareRobustCommand{\cshref}[1]{{\abbrevcrefs\cref{#1}}}
\DeclareRobustCommand{\Cshref}[1]{{\abbrevcrefs\Cref{#1}}}
\def\CFR{CFR+}

\usepackage{slashbox}

\usepackage{tabularx}
\usepackage{tablefootnote}

\usetikzlibrary{tikzmark}
\usepackage{tabstackengine}[2016-11-30]
\usepackage{longtable}

\def\tree{\psi}

\usepackage{wrapfig}
\theoremstyle{plain}
\newtheorem{theorem}{Theorem}[section]
\newtheorem{proposition}[theorem]{Proposition}

\theoremstyle{definition}
\newtheorem{definition}[theorem]{Definition}

\theoremstyle{remark}
\newtheorem{remark}[theorem]{Remark}
\newtheorem{example}{Example}%
\usepackage[american]{babel}
\usepackage{natbib} % has a nice set of citation styles and commands
\usepackage{mathtools} % amsmath with fixes and additions
\usepackage{booktabs} % commands to create good-looking tables
\usepackage{tikz} % nice language for creating drawings and diagrams

\title{HSVI can solve\\  zero-sum Partially Observable Stochastic Games}

\begin{document}
\maketitle

\begin{abstract}
State-of-the-art methods for solving 2-player zero-sum imperfect information games rely on linear programming or regret minimization, though not on dynamic programming (DP) or heuristic search (HS), while the latter are often at the core of state-of-the-art solvers for other sequential decision-making problems. In partially observable or collaborative settings (\eg, POMDPs and Dec-POMDPs), DP and HS require introducing an appropriate statistic that induces a fully observable problem as well as bounding (convex) approximators of the optimal value function.
This approach has succeeded in some subclasses of 2-player zero-sum partially observable stochastic games (zs-POSGs) as well, but how to apply it in the general case still remains an open question.
% \sout{in the general case}.
  %
%  \aurelien{prop : We answer it by (i) rigorously defining an game, (ii) proving mathematical properties of the value function that allow us to derive for the first time a generic and prototypical HSVI-like solver that [...].}
%  \aurelien{jilles : contributions (i) re-formulation rigoureuse (ii) "algorithmique" (les W) (iii) hsvi (premier algo générique de cette famille, tps fini, solution safe) (iv) perfs par rapport à Wiggers et al. }
%\olivier{{\bf [best of]}
  We answer it by (i) rigorously defining an equivalent game to work with, (ii) proving mathematical properties of the optimal value function that allow deriving bounds that come with solution strategies, (iii) proposing for the first time an HSVI-like solver that provably converges to an $\epsilon$-optimal solution in finite time, and (iv) empirically analyzing it.
  This opens the door to a novel family of promising approaches
  complementing those relying on linear programming or iterative
  methods.
\end{abstract}

%!TEX encoding = UTF-8 Unicode
%!TEX root = ./article.tex

%\message{<<< Entering \currfilename <<<}

%\tableofcontents 

\section{Introduction}
\label{sec|introduction}

Solving imperfect information sequential games is a challenging field with many applications from playing Poker \citep{Kuhn-ctg50} to security games \citep{BasNitGat-aaai16}.
We focus on finite-horizon 2-player zero-sum partially observable stochastic games (zs-POSGs), an important class of games that is equivalent to that of zero-sum extensive-form games (zs-EFGs) \citep{OliVla-tr06}\footnote{Note: POSGs are equivalent to the large class of ``well-behaved'' EFGs as defined by \citet{KovSBBL-corr20}.}.
From the viewpoint of (maximizing) player~$1$, we aim at finding a strategy with a worst-case expected return (\ie, whatever player~$2$'s strategy) within $\epsilon$ of the Nash equilibrium value (NEV).

A first approach to solving a zs-POSG is to turn it into a zs-EFG 
addressed as a {\em sequence form} linear program  (SFLP) \citep{KolMegSte-geb96,Stengel-geb96,BosEtAl-jair14}, giving rise to an exact algorithm.
A second approach is to use an iterative game solver, \ie, %
either a counterfactual-regret-based method (CFR) \citep{ZinJohBowPic-nips07,BroSan-science18}, %
or a first-order method \citep{HodGilPenSan-mor10,KroWauKKSan-mp20},
both coming with asymptotic convergence properties.
CFR-based approaches now incorporate deep reinforcement learning and search, some of them winning against top human players at heads-up no limit hold’em poker \citep{MorEtAl-science17,BroSan-science18,BroBakLerQon-nips20}.
A third approach, proposed by \citet{Wiggers-msc15}, is to use two parallel searches in strategy space, one per player, so that the gap between both strategies' security levels (\ie, the values of their opponent's best responses) bounds the distance to the NEV. %Nash equilibrium value.

In contrast, dynamic programming and heuristic search have not been
applied to general zs-POSGs, while often at the core of
state-of-the-art solvers in other problem classes that involve
Markovian dynamics, partial observability and multiple agents (POMDP
\citep{Astrom-jmaa65,Smith-phd07}, Dec-POMDP
\citep{SzeChaZil-uai05,DibAmaBufCha-jair16}, or subclasses of
zs-POSGs with simplifying observability assumptions
\citep{GhoMcDSin-jota04,ChaDoy-tcl14,BasSte-jco15,HorBosPec-aaai17,ColKoc-jet01,HorBos-aaai19}).
They all rely on some statistic that induces a fully observable problem whose value function ($V^*$) exhibits continuity properties that allow deriving bounding approximations.
\citet{WigOliRoi-ecai16,WigOliRoi-corr16} progress in this direction for zs-POSGs by demonstrating an important continuity property of the optimal value function, and proposing a reformulation as a particular equivalent game.
%{
We work in a similar direction,
\begin{enumerate*}
\item using a game with different observability hypotheses, %
\item proving theoretical results they implicitly rely on, and
\item building on some of their results to derive an HSVI-like algorithm solving the zs-POSG.
\end{enumerate*}
%}

\Cref{sec|background} presents some necessary background, including the concept of {\em occupancy state} \citep{DibAmaBufCha-jair16,WigOliRoi-corr16}
(\ie, the probability distribution over the players' past action-observation histories), and properties that rely on it.
Then, \Cref{sec|theory} describes theoretical contributions.
First, we rigorously reformulate the problem as a non-observable game, and demonstrate that the Nash equilibrium value can be expressed with a recursive formula, which is a required tool for DP and HS
(\Cshref{sec|oMGdpp}).
Second, we exhibit novel continuity properties of optimal value functions and derive bounding approximators, a second tool made necessary due to the continuous state space of the new game, before showing that these approximators come with valid solution strategies for the zs-POSG
(\Cshref{sec|solvingOMGs}).
Third, we adapt \citeauthor{SmiSim-uai05}' \citep{SmiSim-uai05} HSVI's algorithmic scheme to $\epsilon$-optimally solve the problem in finitely many iterations %, 
(\Cshref{sec|HSVI}).
\Cref{sec|XPs} presents an empirical analysis of the approach.
\Cref{sec|CFR} discusses similarities and differences of our work with CFR-based continual resolving methods before concluding.
% =============================

\section{Background}
\labelT{sec|background}

Here, we first give basic definitions about zs-POSGs, including
the solution concept at hand.
Then we introduce an equivalent game where a state corresponds to a statistic summarizing past behaviors, which leads to some important properties of the game's optimal value.

\subsection{zs-POSGs}

\begin{definition}[zs-POSGs]
As illustrated through a dynamic influence diagram in \Cref{fig:influenceDiagram},
a (2-player) zero-sum partially observable stochastic game (zs-POSG) is defined by a tuple
$\langle \cS, \cA^1, \cA^2, \cZ^1, \cZ^2, P, r, H, \gamma, b_0 \rangle$, where
\begin{itemize} \item %
  $\cS$ is a finite set of states;
  \item %
  $\cA^i$ is (player) $i$'s finite set of actions;
  \item  %
  $\cZ^i$ is \player $i$'s finite set of observations;
  \item %
  $\PP{s}{a^1,a^2}{s'}{z^1,z^2}$ is the probability to transition
  to state $s'$ and receive observations $z^1$ and $z^2$ when actions
  $a^1$ and $a^2$ are performed while in state $s$;
  \item %
  $r(s,a^1,a^2)$ is a (scalar) reward function;
  \item %
  $H \in \mathbb{N}$ is a (finite) temporal horizon;
  \item %
  $\gamma\in [0,1]$ is a discount factor; and %
  \item %
  $b_0$ is the initial belief state, \ie, a probability distribution over states at $t=0$.
  \end{itemize}%
\end{definition}

\begin{figure}[ht]
\centering
\adjustbox{width = 1.0\linewidth}{
\begin{tikzpicture}[->,>=triangle 45,shorten >=2pt,auto,node distance=4cm,semithick]
  \tikzstyle{every state}=[draw=black, text=black, inner color=white, outer color=white, draw=black, text=black]
  \tikzstyle{place circle blue}=[circle,thick,draw=blue!50!black,fill=blue!20,minimum size=6mm]
  \tikzstyle{place circle red}=[circle,thick,draw=red!50!black,fill=blue!20,minimum size=6mm]
  %\tikzstyle{red place}=[place,square,draw=red!50!black,fill=red!20]
  \tikzstyle{red place}=[place,circle,draw=red!50!black,fill=red!20]
  \tikzstyle{blue place}=[place,regular polygon,regular polygon sides=4,draw=blue!50!black,fill=green!20]
  \tikzstyle{red place}=[place,regular polygon,regular polygon sides=4,draw=red!50!black,fill=green!20]
  
  \draw[rounded corners, red, fill=red!10] (-2,-1) rectangle (12,1);
  
  \node[initial,state,place,fill=gray!20] (S0)               {$\ S_t\ $};
  \node[state,place,fill=gray!20]         (S1) [right of=S0] {$S_{t+1}$};
  \node[state,place,fill=gray!20]         (S2) [right of=S1] {$S_{t+2}$};
  \node[]         (S3) [ right of=S2] {};

  \node[state,place circle red, scale=1]         (O0) [above  of=S0,node distance=2.3cm] {$\ Z^1_t\ $};
  \node[state,place circle red, scale=1]         (O1) [above  of=S1,node distance=2.3cm] {$Z^1_{t+1}$};
  \node[state,place circle red, scale=1]         (O2) [above  of=S2,node distance=2.3cm] {$Z^1_{t+2}$};
  \node[state,place circle blue, scale=1]         (O3) [below  of=S0,node distance=2.3cm] {$\ Z^2_t\ $};
  \node[state,place circle blue, scale=1]         (O4) [below  of=S1,node distance=2.3cm] {$Z^2_{t+1}$};
  \node[state,place circle blue, scale=1]         (O5) [below  of=S2,node distance=2.3cm] {$Z^2_{t+2}$};

 \node[state,red place]         (A0) [above right of=O0,node distance=2cm] {$\ A^1_t\ $};
 \node[state,red place]         (A1) [right of=A0]       {$A^1_{t+1}$};
 \node[]         (A2) [right of=A1]       {};
 \node[state,blue place]         (A3) [below right of=O3,node distance=2cm] {$\ A^2_t\ $};
 \node[state,blue place]         (A4) [right of=A3]       {$A^2_{t+1}$};
 \node[]         (A5) [right of=A4]       {};

 \node[]         (Time) at (-1,-5) {Time};
 \node[]         (T0) [below  of=A3,node distance=1.25cm] {$t$};
 \node[]         (T1) [below  of=A4,node distance=1.25cm] {$t+1$};
 \node[]         (T2) [below  of=A5,node distance=1.25cm] {$t+2$};

 \node[]         (N0) at (3,+4.75) {};
 \node[]         (N1) at (3,-5.75) {};
 \draw[-,dotted] (N0)-- (N1);
 \node[fill=red!10,text=black,draw=none,scale=.7]  at (1.5,+.25) {$r(s,a^1,a^2)$};
 \node[fill=red!10,text=black,draw=none,scale=.7]  at (1.7,-.25) {$P^{z^1,z^2}_{a^1,a^2}(s'|s)$};
% \node[fill=white,text=black,draw=none,scale=.7]  at (-.75,+1.35) {$o^{a,z}(s')$};
% \node[fill=white,text=black,draw=none,scale=.7]  at (-.75,-1.35) {$o^{a,z}(s')$};
 \draw[-,dotted] (-2,-1.1)--(12,-1.1);
 \draw[-,dotted] (-2,1.1)--(12,1.1);
 \node[]         (N2) at (7,+4.75) {};
 \node[]         (N3) at (7,-5.75) {};
 \draw[-,dotted] (N2)--(N3);
 \node[fill=red!10,text=black,draw=none,scale=.7]  at (5.5,+.25) {$r(s,a^1,a^2)$};
 \node[fill=red!10,text=black,draw=none,scale=.7]  at (5.7,-.25) {$P^{z^1,z^2}_{a^1,a^2}(s'|s)$};
 % \node[fill=white,text=black,draw=none,scale=.7]  at (3.25,-1.35) {$o^{a,z}(s')$};
 % \node[fill=white,text=black,draw=none,scale=.7]  at (3.25,+1.35) {$o^{a,z}(s')$};
 % \node[fill=white,text=black,draw=none,scale=.7]  at (7.25,-1.35) {$o^{a,z}(s')$};
 % \node[fill=white,text=black,draw=none,scale=.7]  at (7.25,+1.35) {$o^{a,z}(s')$};
 \node[fill=red!10,text=black,draw=none,scale=.8]  at (11,-.5) {Hidden};
 \node[fill=white,text=black,draw=none,scale=1]  at (10.75,+2.5) {$1$'s viewpoint};
 \node[fill=white,text=black,draw=none,scale=1]  at (10.75,-2.5) {$2$'s viewpoint};
  
  \path (S0) edge              node {} (S1)
            edge    node {} (O0)
            edge    node {} (O3)
        (S1) edge node {} (S2)
            edge    node {} (O1)
            edge    node {} (O4)
        (S2) edge              node {} (S3)
            edge    node {} (O2)
            edge    node {} (O5)
        (A0) edge [out=-90, in=-180]  node {} (O1)
            edge  [out=-90, in=-205] node {} (S1)
        (A1) edge [out=-90, in=-180]  node {} (O2)
            edge  [out=-90, in=-205] node {} (S2)
        (A3) edge [out=90, in=-180]  node {} (O4)
            edge  [out=90, in=-155] node {} (S1)
        (A4) edge [out=90, in=-180]  node {} (O5)
            edge  [out=90, in=-155] node {} (S2);
\end{tikzpicture}
}
\caption{Dynamic influence diagram representing the evolution of a zs-POSG}\label{fig:influenceDiagram}
\end{figure}
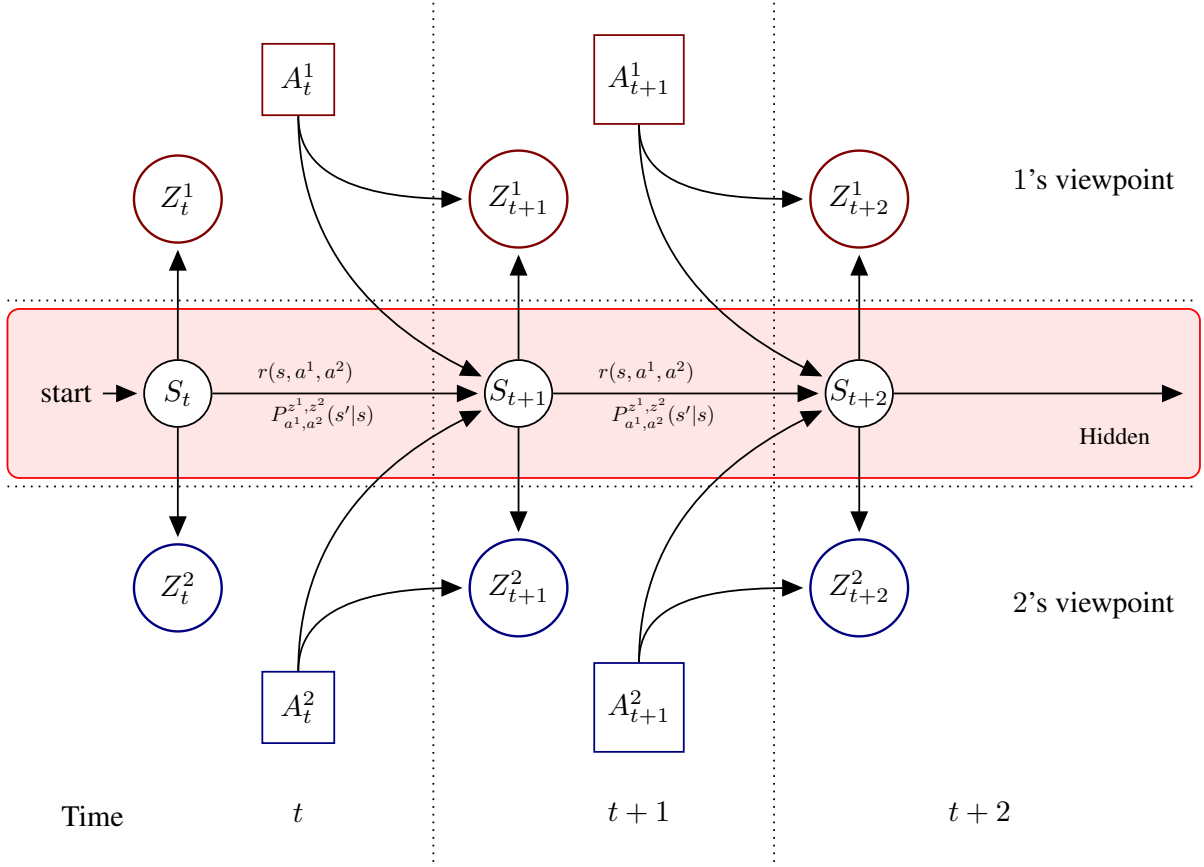

From the Dec-POMDP, POSG and EFG literature, we use the
following concepts and definitions: %, where $i \in \{1,2\}$:
\begin{description}[leftmargin=!,labelwidth=\widthof{$\vbeta_{\depth:\depth'}$},labelsep=.25em] % see enumitem doc
\item[$\theta^i_\depth $]
  $ = (a^i_0, z^i_1, \dots , a^i_{\depth-1}, z^i_\depth) $ 
  is a length-$\depth$ {\em action-observation history} (\aoh) for \player $i$. We note $\Theta_\depth^i$ the set of all \aoh{}s for player $i$ at horizon $\depth$ such that any \aoh $\theta_\depth^i$ is in $\cup_{t=0}^{H-1} \Theta^i_t$.
\item[$\beta^i_\depth$] is a {\em (behavioral) decision rule} (\dr) at
  $\depth$ for \player $i$, \ie, a mapping from private \aoh{}s in $\Theta^i_\depth$ to {\em distributions} over private actions.
  $\beta^i_\depth(\theta^i_\depth,a^i)$ is the probability to pick $a^i$ when facing $\theta^i_\depth$.
\item[$\beta^i_{\depth:\depth'}$]
  $= (\beta^i_\depth, \dots, \beta^i_{\depth'})$ is a {\em behavioral strategy} for \player $i$ from time step $\depth$ to $\depth'$ (included).
  \end{description}
   When considering both players, the last 3 concepts become:
  \begin{description}[resume,leftmargin=!,labelwidth=\widthof{$\vbeta_{\depth:\depth'}$},labelsep=.25em]
  % \begin{description}
  \item[$\vth_\depth$] $ =(\theta^1_\depth,\theta^2_\depth)$ %
  ($\in \vTh = \cup_{t=0}^{H-1} \vTh_t$), %
  a {\em joint \aoh} at $\depth$,
\item[$\vbeta_\depth$]
  $= \langle \beta^1_\depth, \beta^2_\depth \rangle$
  ($\in \cB = \cup_{t=0}^{H-1} \cB_t$), a {\em decision rule profile}, and
\item[$\vbeta_{\depth:\depth'}$]
  $= \langle \beta^1_{\depth:\depth'}, \beta^2_{\depth:\depth'} \rangle$, a {\em behavioral strategy profile}.
  %%%%%%%%%%%%%%%%%%%%%%%
\end{description}

\subsubsection*{Nash Equilibria}
%\paragraph{Nash Equilibria}

Here, player~$1$ (respectively $2$) wants to maximize (resp. minimize) the expected return, or {\em value}, of strategy profile $\vbeta_{0:H-1}$, defined as the discounted sum of future rewards, \ie,
%\ifextended{
    \begin{align*}
      V_0(\vbeta_{0:H-1}) 
      & = E\left[\sum_{t=0}^{H-1} \gamma^t R_t \mid \vbeta_{0:H-1}\right],
    \end{align*}
where $R_t$ is the random variable associated to the instant reward at $t$.
This leads to the solution concept of Nash equilibrium strategy (NES).
\begin{definition}[Nash Equilibrium]
The strategy profile
$\vbeta^*_{0:H-1}=\langle\beta^{1*}_{0:H-1},\beta^{2*}_{0:H-1}\rangle$ is a \nes{} if no player has an
incentive to deviate, which can be written: %(replacing ``$0:H-1$'' by ``$0:$''): 
\begin{align*}
    \forall \beta^1_{0:H-1},\ V_0(\beta^{1*}_{0:H-1}, \beta^{2*}_{0:H-1})
    & \geq V_0(\beta^{1}_{0:H-1}, \beta^{2*}_{0:H-1})
    \text{ and } \\
    \forall \beta^2_{0:H-1},\ V_0(\beta^{1*}_{0:H-1}, \beta^{2*}_{0:H-1}) 
    & \leq V_0(\beta^{1*}_{0:H-1}, \beta^{2}_{0:H-1}).
\end{align*}
\end{definition}
In such a game, all NESs have the same Nash-equilibrium value (NEV),
$V^*_0 \eqdef V_0(\beta^{1*}_{0:H-1}, \beta^{2*}_{0:H-1})$.
Our specific objective is to find an $\epsilon$-\nes{}, \ie, a behavioral strategy profile such that any player would gain at most $\epsilon$ by deviating.

\subsubsection*{Why writing a Bellman Optimality Equation is Hard}

Our approach requires writing Bellman optimality equations.
The main obstacle to achieve this is to find an appropriate characterization of a {\em subproblem} that allows
\begin{enumerate}
\item {\em predicting} both the immediate reward and the next possible subproblems given an immediate decision;
\item {\em connecting} a subproblem's solution with solutions of its own (lower-level) subproblems; and
\item {\em prescripting} a solution strategy for the subproblem built on solutions of lower-level subproblems.
\end{enumerate}

In our setting, a player's \aoh{} does not characterize a subproblem since her opponent's strategy is also required to predict the expected reward and the next \aoh{}s.
For their part, joint \aoh{}s allow predicting next joint \aoh{}s given both player's immediate decision rules, but would not be appropriate either, since player $i$ cannot decide how to act when facing some individual \aoh{} $\theta^i_\depth$ without considering all possible \aoh{}s of his opponent $\neg i$. 

Partial behavioral strategy profiles (sequences of behavioral decision rule profiles from $t=0$ to some $\depth$) contain enough information to completely describe the situation at $\depth$, and are thus necessarily {\em predictive}.
We still need to demonstrate that they are {\em connected}, despite decision rules not being public, and {\em prescriptive}, despite the need to address global-consistency issues illustrated in the following example.

\begin{example}
\label{ex|MP}
\def\none{z_n}

Matching pennies is a well-known 2-player zero-sum game in which each player has a penny and secretly chooses one side (head or tail).
Then, both penny's sides are revealed, and player $1$ wins (payoff of $+1$) if both chosen sides match and looses (payoff of $-1$) if not.

We here formalize this game as a zs-POSG (as illustrated in \Cref{fig|matchingPennies}) where player $1$ actually picks his action at $t=0$, and player $2$ at $t=1$.
Hence the tuple $\langle \cS, \cA^1, \cA^2, \cZ^1, \cZ^2, P, r, H, \gamma, b_0 \rangle$ where:
\begin{itemize}
  \item $\cS = \{s_i, s_h, s_t\}$, where $s_i$ is the initial state, and $s_h$ and $s_t$ represent a memory of $1$'s last move: respectively "head" or "tail";
  \item $\cA^1 = \cA^2 = \{a_h,a_t\}$ for playing "head" ($a_h$) or "tail" ($a_t$);
  \item $\cZ^1 = \cZ^2 = \{\none\}$ a "none" trivial observation;
  \item $\PP{s}{\va}{s'}{\vz}=T(s,\va,s')\cdot \mathcal{O}(\va,s',\vz)$, using the next two definitions;
  \item $T$ is deterministic and such that ($\cdot$ is used to denote "for all")
        \begin{itemize}
            \item $T(\cdot,\cdot,a_h) = s_h$,
            \item $T(\cdot,\cdot,a_t) = s_t$;
        \end{itemize}
  \item $\mathcal{O}$ is deterministic and always returns "$\none$";
  \item $r$ is such that 
        \begin{itemize}
        \item $r(s_i,\cdot,\cdot) = 0$,
        \item $r(s_t,\cdot,a_t) = r(s_h,\cdot,a_h) = +1$,
        \item $r(s_t,\cdot,a_h) = r(s_h,\cdot,a_t) = -1$; % was -2 for (s_h,a_t)
        \end{itemize}
    \item $H=2$;
    \item $\gamma = 1$;
    \item $b_0$ is such that the initial state is $s_i$ with probability $1$.
\end{itemize}

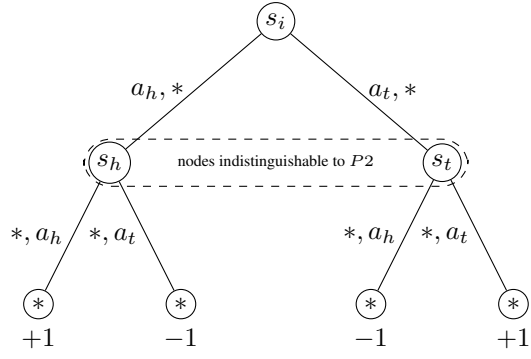
\begin{figure}
    \centering
    \tikzset{
% Two node styles for game trees: solid and hollow
solid node/.style={circle,draw,inner sep=1.5,fill=black},
hollow node/.style={circle,draw,inner sep=1.5}
}

\begin{tikzpicture}[scale=1.25] %,font=\footnotesize]
% Specify spacing for each level of the tree
\tikzstyle{level 1}=[level distance=15mm,sibling distance=35mm]
\tikzstyle{level 2}=[level distance=15mm,sibling distance=15mm] %7.5mm]
%\tikzstyle{level 3}=[level distance=15mm,sibling distance=4.5mm]
% The Tree
%\node(0)[solid node,label=above:{$P1$}]{}
\node(0)[hollow node]{$s_i$}
child{
    node(1)[hollow node]{$s_h$}
    child{
        node(3)[hollow node,label=below:{$+1$}]{$*$}
        edge from parent node[left,align=center]{$*,a_h$}
    }
    child{
        node(4)[hollow node,label=below:{$-1$}]{$*$}
        edge from parent node[left,align=center]{$*,a_t$}
    }
    edge from parent node[left]{$a_h,*$}
}
child{
    node(2)[hollow node]{$s_t$}
    child{
        node(5)[hollow node,label=below:{$-1$}]{$*$}
        edge from parent node[left,align=center]{$*,a_h$}
    }
    child{
        node(6)[hollow node,label=below:{$+1$}]{$*$}
        edge from parent node[left,align=center]{$*,a_t$}
    }
edge from parent node[right]{$a_t,*$}
    };

% information set
\draw[dashed,rounded corners=10]($(1) + (-.27,.25)$)rectangle($(2) +(.27,-.25)$);
% specify mover at 2nd information set
\node at ($(1)!.5!(2)$)[font=\tiny] {nodes indistinguishable to $P2$};
\end{tikzpicture}
    \caption{Simplified tree representation of the sequentialized Matching Pennies game.
    Irrelevant actions, noted $*$, allow merging edges with the same action for (i) player $2$ at $t=0$, and (ii) player $1$ at $t=1$.
    Notes:
    (a) Due to irrelevant actions, this game can be seen as an Extensive Form Game, despite players acting simultaneously.
    (b) Players only know about their past action history (in this observation-free game).}
    \label{fig|matchingPennies}
\end{figure}

Let us then assume that both players' \dr{}s at $t=0$ are fixed, with $\beta^1_0$ randomly picking $a_t$ or $a_h$ (\ie, it induces a NES whatever his \dr at $t=1$).
Then, we face a "subgame" at $t=1$ where any strategy profile $\langle \beta^1_{1:1}, \beta^2_{1:1} \rangle$ is a NES profile with Nash equilibrium value $0$.
In particular, $2$ can pick a deterministic strategy $\beta^2_{1:1}$, which will be said to be {\em locally consistent}.
Yet, for $2$, such a NES in the subgame at $\depth=1$ is not necessarily {\em globally consistent}, \ie, it may not be part of a NES for the original game (\ie, starting from $\depth=0$).
Intuitively, in such global-consistency issues \cite{KovEtAl-aij-2022,Schmid-phd21} (also called {\it safety issues} \cite{Burch-aaai-2014}),
the choices made at latter time steps do not account for possible deviations from the opponent at earlier time steps. 
\end{example}

As detailed in the next section,
we will characterize a subproblem not with the raw data of partial strategy profiles, but with a sufficient statistic,
and this characterization will be used as the state of a new dynamic game equivalent to the zs-POSG.

% =====================================

\subsection{Occupancy State and Occupancy Markov Game}

We now introduce an equivalent game, in which trajectories correspond to behavioral strategy profiles, and which we will be able to decompose temporally (and recursively), a first key tool for DP and HS.

To cope with the necessarily continuous nature of its state space, we will set this game in occupancy space, \ie, a statistic that sums up past \dr{} profiles. This will let us derive continuity properties on which to build point-based approximators.

As \citet{WigOliRoi-corr16}, let us formally define an {\em occupancy state} (\os) $\occ_{\vbeta_{0:\depth-1}}$
as the probability distribution over joint \aoh{}s $\vth_\depth$ given partial strategy profile $\vbeta_{0:\depth-1}$.
This statistic exhibits the usual Markov and sufficiency properties:

\begin{restatable}[Adapted from  
%\citep[Thm.~1]{DibAmaBufCha-jair16}
\citeauthor{DibAmaBufCha-jair16}
{\citep[Thm.~1]{DibAmaBufCha-jair16}} -- %Markov dynamics and rewards --
Proof in \Cshref{app|occStates}]{proposition}{lemOccSufficient}
  \IfAppendix{
    {\em (originally stated on page~\pageref{lem|occSufficient})}
  }{
      \labelT{lem|occSufficient}
  }
  $\occ_{\vbeta_{0:\depth-1}}$, together with $\vbeta_\depth$, is a sufficient statistic to compute %
  (i) the next \os, $T(\occ_{\vbeta_{0:\depth-1}},\vbeta_\depth) \eqdef \occ_{\vbeta_{0:\depth}}$, and %
  (ii) the expected reward at $\depth$:
  $r(\occ_{\vbeta_{0:\depth-1}},\vbeta_\depth) \eqdef \E\left[ R_\depth \mid \vbeta_{0:\depth-1} \oplus \vbeta_\depth \right]$, where $\oplus$ denotes a concatenation.
\end{restatable}

Writing from now on $\occ_\depth$, as short for $\occ_{\vbeta_{0:\depth-1}}$, the \os{}
associated with some prefix strategy profile $\vbeta_{0:\depth-1}$, the proof essentially relies on deriving the following formulas:
  $\forall \theta^1_\depth, a^1, z^1, \theta^2_\depth, a^2, z^2$,
  \begin{align}
    \label{eq|transition}
    T(\occ_\depth,\vbeta_\depth)((\theta^1_\depth,a^1,z^1),(\theta^2_\depth,a^2,z^2)) 
    \hspace{-2.5cm}
    \\ \nonumber
    & \eqdef Pr((\theta^1_\depth,a^1,z^1),(\theta^2_\depth,a^2,z^2) | \occ_\depth, \vbeta_\depth)
    \\ \nonumber
    & = \beta^1_\depth(\theta^1_\depth,a^1) \beta^2_\depth(\theta^2_\depth,a^2) \occ_\depth (\vth_\depth)
    \sum_{s,s'} P^{\vz}_{\va}(s'|s) b(s|\vth_\depth)
    ,
\intertext{where $b(s|\vth_\depth)$ is a {\em belief state} obtained by Hidden Markov Model filtering;
and}
    r(\occ_\depth,\vbeta_\depth) 
    \label{eq|reward}
    & \eqdef E[r(S,A^1,A^2) | \occ_\depth, \vbeta_\depth ] 
    \\ \nonumber %
    & = \sum_{s,\vth_\depth, \va} \occ_\depth(\vth_\depth) b(s | \vth_\depth)
    \beta^1_\depth(a^1|\theta^1_\depth) \beta^2_\depth(a^2|\theta^2_\depth) r(s,\va).
  \end{align}

We can then derive, from a zs-POSG, a non-observable zero-sum game similar to \citeauthor{WigOliRoi-corr16}'s {\em plan-time NOSG} \cite[Definition~4]{WigOliRoi-corr16}, but without assuming that the players' past strategies are public.

\begin{definition}[zero-sum occupancy Markov Game (zs-oMG)]
A {\em zero-sum occupancy Markov game} (\zsomg)\footnote{We use (i) ``Markov game'' instead of
  ``stochastic game'' because the dynamics are not stochastic, and
  (ii) ``partially observable stochastic game'' to stick with the
  literature.}
is defined by the tuple
$\langle \Occ, \cB, T, r, H, \gamma \rangle$, where:
\begin{itemize} \item %
  $\Occ (= \cup_{t=0}^{H-1} \Occ_t)$
  is the set of \os{}s induced by the zs-POSG;
  %, with one subset per time step;
  \item %
  $\cB$ % = \cB_0 \cup \cB_1 \cup \dots $
  is the set of \dr profiles of the zs-POSG;
  %\uline{--- a player's \dr{}s remaining {\em private}};
  % 
  \item %
  $T$ is the deterministic transition function in \Cshref{eq|transition};
  \item %
  $r$ is the reward function in \Cshref{eq|reward}; and
  \item %
  $H$ and $\gamma$ are as in the zs-POSG
  \end{itemize} %
  ($b_0$ is not in the tuple but serves to define $T$ and $r$).
\end{definition}
In this game, as in the zs-POSG, a player's solution is a behavioral strategy.
Besides, the value of a strategy profile $\vbeta_{0:H-1}$ is the same for both \zsomg and zs-POSG, so that they share the same $\epsilon$-\nev{} and $\epsilon$-\nes{}s.
We can thus work with zs-oMGs as a means to solve zs-POSGs.

The following aims at deriving a recursive expression of $V^*_0$, as well as continuity properties.

\subsubsection*{Bellman Optimality Equation}

Despite the \os{} at $\depth>0$ not being accessible to any player, let us define a {\em subgame} at $\occ_\depth$ as the restriction starting from time step $\depth$ under this particular occupancy state, meaning that we are seeking strategies $\beta^1_{\depth:H-1}$ and $\beta^2_{\depth:H-1}$.
$\occ_\depth$ tells us which \aoh{}s each player could be facing with non-zero probability, and are thus relevant for planning.
We can then define the value function in any
\os $\occ_\depth$ for any strategy profile $\vbeta_{\depth:H-1}$ as follows:
\begin{align}
  \label{eq|subgame}
  V_\depth(\occ_\depth,\vbeta_{\depth:H-1}) 
  & \eqdef E[\sum_{t=\depth}^\infty \gamma^{t-\depth} r(S_t,A_t) | \occ_\depth, \vbeta_{\depth:H-1}]. %\\
\end{align}

%\cbstart
The optimal value of a subgame rooted at $\occ_\depth$, $V^*(\occ_\depth)$, is thus the unique NEV for the previous criterion\footnote{We will come back to the validity of this point in \Cref{sec|oMGdpp}.}. \citeauthor{WigOliRoi-ecai16} then proved key continuity properties of $V^*$ discussed next.

\subsubsection*{Concavity and Convexity Results}

As a preliminary step, \citeauthor{WigOliRoi-corr16} decompose an occupancy state $\occ_\depth$ into a {\em marginal term} $\occ_\depth^{m,1}$ and a {\em conditional term} $\occ_\depth^{c,1}$, where
\begin{itemize}
    \item %a marginal term, 
    $\occ_\depth^{m,1}(\theta_\depth^1) = \sum_{\theta_\depth^2} \occ_\depth(\theta_\depth^1,\theta_\depth^2)$
    is the probability of $1$ facing $\theta_\depth^1$ under $\occ_\depth$, and
    \item %a conditional one,
    $\occ_\depth^{c,1}(\theta_\depth^2 | \theta_\depth^1) = \frac{\occ_\depth(\theta_\depth^1,\theta_\depth^2)}{\occ_\depth^{m,1}(\theta_\depth^1)}$
    is the probability of $2$ facing $\theta_\depth^2$ under $\occ_\depth$ given that $1$ faces $\theta_\depth^1$,
\end{itemize}
so that $\occ_\depth(\theta_\depth^1,\theta_\depth^2) = \occ_\depth^{m,1}(\theta_\depth^1) \cdot \occ_\depth^{c,1}(\theta_\depth^2 | \theta_\depth^1)$.
(Symmetric definitions apply by swapping players $1$ and $2$.)
In addition, let us denote $\Tm{1}(\occ_\depth,\vbeta_\depth)$ and $\Tc{1}(\occ_\depth,\vbeta_\depth)$ the marginal and conditional terms associated to $T(\occ_\depth,\vbeta_\depth)$.

Now,
if $1$ %
faces \aoh{} $\theta^1_\depth$, %
knows $2$'s future strategy $\beta^2_{\depth:H-1}$, and %
has access to %the conditional term
$\occ_\depth^{c,1}(\theta^2_\depth|\theta^1_\depth)$ for any $\theta_\depth^2$, %
then she faces a POMDP whose optimal value we denote 
$ \nu^2_{[\occ_\depth^{c,1},\beta^{2}_{\depth:H-1}]} (\theta^1_\depth)$.
This leads to defining the best-response value vector $\nu^2_{[\occ_\depth^{c,1},\beta^{2}_{\depth:H-1}]}$, which contains one component per \aoh{} $\theta^1_\depth$, and writing the value of $1$'s best response against $\beta^2_{\depth:H-1}$ under $\occ_\depth$ as $\occ_\depth^{m,1} \cdot \nu^2_{[\occ_\depth^{c,1}, \beta_{\depth:H-1}^2]}$. 
But then, because $2$ also knows $\occ_\depth$, she can in fact pick $\beta^2_{\depth:H-1}$ to minimize this value, so that we get the following theorem.

\begin{theorem}[%Concavity and convexity (CC) of $V^*_\depth$
{\citep[Thm.~2]{WigOliRoi-corr16}}] %the optimal value function]
  \labelT{theo|ConvexConcaveV}
  For any $\depth \in \{0\twodots H-1\}$, %
  $V_\depth^*$ is
  (i) concave w.r.t. $\occ_\depth^{m,1}$ for a fixed
  $\occ_\depth^{c,1}$, and %
  (ii) convex w.r.t. $\occ_\depth^{m,2}$ for a fixed
  $\occ_\depth^{c,2}$.
  More precisely,
  \begin{align*}
    V_\depth^*(\occ_\depth)
    & = \min_{\beta_{\depth:H-1}^2}\left[ \occ_\depth^{m,1} \cdot \nu^2_{[\occ_\depth^{c,1}, \beta_{\depth:H-1}^2]} \right]
    = \max_{\beta_{\depth:H-1}^1}\left[ \occ_\depth^{m,2} \cdot \nu^1_{[\occ_\depth^{c,2}, \beta_{\depth:H-1}^1]} \right], \text{where}
  \end{align*}
\begin{adjustbox}{max width=\textwidth}
$\nu^2_{[\occ_\depth^{c,1},\beta^{2}_{\depth:H-1}]} (\theta^1_\depth)  \eqdef \max_{\beta_{\depth:H-1}^1}  \mathbb{E}_{\theta_\depth^2 \sim \occ_\depth^{c,1}(\theta_\depth^1)} \left\{ \sum_{t = \depth}^{H-1} \gamma^{t-\depth} r(S_t,A_t^1,A_t^2) \mid \beta_{\depth:H-1}^1, \beta_{\depth:H-1}^{2} \right\}$.
\end{adjustbox}
%}
\end{theorem}

\begin{proof}(Sketch)
We start from \citeauthor{Neu-ma28}'s Minimax theorem \citep{Neu-ma28} giving the following equation:
 \stepcounter{footnote}
\begin{align*}
    & V_\depth^*(\occ_\depth)
    = \min_{\beta_{\depth:H-1}^2}\max_{\beta_{\depth:H-1}^1} \left[ 
    V_\depth(\occ_\depth,\beta_{\depth:H-1}^1,\beta_{\depth:H-1}^1) \right] \\
    & = \min_{\beta_{\depth:H-1}^2}\max_{\beta_{\depth:H-1}^1} \left[ 
    \mathbb{E} \left\{ \sum_{t = \depth}^{H-1} \gamma^{t-\depth} r(S_t,A_t^1,A_t^2) \mid \theta_\depth^1,\beta_{\depth:H-1}^1, \beta_{\depth:H-1}^{2}, \occ^{c,1}_\depth \right\} \right], \\
    \intertext{then, observing that $1$'s best response to $\beta_{\depth:}^2$ can be computed for each \aoh{} $\theta_\depth^1$ independently, we can swap the $\max$ operator and part of the expectation one ($\mathbb{E}$) as follows:\footnotemark[\value{footnote}]}
    & = \min_{\beta_{\depth:H-1}^2} 
    \mathbb{E}_{\theta^1_\depth \sim \occ^{m,1}_\depth}
    \underbrace{\left\{
    \max_{\beta_{\depth:H-1}^1} 
    \mathbb{E}_{\theta^2_{\depth} \sim \occ_\depth^{c,1}(\theta^1)}
    \left[
    \sum_{t = \depth}^{H-1} \gamma^{t-\depth} r(S_t,A_t^1,A_t^2) \mid \beta_{\depth:H-1}^1, \beta_{\depth:H-1}^{2} \right] \right\} }_\text{best-response of $1$ to $\beta_{\depth:}^2$ under $\theta_\depth^1$}  \\
    \intertext{and, recognizing the components of vector $\nu^2_{[\occ_\depth^{c,1}, \beta_{\depth:H-1}^2]}$ and writing the expectation over \aoh{}s $\theta_{\depth:}^1$ as a scalar product:}
    & = \min_{\beta_{\depth:H-1}^2}\left[ \occ_\depth^{m,1} \cdot \nu^2_{[\occ_\depth^{c,1}, \beta_{\depth:H-1}^2]} \right].
  \end{align*}
\end{proof}
\footnotetext[\value{footnote}]{Note that this property is well known in Bayesian games, where \aoh{}s correspond to {\em types}, \cf \cite[Th.~1, p.~321]{HarsanyiBG-II-ms68}.}

An important observation that ensues from this theorem is that $V^*_\depth$ is concave in $\occ_\depth^{m,1}$ and convex in $\occ_\depth^{m,2}$.
In practice, however, such continuity properties alone only allow upper-bounding $V^*_\depth$ for finitely many
conditional terms $\occ^{c,i}_\depth$, thus {\em not} for the whole occupancy space, as required to enable DP and HS in our game.

%\cbstart 
%\sout{
%In the following, we come back to \mbox{\citeauthor{WigOliRoi-corr16}}'s assumptions to (i) uniquely define $V^*$ and (ii) apply Bellman's optimality equation.
%}
%\olivier{Je suis gêné dans la version ci-dessus du fait qu'elle se focalise autant sur les "hypothèses" de Wiggers et al. Au contraire, je serais plus évasif là-dessus, et parlerais aussi de continuité et d'approximateurs (là où on veut en venir). Cf. ci-dessous}
%
In the following, we complement \mbox{\citeauthor{WigOliRoi-corr16}}'s results with properties of $V^*$ in subgames, plus continuity properties that help designing bounding approximators, which will lead us to an HSVI-like solver.
%\cbend

\paragraph{Note:}
To help the reader, \Cref{app|syntheticTables} provides two synthetic tables: 
\Cref{tab|PropertyTable} (p.~\pageref{tab|PropertyTable}) to sum up various theoretical properties that are stated in this paper (assuming a finite temporal horizon), and 
\Cref{tab|NotationTable} (p.~\pageref{tab|NotationTable}) to sum up the notations used in this paper, including some that are used only in the appendix.

Also, for convenience, we may replace in the following:
(i) subscript ``$\depth:H-1$'' with ``$\depth:$'',
(ii) any function $f(\vx)$ linear in vector $\vx$ with either $f(\cdot) \cdot \vx$ or ${\vx}^{\top} \! \cdot f(\cdot)$,
(iii) a full tuple with its few elements of interest,
and
(iv) an element (a "field") $x$ of a specific tuple $t$ by $x[t]$.

\section{Theoretical Contributions}
\label{sec|theory}

In this section, we demonstrate how to implement dynamic programming and heuristic search by
(1) rigorously showing that 
Bellman optimality equation (\Cshref{sec|oMGdpp}) holds,
(2) deriving bounding approximators of two novel optimal value functions, which come with solution strategies (\Cshref{sec|solvingOMGs}), and
(3) proposing a variant of HSVI that computes (in finite time) a player's strategy whose value is within $\epsilon$ of the zs-POSG's \nev{} (\Cshref{sec|HSVI}).

\subsection[The Optimal Value Function \texorpdfstring{$V^*$}{V*} and its Recursive Expression]{The Optimal Value Function \texorpdfstring{$V^*$}{V*} \\
and its Recursive Expression}
\label{sec|oMGdpp}

Let us first recall that, contrary to \citeauthor{WigOliRoi-corr16} \cite[Section~5, Lemma~4]{WigOliRoi-corr16}, we do not make the strong assumption that past decision rules can be considered as public 
(and, thus, we do not assume that any player knows $\occ_\depth$).
Indeed, while it is valid in Dec-POMDPs because the players are willing to coordinate their behaviors, it is {\em a priori} not valid in zs-POSGs, since players are, in the contrary, willing to deceive one another. Safety issues as presented in \Cref{ex|MP} illustrate the possible flaws of such an assumption.

\label{sec|CCV}

We now discuss the existence of an optimal value function $V_\depth^*$ and its properties.
These results are implicitly used by \citeauthor{WigOliRoi-corr16}, but it seems important to state and demonstrate them.
A first step is to demonstrate that \citeauthor{Neu-ma28}'s minimax theorem \citep{Neu-ma28} applies when in $\occ_\depth$,
thus justifying the definition of the optimal (Nash equilibrium) value of a subgame.

\begin{restatable}[Minimax theorem -- %Unicity of \nev{} -- 
Proof in \Cshref{app|backToMixedStrategies}]
{theorem}{thUnicityNev}
  \labelT{th|unicityNev}
  \IfAppendix{{\em (originally stated on
    page~\pageref{th|unicityNev})}}{}
The subgame defined in \Cshref{eq|subgame} admits a unique \nev{} %Nash equilibrium value
\begin{align}
    V_\depth^*(\occ_\depth) 
  & \eqdef \max_{\beta_{\depth:H-1}^1} \min_{\beta_{\depth:H-1}^2} V_\depth ( \occ_\depth, \beta_{\depth:H-1}^1,\beta_{\depth:H-1}^2).
\end{align}
\end{restatable}

$V(\occ_\depth,\cdot,\cdot)$ not being bilinear in the space of behavioral strategies (\Cref{ex|VnotBilinearInBetas}), the proof requires 
%The proof requires 
reasoning with mixed strategies (for which the bilinearity holds), \ie, distributions over pure (deterministic) strategies defined over {\em all} time steps.
Yet, when in a subgame, we have to reason only on mixed strategies {\em compatible} with the associated occupancy state $\occ_\depth$ (\ie, which ensure that the \os at $\depth$ is $\occ_\depth$), one step being to extend \citeauthor{Kuhn-ctg50}'s equivalence results between behavioral and mixed strategies \citep{Kuhn-ctg50} to the subgames.

Then, defining the optimal action-value function:
\begin{align}
  Q^*_\depth(\occ_\depth, \vbeta_\depth) 
  & \eqdef r(\occ_\depth, \vbeta_\depth) + \gamma V^*_{\depth+1}( T(\occ_\depth,\vbeta_\depth) ),
    \label{eq|localGames}
\end{align}
we can now prove that a Bellman optimality equation exists, %\sout{\mbox{(\citeauthor{WigOliRoi-corr16}}'s second postulate)},
 which justifies reasoning on subgames despite the non-observability.

\begin{restatable}[Bellman optimality equation -- %Bellman's optimality principle --
Proof in \Cshref{app|backToMixedStrategies}]
{theorem}{bellmanPpe}
  \labelT{th|bellmanPpe}
  \IfAppendix{{\em (originally stated on
    page~\pageref{th|bellmanPpe})}}{}
$V_\depth^*(\occ_\depth)$ satisfies the following functional equation:
  \begin{align*}
    V_\depth^*(\occ_\depth)
    &
    = \max_{\beta_\depth^1} \min_{\beta_\depth^2}
    r(\occ_\depth, \vbeta_\depth) + \gamma V^*_{\depth+1}( T(\occ_\depth,\vbeta_\depth) )
   % \\
   % &
   = \max_{\beta_\depth^1} \min_{\beta_\depth^2}
     Q^*_\depth(\occ_\depth, \vbeta_\depth).
 \end{align*}
\end{restatable}

The proof requires decomposing min and max operators over different time steps before swapping them appropriately to end up recognizing the optimal value function at the next time step ($V^*_{\depth+1}$).

\Cref{th|unicityNev,th|bellmanPpe} together show that \Cref{theo|ConvexConcaveV} holds even without player's strategies being public so that we can now build on the convex-concave property to solve zs-oMGs.

\subsection{Towards Solving zs-OMGs}
\label{sec|solvingOMGs}

This section aims at providing the second tool for DP and HS with continuous state spaces, \ie, bounding approximators of optimal value functions which will allow generalization across occupancy space.
Their update and selection operators are written as linear programs, and they turn out to come with solution strategies.

\subsubsection{Bounding value functions}
\label{sec|boundingValueFunctions}

So far, several issues prevented to apply the HSVI scheme to zs-POSGs, starting with the continuous spaces of
1. occupancy states (zs-OMG states) and
2. decision rules (zs-OMG actions).
One can address (1) by introducing the bounding functions $\upb{V}_\depth(\occ_\depth)$ and $\lob{V}_\depth(\occ_\depth)$ of $V_\depth^*(\occ_\depth)$ (\cf \Cshref{app|derivingApproximations}), for instance:
\begin{align*}
%\label{eq|upV} % [ob] label in appendix
  \upb{V}_\depth(\occ_\depth)
  & = \min_{   \langle \tilde\occ_\depth^{c,1}, \langle \upb\nu^2_\depth, \beta_{\depth:}^2 \rangle  \rangle  \in \upb{\cJ}_\depth } %
  \left[ \occ_\depth^{m,1} \cdot \upb\nu^2_\depth + \lt{\depth} \norm{ \occ_\depth - \occ_\depth^{m,1}\tilde\occ_\depth^{c,1} }_1 \right],
  %\text{ and} \\
  %\lob{V}_\depth(\occ_\depth) 
  %& = \max_{ \langle \tilde\occ_\depth^{c,2}, \langle \lob\nu^1_\depth, \beta_{\depth:}^1 \rangle \rangle \in \lob{\cJ}_\depth } %
  %\left[ \occ_\depth^{m,2} \cdot \lob\nu^1_\depth - \lt{\depth} \norm{ %\occ_\depth - \occ_\depth^{m,2}\tilde\occ_\depth^{c,2} }_1 \right],
\end{align*}
where
$\upb\nu^2_\depth$ component-wise upper-bounds $\nu^2_{[\tilde\occ^{c,1}_\depth,\beta_{\depth:}^2]}$ for some $\beta_{\depth:}^2$.
%, and
%$\lob\nu^1_\depth$ component-wise lower-bounds $\nu^1_{[\occ^{c,2}_\depth,\beta_{\depth:}^1]}$ for some $\beta_{\depth:}^1$.
%
They allow generalizing knowledge from the subgame rooted at $\occ_\depth$ to any other one rooted at $\tilde \occ_\depth$.
To do so, we use $V^*$'s Lipschitz-Continuity proven below.
%\cbend

\begin{restatable}[Lipschitz-Continuity of $V^*$ - proof in \Cshref{app|LC|V}]{theorem}{corVLCOcc}
  \labelT{app|V|LC|occ}
  \IfAppendix{{\em (originally stated on
      page~\pageref{sec|CCV})}}{}
  Let
  $\h{H}{\depth}{\gamma} \eqdef \frac{1-\gamma^{H-\depth}}{1-\gamma}$
  (or $\h{H}{\depth}{\gamma} \eqdef H-\depth$ if $\gamma=1$).
  Then $V^*_\depth(\occ_\depth)$ is $\lt{\depth}$-Lipschitz continuous in
  $\occ_\depth$ at any depth $\depth \in \{0 \twodots H-1\}$, where
  $\lt{\depth} = \frac{1}{2} \h{H}{\depth}{\gamma} %
  \left( r_{\max} - r_{\min} \right)$. 
\end{restatable}

Yet, this yields (generally non-convex) Lipschitz-continuous functions whose $\max$-$\min$ optimization would be intractable, so that (2) remains an issue.
Also, we do not know how to retrieve valid solution strategies.
In particular, and as illustrated in \mbox{\Cref{ex|MP}}, simply concatenating decision rules backwards from $\depth=H-1$ to $0$ would not guarantee globally-consistent solutions, and could result in exploitable strategies.

But then, combining \Cref{th|bellmanPpe,theo|ConvexConcaveV} leads to introducing a novel value function (denoted $W^{1,*}_\depth$) through writing, for any \os{} $\occ_\depth$:
\begin{align*}
      V_\depth^*(\occ_\depth) & =  \max_{\beta_\depth^1} \underbrace{\min_{\beta_{\depth:H-1}^2 \in \cB_\depth^2} \left[ r(\occ_\depth,\vbeta_\depth) + \gamma \occ_{\depth+1}^{m,1} \cdot \nu_{[\occ_{\depth+1}^{c,1}, \beta_{\depth+1:H-1}^2]}^2 \right]}_{ \eqdef W^{1,*}_\depth(\occ_\depth,\beta^1_\depth) }.
\end{align*}

Assuming that player $2$ can only respond with one of finitely many stored strategies,
the concavity and $\lambda_\depth$-Lipschitz-continuity of $W^{1,*}_\depth$ allow upper-bounding it with finitely many tuples $w=\langle \tilde\occ_\depth, \beta_{\depth}^2, \langle \upb{\nu}_{\depth+1}^2, \beta_{\depth+1:}^2 \rangle \rangle$ stored in sets $\upb{\cI}_\depth$, and where $\upb{\nu}_{\depth+1}^2$   upper-bounds $\nu^2_{[\tilde\occ_{\depth+1}^{c,1}, \beta_{\depth+1:}^2]}$.

\begin{restatable}[proof in \Cshref{app|derivingApproximationsW}]{proposition}{coreThUpperBounds}
\labelT{core|thUpperBounds}
\IfAppendix{{\em (originally stated on page~\pageref{core|thUpperBounds})}}{}
Let $\upb{\cI}_\depth$ be a set of tuples $w=\langle \tilde\occ_\depth, \beta_{\depth}^2, \langle \upb{\nu}_{\depth+1}^2, \beta_{\depth+1:}^2 \rangle \rangle$.
Then, 
\begin{align}
\nonumber
     \upbW{\depth}{\cI}(\occ_\depth, \beta_\depth^1)
     & \eqdef 
     \min_{\langle \tilde\occ_\depth, \beta_{\depth}^2, \langle \upb{\nu}_{\depth+1}^2, \beta_{\depth+1:}^2 \rangle \rangle \in \upb{\cI}_\depth }
    \left[ r(\occ_\depth, \beta_\depth^1, \beta^2_\depth)
     + \gamma \Tm{1}(\occ_\depth, \beta_\depth^1, \beta^2_\depth) \cdot \upb\nu^2_{\depth+1}  \right. 
     \\ 
    & \qquad \left. +  \lt{\depth+1} \norm{ T(\occ_\depth, \beta_\depth^1, \beta^2_\depth) - \Tm{1}(\occ_\depth, \beta_\depth^1, \beta^2_\depth) \Tc{1}(\tilde\occ^{c,1}_\depth, \beta^2_\depth) }_1 \right]
    \label{eq|upW}
\end{align} 
upper-bounds $W_\depth^{1,*}$ over the whole space $\Occ_\depth \times \cB_{\depth}^1$.
\end{restatable}

Symmetrically, we define $\lob{W}_\depth$ as the lower bound of the symmetrically defined $W_\depth^{2,*}$.
As explained in the next two sections, $\upb{W}_\depth$ will be easier to deal with compared to $\upb{V}_\depth$, allowing 1 to seek for decision rules optimistically, and providing valid solution strategies for 2 for the subgame at $\depth$, \ie, ignoring consistency with higher-level subgames.

\subsubsection{Action Selection and Backup Operators}
\label{subsection:actionSelectionAndBackupOperators}

We now detail the decision rule selection for $1$ using $\upb{W}_\depth$ to optimistically guide a trajectory in occupancy space, and how to update $\upb{W}_\depth$ by providing backup operations.

To that end, first note that linearities in $\beta^1_\depth$ within \Cshref{eq|upW} allow writing $\upb{W}_\depth(\occ_\depth,\beta^1_\depth)= \min_{   \substack{ w
        % =\langle \tilde \occ_\depth, \beta^2_{\depth:H-1}, \\
      % \upb\nu^2_{\depth+1} \rangle
      \in \upb{\cI}_\depth }
}
{\beta^1_\depth}^\top \cdot M^{\occ_\depth}_{(\cdot,w)}$, where $\beta^1_\depth$ and $M^{\occ_\depth}_{(\cdot,w)}$ (for each $w$) are column vectors of dimension $\abs{\Theta^1 \times \cA^1}$.
$M^{\occ_\depth}$ (see developed formula in \Cshref{sec|getLP}) is thus a $|\Theta_\depth^1 \times \cA^1| \times |\upb{\cI}_\depth|$ matrix. %,
Then, $\upbW{\depth}{\cI}$ being a lower envelope of hyperplanes leads to a convenient way of computing $\max_{\beta_\depth^1} \upbW{\depth}{\cI}(\occ_\depth,\beta_\depth^1)$.

\begin{restatable}[%Linear programming for action selection -- 
Proof in \Cshref{sec|getLP}]{proposition}{core|lpth}
For any given $\occ_\depth$ and any set $\upb{\cI}_\depth$ of tuples $w=\langle \tilde\occ_\depth, \beta_{\depth}^2, \langle \upb{\nu}_{\depth+1}^2, \beta_{\depth+1:}^2 \rangle \rangle$, $\max_{\beta_\depth^1} \upbW{\depth}{\cI}(\occ_\depth,\beta_\depth^1)$ is equivalent to the LP and dual LP:
\begin{align}
  \label{eq|LP}
  %\label{eq|DLP}
  & \begin{array}{ll@{\ }llllll}
  \lp{\upbW{\depth}{}}(\occ_\depth)\ : &
      \displaystyle
      \max_{\beta_\depth^1,v}
      & v
      & \text{s.t.} & \text{(i)}
      & 
        \forall w \in \upb{\cI}_\depth,
      & v \leq {\beta_\depth^1}^{\top} \! \cdot M^{\occ_\depth}_{(\cdot,w)}
      & \text{and}
      \\
      & & & &
      \text{(ii)}
      & \forall \theta_\depth^1 \in \Theta_\depth^1,
      & {\displaystyle \sum_{a^1}} \beta_\depth^1(a^1|\theta_\depth^1)
       = 1,
      \\ \\
      % DUAL
      % 
\dlp{\upbW{\depth}{}}(\occ_\depth)\ : &
      \displaystyle
      \min_{\tree^2_\depth,v}
      & v
       & \text{ s.t.} 
      &  \text{(i) } &
        \forall (\theta^1_\depth, a^1 ),
      & v \geq  M^{\occ_\depth}_{((\theta^1_\depth,a^1),\cdot)} \cdot \tree^2_\depth
      & \text{and} \\
      & & & &  \text{(ii)}
      & &
       {\displaystyle \sum_{w \in \upb{\cI}_\depth}} \! \tree^2_\depth( w )
       = 1.
    \end{array}
\end{align}
\end{restatable}

\begin{remark}[Outcomes of this game] %[Unusual Bayesian Game]
Since $\upb{W}_\depth$ upper-bounds $W^{1,*}_\depth$, solving this LP provides $1$ with an {\em optimistically} selected immediate decision rule $\beta^1_\depth$.
For $2$, $\tree_{\depth}^2$ is a probability distribution over tuples containing strategies $\beta_{\depth}^{2} \oplus \beta_{\depth+1:H-1}^{2}$, thus
recursively induces a strategy, as illustrated by \Cshref{fig|dag}, which can be turned into a
behavioral strategy $\beta_{\depth:H-1}^2$ (more details in \Cshref{App|th:eqStrategies}) whose value is {\em at worst} (from $2$'s viewpoint) the LP's value, \ie, against $1$'s best response to it.
\end{remark}

%\aurelien{mettre la première arrête droite ! Là c'est vraiment moche.}
%\aurelien{rapprocher les deux "colonnes"}
%\aurelien{passer les $\delta$ en $\tree$}
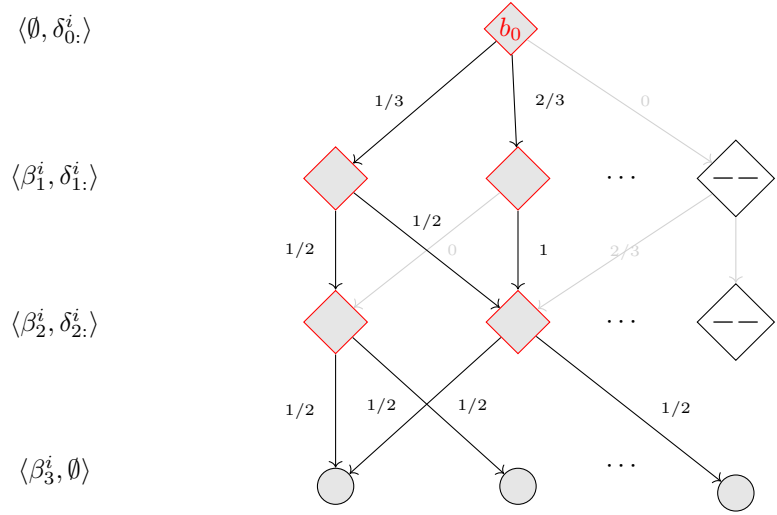
\begin{figure}[ht]
    \centering
    \adjustbox{max width = .85\linewidth}{
        \def\sScript{1.2} % useless if =1
\def\sTiny{0.1} % currently unused
\def\width{11.25}
\def\leveldist{1.5}
\tikzstyle{open circle}=[diamond, draw, inner sep=0.5]

\newcommand{\nodeW}[2]{
    \phantom{}
}
\newcommand{\nodeWLast}[2]{
  \phantom{}
}

\colorlet{lightgray}{gray!35}

\begin{tikzpicture}

  \def\radA{.6cm} % {.6cm}
  \def\radB{3cm} % {3cm}
  \def\radC{7.5cm}
  %\def\myshape{rectangle} % circle/ellipse/rectangle/...
  
  % 0th line: root node only
  \node [pienode={2em}{100}{diamond}{gray!20}{gray!40}, draw, scale=1] (n0) at (0,0) {$b_0$};
%  \node [scale=\sScript] at (\radC,0) {$\depth=0$};
  
  % 1st line
  \node [below = \leveldist of n0] (ntemp) {};
  \node [right = 1cm of ntemp] (n1dots) { $\cdots$ };
  \node [pienode={2em}{100}{diamond}{gray!20}{gray!40}, scale=\sScript, left = \radB of n1dots] (n11) { \nodeW{1}{0} };
  \node [pienode={2em}{100}{diamond}{gray!20}{gray!40}, scale=\sScript, left = \radA of n1dots] (n12) { \nodeW{2}{0} };
  \node [open circle, scale=\sScript, right = \radA of n1dots] (n1n) { $--$ };
  
  \path (n0) edge[-> ] node[left = 0.15cm] {\tiny $1/3$} (n11) ;
  \path (n0) edge[-> ] node[right = 0.15cm] {\tiny $2/3$} (n12) ;
  \path (n0) edge[->, lightgray] node[right = 0.15cm] {\tiny $0$} (n1n) ;
%  \draw[->] (n0) -- (n1n1);
  
  % 2nd line
  \node [below = \leveldist of n1dots] (n2dots) { $\cdots$ };
  %\node at (0,-5) (n2dots) { $\cdots$ };
  \node [pienode={2em}{100}{diamond}{gray!20}{gray!40}, scale=\sScript, left = \radB of n2dots] (n21) { \nodeW{1}{1} };
  \node [pienode={2em}{100}{diamond}{gray!20}{gray!40}, scale=\sScript, left = \radA of n2dots] (n22) { \nodeW{2}{1} };
  \node [open circle, scale=\sScript, right = \radA of n2dots] (n2n) {$--$};
  
  \path (n12) edge[->, lightgray] node[right=0.15cm] {\tiny{$0$}} (n21) ;
  \path (n12) edge[->] node[right=0.15cm] {\tiny{$1$}} (n22) ;

  \path (n1n) edge[->, lightgray] node[] {\tiny{$2/3$}} (n22) ;
  \path (n1n) edge[->, lightgray] node[] {} (n2n) ;
  
  \path (n11) edge[->] node[left = 0.15cm] {\tiny $1/2$} (n21) ;
  \path (n11) edge[->] node[above = 0.15cm] {\tiny $1/2$} (n22) ;

  %\node [below of = n21] (n21vdots) { $\vdots$ };
  %\node [below of = n2dots] (n2dotsvdots) { $\vdots$ };
  %\node [below of = n2n] (n2nvdots) { $\vdots$ };
  
  % Hth line
  \node [below = \leveldist of n2dots] (nH1dots) { $\cdots$ };
  %\node [pienode={2em}{100}{diamond}{gray!20}{gray!40}, scale=\sScript, left = \radB of nH1dots] (nH11) { \nodeWLast{1}{H-1} };
  %\node [pienode={2em}{0}{diamond}{gray!20}{gray!40}, scale=\sScript, left = \radA of nH1dots] (nH12) { \nodeWLast{2}{H-1} };
  %\node [pienode={2em}{0}{diamond}{gray!20}{gray!40}, scale=\sScript, right = \radA of nH1dots] (nH1n) { \nodeWLast{n_{H-1}}{H-1} };
  \node [circle, draw, fill = gray!20, scale=\sScript, below = \leveldist of n21] (nH11) { \nodeWLast{1}{H-1} };
  \node [circle, draw, fill = gray!20, scale=\sScript, below = \leveldist of n22] (nH12) { \nodeWLast{2}{H-1} };
  \node [circle, draw, fill = gray!20, scale=\sScript,  below = \leveldist of n2n] (nH1n) { \nodeWLast{n_{H-1}}{H-1} };
  
  \path (n21) edge[->] node[left = 0.15cm] {\tiny $1/2$} (nH11) ;
  \path (n21) edge[->] node[left = 0.3cm] {\tiny $1/2$} (nH12) ;
  \path (n22) edge[->] node[right = 0.3cm] {\tiny $1/2$} (nH11) ;
  \path (n22) edge[->] node[right = 0.3cm] {\tiny $1/2$} (nH1n) ;
  
  %\draw[->, lightgray] (nH12) -- (nH12);
\def\rootNode{n0}
\draw[]
	node at ($(\rootNode) + (-2*\radB,0)$) {$\langle \emptyset, \delta_{0:}^i \rangle$}
	node at ($(\rootNode) + (-2*\radB,-1.3*\leveldist)$) {$\langle \beta_1^i, \delta_{1:}^i \rangle$}
	node at ($(\rootNode) + (-2*\radB,-2.6*\leveldist)$) {$\langle \beta_2^i, \delta_{2:}^i \rangle$}
	node at ($(\rootNode) + (-2*\radB,-3.9*\leveldist)$) {$\langle \beta_3^i, \emptyset \rangle$}
;

\end{tikzpicture}
    }
    \caption{
    Representation of the strategy recursively induced by some $\tree_0^1$.
    At each time step $\depth$, one must
(i) sample a next tuple/node $w_\depth^1$ from current distribution $\tree_\depth^1$,
(ii) apply \dr{} $\beta_\depth^1[w_\depth^1]$, and
(iii) make $\tree_{\depth+1}^1[w_\depth^1]$ the new current distribution (unless reaching a leaf).
    }
    \label{fig|dag}
\end{figure}

Then, the following properties allow performing backups, \ie, filling up
the set $\upb{\cI}_{\depth-1}$ with new tuples $w$ containing, in particular, vectors $\upb\nu_\depth^2$.

\begin{restatable}[%Backup -- 
Proof in \Cshref{sec|compute|nu}]{lemma}{core|propBackup}
\labelT{backup}
  For any $\tree_{\depth}^2 = \dlp{\upbW{\depth}{}}(\occ_\depth)$,
  the vector $\nu_{[\occ_\depth^{c,1},\tree_{\depth}^2]}^{2}$ is component-wise upper-bounded by \begin{align*}
  \upb{\nu}_\depth^2
  & \eqdef \frac{1}{\occ_\depth^{m,1}} M^{\occ_\depth}_{((\theta^1_\depth,a^1),\cdot)} \cdot \tree^2_{\depth}.
  \end{align*}
\end{restatable}

%\aurelien{changer les $\cJ$ pour n'y mettre que $\occ_\depth^{c,i}$}
\begin{proposition}[update]
Let us assume that
\begin{itemize}
    \item a transition $\occ_{\depth-1} \to \occ_{\depth}$ has been performed through playing $\langle \beta_{\depth-1}^1,\beta_{\depth-1}^2 \rangle$, and
    \item solving $\dlp{\upbW{\depth}{}}(\occ_\depth)$ provides both
    \begin{itemize}
        \item a tree strategy $\tree_\depth^2$ (as the main solution of the DLP), and 
        % through a distribution over tuples $w_\depth$ in $\upb\cI_\depth$, and
        \item a vector $\upb{\nu}_\depth^2 = \frac{1}{\occ_\depth^{m,1}} M^{\occ_\depth}_{((\theta^1_\depth,a^1),\cdot)} \cdot \tree^2_{\depth}$ (as a by-product).
    \end{itemize}
\end{itemize}
Then,
\begin{enumerate}
    \item $\upb\cI_{\depth-1} \gets \upb\cI_{\depth-1} \cup \{ \langle \occ_{\depth-1}^{c,1}, \beta_{\depth-1}^2, \langle \upb{\nu}^2_\depth, \tree_{\depth}^2 \rangle \rangle \}$ 
    is a valid update operator in the sense that it preserves $\upb W_\depth$'s upper-bounding property, and
    \item similarly, $\upb\cJ_\depth \gets \upb\cJ_\depth \cup \{\langle \occ_\depth^{c,1}, \langle \upb\nu_\depth^2, \tree_\depth^2 \rangle \rangle \}$ is a valid update operator for $\upb V_\depth$.
\end{enumerate}
\end{proposition}
%as the algorithm iterates.

\subsubsection{Initialization}
\label{sec|initialization}

To initialize the bounds $\upb{W}_\depth$ and $\upb{V}_\depth$ for any time step, we begin by generating a trajectory in a forward phase. %performing the forward search.
At each time step, a uniform decision rule is picked for both players to derive a sequence of occupancy states $\occ_0, \dots, \occ_{H-1}$.
Then, during a backward phase, for each time step $\depth = H-1, \dots, 1$, we create a tuple $w_{\depth-1,init} = \langle \occ_{\depth-1}^{c,1}, \beta_{\depth-1}^2, \langle \upb\nu_\depth^2, \tree_{\depth}^2 \rangle \rangle$, 
where
\begin{itemize}
    \item $\occ_{\depth-1}^{c,1}$ is the conditional term associated to $\occ_{\depth-1}$;
    \item $\beta_{\depth-1}^2$ is a uniform decision rule;
    \item $\tree_\depth^2$ is
    \begin{itemize}
        \item a degenerate distribution over the only next tuple $w_{\depth+1}$ if $\depth<H-1$ (which induces a concatenation of uniform decision rules for all future time steps);
        \item undefined if $\depth=H-1$;
    \end{itemize}
and
%    \item for $\depth<H-1$, $\tree_\depth^2$ is a degenerate distribution over the only next tuple $w_{\depth+1}$; this induces a concatenation of uniform decision rules for all future time steps; at $\depth=H-1$, $\tree_{H-1}^2$ is undefined;
%and
    \item $\upb\nu_\depth^2(\theta_\depth^1) = r_{max} \cdot (H-\depth)$ for any history $\theta_\depth^1$ that player $1$ could face. 
\end{itemize}
Tuples $w_{\depth-1,init}$ are added to sets $\upb\cI_{\depth-1}$.
For any time step $\depth \geq 0$, we similarly create tuples $\langle \occ_\depth^{c,1}, \langle \upb\nu_\depth^2, \tree_{\depth:}^2 \rangle \rangle$ and add them to sets $\upb\cJ_\depth$.
The lower bounds are initialized symmetrically.

We now show that occupancy states can also be prescriptive, allowing one to retrieve an $\epsilon$-NES for the subgame at occupancy state $\occ_\depth$ once the bounds are withing $\epsilon$ from each other, in particular at $\depth=0$.

\subsubsection{Extracting a NES}
\label{sec|extraction}

Vectors $\upb{\nu}^2_0$ upper bounding the value of their associated strategies,
the following result tells when and how to extract an $\epsilon$-optimal solution strategy for this player.

\begin{restatable}[]{theorem}{thInclusionWandV}
\label{th|thinclusionWandV}
\IfAppendix{{\em (originally stated on page~\pageref{th|thinclusionWandV})}}{}
If sets $\upb{\mathcal{J}}_{0}$ and $\lob{\mathcal{J}}_{0}$ are such that
$ %\begin{align*}
    \upb{V}_0(\occ_0) - \lob{V}_0(\occ_0) \leq \epsilon,
$ %\end{align*}
then 
$\argmax_{\lob{w} \in \lob{\mathcal{J}}_{0}} \lob{\nu}^2_0 $
and
$\argmin_{\upb{w} \in \upb{\mathcal{J}}_{0}} \upb{\nu}^2_0 $ respectively provide strategies
$\tree_{0}^1$ and $\tree_{0}^2$ that form an $\epsilon$-\nes{} of the zs-POSG. 
\end{restatable}

\begin{proof}
\labelT{app|thinclusionWandV}
First, let us notice that, at $\depth=0$, the occupancy-state space is reduced to a singleton, $\{ \occ_0= \langle 1 \rangle \}$, because of the single (empty) joint \aoh{}. The value vectors $\nu$ are thus one-dimensional, and here considered as scalar numbers.

Let us assume that sets $\upb{\mathcal{J}}_{0}$ and $\lob{\mathcal{J}}_{0}$ are such that
\begin{align*}
    \upb{V}_0(\occ_0) - \lob{V}_0(\occ_0)
    & \leq \epsilon,
\end{align*}
and let
$\lob{w}^* = \langle \occ_0^{c,1}, \langle \lob{\nu}_0^*, \tree_0^{1,*} \rangle \rangle$ and
$\upb{w}^* = \langle \occ_0^{c,1} \langle \upb{\nu}_0^*, \tree_0^{2,*} \rangle \rangle$ be the tuples returned by $\argmax_{\lob{w}\in \lob{\cJ}_0} \lob{\nu}_0^2$ and $\argmin_{\upb{w}\in \upb{\cJ}_0} \upb{\nu}_0^1$.
Then, noting that $\occ_0= \langle 1 \rangle$, 
\begin{align*}
    \nu^1_{[\occ_0^{c,2},\tree_0^{1,*}]} - \nu^2_{[\occ_0^{c,1},\tree_0^{2,*}]} 
    & \leq \upb{\nu}_0^* - \lob{\nu}_0^*
    \\
    & = \max_{\lob{w}\in \lob{\cJ}_0} \lob{\nu}_0^2
    - \min_{\upb{w}\in \upb{\cJ}_0} \upb{\nu}_0^1
    \\
    & = 
    \upb{V}_0(\occ_0) - \lob{V}_0(\occ_0)
    \\
    & \leq \epsilon.
\end{align*}
Thus, $\tree_{0}^1$ and $\tree_{0}^2$ are two strategies whose security levels (values against best-responding opponents) are $\epsilon$-close, and thus form an $\epsilon$-\nes{} of the zs-POSG.
\end{proof}

Note: This result can be generalized to any $\occ_\depth$ at later time steps, but this generalization is not used in practice.

Distributions $\tree_{0}^2$ are stored and could be executed as is.
\Cref{app|stratExtraction} still presents a conversion process to retrieve a behavioral strategy $\beta_{0:H-1}^2$ from a distribution $\tree_{0}^2$ over tuples $w \in \upb{\cI}_0$. 
Next, we see how to design a practical HSVI-based algorithm that provably returns sets $\upb{\cJ}_0$ and $\lob{\cJ}_0$ satisfying \Cref{th|thinclusionWandV} after finitely many iterations.

\subsection{HSVI for zs-POSGs}
\label{sec|HSVI}

This section details our adaptation of the general HSVI scheme for $\epsilon$-optimally solving zs-POSGs, and presents a theoretical finite-time convergence property.

\subsubsection{Algorithm}

HSVI for zs-POSGs
is described in \Cref{alg|zsPOSGwithLP+VWs+}.
As vanilla HSVI, it relies on
(i) generating trajectories while acting optimistically (lines
\ref{alg|greedP1}+\ref{alg|greedP2}), \ie, player $1$ (resp. $2$)
acting ``greedily'' w.r.t. $\upbW{\depth}{}$ (resp. $\lobW{\depth}{}$),
and
(ii) locally updating the upper and lower bounds
(lines~\ref{alg|updateUpB}+\ref{alg|updateLoB}).
Both phases rely on solving the same games described by LP~(\ref{eq|LP}).
%\olivier{ou plutôt "by LP~(\ref{eq|LP})" ?}
%
At $\depth=H-1$, \cref{alg|oneShot} selects \dr{}s by solving 
%the exact game (\Cref{sec|approximations})
an exact game, and \cref{line|computeDelta}
returns a distribution reduced to the single element added in
\cref{alg|oneShotAddW}.

A key difference with \citeauthor{SmiSim-uai05}'s HSVI algorithm \citep{SmiSim-uai05} lies in the criterion for stopping trajectories.
The branching factor for zs-oMGs being infinite, we make use of $V^*$'s Lipschitz-continuity to implement the same adaptations as  \cite{HorBosPec-aaai17} used for zs-OS-POSGs.
The Lipschitz-continuity allows controlling the variations of the value function within small balls of radius $\radius$ around a previously visited occupancy-state. A finite number of such balls is sufficient to cover the whole space.
Then, \Cref{thm|termination} (below) ensures $\epsilon$-optimality in finite time if stopping trajectories when $\upb{V}_\depth(\occ_\depth) - \lob{V}_\depth(\occ_\depth) \leq \thr(\depth)$, with the threshold function $\thr(\depth) \eqdef
    \gamma^{-\depth}\epsilon - \sum_{i=1}^\depth 2 \radius \lt{\depth-i} \gamma^{-i}$.

%\end{tcolorbox}
 
\begin{algorithm}
  \caption{\zsomg-HSVI($b_0, [ \epsilon, \radius ]$) \newline
    \small [here returning a tuple $w_0$ containing a solution strategy $\lob\tree^1_0$ for player $1$]}
  \label{alg|zsPOSGwithLP+VWs+}
  \DontPrintSemicolon
  \SetKwFunction{FupbUpdate}{$\upb{\text{\bf Update}}$}
  \SetKwFunction{FlobUpdate}{$\lob{\text{\bf Update}}$}
  \SetKwFunction{FRecursivelyTry}{\textbf{Explore}} %RecursivelyTry}}
  \SetKwFunction{FzsOMGHSVI}{\textbf{\zsomg-HSVI}}
  % 
  %\SetInd{.25em}{.5em}
 
       % --------
      \Fct{\FzsOMGHSVI{$b_0 \simeq \occ_0$}}{
        \ForEach{$\depth \in 0\twodots H-1$}{
            Initialize 
            $\upb{V}_\depth$, $\lob{V}_\depth$, 
            $\upbW{\depth}{}$, \& $\lobW{\depth}{}$
        }
        \While{ $\left[ \upb{V}_0(\occ_0)-\lob{V}_0(\occ_0)  > thr(0) \right]$ \nllabel{alg|valSigma0} }{
          \FRecursivelyTry{$\occ_0, 0, -, -$}
        }
        \Return{$\argmax_{w_0 \in \lob{\mathcal{J}}_0} \lob{\nu}^1_0$}
      }
        
      % --------
      \medskip
      % --------
    
      \Fct{\FRecursivelyTry{$\occ_\depth, \depth, %
          \occ_{\depth-1}, \vbeta_{\depth-1} $}}
      {
        \If{$\left[\upb{V}_\depth(\occ_\depth) - \lob{V}_\depth(\occ_\depth) > thr(\depth) \right]$ \nllabel{alg|valSigmaT}}{
          \eIf{$\depth<H-1$}{
            $\upb\beta_\depth^1 \gets \lp{\upbW{\depth}{}}(\occ)$
            \nllabel{alg|greedP1}\;
            $\lob{\beta}_\depth^2 \gets \lp{\lobW{\depth}{}}(\occ)$
            \nllabel{alg|greedP2}\;
            \FRecursivelyTry{$ T(\occ_\depth, \upb\beta_\depth^1, \lob\beta_\depth^2), %
              \depth+1,%
              \occ_\depth, \langle \upb\beta^1_\depth, \lob\beta^2_\depth \rangle
              $} \nllabel{alg|callExploreRec}
          }({($\depth=H-1$)}){
            $(\upb\beta_\depth^1, \lob\beta^2_\depth) \gets \nes \left( r(\occ, \beta^1_\depth, \beta^2_\depth) \right)$ \nllabel{alg|oneShot}\;

            $\upb{\cI}^1_{\depth} \gets \upb{\cI}^1_{\depth} \cup \{ \langle { %
              \occ^{c,1}_{\depth}, \lob\beta^2_\depth, -
            } \rangle \} $ \nllabel{alg|oneShotAddW}\;

            $\lob{\cI}^2_{\depth} \gets \lob{\cI}^2_{\depth} \cup \{ \langle { %
              \occ^{c,2}_{\depth}, \upb\beta^1_\depth,  -
            } \rangle \} $
            
          }
          $\FupbUpdate ({\upbW{\depth-1}{}, %
            \langle \occ_\depth, %
            \occ^{c,1}_{\depth-1}, \lob\beta^2_{\depth-1} %,
            \rangle }) $ \nllabel{alg|updateUpB}\;

          $\FlobUpdate ({\lobW{\depth-1}{}, %
              \langle \occ_\depth, %
              \occ^{c,2}_{\depth-1}, \upb\beta^1_{\depth-1} %,
              \rangle }) $  \nllabel{alg|updateLoB}\;
        }
      }
      
      % --------
      \medskip
      % --------
       
      \Fct{\FupbUpdate{$\upbW{\depth-1}{}, \langle { %
            \occ_\depth, %
            \occ^{c,1}_{\depth-1}, \upb\beta^2_{\depth-1}%, %} %, \tree^{2}_{\depth-1}  %
          } \rangle $ } %
      }
      {
        \nllabel{alg|UpdateFunction}
        $\langle \upb\nu^2_\depth, \tree_\depth^2 \rangle %
        \gets \dlp{\upbW{\depth}{}}(\occ_\depth, )$
        \nllabel{line|computeDelta}
        
        $\upb{\cI}_{\depth-1} \gets \upb{\cI}_{\depth-1} \cup \{ \langle { %
          \occ^{c,1}_{\depth-1}, \lob\beta^2_{\depth-1}, \langle \upb\nu^2_\depth, \tree^2_\depth \rangle
        } \rangle \} $
        \nllabel{alg|UpdateFunction|end}
        
        $\upb{\mathcal{J}}_{\depth} \gets \upb{\mathcal{J}}_{\depth} \cup \{ \langle { %
          \occ_\depth^{c,1}, \langle \upb\nu^2_\depth, \tree_\depth^2 \rangle
        } \rangle \}$
                \nllabel{alg|UpdateFunction|endJ}

      }

\end{algorithm}

%-----------------------

\paragraph{Setting \texorpdfstring{$\radius$}{Rho}}

As can be observed, this threshold function should always return
positive values, which requires a small enough (but $>0$) $\radius$.
For a given problem (\cf \Cshref{lem|MaxRadius},
\Cshref{sec|settingRadius}), the maximum possible value
$\radius_{\max}$ depends on the Lipschitz constants at each time step,
which themselves depend on the initial upper and lower bounds of the
optimal value function.
Setting $\radius \in (0,\radius_{\max})$ means making a trade-off between generating many trajectories (small $\radius$) and long ones
(large $\radius$).

\subsubsection{Finite-Time Convergence} % (Proof)}

\begin{restatable}[Proof in \Cshref{sec|ConvergenceProof}]{theorem}{thmTermination}
  \labelT{thm|termination}
  \IfAppendix{{\em (originally stated on page~\pageref{thm|termination})}}{}
  \zsomg-HSVI (\Cshref{alg|zsPOSGwithLP+VWs+}) terminates in
  finite time with an $\epsilon$-approximation of $V^*_0(\occ_0)$
  % , \ie, $\upb{V}(b_0) - \lob{V}(b_0) \leq \epsilon$.
  that statisfies \Cref{th|thinclusionWandV}.
\end{restatable}

The finite time complexity suffers from the same combinatorial 
explosion as for Dec-POMDPs, and is even worse as we have to handle 
``infinitely branching'' trees of possible futures.
More precisely, the bound on the number of iterations depends on the 
number of balls of radius $\radius$ required to cover occupancy 
simplexes at each depth.

Also, the following proposition allows solving infinite horizon problems as 
well (when $\gamma<1$) by bounding the length of HSVI's trajectories 
using the boundedness of $\upb{V}-\lob{V}$ and the exponential growth of 
$thr(\depth)$.

\begin{restatable}[Proof in \Cshref{proofLemFiniteTrials}]{proposition}{lemFiniteTrials}
  \labelT{lem|finiteTrials}
  \IfAppendix{{\em (originally stated on
      page~\pageref{lem|finiteTrials})}}{}
  When $\gamma<1$,
  the length of trajectories is
  upper bounded by
  $ %\begin{align*}
    T_{\max}
    \eqdef \ceil*{
      \log_{\gamma} %\left(
      \frac{
        \epsilon - \frac{2 \radius \l^\infty}{1-\gamma}
      }{
        \WUL - \frac{2 \radius \l^\infty}{1-\gamma}
      }
    } $,
where $\l^\infty$ is a depth-independent Lipschitz constant and $\WUL \eqdef \norm{ \upb{V}^{(0)}-\lob{V}^{(0)}}_\infty$ is the maximum width between initializations.
\end{restatable}

\section{Experiments}
\label{sec|XPs}

Experiments presented in this section aim at validating the proposed approach and comparing its behavior to the behavior of some reference algorithms.
%
%\aurelien{ça me fait penser que Jilles insite sur l'importance de dire et redire que les exp. ne sont pas une "contribution", qu'elles ne sont là que pour valider les trouvailles théoriques. Est-ce qu'on a suffisament insisté dessus avant?}
%
%\olivier{On doit déjà avoir des choses pas mal avant, mais redire un mot ici (quitte à se répéter) semblerait bien.}

\subsection{Setup}

\subsubsection*{Benchmark Problems}

Five benchmark problems were used.
Adversarial Tiger and Competitive Tiger were introduced by
\citet{Wiggers-msc15}.
Mabc and Recycling Robot are well-known Dec-POMDP benchmark problems (\cf \url{http://masplan.org}) %
and were adapted to our competitive setting by making player $2$ minimize (rather than maximize) the objective function. 
The fifth benchmark is the adaptation of the well-known Matching Pennies game detailed in \Cref{ex|MP}, with a small difference in that $r(s_h,\cdot,a_h)=+2$ instead of $+1$;
this change breaks the symmetry in the optimal strategy, so that HSVI can not find the NES by "chance" by trying uniform strategies.
We only consider finite horizons $H$ and $\gamma=1$.
\Cref{tab|BenchmarkNumberStatesActionsObservations} gives the cardinal of the state, action and observation sets for each of these problems.

\begin{table}[ht]
    \caption{Number of states/actions/observations for each benchmark problem}
    \centering
    \label{tab|BenchmarkNumberStatesActionsObservations}
%\resizebox{\linewidth}{!}{%
    \begin{tabular}{l|cc@{ }cc@{ }c}
      \toprule
      % Variable $\to$ 
      & $\mathcal{S}$ &  $\mathcal{A}^1$ &  $\mathcal{A}^2$ & $\mathcal{O}^1$ & $\mathcal{O}^2$ \\
      \midrule
      Competitive Tiger & 2 & 4 & 4 & 3 & 3  \\
      Adversarial Tiger & 2 & 3 & 2 & 2 & 2 \\
      Recycling Robot & 4 & 3 & 3 & 2 & 2 \\
      Mabc & 4 & 2 & 2 & 2 & 2 \\
      Matching Pennies & 3 &  2 & 2 & 1 & 1 \\
      \bottomrule
%}
    \end{tabular}
%}
\end{table}

\subsubsection*{Algorithms}
%\paragraph{Algorithms}
\label{sec|XP|algorithms}

\def\checklistCiteCreators{
%\olivier{Version mise-à-jour à vérifier:}
For conciseness, \Cref{alg|zsPOSGwithLP+VWs+} is here denoted HSVI, and compared against %
{\em Random} search and {\em Informed} search \cite{Wiggers-msc15}
(both using \citeauthor{Wiggers-msc15}'s  implementation {\scriptsize (unlicensed and unreleased})), 
SFLP \citep{KolMegSte-geb96}, %
and
\CFR{} \cite{Tammelin-arxiv14} %
(both using open\_spiel \citep{LanctotEtAl2019OpenSpiel} {\scriptsize (Apache license)}).
}
\checklistCiteCreators

All algorithms (but SFLP, which is exact) used a target error $\epsilon = 1$\% of the initial gap $H \cdot (r_{\max}-r_{\min})$.
HSVI ran with %
$\lambda_\depth = (H-\depth) \cdot ( r_{\max} - r_{\min} )$, %
and
$\rho$ the middle of its feasible interval.
We also use FB-HSVI's LPE lossless compression of probabilistically equivalent action-observation histories in occupancy states, %
so as to reduce their dimensionality \citep{DibAmaBufCha-jair16}.
Experiments ran on an Ubuntu machine with i7-10810U 1.10\,GHz Intel processor and 16\,GB available RAM, and the code is available under MIT license at \url{https://gitlab.com/aureliendelage1/hsviforzsposgs}.

Random and Informed, only ran once, providing fairly representative results.

\subsection{Results}

\paragraph{Performance Measures}

A common performance measure in 2-player zero-sum games is the {\em exploitability} of a strategy $\beta^i_{0:}$, \ie, the difference between the strategy's {\em security level} (the value of $\neg i$'s best response to $\beta^i_{0:}$) and the Nash equilibrium value $V^*_0(\occ_0)$:
\begin{align*}
    \text{exploitability}(\beta^i_{0:})
     & = \abs{ V^*(\occ_0) - \occ^{m,1}_0 \cdot \nu^i_{[\occ^{c,\neg i}_0,\beta^i_{0:}]} } \\
     & = \abs{ V^*(\occ_0) - \nu^i_{[\occ^{c,\neg i}_0,\beta^i_{0:}]} },
\end{align*}
noting that $\occ_0$ is a degenerate distribution over a single element, the pair of empty action-observation histories.
In our setting, it will be convenient to look at the (average) {\em exploitability of a strategy profile} $\langle \beta_{0:}^1, \beta_{0:}^2 \rangle$:
\begin{align*}
    \text{exploitability}(\beta_{0:}^1, \beta_{0:}^2)
    & = \frac{
        ( V^*(\occ_0) - \nu^1_{[\occ^{c,2}_0,\beta^1_{0:}]} )
        +
        ( \nu^2_{[\occ^{c,1}_0,\beta^2_{0:}]} - V^*(\occ_0) )
        }{2} \\
    & = \frac{
        \nu^2_{[\occ^{c,1}_0,\beta^2_{0:}]} - \nu^1_{[\occ^{c,2}_0,\beta^1_{0:}]}
        }{2}.
%        = \frac{\text{SL-gap}}{2}.
\end{align*}
%\olivier{En revoyant les formules, je pense avoir corrigé des erreurs (identifiants des joueurs parfois échangés).}
This quantity is a more concise statistic than both individual exploitabilities, and can be obtained by solving two POMDPs (fixing one player's strategy or the other) without requiring to know the actual NEV.

This exploitability can also be defined as half of the {\em gap between security levels} (SL-gap).
To analyze the convergence of algorithms with respect to the initial gap, we will look at the {\em SL-gap percentage}, \ie,
\begin{align*}
    \text{SL-gap percentage}(\beta_{0:}^1, \beta_{0:}^2)
    & = \frac{
        \nu^2_{[\occ^{c,1}_0,\beta^2_{0:}]} - \nu^1_{[\occ^{c,2}_0,\beta^1_{0:}]}
    }{
        H\cdot (R_{\max} - R_{\min})
    } \\
    & = \frac{
        2 \cdot \text{exploitability}(\beta_{0:}^1, \beta_{0:}^2)
    }{
        H\cdot (R_{\max} - R_{\min})
    }
    .
\end{align*}

\subsubsection{Comparison with the state of the art}

\Cref{tab|resExp} gives the convergence time of \citeauthor{Wiggers-msc15}'s two heuristic algorithms, \CFR, SFLP, and HSVI on the benchmark problems with various horizons, or the SL-gap percentage when reaching a 1\,h time limit.
%, a strategies's security level being the value of the opponent's best response.
%
Executions not returning any result (\ie, for Random, Informed and \CFR, not performing a single iteration) are noted out-of-time {\sc [oot]}.

This table first shows that HSVI always outperforms the heuristic baseline provided by \citeauthor{Wiggers-msc15}'s algorithms, thus proving the interest of an HSVI scheme.
However, HSVI is outperformed by both SFLP and \CFR{}, unless they run out of time. 
As can be noted, HSVI is able to keep improving even when the horizon grows thanks to the LPE compression, taking advantage of underlying structure in some games (\eg, Recycling Robot, a problem with transition+observation independence (TOI), when scaling to larger horizons). 

\def\compTigerInitGap{12.0}
%0.012,0.024,0.036,0.048
\def\advTigerInitGap{8.0}
\def\recyclingInitGap{3.552}
\def\mabcInitGap{0.4}
\def\mpInitGap{3.0}

\begin{table}
  \def\tunits{time/gap}
  \def\oom{{\scriptsize \sc [oom]}}
  \def\oot{{\scriptsize \sc [oot]}}
  \def\eer{{\scriptsize \sc [eer]}}
  \def\pcent{{\%}}
  \newcommand{\pgap}[3]% (1) gap, (2) maxRGap, (3) H
  {\fpeval{#1/(#2*#3) * 100}\,{\pcent}}
  \newcommand{\pgapB}[3]% (1) gap, (2) maxRGap, (3) H
  {\fpeval{#1/(#2*#3) * 100}\,{\pcent}}
  \caption{% \olivier{Contre-proposition qui (1) aligne les pourcentages plus joliment, et (2) permet de régler le nombre de chiffres après la virgule (pour les pourcentages) aisément (ici 1 et non 2).}
  Comparison of different solvers on various benchmark problems.
    Reported values are %
    %    %\uwave{[gap] values [$\upb{V}_0(\occ_0)-\lob{V}_0(\occ_0)$]f}
    %gap "percentages" $\frac{\upb{V}_0(\occ_0)-\lob{V}_0(\occ_0)}{H \cdot (R_{\max}-R_{\min})}$
    %if the timeout limit is reached, or running times (when gap is lower than 1\,\pcent{}).
    the running times until the algorithm's error gap (based on bounds for HSVI) is lower than 1\,\pcent{}, or,
    if the timeout limit is reached,
    the security-level gap percentages (100\,\% if gap $=H \cdot (R_{\max}-R_{\min})$).
    {% \small 
    %``(ni)'' indicates no improvement over the initialization after 24\,\si{\hour}.
    % 
    %\sout{``\oom''}
    %``\oot'' indicates an out-of-time error (when no solution is returned).
    %``\oot'' and ``\eer'' indicate either an out-of-time or an execution error.
    % 
    %``n/a'' indicates an unavailable result.
    } %For each problem, algorithms were respectively given $10$ \si{\minute}, $10$ \si{\minute} and $1$ \si{\hour} for the three considered time horizons.
    Notes:
    (1) Horizons with a star exponent ($H^*$) are those for which the security-level computations ran out of time so that, for HSVI, we give the gap between the pessimistic bounds. % \aurelien{Also, note that results for \CFR{} are only indicative.}
    (2) Even though Random and Informed contain randomness, we ran them only once, getting fairly representative results.
  }
  \label{tab|resExp}
    \centering
    \adjustbox{max width=\linewidth}
    {
\sisetup{
round-mode = places,
round-precision = 1,
table-format=2.1,
%  round-pad = false %% true by default
}%
    \begin{tabular}[t]{llSSSSSS}
    \toprule
    Domain & H & \multicolumn{2}{c}{Wiggers} & HSVI & SFLP & {\CFR} \\ %{\omgHSVIlccc} \\
    \cmidrule(lr){3-4}
    & & {Rand.} & {Inf.} \\ % & {HSVI}  &  & \\
    \midrule
    %0.02\ \si{\second} & 0.17\ \si{\second} & 3\ \si{\second}
    \multirow{3}{*}{\shortstack[l]{Comp\\Tiger}}
    & 2 & \pgap{0.63}{\compTigerInitGap}{2} & \pgap{2.00}{\compTigerInitGap}{2} &  {6 } \si{\second} & \BS {1 } \si{\second} & {18 } \si{\second}\\
    & 3 & \pgap{2.53}{\compTigerInitGap}{3} & \pgap{2.20}{\compTigerInitGap}{3} & \pgap{1.38}{\compTigerInitGap}{3} & \BS {48 } \si{\second}  & {30 } \si{\minute}\\%& \pgap{1.14}{\compTigerInitGap}{3}\\
    & 4 & \pgap{5.81}{\compTigerInitGap}{4} & \pgap{3.71}{\compTigerInitGap}{4} & \pgap{2.28}{\compTigerInitGap}{4} & \BS {14 } \si{\minute} &\oot\\
     & $5^*$ & \oot & \oot & \pgapB{32}{\compTigerInitGap}{5} & \oot & \oot\\
    \midrule
    \multirow{3}{*}{\shortstack[l]{Rec.\\Robot}}
& 2 & \pgap{0.24}{\recyclingInitGap}{2} & \pgap{0.36}{\recyclingInitGap}{2} & {5 } \si{\second} & \BS {1 } \si{\second} & {30 } \si{\second}\\
    & 3 & \pgap{0.98}{\recyclingInitGap}{3} & \pgap{1.62}{\recyclingInitGap}{3} & {4 }\si{\minute} & \BS {1 } \si{\second} & {13 } \si{\minute}\\
    & 4 & \pgap{2.0}{\recyclingInitGap}{4} & \pgap{2.78}{\recyclingInitGap}{4} & \pgap{0.7}{\recyclingInitGap}{4} & \BS {13 } \si \second & \pgap{0.21}{\recyclingInitGap}{4}\\
    & 5 & \oot & \oot & \BS \pgap{1.9}{\recyclingInitGap}{5} & \oot & \oot \\
    & $6^*$ & \oot & \oot & \BS \pgap{9.70}{\recyclingInitGap}{6} & \oot & \oot \\
    \midrule
%    \end{tabular}
%    \hspace*{.5cm}
%    \hfill % %%%%%%%%%%%%%%%%%%%%%%%%%%
%    \begin{tabular}[t]{llSSSSSS}
%    \toprule
%    Domain & H & \multicolumn{2}{c}{Wiggers} & HSVI & {SFLP} & {\CFR} \\ %{\omgHSVIlccc}
%    \cmidrule(lr){3-4}
%    & & {Rand.} & {Inf.} & \\ % HSVI \\
%    \midrule
    \multirow{3}{*}{\shortstack[l]{Adv\\Tiger}}
    & 2 & {1 } \si{\second} & \pgap{0.59}{\advTigerInitGap}{2} & \BS {1 }\si{\second} & \BS {1 } \si{\second} & \BS {1 } \si{\second}\\
    & 3 & \pgap{0.35}{\advTigerInitGap}{3} & \pgap{1.05}{\advTigerInitGap}{3} & {2 }\si{\minute} & \BS {1 } \si{\second} & {8 } \si{\second}\\
    & 4 & \pgap{0.92}{\advTigerInitGap}{4} & \pgap{1.79}{\advTigerInitGap}{4} & \pgap{0.83}{\advTigerInitGap}{4} & \BS {8 } \si{\second}& {13 } \si{\minute} \\%&\pgap{0.06}{\advTigerInitGap}{4}\\
    \midrule
    %SFLP & 0.14\ \si{\second} & 48\ \si{\second}  & 14\ \si{\min}\\
    \multirow{3}{*}{MABC}
    & 2 & {45 } \si{\second} & \pgap{0.15}{\mabcInitGap}{2} & {8 }\si{\second} & \BS {3 } \si{\second} & {5 } \si{\second}\\
    & 3 & \pgap{0.05}{\mabcInitGap}{3} & \pgap{0.11}{\mabcInitGap}{3} & {27 }\si{\second} & \BS {1 }\si{\second} &{1 } \si{\minute}\\
    & 4 & \pgap{0.29}{\mabcInitGap}{4} & \pgap{0.58}{\mabcInitGap}{4} & \pgap{0.07}{\mabcInitGap}{4} & \BS {3 } \si{\second} & {47 } \si{\minute}\\
    \midrule
    \multirow{3}{*}{MP}
    & 4 & {2 } \si{\minute} & \pgap{5.6}{\mpInitGap}{4}& {5 }\si{\second} & \BS {1 } \si{\second} & {2 } \si{\second}\\
    & 5 & {9 } \si{\minute}& \pgap{6.87}{\mpInitGap}{5} & {1 }\si{\minute} & \BS {1 } \si{\second} & {10 } \si{\second}\\
    & 6 & \pgap{0.40}{\mpInitGap}{6} & \pgap{8.03}{\mpInitGap}{6} & {8 }\si{\minute} & \BS {2 } \si{\second} & {1 } \si{\minute}\\
    \bottomrule
    \end{tabular}
    }
\end{table}

\medskip

We now study the dynamic behavior of the algorithms at hand by providing and analyzing the bounds and exploitability graphs for the same benchmarks.

\subsubsection{Bounding Graphs}
\label{core|boundsGraphs}

Left-side graphs in \Cref{fig|Graphs|AdvTiger,fig|Graphs|CompTiger,fig|Graphs|Mabc,fig|Graphs|Recycling,fig|Graphs|RecyclingH56,fig|Graphs|MP} show how the computed upper- and
lower-bounding values $\upb{V}_0(\occ_0)$ and $\lob{V}_0(\occ_0)$ (respectively the {\em dotted} dark and light green curves) evolve as a function of computation time (always given in seconds).
The {\em solid} dark and light green curves show the security levels $\nu_{[\occ_0,\tree_{0}^2]}^{1}$ and $\nu_{[\occ_0,\tree_{0}^1]}^{2}$ of the current returned strategies $\tree_{0}^2$ and $\tree_0^1$. %\sout{(the values of their opponent's best response, obtained by solving POMDPs)}.
%\olivier{La suite a été remontée pour que toute l'explication de la construction des courbes soit avant l'analyse.}
Note that, when best-response computations to obtain security levels are expensive (\eg, for the competitive tiger problem, with $H=4$), they are performed either periodically (\eg every $10$ iterations) or only once, at the end.
In the captions, we indicate the (arbitrary) frequency of the POMDP evaluations. For example, $(1,1,once)$ means that, for the first two horizons, the POMDP evaluations were done after each iteration, and, for the last one, only once (at the end).

Overall, we observe consistent curves with
(i) security levels in-between bounds and around the NEV, and
(ii) bounds converging monotonically.
Note that HSVI stops when the gap between bounds is small enough, while the gap between SLs (used by Informed, Random and CFR, and whose computation can be time-consuming) can be much smaller.
As a matter of fact, one can notice that strategies $\tree_{0}^i$ returned at each iteration
%\olivier{Je préfèrerais réserver "time step" pour les pas de temps du zs-POSG.}
by HSVI are often better (in terms of security level) %\sout{(\ie their value (when played against a best response of their opponent) is often closer to the \nev{})}
than their pessimistic lower- or upper-bounding guarantees $\lob{\nu}_{[\occ_0,\tree_{0}^1]}^1$ and $\upb{\nu}_{[\occ_0,\tree_{0}^2]}^2$.

\newcommand\evolGraphCaption[1]{\scriptsize
Evolution of the upper- and
  lower-bound value $\upb{V}_0(\occ_0)$ (in red) and
  $\lob{V}_0(\occ_0)$ (in blue) of \omgHSVIlccc{} for the #1 problem
  as a function of time (in seconds).
  Optimal value found by Sequence Form LP in green for reference (when available). 
}
\newcommand\evolGraphSubCaption[1]{\scriptsize #1}
\def\evolGraphScale{0.32} % was .42 or .38

\newcommand{\dotvfill}{%
  \par\leaders\hbox{$\cdot$}\vfill}

\subsubsection{Exploitability Graphs}
\label{core|exploitabilityGraphs}

Right-side graphs in \Cref{fig|Graphs|AdvTiger,fig|Graphs|CompTiger,fig|Graphs|Mabc,fig|Graphs|Recycling,fig|Graphs|RecyclingH56,fig|Graphs|MP}
show the exploitability of the returned strategy profile as a function of computation time for HSVI, Random, Informed, and \CFR{} for the different benchmarks considered. A limit precision of $10^{-7}$ (chosen empirically, according to the LP solver's precision) was applied to HSVI's exploitability.

As can be observed, Random and Informed tend to produce reasonable strategies quickly, but struggle to improve them so as to converge towards an $\epsilon$-NES with $\epsilon \simeq 0$.
In contrast, our algorithm keeps improving as computation time increases.
The exploitation graphs support the observed behavior in \Cref{tab|resExp} that HSVI converges in reasonable time compared to \citeauthor{Wiggers-msc15}'s algorithms.
However, the graphs also show that \CFR{} essentially outperforms HSVI when the problems are difficult enough
(\ie, when the temporal horizon grows) but the traversal of the whole tree still remains tractable (thus allowing \CFR{} to perform iterations).
An interesting observation is that, on small enough problems, HSVI achieves very low exploitabilities earlier than \CFR{}.

Finally, HSVI's exploitability graph shares strong similarities with those of \citeauthor{BosEtAl-jair14}'s double-oracle algorithms \citep[Fig.~8 and 11]{BosEtAl-jair14}.
This can be understood as HSVI iteratively building two sets of strategies, one per player, until they are sufficient to support NES profiles, so that the average exploitability is almost zero.
But note that \citeauthor{BosEtAl-jair14} construct LPs using pure strategies (deterministic best responses), while HSVI's strategies are stochastic.

\def\evolGraphScale{.49}

\newcommand{\myEvoExploCaption}[1]{
    \uline{\textbf{#1:}} %
    \textbf{(left)} Evolution of (in dotted lines) the upper- and lower-bound values, and (in solid lines) the security levels of the returned strategies for HSVI as a function of time (ms).
    \textbf{(right)} Exploitability ($=\frac{\text{SL-gap}}{2}$) as a function of time (s) for Random, Informed, \CFR{}, and HSVI.
}

\begin{figure}[ht]%
    \centering    % Adv Tiger
    
    \includegraphics[width=\evolGraphScale\linewidth]{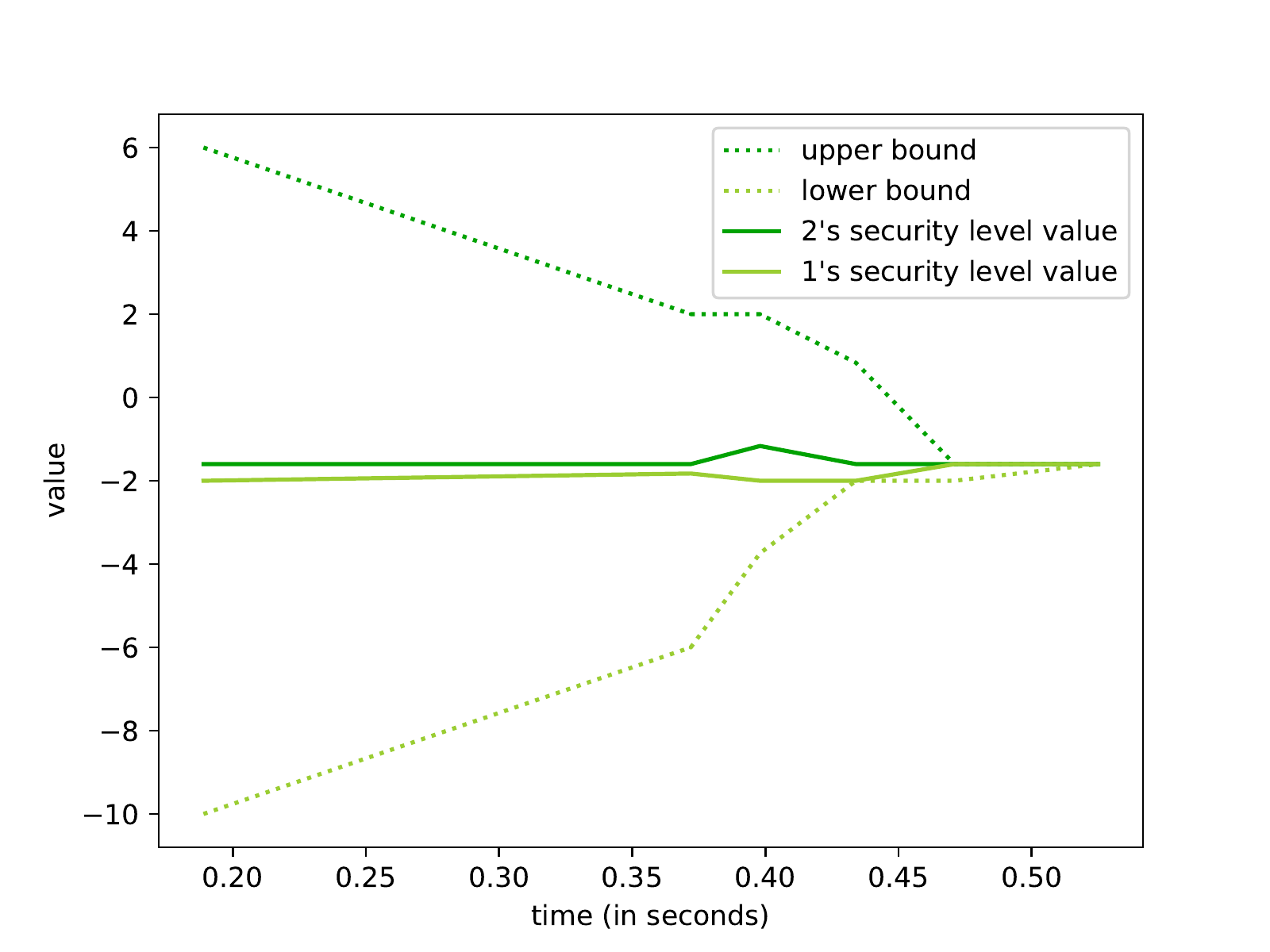}%
    \hfill % H=2
    \includegraphics[width=\evolGraphScale\linewidth]{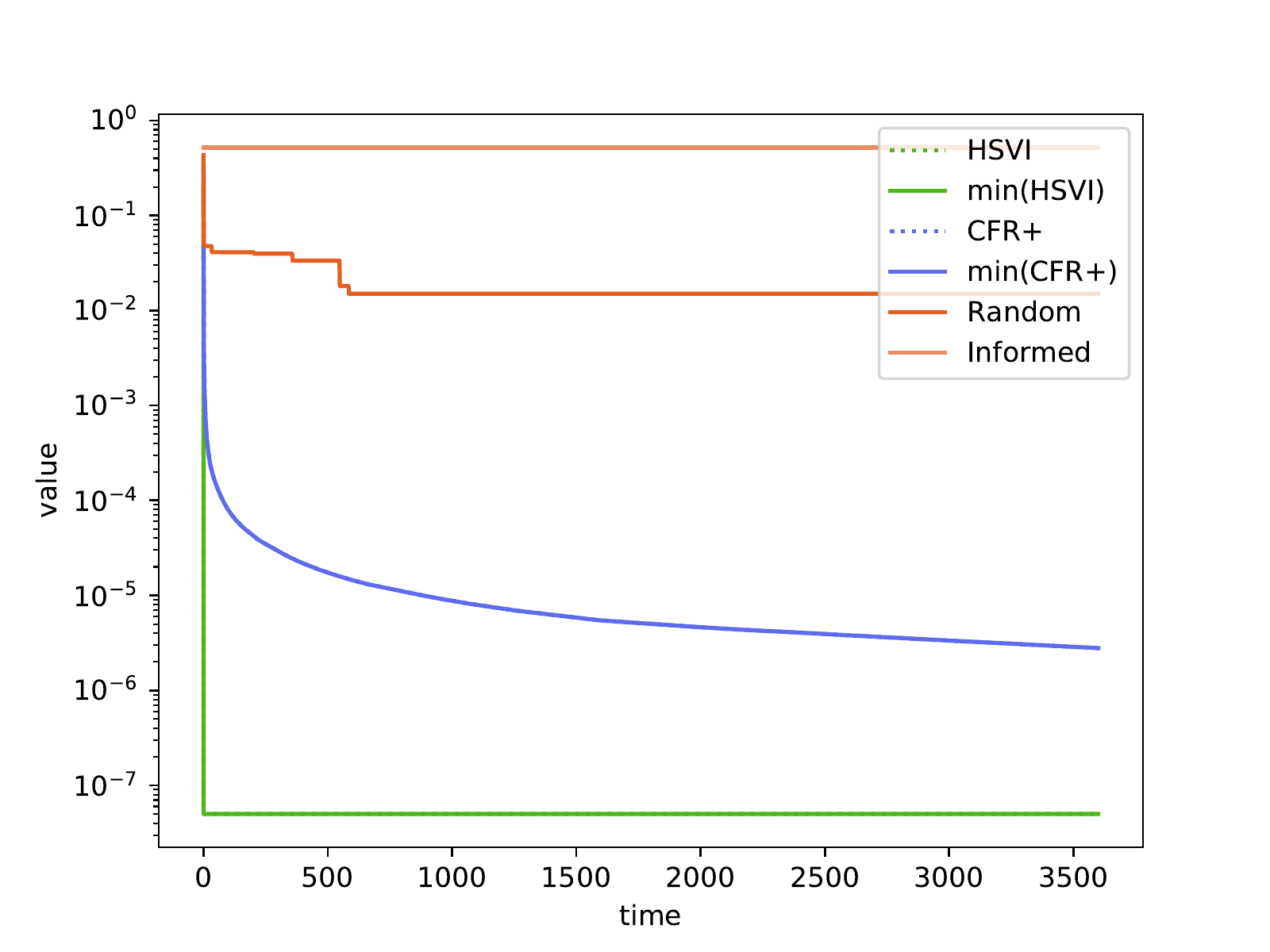}%
        
    \includegraphics[width=\evolGraphScale\linewidth]{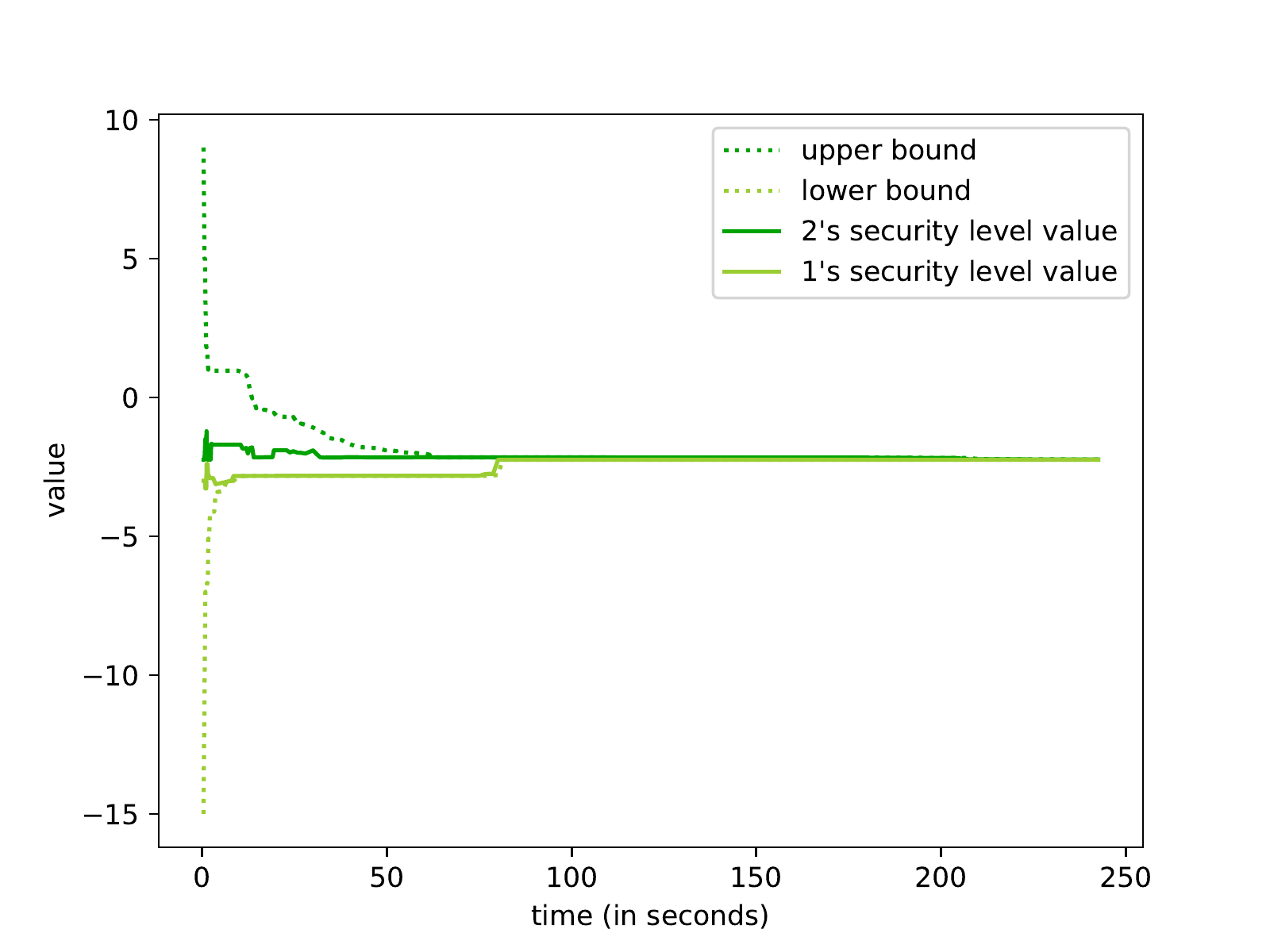}
    \hfill % H=3
    \includegraphics[width=\evolGraphScale\linewidth]{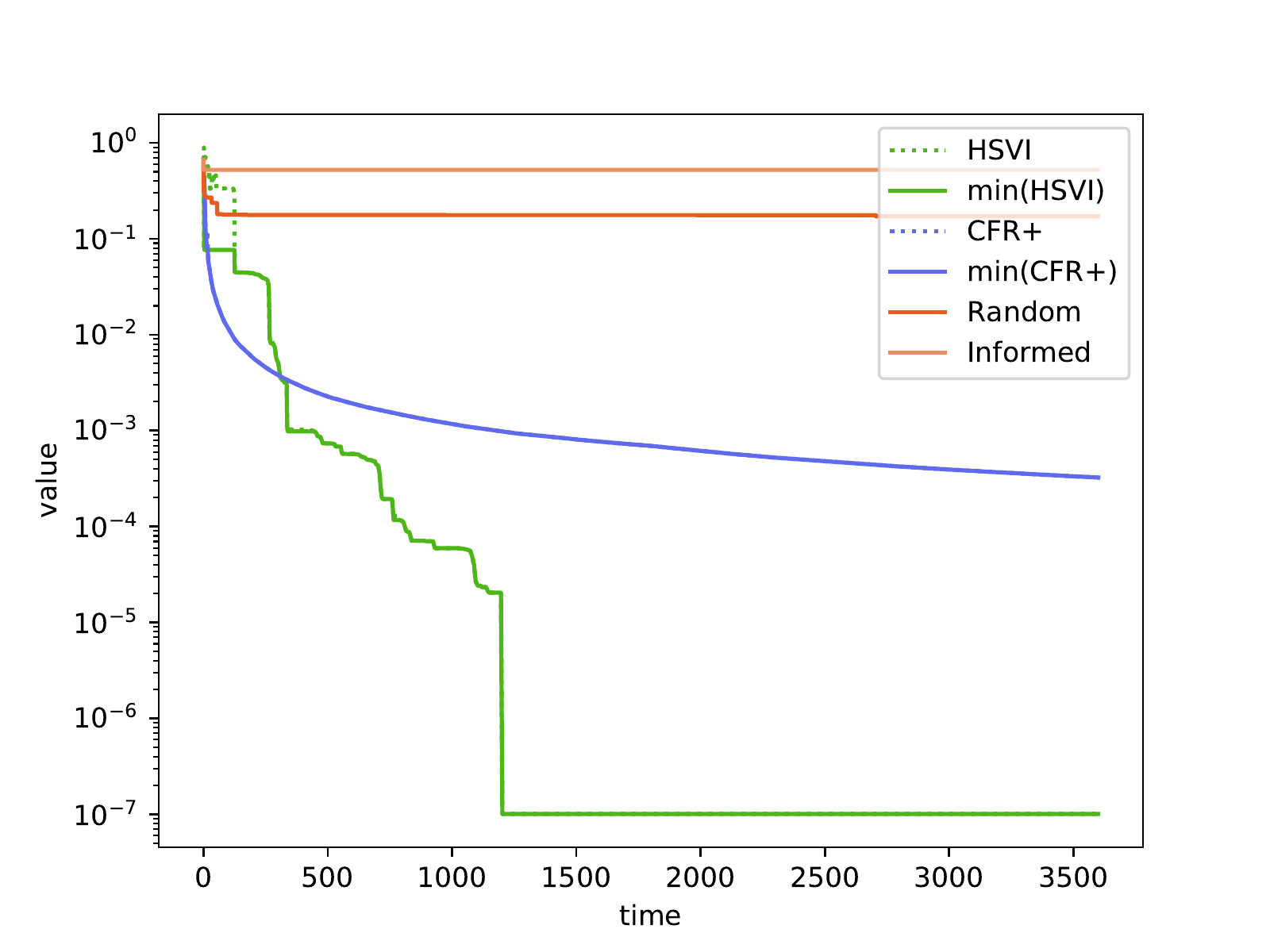}%
    
    \includegraphics[width=\evolGraphScale\linewidth]{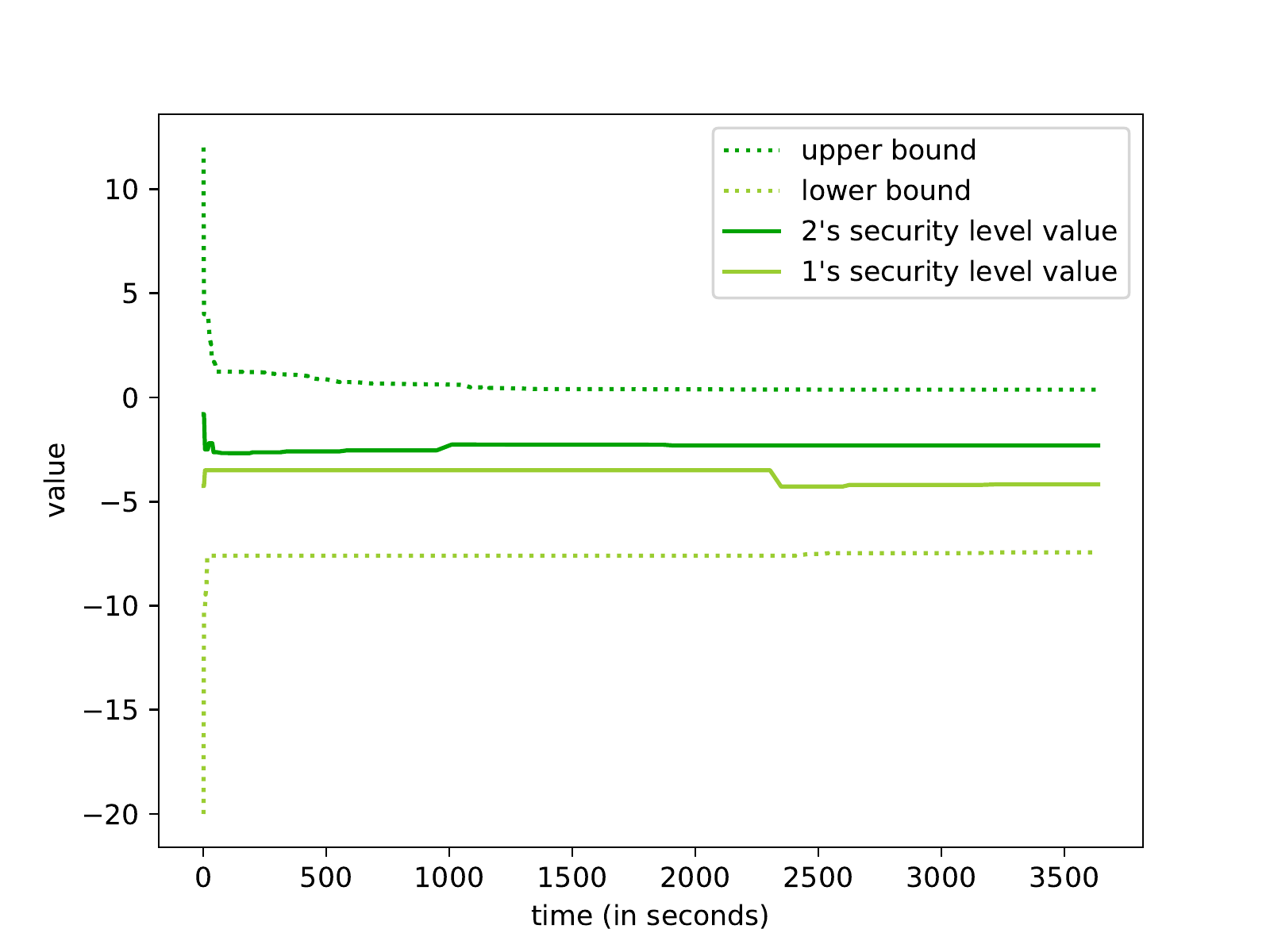}
    \hfill % H=4
    \includegraphics[width=\evolGraphScale\linewidth]{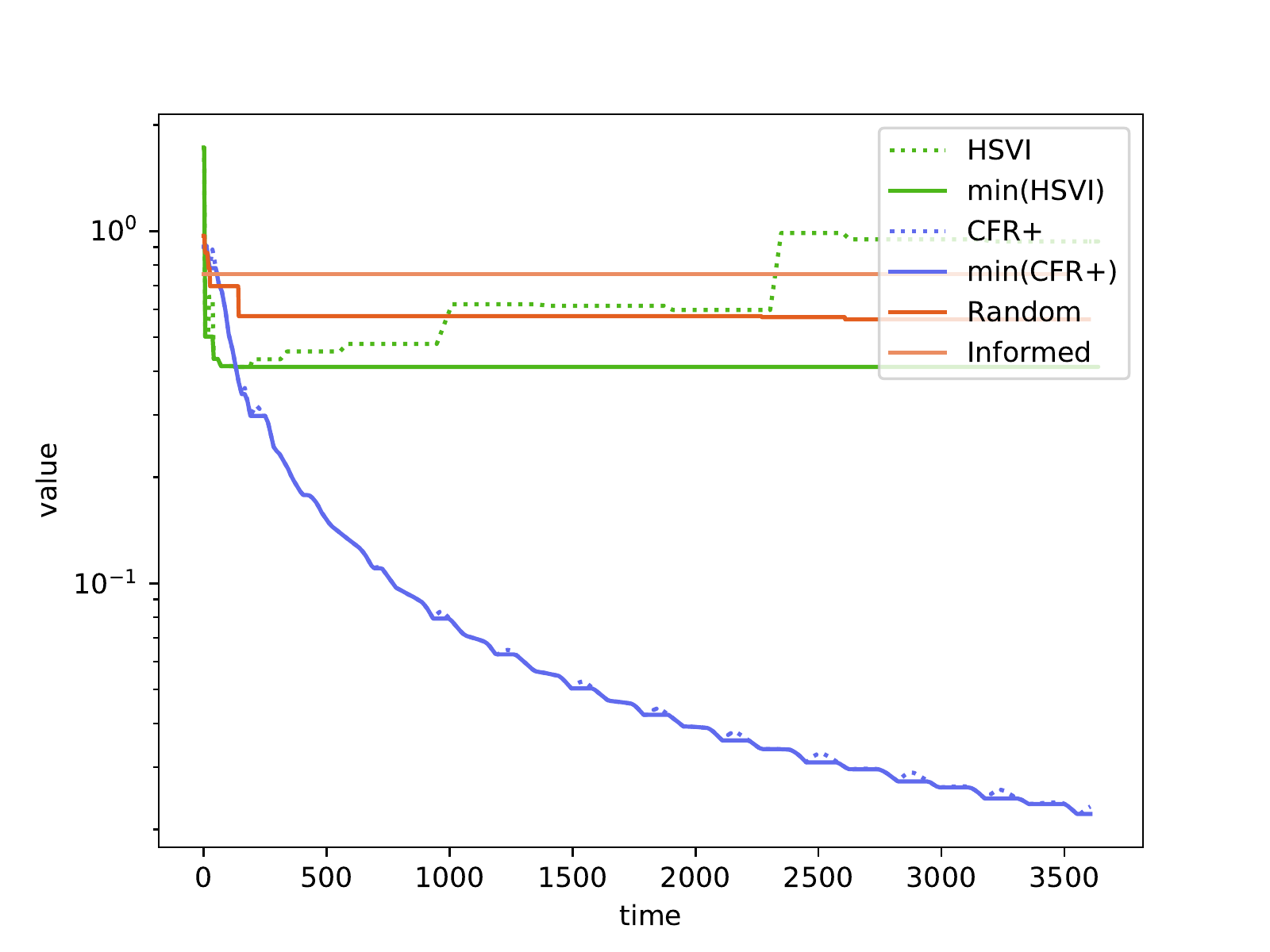}%
        
    \caption{\myEvoExploCaption{Adversarial Tiger ($H=2,3,4$) (1,1,10)}}
    \label{fig|Graphs|AdvTiger}
\end{figure}

\begin{figure}[ht]
    \centering    % CompTiger

    \includegraphics[width=\evolGraphScale\linewidth]{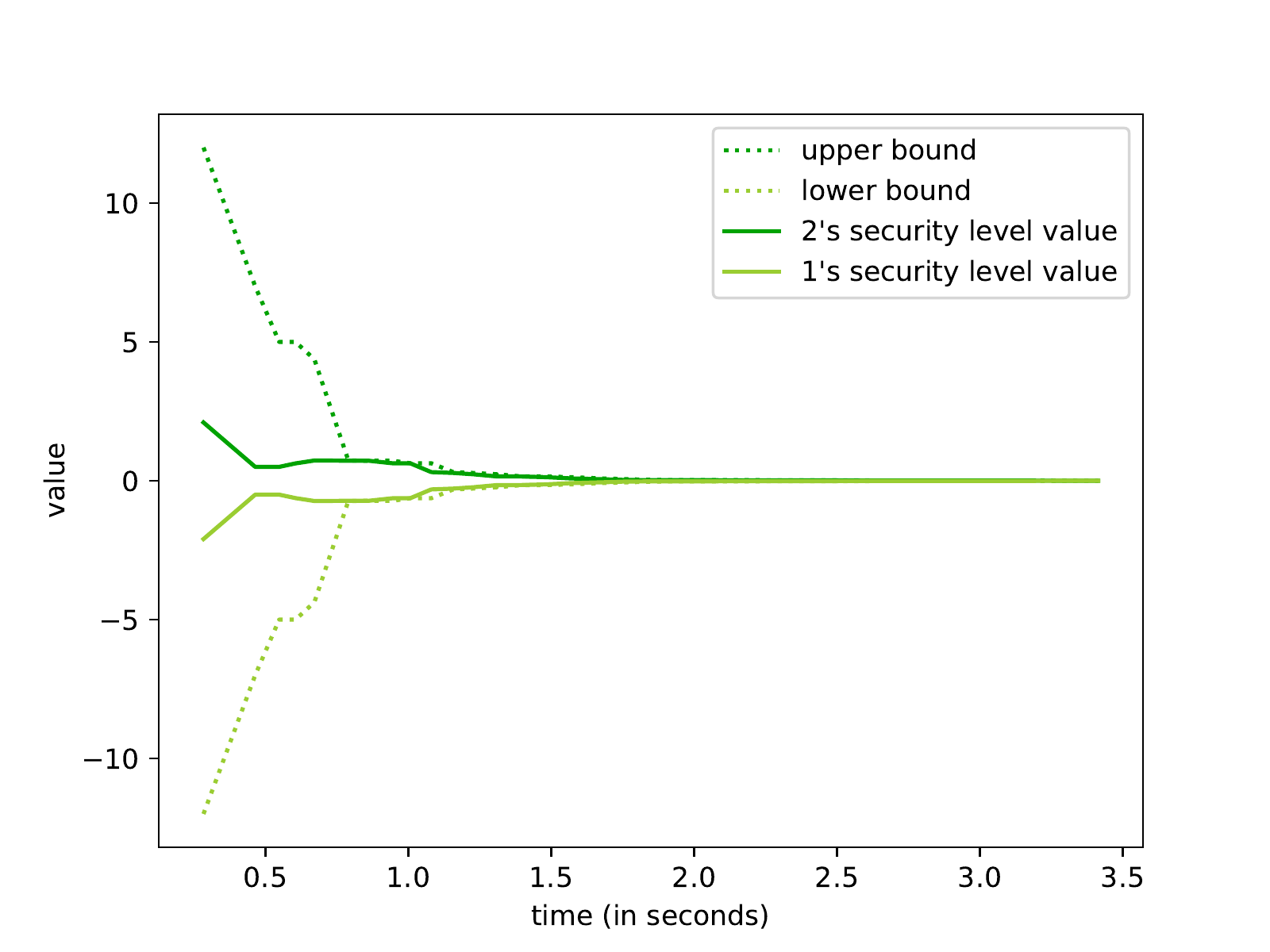}
    \hfill % H=2
    \includegraphics[width=\evolGraphScale\linewidth]{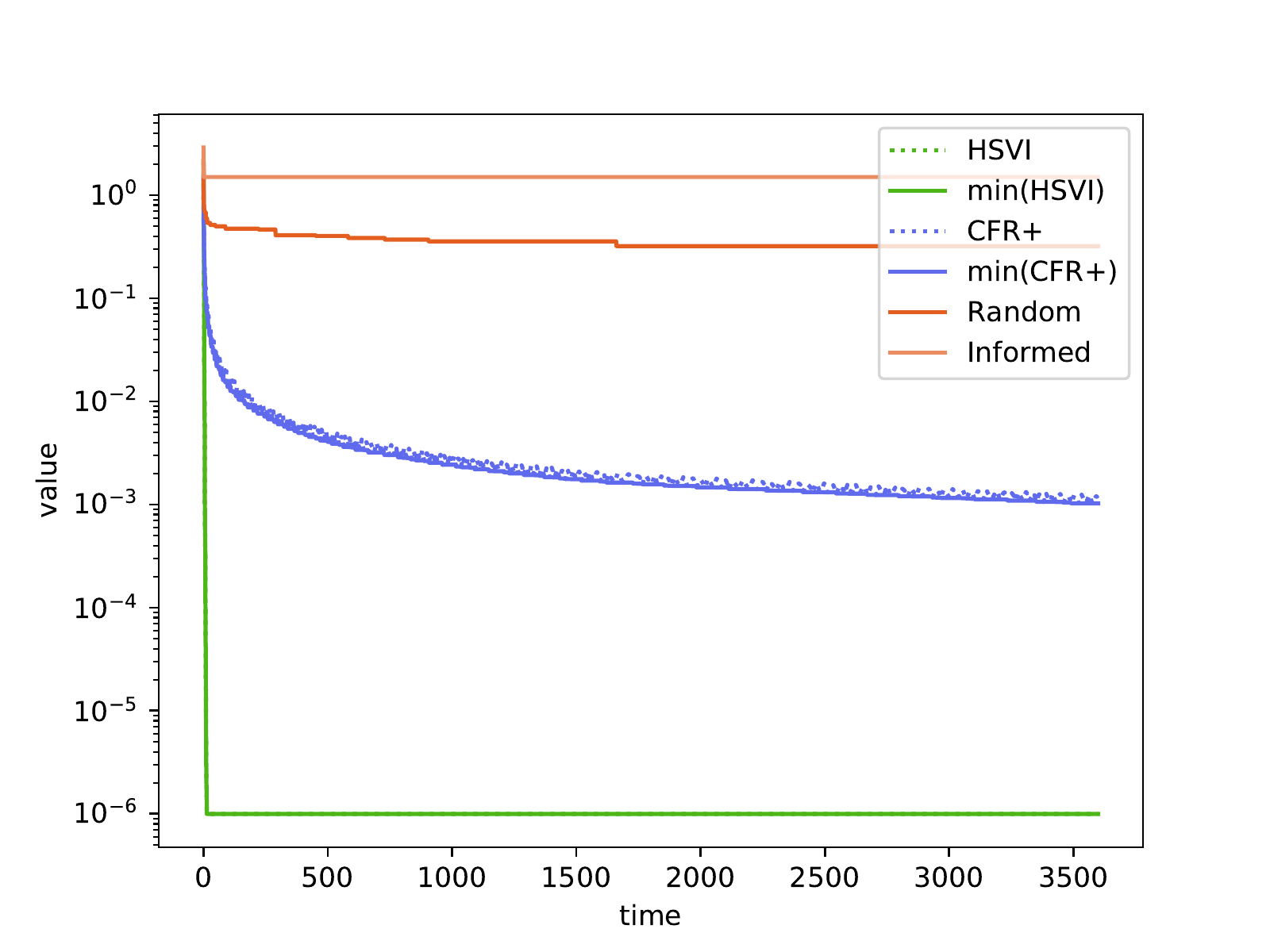}
    
    \includegraphics[width=\evolGraphScale\linewidth]{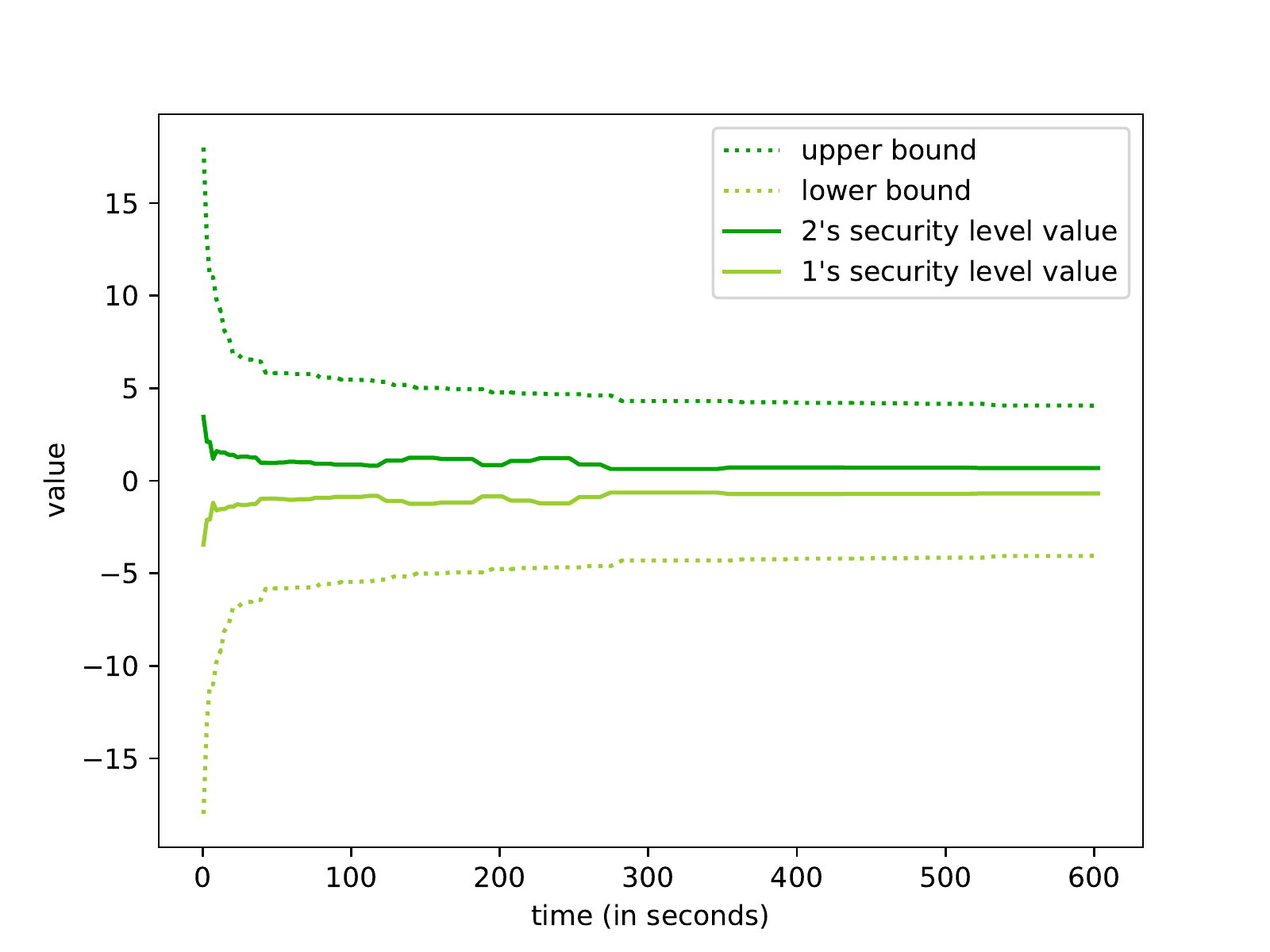}
    \hfill % H=3
    \includegraphics[width=\evolGraphScale\linewidth]{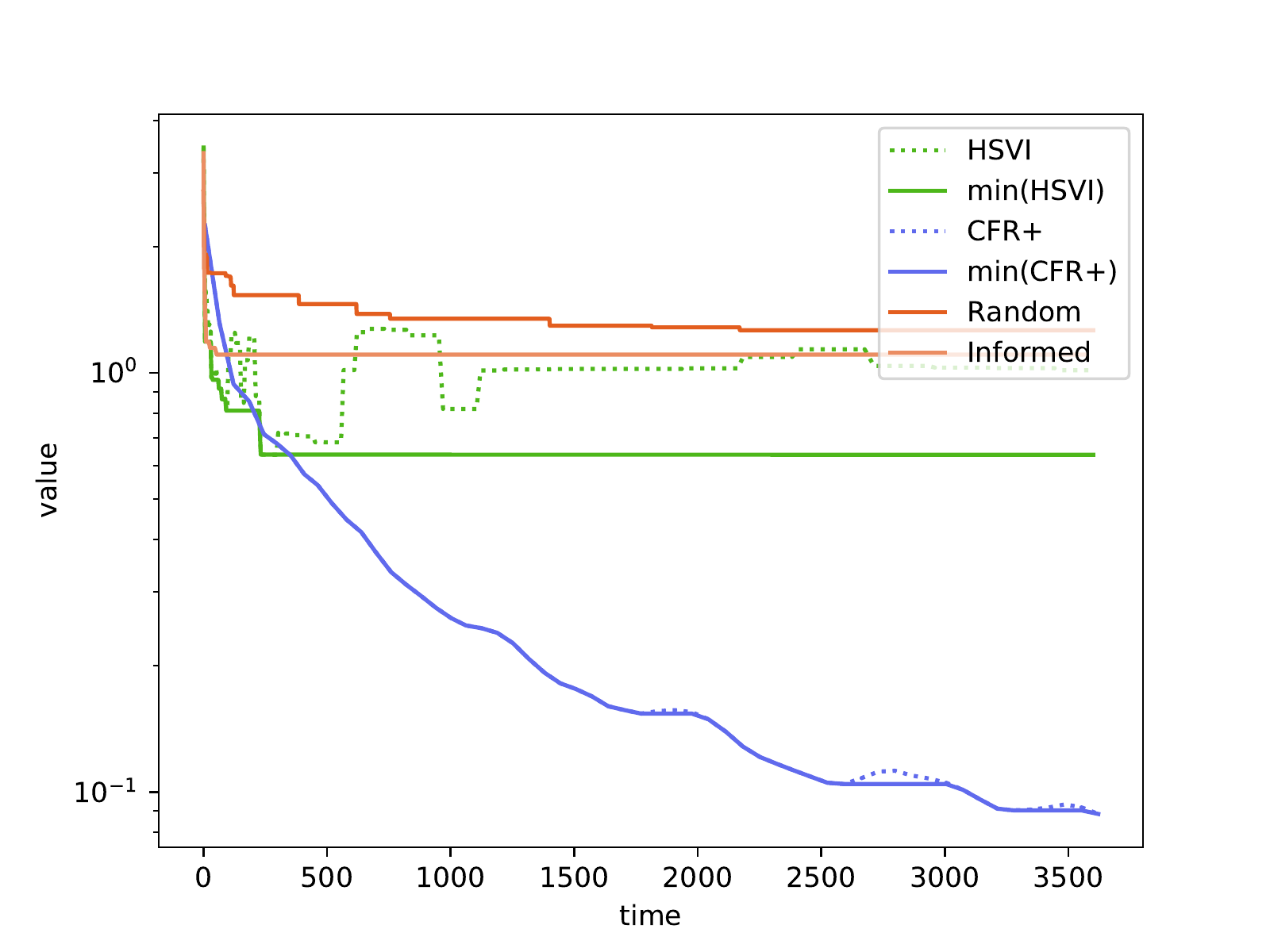}
    
    \includegraphics[width=\evolGraphScale\linewidth]{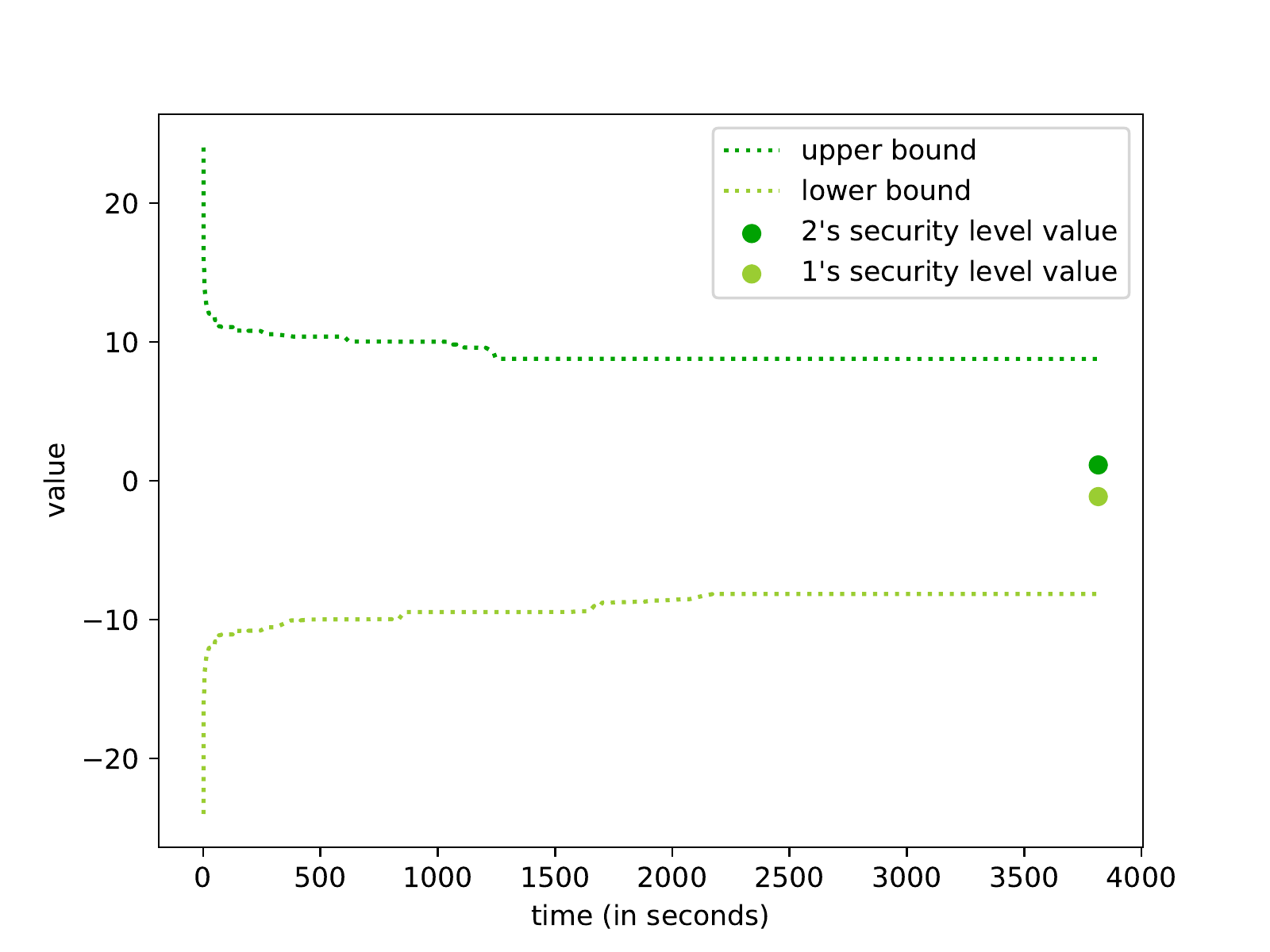}%
    \hfill % H=4
    ~ 
    
    \caption{\myEvoExploCaption{Competitive Tiger ($H=2,3$) (1,1)}}
    \label{fig|Graphs|CompTiger}
\end{figure}

\begin{figure}[ht]
    \centering    % Mabc

    \includegraphics[width=\evolGraphScale\linewidth]{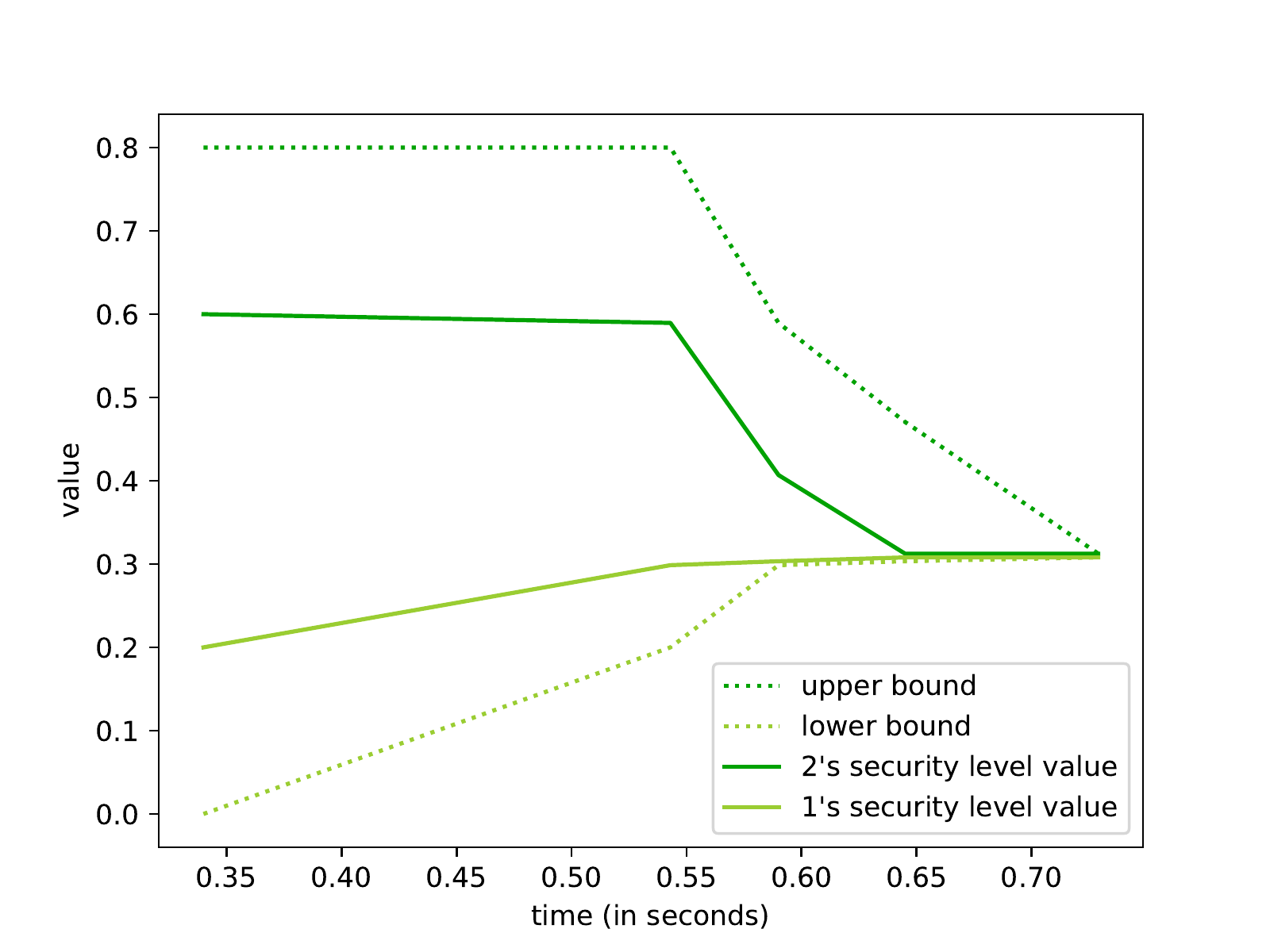}%
    \hfill % H=2
    \includegraphics[width=\evolGraphScale\linewidth]{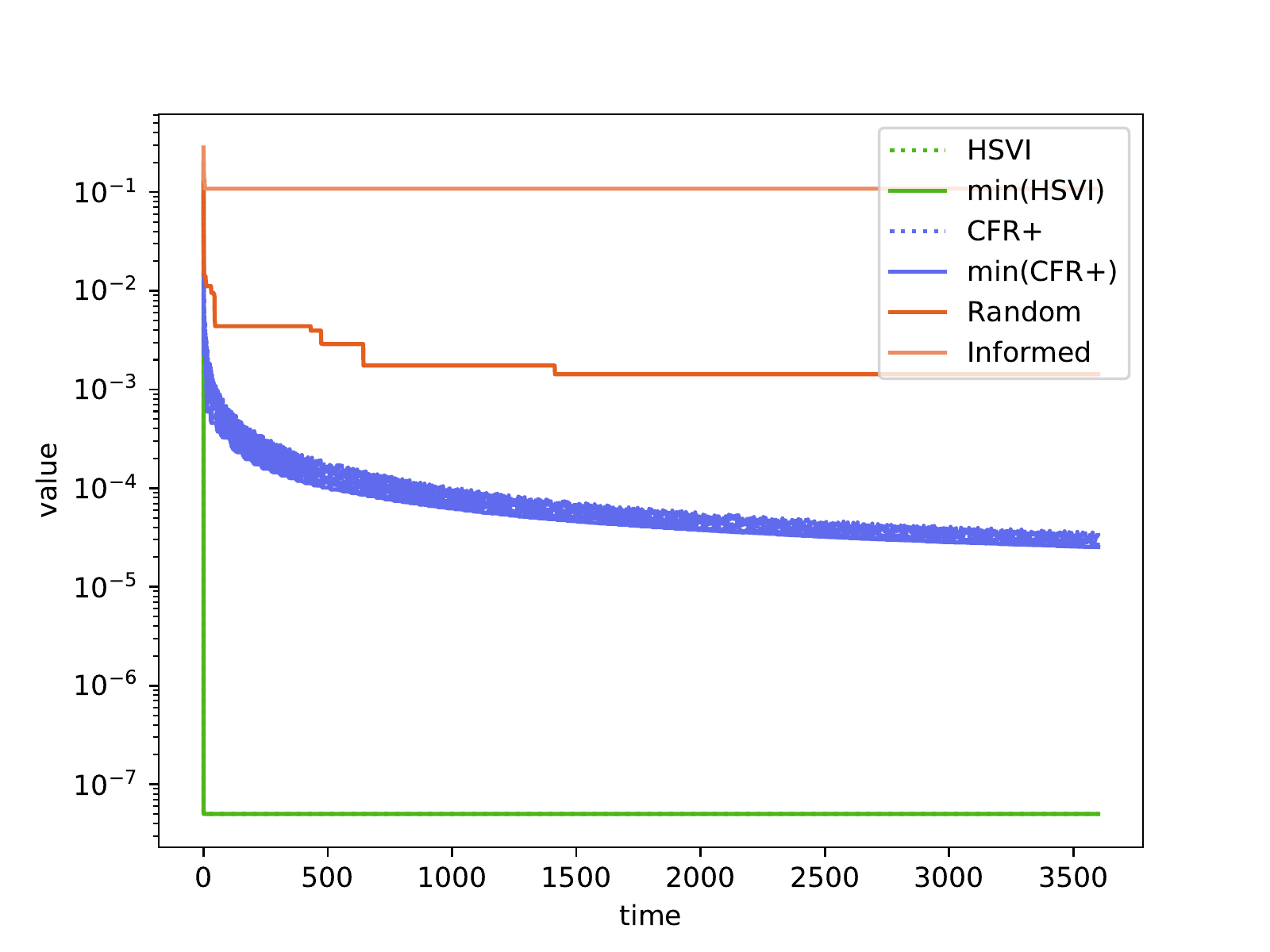}%
    
    \includegraphics[width=\evolGraphScale\linewidth]{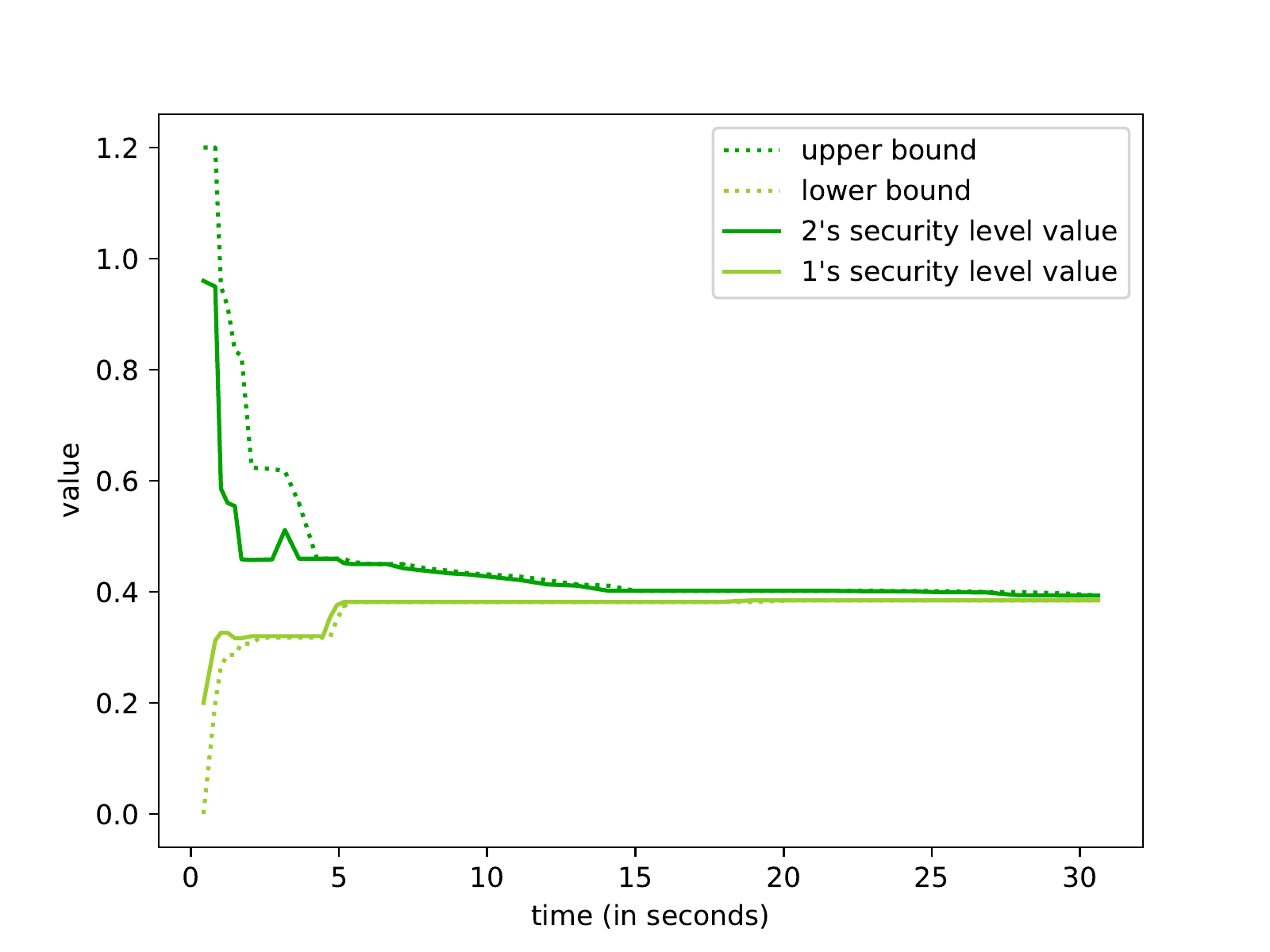}%
    \hfill % H=3
    \includegraphics[width=\evolGraphScale\linewidth]{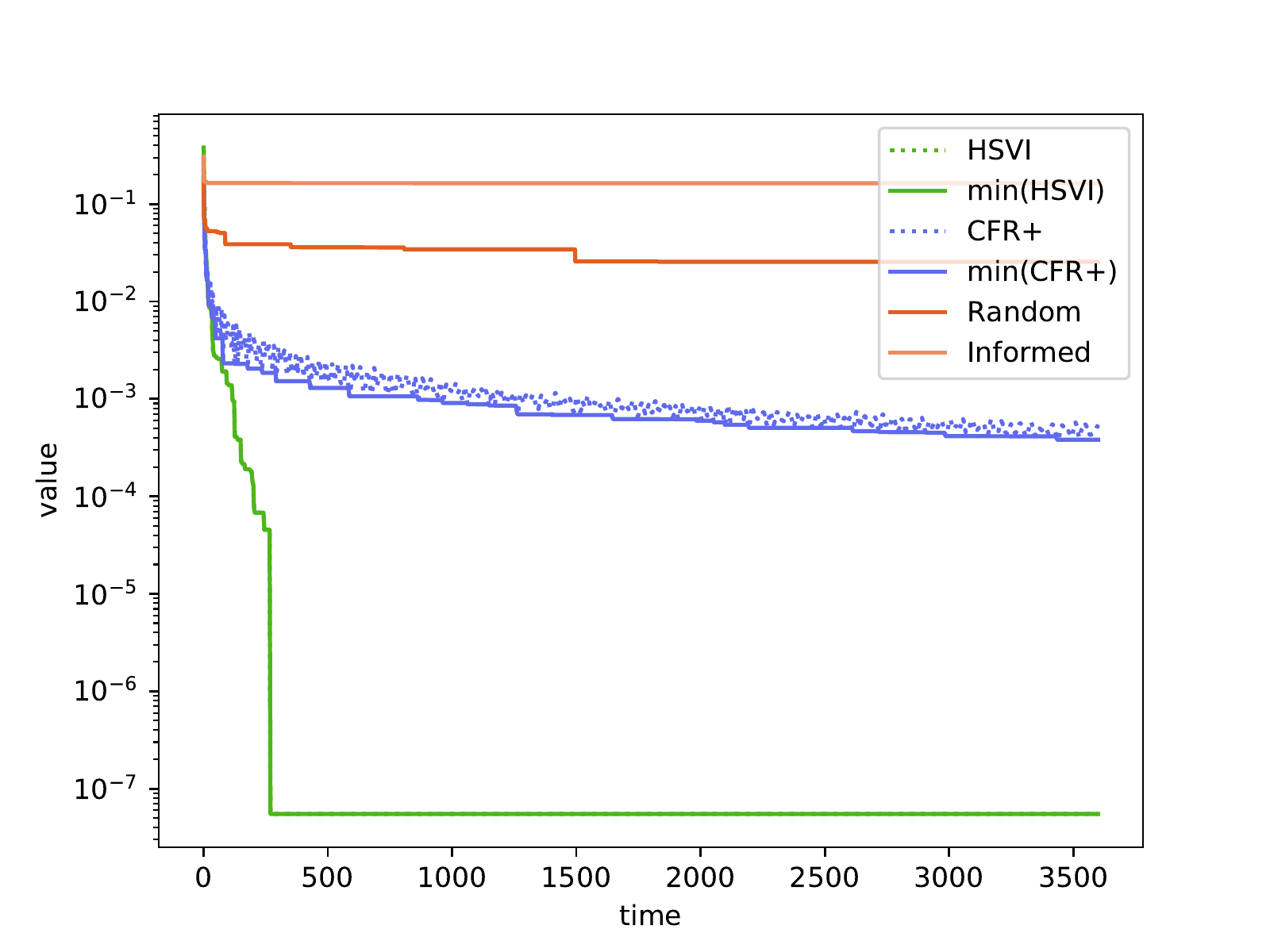}%
    
    \includegraphics[width=\evolGraphScale\linewidth]{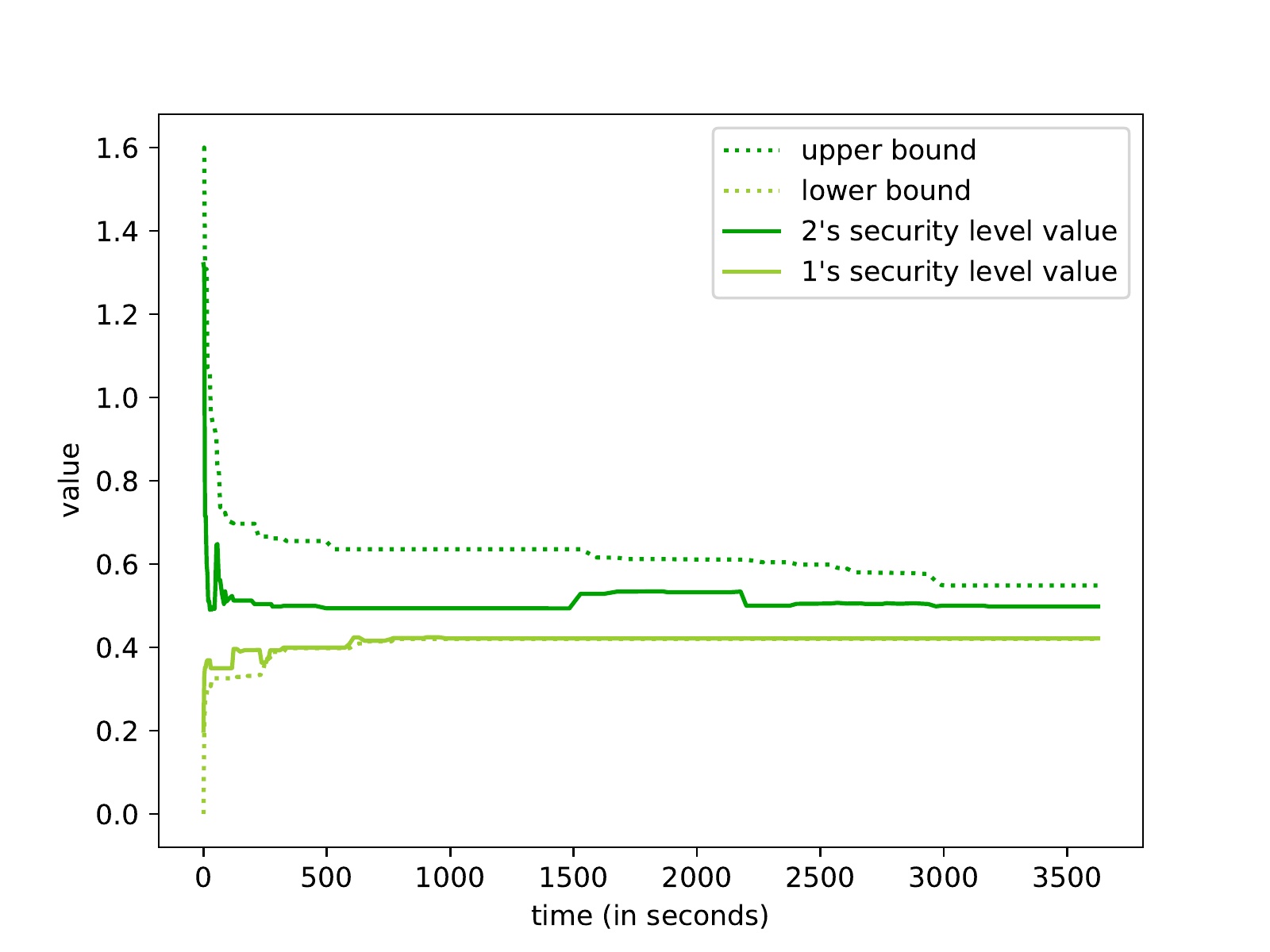}%
    \hfill % H=4
    \includegraphics[width=\evolGraphScale\linewidth]{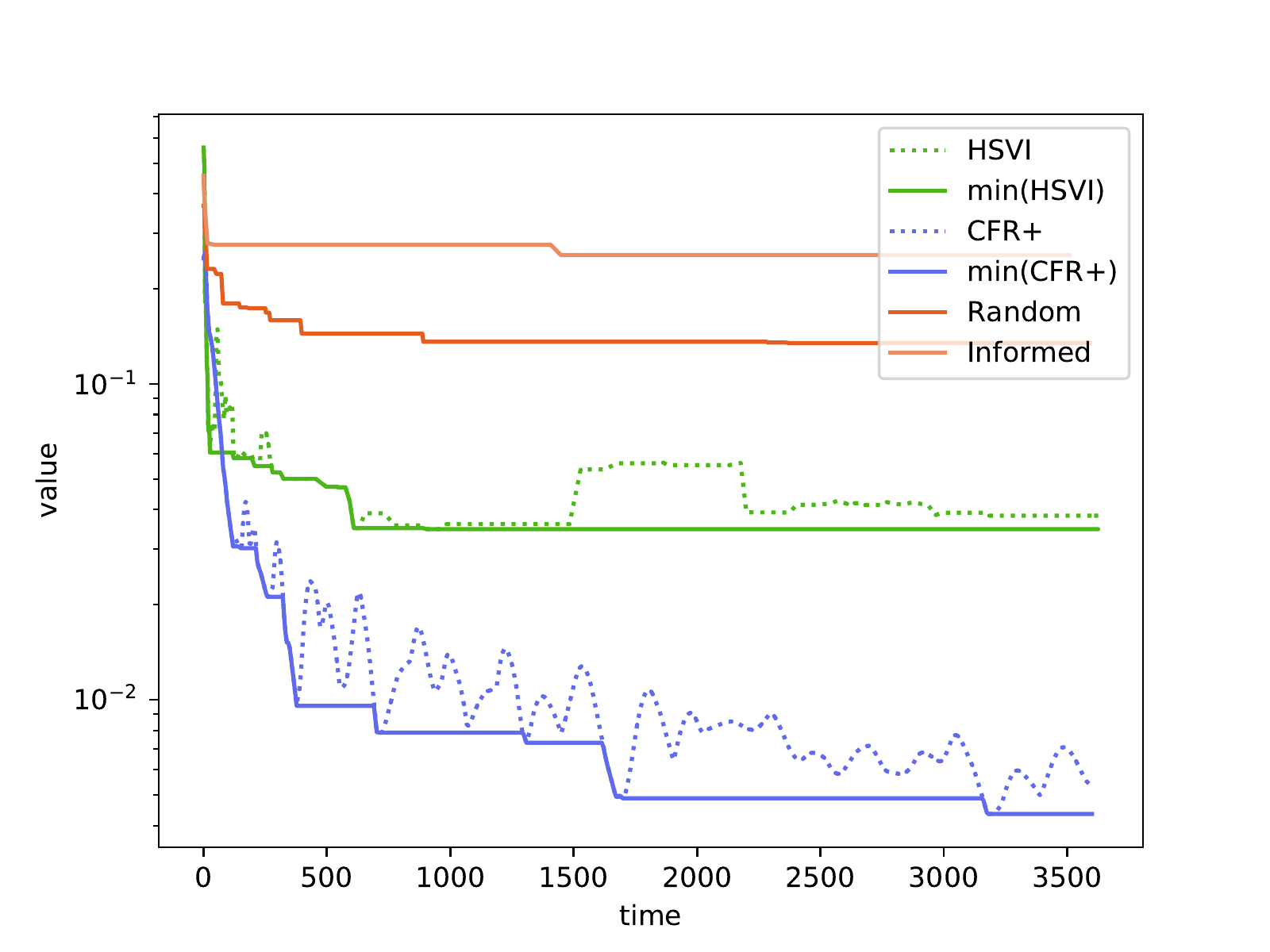}%
    
    \caption{\myEvoExploCaption{Mabc ($H=2,3,4$) (1,1,10)}}
    \label{fig|Graphs|Mabc}
\end{figure}

\begin{figure}[ht]
    \centering    % Recycling
    
    \includegraphics[width=\evolGraphScale\linewidth]{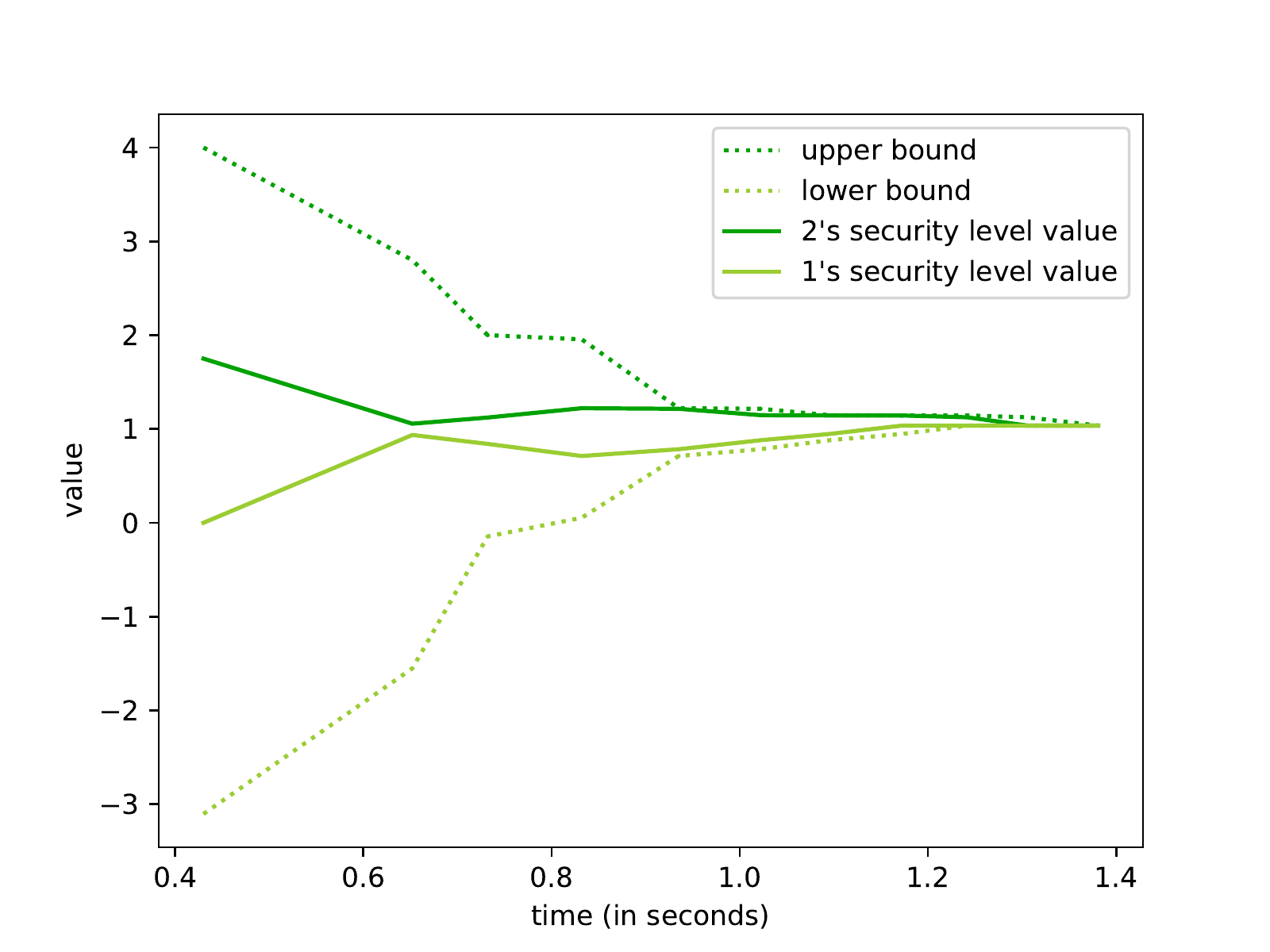}%
    \hfill
    \includegraphics[width=\evolGraphScale\linewidth]{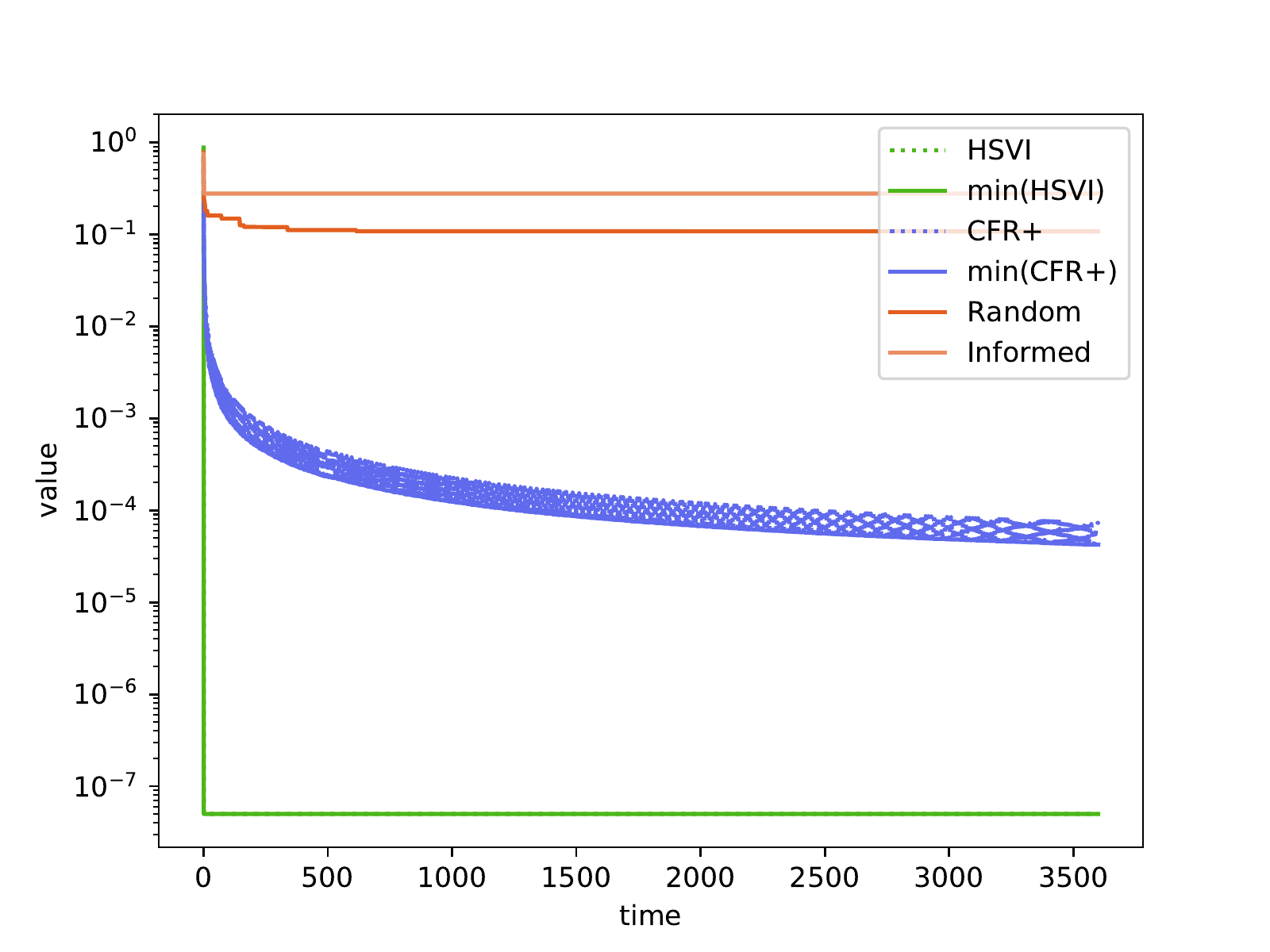}
    
    \includegraphics[width=\evolGraphScale\linewidth]{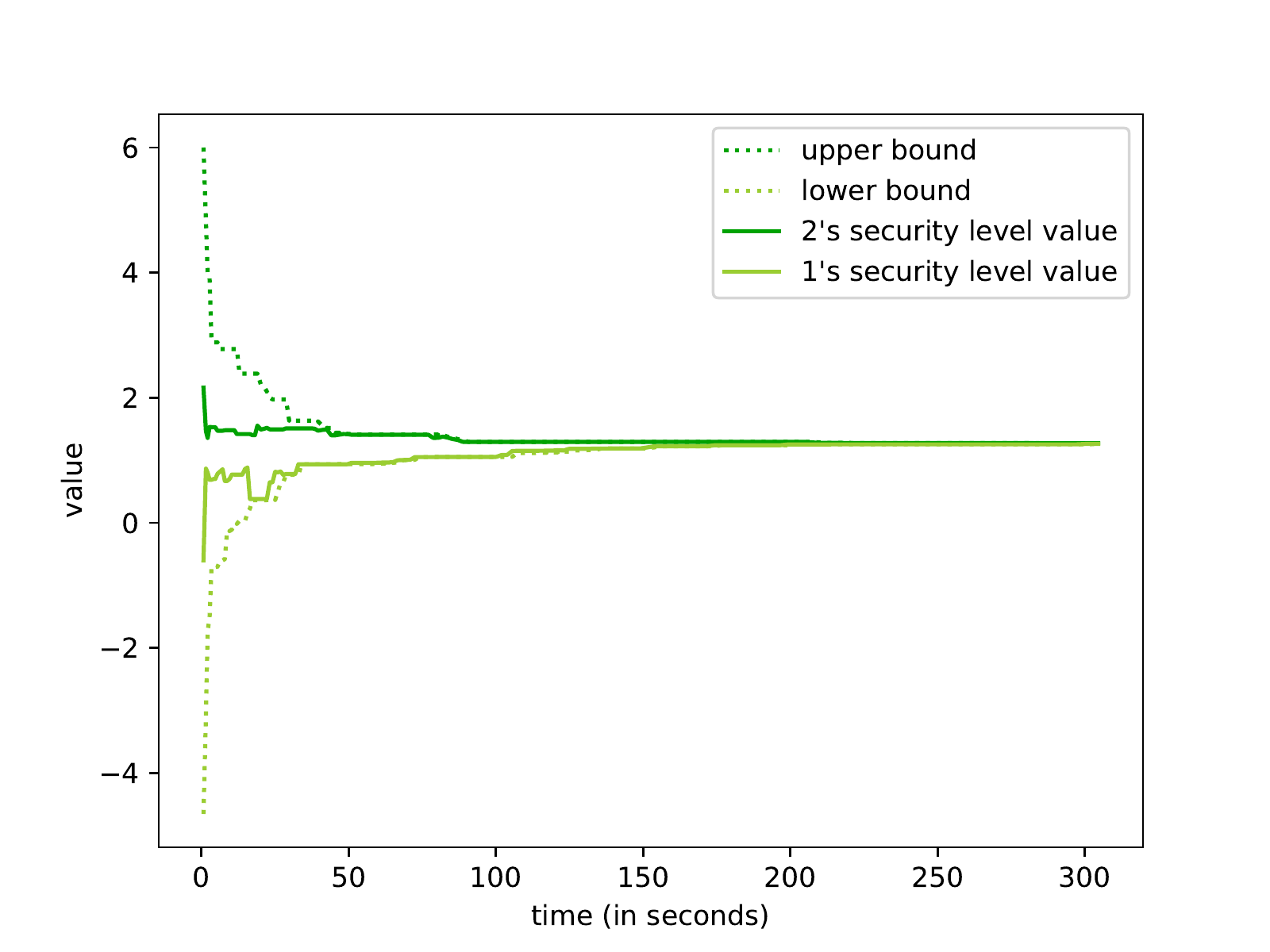}
    \hfill
    \includegraphics[width=\evolGraphScale\linewidth]{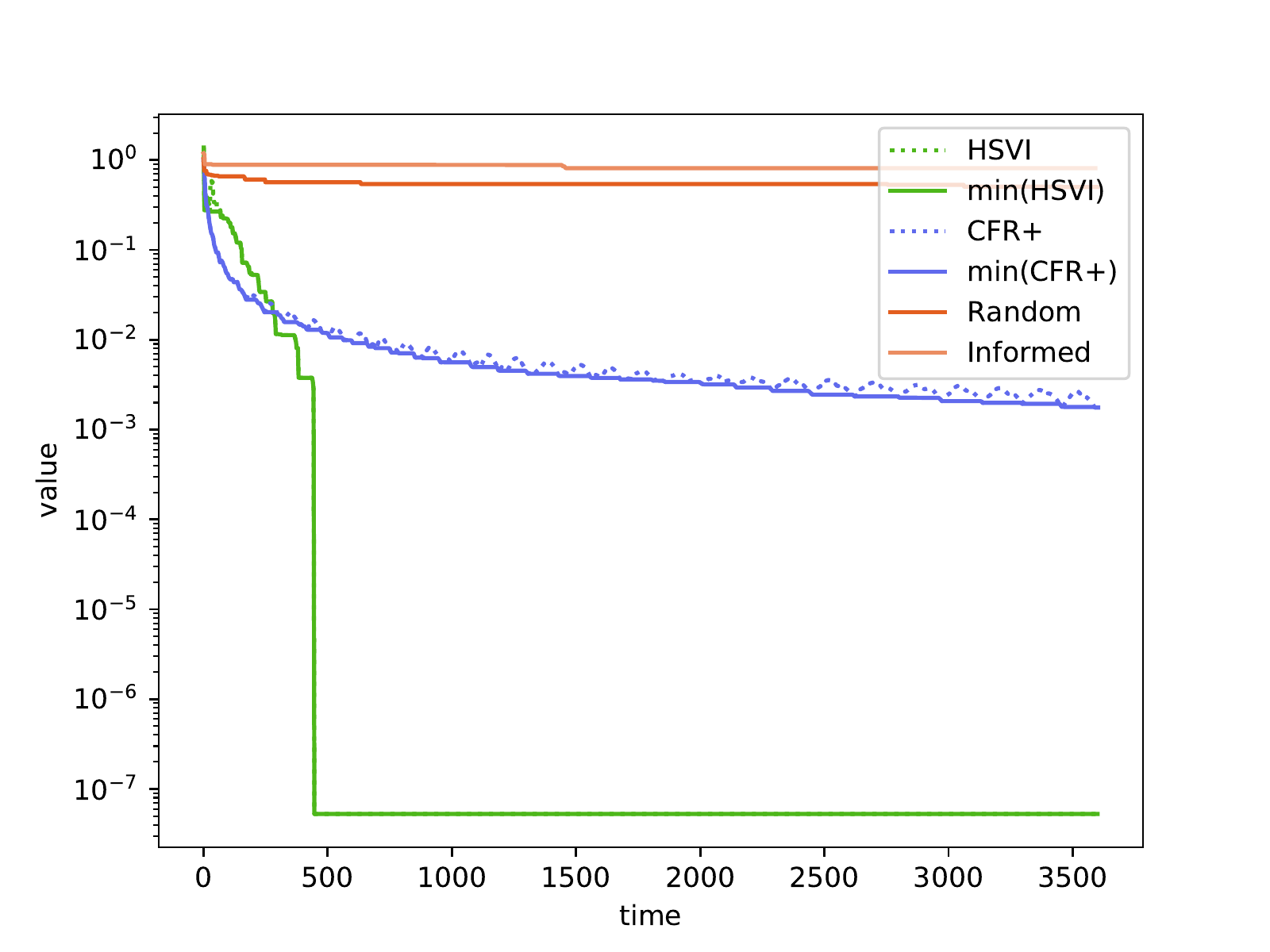}
    
    \includegraphics[width=\evolGraphScale\linewidth]{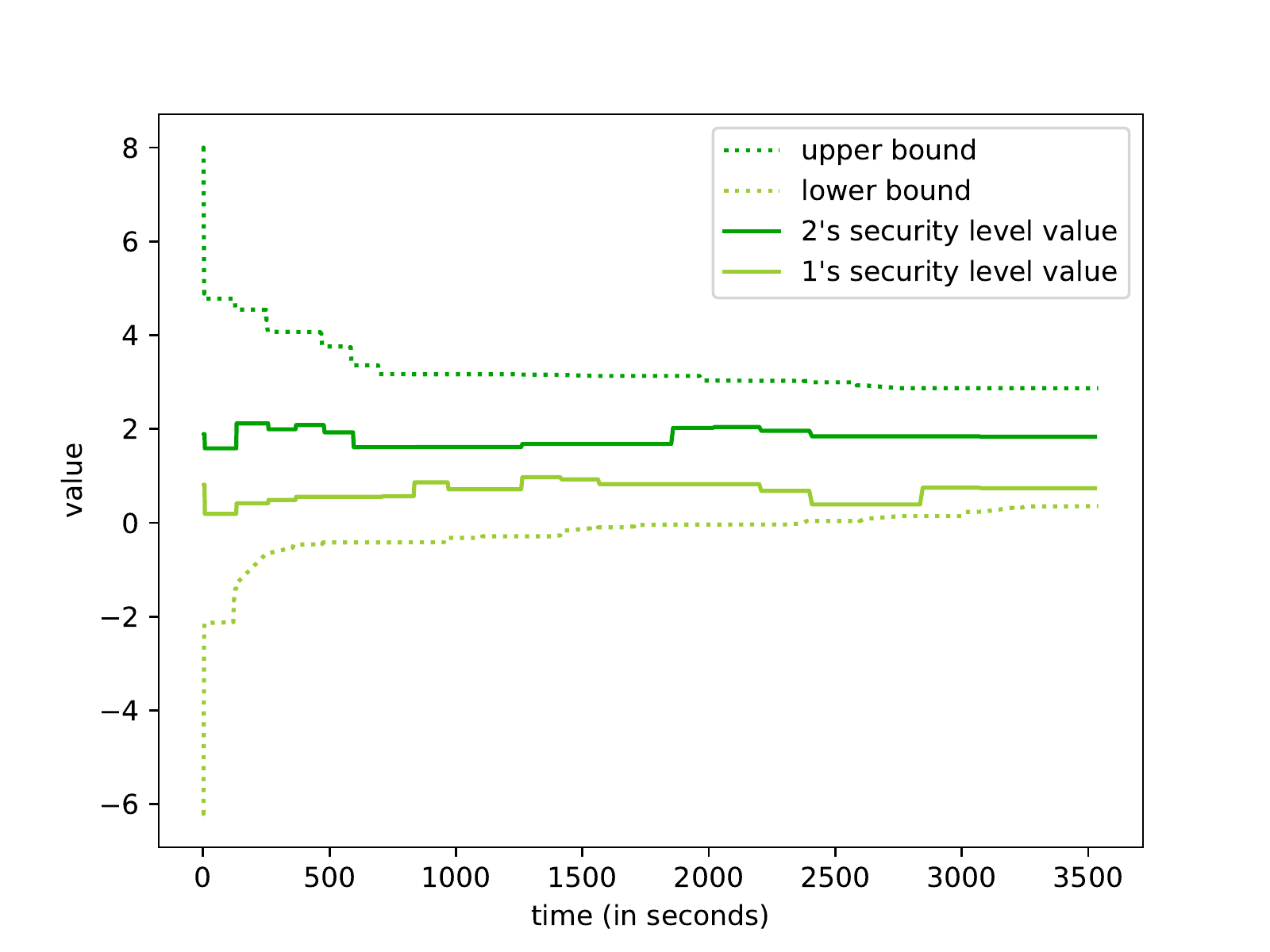}
    \hfill
    \includegraphics[width=\evolGraphScale\linewidth]{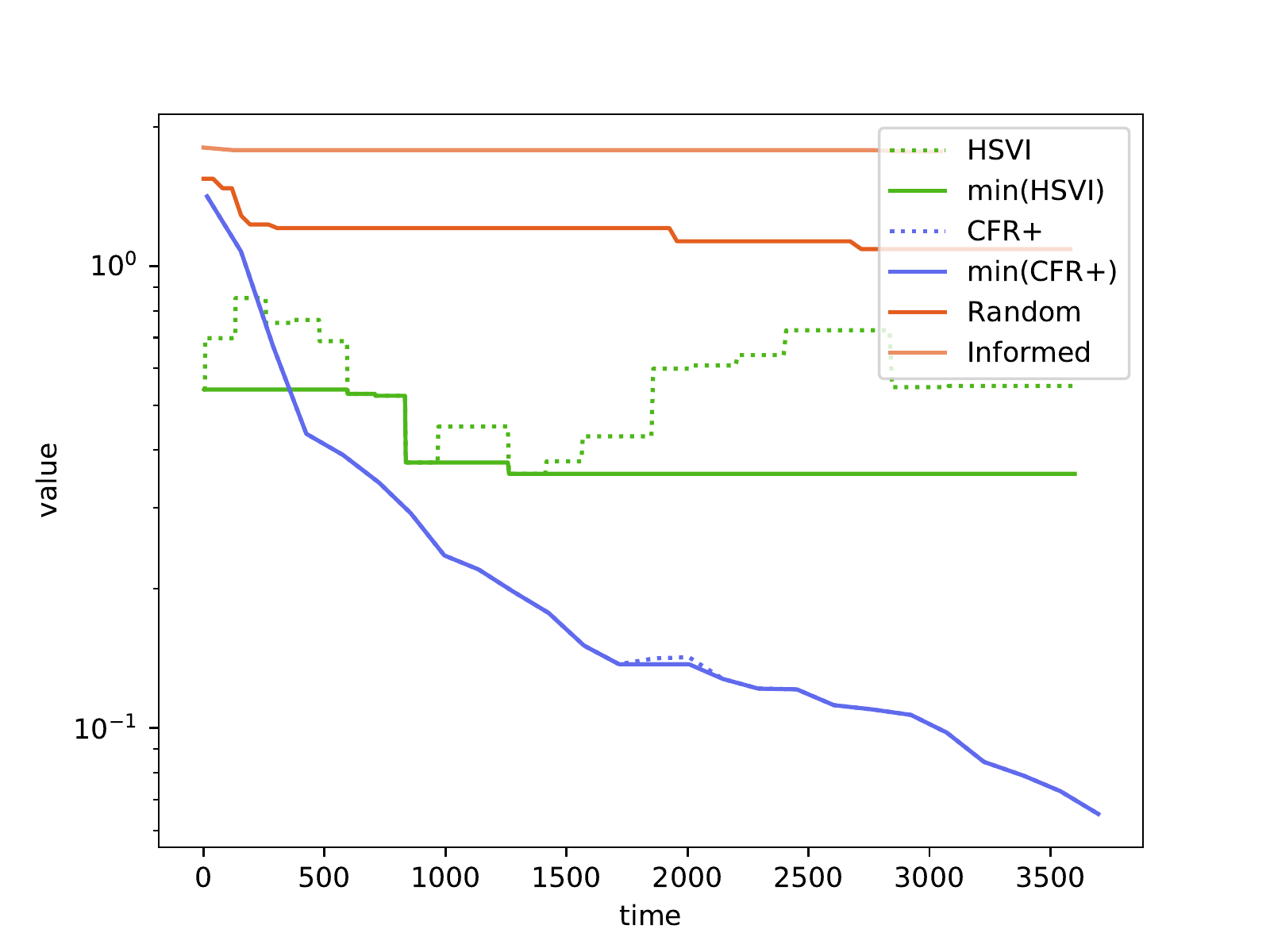}

    \hfill
    ~ 

    \hfill
    ~ 
    
    \caption{\myEvoExploCaption{Recycling Robot ($H=2,3,4$) (1,1,10)}}
    \label{fig|Graphs|Recycling}
\end{figure}

\begin{figure}[ht]
    \includegraphics[width=\evolGraphScale\linewidth]{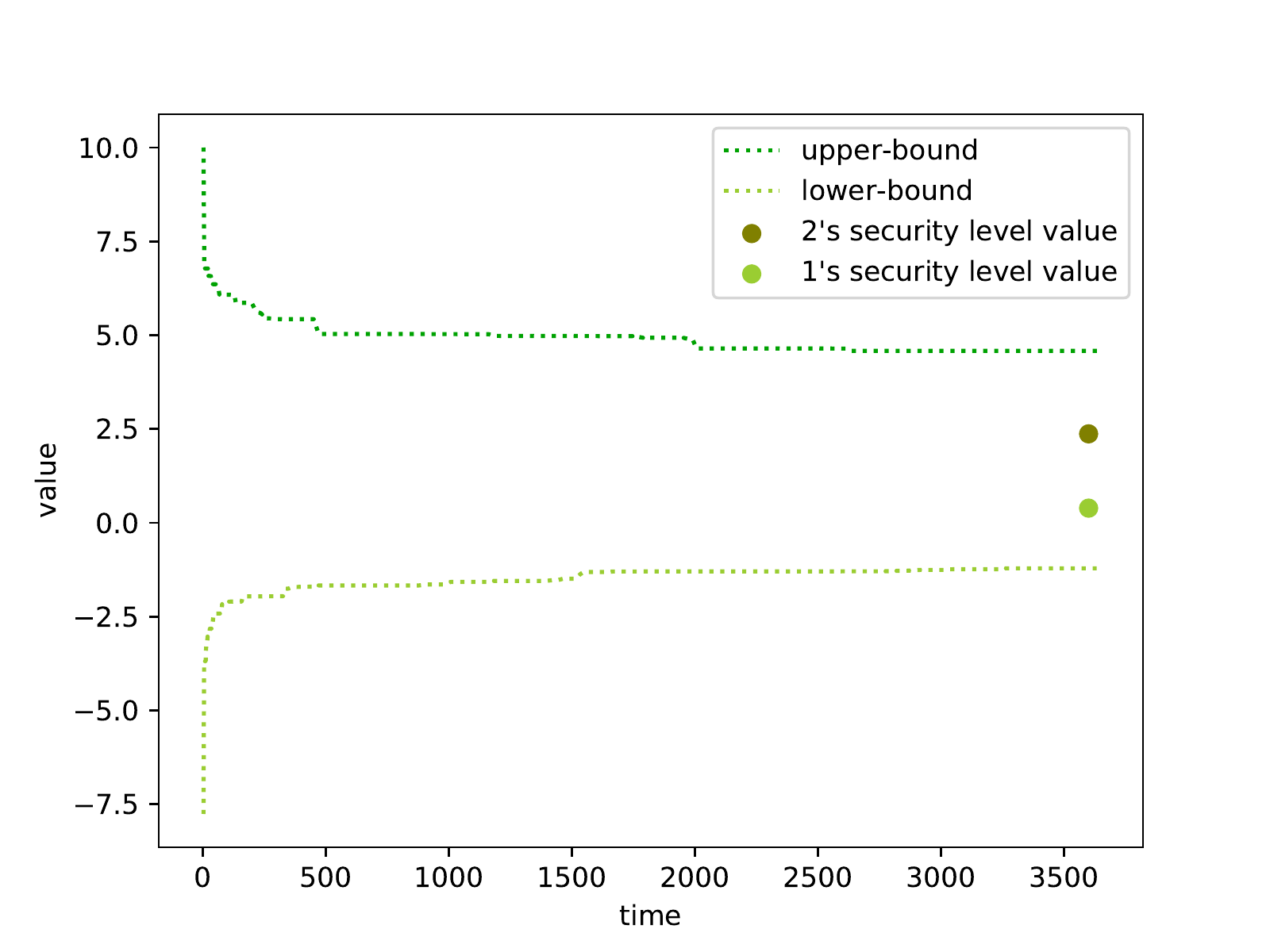}
    %\hfill
    \vfill
    \includegraphics[width=\evolGraphScale\linewidth]{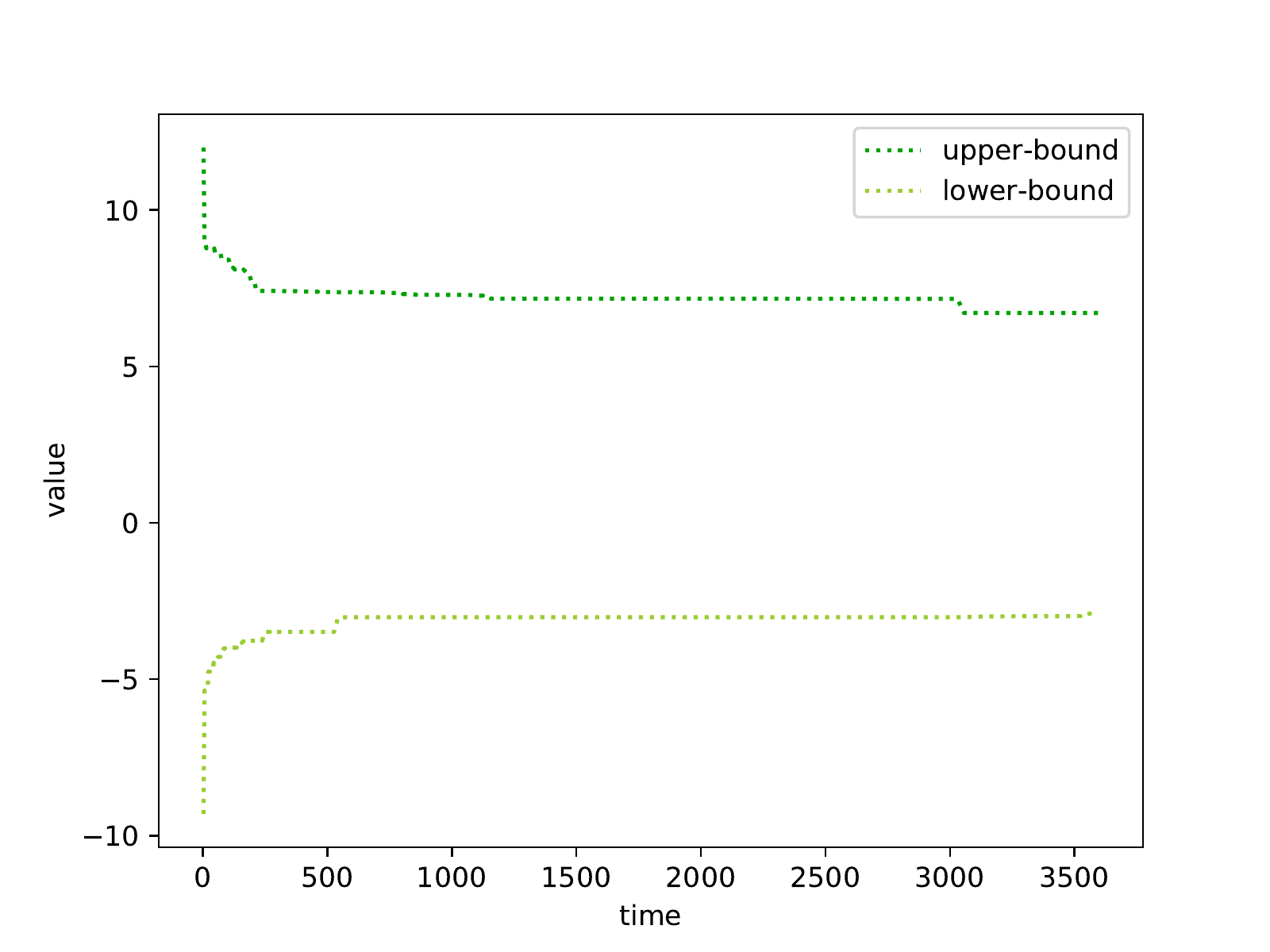}
        %\caption{\myEvoExploCaption{Recycling Robot ($H=5,6$) (1,1,10)}}
    \caption{\uline{\textbf{Recycling Robot ($H=5,6$) (once,none):}} %
    \textbf{(left)} Evolution of (in dotted lines) the upper- and lower-bound values, and (in solid lines) the security levels of the returned strategies for HSVI as a function of time (ms).
    }
    \label{fig|Graphs|RecyclingH56}
\end{figure}

\begin{figure}[ht]
    \centering % Matching Pennies

    \includegraphics[width=\evolGraphScale\linewidth]{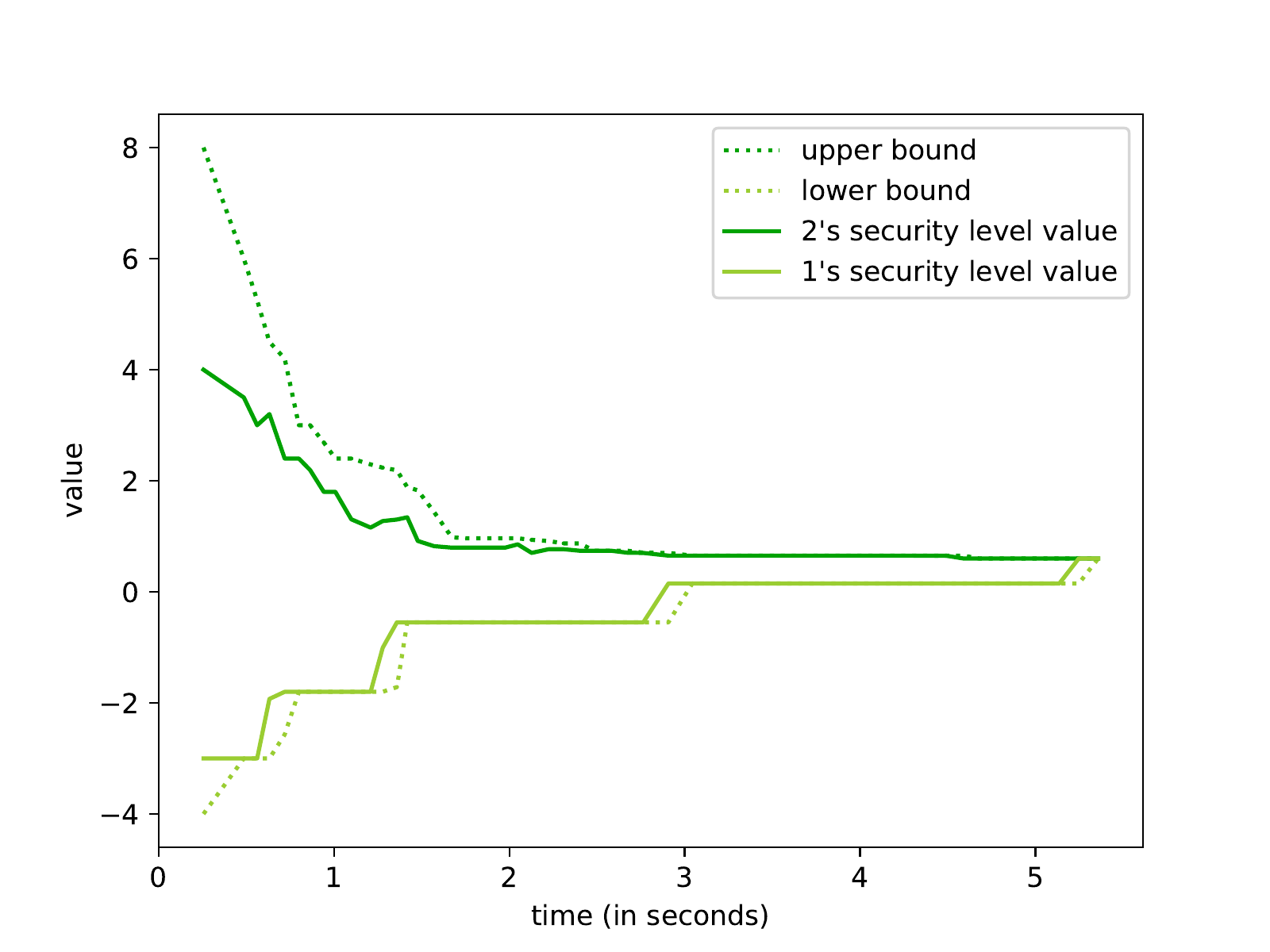}
    \hfill
    \includegraphics[width=\evolGraphScale\linewidth]{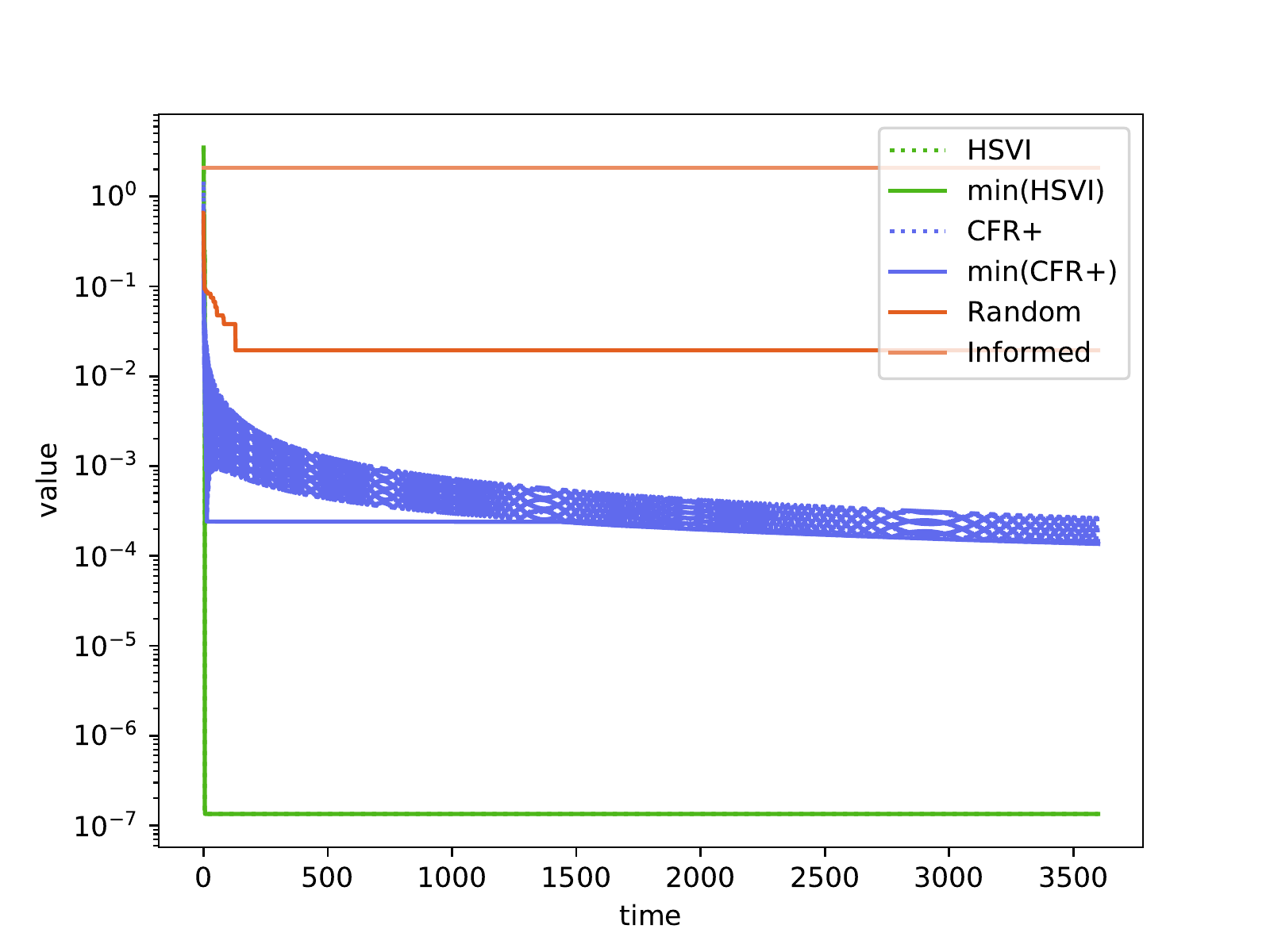}
    
    \includegraphics[width=\evolGraphScale\linewidth]{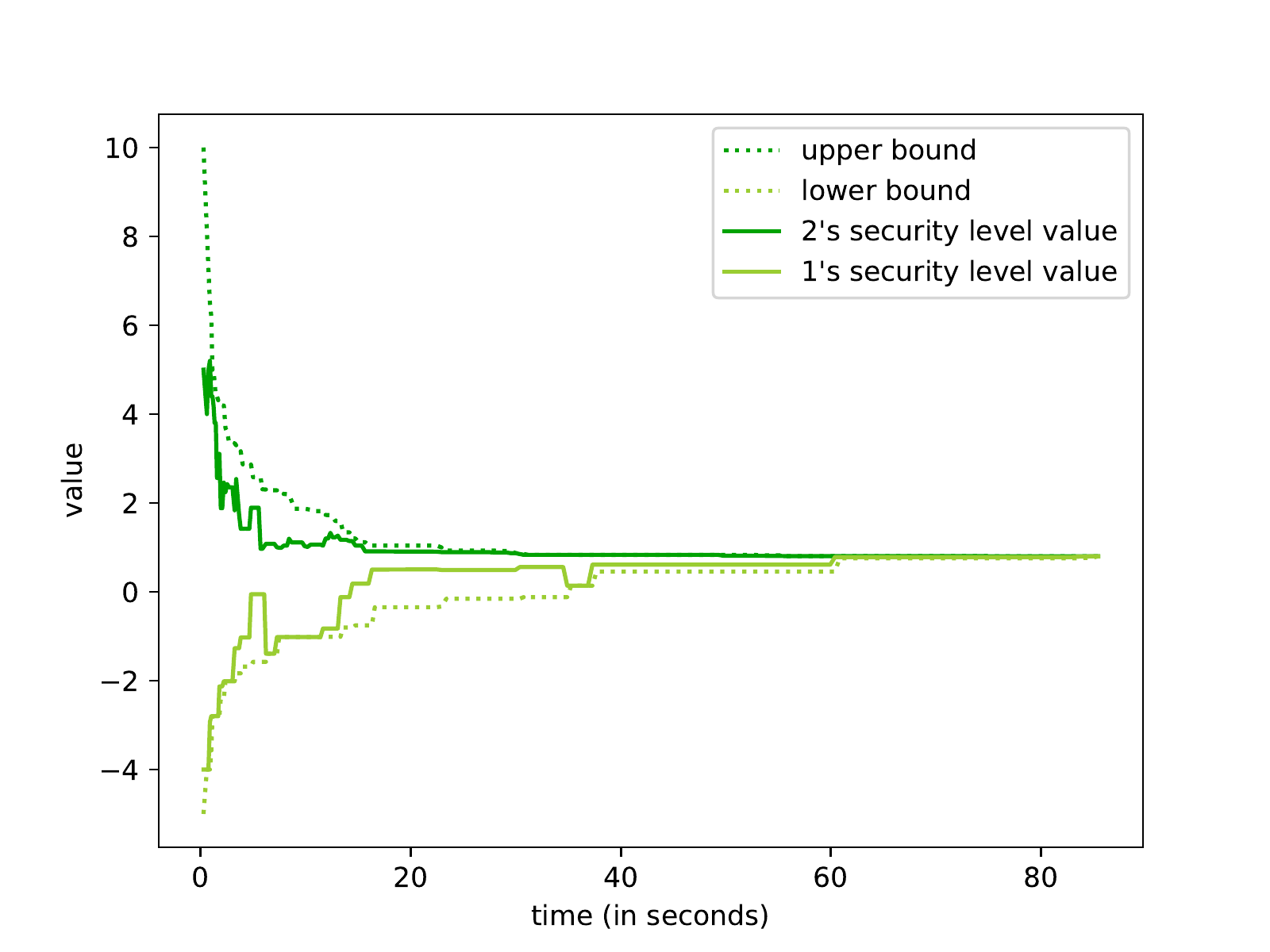}
    \hfill
    \includegraphics[width=\evolGraphScale\linewidth]{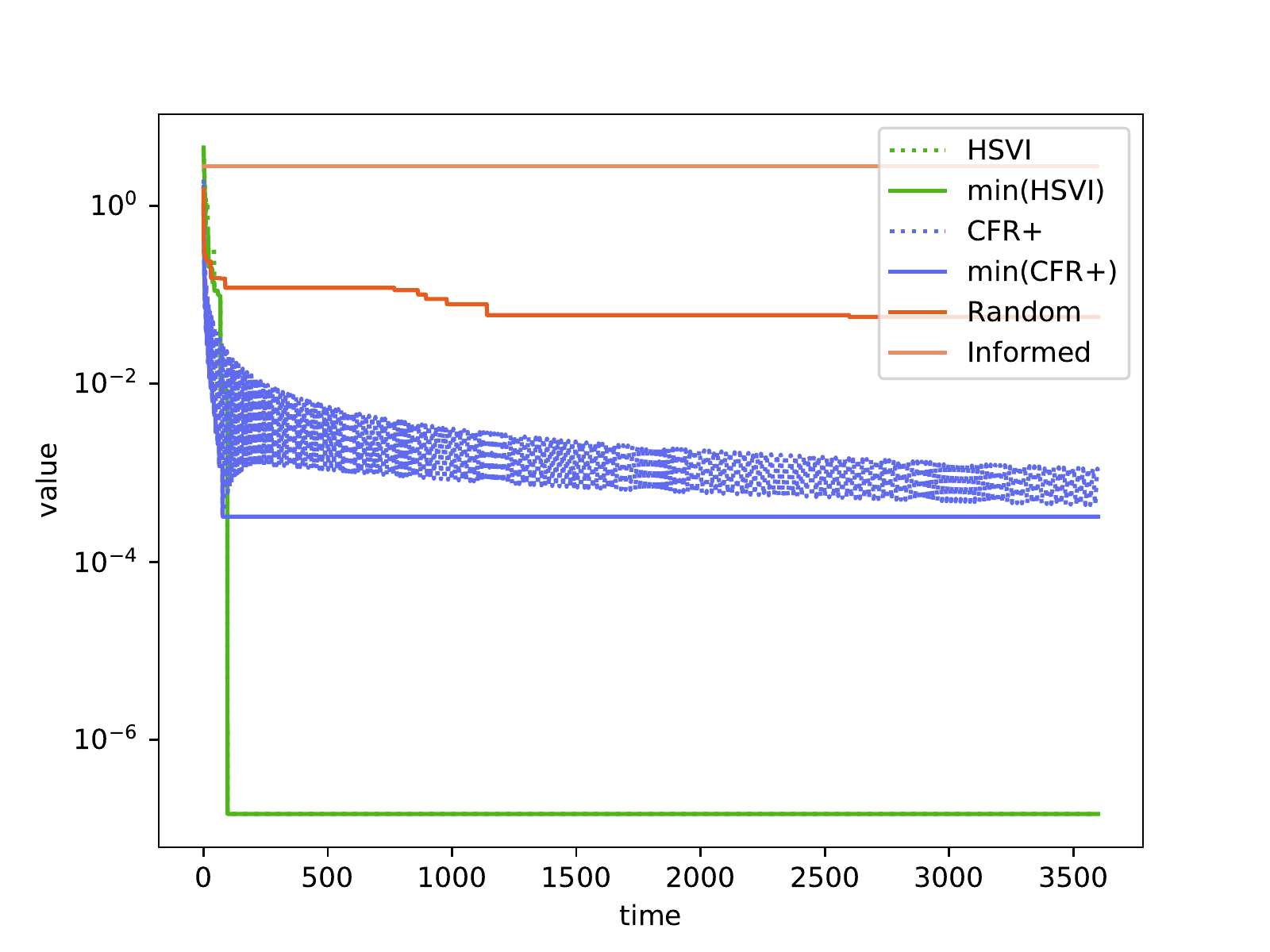}
    
    \includegraphics[width=\evolGraphScale\linewidth]{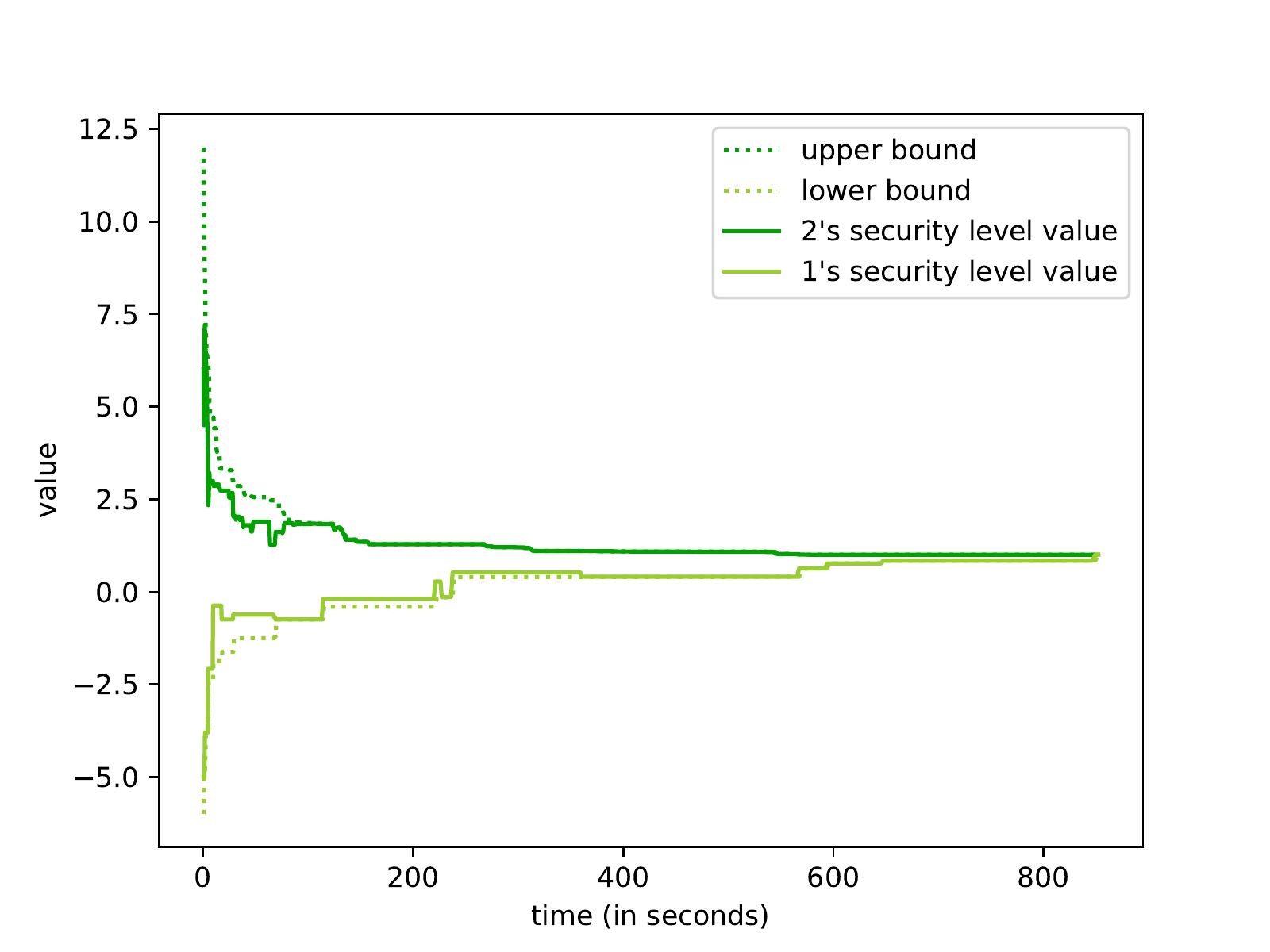}
    \hfill
    \includegraphics[width=\evolGraphScale\linewidth]{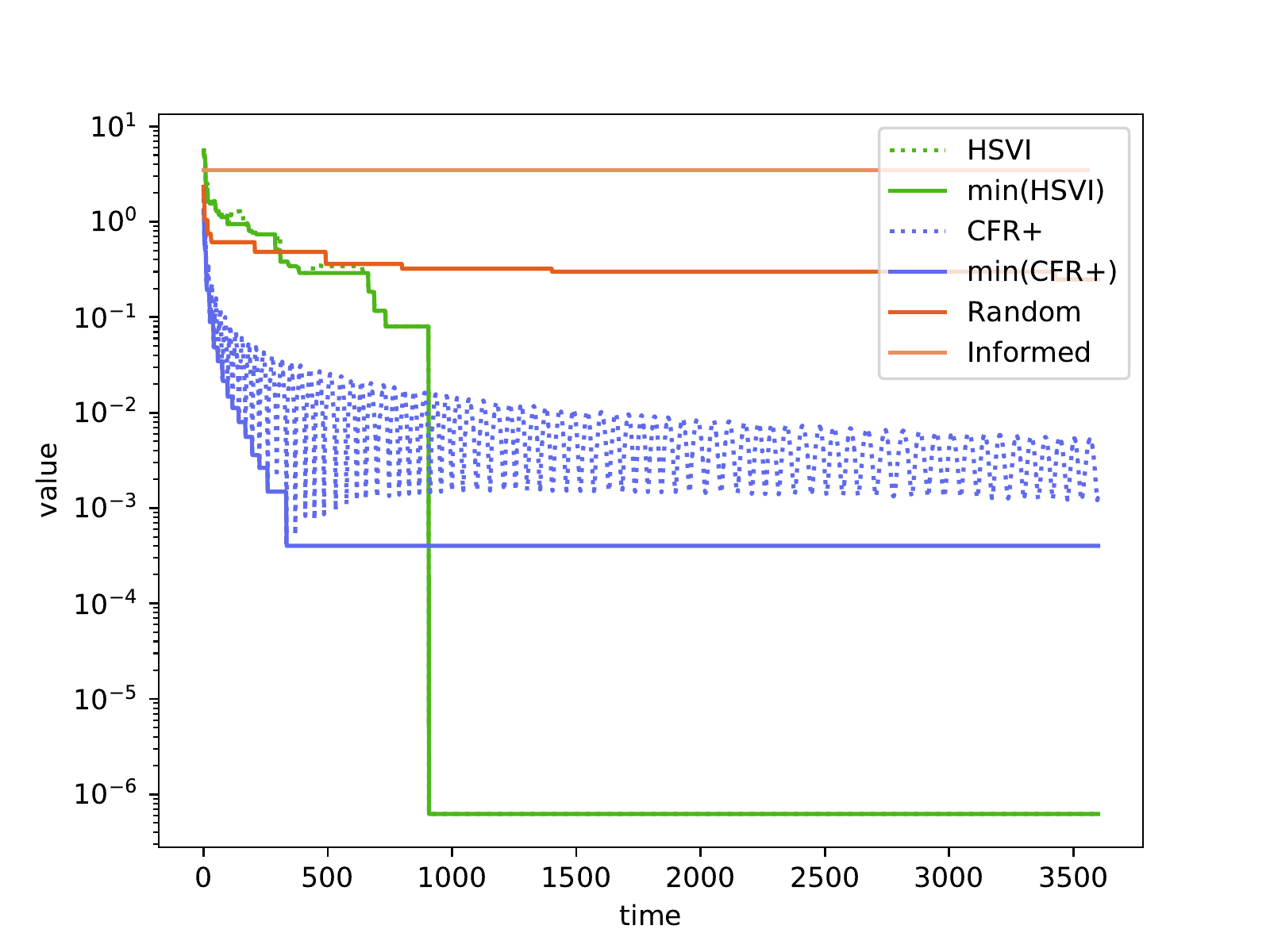}
    
    \caption{\myEvoExploCaption{Matching Pennies ($H=4,5,6$) (1,1,1)}}
    \label{fig|Graphs|MP}
\end{figure}

Having empirically studied the behavior of HSVI compared to other basic offline solvers, we now provide insight about the connections between HSVI and continual (thus online) resolving methods.

\section{Comparison with Continual Resolving}
\label{sec|CFR}

Continual Resolving techniques share some similarities with our approach, but also important differences.
The purpose of this section is to clarify these points.
It starts with a quick description of CFR, on which Continual Resolving is built.

\subparagraph*{Counterfactual Regret Minimization (CFR)} \cite{ZinJohBowPic-nips07} belongs to the self-play family of algorithms, which gave rise to several CFR-based approaches \citep{Tammelin-arxiv14,BurMorSch-jair19,LanctotEtAl-nips09,BroKroSan-aaai17}).
It iteratively traverses the whole game tree and applies, in each private history, a regret-matching update rule based on a specific type of regret called counterfactual regrets.
Iteratively updating an initially uniform strategy asymptotically converges towards a NES.
However, the tree traversal becomes intractable when the tree size is large.

Built on top of the CFR framework, approaches based on {\em limited-lookahead continual resolving} (LLCR) (inspired by Burch et al.'s decomposition \cite{Burch-aaai-2014}) such as DeepStack \cite{MorEtAl-science17}, Libratus \cite{BroSan-science18}, ReBeL \cite{BroBakLerQon-nips20} and Player of Games \cite{SchmidEtAl-arxiv21}, perform well by exploiting a temporal decomposition in subgames, which are specified through knowledge about both players' past strategies.
Our approach thus shares similarities with these works.

Yet, as we will see in the next sections, a closer look at LLCR \citep{Schmid-phd21} demonstrates how fundamentally different they are, starting with the fact that LLCR is an {\em online} search algorithm, \ie, is meant to make good decisions at each time step, based on the current knowledge about the state of the game, while HSVI, as SFLP or CFR (and its variants), is an {\em offline} algorithm returning a complete solution strategy.

\subsection{Continual Resolving}

{\em (Continual) Resolving} techniques have been the first ingredient to adapt CFR to online settings.
They address the problem of solving the complete subgame (down to its end) starting in the current situation at $\depth$,
while maintaining the global consistency (aka safety) of the whole strategy, \ie, not making choices that could encourage the opponent to deviate {\em in the past}, before $\depth$.
This is achieved by introducing constraints, called {\em gadgets}, in a preliminary stage of the subgame that represent possible deviations and their values, but increase the size of the game tree so that it is practically intractable \cite{MorEtAl-science17}.
%%} \aurelien{je pense que c'est important, c'est au coeur des contributions de DeepStack et Libratus d'arriver à utiliser re-solving en faisan un "compromis" entre safety et temps de calcul}.
For its part, HSVI solves similar subgames, but ensuring only local consistency, \ie, only considering the subgame. Global consistency comes from the way HSVI combines "lower-level" solutions in its backtracking process, without adding any gadget modifying the game.

For both Resolving and HSVI, solving a subgame requires {\em sufficient statistics} that represent a prefix strategy profile from $t=0$ to $\depth$.
In Resolving's online setting, this may seem surprising, since the opponent's actual strategy is not public.
Yet, Resolving does not actually require knowing or guessing the opponent's actual strategy.
In Resolving, any opponent Nash equilibrium strategy is appropriate, since the purpose is to verify that the opponent has no incentive to deviate from the Nash equilibrium.
A requirement for Resolving is for the sufficient statistics to represent {\em complete} strategies (given the current public information), so that decisions are anticipated for player $i$ even in \aoh{}s (infostates) not reachable given player $\neg i$'s strategy.
This leads to using {\em ranges} \cite{KovEtAl-Corr-2019}, a vector that gives, for each player, her contribution to the probability of any history she could face at time step $\depth$.
In contrast, HSVI's occupancy state is not necessarily related to a Nash equilibrium strategy in any manner, and leads to ignoring unreachable \aoh{}s, which helps to reduce the size of the decision-making (sub)problem.

\subsection{Limited Lookahead}

Continual Resolving alone solves complete subgames, thus larger problems at early stages of the game than at the end, which is not appropriate in an online setting.
To address this issue through limiting the lookahead of subgames, one needs to estimate the value of the leaves of any truncated subgame.
This is achieved through learning offline, for each player $i$, deep networks that, given the current public belief state, map each \aoh{} to its value under some Nash equilibrium strategy profile.
Note that the target function is not unique \cite[Proposition~A.1]{KovEtAl-Corr-2019}, since each NES profile maps to different value vectors.
Still, according to \cite{KovEtAl-Corr-2019}, this does not seem to cause problems in practice.

In contrast, the individual value functions HSVI considers (the "$\nu$" functions) are uniquely defined since they correspond to the best responses to given (not necessarily Nash equilibrium) strategies of the opponent.

\subsection{Limited Lookahead Continual Resolving as a General Scheme?}

The previous subsections highlight to what extent HSVI and LLCR are fundamentally different, in particular because they are not on the same algorithmic level.
LLCR should be seen as a general scheme in which the subgame solver used, namely CFR, could be replaced by other "basic" offline algorithms such as HSVI or SFLP.
But we leave further investigation on this topic for future work.

\section{Discussion}
\label{sec|discussion}

This paper addresses the problem of $\epsilon$-optimally solving zs-POSGs.
In contrast to SFLP or \CFR{}, we provide the necessary foundational building blocks to apply dynamic programming (in tandem with heuristic search) to solve zs-POSGs.
We introduce Bellman optimality equations and uniform-continuity properties of the optimal value function.
Next, we exhibit rules for updating value functions while preserving uniform continuity and the ability to extract globally-consistent solutions. 
Finally, we describe the first effective DP algorithm for zs-POSGs, zs-oMG-HSVI, with finite-time convergence to an $\epsilon$-optimal solution.
Experiments support our theoretical findings.

We believe our approach complements existing ones, e.g., SFLP and \CFR{}, in two dimensions.
First, it breaks the original zs-POSG into subgames.
Second, it generalizes values from visited subgames to unvisited ones.
Our performances are as good as or better than those from SFLP and \CFR{} for small-dimensional subgames (\eg, with TOI structure).
Unfortunately, the advantage of breaking the original problem into subgames and exploiting uniform continuity properties often fails to fully manifest in the overall computational time.

Despite some similarities, our (offline) approach is fundamentally different from (online) continual resolving approaches.
The latter could even possibly be adapted to use other offline methods than CFR-based ones, including HSVI.

We hope that this approach will lay the foundation for further work in the area of both exact and approximate DP solutions for zs-POSGs.
In the short term, we shall investigate pruning techniques, better Lipschitz constants, and improved initial bounding approximators using solutions from relaxations of zs-POSGs, e.g., zs-OS-POSGs.
In the long term, we shall investigate (deep) RL for zs-POSGs, similarly to a recent approach for Dec-POMDPs \citep{BonDibLaePerSim-ECML-18}.
The latter shall investigate the trade-off between the update-rule accuracy and the computational efficiency when facing high-dimensional subgames, hence providing competitive solvers.
%

% [Olivier] One line to tell emacs to use french/american/\dots spelling:
% Local IspellDict: american

%\bibliography{uai2022-template}
\bibliography{article}

% \onecolumn <- inutile pour NeurIPS

% \newpage

%\input{checklist.tex}

\newpage

\appendix

\def\l{\lambda}
\def\Appendix{}

%\aurelien{Warning : il y a beaucoup de numérotations d'équations dans l'annexe qui sont pas fait correctement à cause des équations sur plusieurs lignes. J'ai commencé, mais il en reste.
%}

\temporallyHidden{
\olivier{Dans les annexes je recommenderais
\begin{itemize}
\item de commencer par la section "synthetic tables", qui doivent être facilement accessibles; et
\item d'essayer de coller à l'ordre des sections du corps du papier, afin que le lecteur qui suit les deux en parallèles ait un parcours linéaire des deux côtés (donc les expériences iraient plutôt à la fin).
\end{itemize}}
}

\section{Synthetic Tables}
\label{app|syntheticTables}

For convenience, we provide two synthetic tables: 
\Cref{tab|PropertyTable} to sum up various theoretical properties
that are stated in this paper (assuming a finite temporal horizon), and 
\Cref{tab|NotationTable} to sum up the notations used in this paper, adding some notations that appear only in the appendix.
%

% \subsection{Properties table}

%\olivier{Il semble que tous les résultats de la table soient toujours dans le papier. Il faudrait probablement re-parcourir le papier pour voir s'il n'y a pas des résultats nouveaux à ajouter.}
%\aurelien{à discuter (pas compris)}
More precisely, \Cref{tab|PropertyTable} indicates, for various
functions $f$ and variables $x$, properties that $f$ is known to exhibit with
respect to $x$.
We denote by
\begin{description}[leftmargin=!,labelwidth=\widthof{$PWLCv$}]
\item[-] a function with no known (or used) property (see also comment
  below);
\item[{\sc n/a}] a non-applicable case;
\item[$\Lin$] a linear function;
\item[$LC$] a Lipschitz-continuous function;
\item[$Cv$] (resp. $Cc$) a convex (resp. concave) function;
\item[$PWLCv$] (resp. $PWLCc$) a piecewise linear and convex (resp. concave) function;
\item[$\indep$] the function being independent of the variable;
\item[$\neg P$] the negation of some property $P$ %
  (\ie, $P$ is known not to hold).
\end{description}
Note also that, as
$\occ_\depth = \occ_\depth^{c,1} \occ_\depth^{m,1}$, the linearity or
Lipschitz-continuity properties of any function w.r.t. $\occ_\depth$
extends to both $\occ_\depth^{c,1}$ and $\occ_\depth^{m,1}$.
Reciprocally, related negative results extend from $\occ_\depth^{c,1}$ or $\occ_\depth^{m,1}$ to $\occ_\depth$.
In these three columns, we just indicate results that cannot be derived from one of the two other columns.

\def\myfontsize{\tiny}

\begin{table}[ht]
    \caption{Known properties of various functions appearing in this work} %y table for this paper}
    \centering
    \label{tab|PropertyTable}
    \resizebox{\linewidth}{!}{%
    \begin{tabular}{lr@{ }l@{ }r@{ }l@{ }r@{ }l@{ }r@{ }l@{ }r@{ }l}
      \toprule
      % Variable $\to$ 
      & \multicolumn{2}{c}{$\occ_\depth$}
      & \multicolumn{2}{c}{$\occ_\depth^{m,1}$}
      & \multicolumn{2}{c}{$\occ_\depth^{c,1}$}
      & \multicolumn{2}{c}{$\beta_\depth^i$}
      & \multicolumn{2}{c}{$\beta_\depth^{\neg i}$} \\
      \midrule
      $T(\occ_\depth,\vbeta_\depth)$
      & $\Lin$ & {\myfontsize (\cshrefpage{lem|occSufficient})}
      & \multicolumn{2}{c}{-}
      & \multicolumn{2}{c}{-}
      & $\Lin$ & {\myfontsize (\cshrefpage{lem|occSufficient})}
               & $\Lin$ & {\myfontsize (\cshrefpage{lem|occSufficient})} \\
      $\Tm{i}(\occ_\depth,\vbeta_\depth)$
      & $\Lin$ & {\myfontsize (\cshrefpage{lem|T1mlin})}
      & \multicolumn{2}{c}{-}  & \multicolumn{2}{c}{-}
      & $Lin$ & {\myfontsize (\cshrefpage{lem|T1mlin})}
               & $Lin$ & {\myfontsize (\cshrefpage{lem|T1mlin})} \\
      $\Tc{i}(\occ_\depth,\vbeta_\depth)$
      & \multicolumn{2}{c}{-}  % $\neg LC$ & {\myfontsize (lem.~\Crefpage{lemma|NotLipschitzOccC})} %{lem|T1cindep})}?
      & $\indep$ & {\myfontsize (\cshrefpage{lem|T1cindep})}
      % & \multicolumn{2}{c}{-}  
      %& $\neg LC$ & {\myfontsize (\cshrefpage{lemma|NotLipschitzOccC})} 
      & \multicolumn{2}{c}{-} 
      & $\indep$ & {\myfontsize (\cshrefpage{lem|T1cindep})}
      %& $\neg LC$ & {\myfontsize (\cshrefpage{lemma|NotLipschitzBeta})} \\
      & \multicolumn{2}{c}{-} \\
      $V_\depth^{*}(\occ_\depth)$
      & $LC$ & {\myfontsize (\cshrefpage{proof|cor|V|LC|occ})}
      & $PWLCv$ & {\myfontsize (\cshrefpage{theo|ConvexConcaveV})}
      & \multicolumn{2}{c}{-}
      & \multicolumn{2}{c}{\sc n/a} & \multicolumn{2}{c}{\sc n/a} \\
      $W_\depth^{i,*}(\occ_\depth,\beta_\depth^i)$
      & $LC$ & {\myfontsize (from $Q^*_{\depth+1}$ LC)} % ?
      & \multicolumn{2}{c}{-} & \multicolumn{2}{c}{-}
     % & $\neg Lin$ & {\myfontsize (from $Q^*$ $\neg Lin$)} 
     & $\neg Lin$ & {\myfontsize (from $Q^*$ $\neg Lin$)} 
               & \multicolumn{2}{c}{\sc n/a} \\
%      & & & & & &
%               %&  $Cc$ & {\myfontsize (\cshrefpage{lemma|Wconcave})}
%               &  \multicolumn{2}{c}{-}
%      & \\
      %
      $\nu^2_{[\occ^{c,1}_\depth,\beta^2_{\depth:}]}$
      & \multicolumn{2}{c}{{\sc n/a}} & \multicolumn{2}{c}{{\sc n/a}}
      & $LC$ & {\myfontsize (\cshrefpage{lem|nuLC})} % ?
      & \multicolumn{2}{c}{\sc n/a}
      & \multicolumn{2}{c}{-}  \\
      % & & & & & &
      %          &  $Cc$ & {\myfontsize (prop.~\Crefpage{lemma|Wconcave})}
      % & \\
      \bottomrule
%}
    \end{tabular}
}
\end{table}

\newcommand{\multicolumnTWO}[1]{\multicolumn{2}{@{}>{\hsize=\dimexpr2\hsize+4\tabcolsep+2\arrayrulewidth\relax}X}{#1}}

\begin{table}[htbp]
  \caption{Various notations used in this work %\olivier{Ne manque-t-il pas (1) les stratégies mixtes et (2) $\lob{V}_\depth$ et $\upb{V}_\depth$ ?} \aurelien{ai ajouté les stratégies pures/mixtes mais pas les approximations $\upb{V}$ et $\lob{V}$ parce que (i) on s'en sert pas dans l'article sans faire référence dans l'annexe à là où c'ets défini et (ii) la table ne tient plus sur une page...}
%\olivier{Table à couper en deux.}  
} %
  \label{tab|NotationTable}
  % \label{tab|TableOfNotationForMyResearch} 
  \centering % to have the caption near the table
  % \begin{tabular}{r c p{10cm} }
  %\resizebox{.865\linewidth}{!}{
  %\resizebox{1.\linewidth}{!}{
  \begin{tabularx}{\textwidth}{r@{ }c@{ }l@{ }X} %1.17
    \toprule
    $\neg i$ & $\eqdef$ &
    \multicolumnTWO{$i$'s opponent. Thus: $\neg 1=2$, and $\neg 2=1$.}
    \medskip \\
    \multicolumn{4}{c}{\underline{Histories and occupancy states}} % -------------------------------------------------
    \medskip \\
    $\theta^i_\depth$ & $\eqdef$ & 
    \multicolumnTWO{
      $(a^i_0, z^i_1, \dots , a^i_{\depth-1}, z^i_\depth) $ %
      ($\in \Theta^i = \cup_{t=0}^{H-1} \Theta^i_t$) %
      is a length-$\depth$ {\em action-observation history} (\aoh) for
      \player $i$.
    }
    \\
    $\vth_\depth$ & $\eqdef$ & 
    \multicolumnTWO{
      $(\theta^1_\depth,\theta^2_\depth)$ %
      ($\in \vTh = \cup_{t=0}^{H-1} \vTh_t$) %
      is a {\em joint \aoh} at $\depth$.
    }
    \\
    $\occ_\depth(\vth_\depth)$  & $\eqdef$ & 
    \multicolumnTWO{
      {\em Occupancy state} (\os) $\occ_\depth$ ($\in \Occ = \cup_{t=0}^{H-1} \Occ_t$, where $\Occ_\depth \eqdef \Delta(\vTh_\depth)$), %
      \ie, probability distribution over joint \aoh{}s $\vth_\depth$
      (typically for some applied $\vbeta_{0:\depth-1}$).
    }
    \\
    $\occ_\depth^{m,i}(\theta^i_\depth)$ & $\eqdef$ & 
    \multicolumnTWO{
      {\em Marginal term} of $\occ_\depth$ from player $i$'s point of view ($\occ_\depth^{m,i} \in \Delta(\Theta_\depth^i) $).
    }
    \\
    $\occ_\depth^{c,i}(\theta^{\neg i}_\depth | \theta^i_\depth)$ & $\eqdef$ & 
    \multicolumnTWO{
      {\em Conditional term} of $\occ_\depth$ from $i$'s point of view ($\occ_\depth^{c,i} : \Theta_\depth^i \mapsto \Delta(\Theta_\depth^{\neg i}) $).
    }
    \\
    $b(s | \vth_\depth)$ & $\eqdef$ & 
    \multicolumnTWO{
      {\em Belief state}, \ie, probability distribution over states given a joint \aoh ($b(s | \vth_\depth) : \cS \times \vTh_\depth \mapsto \reals $). %\newline
      Can be computed by an HMM filtering process.
    }
    \\
    $o_\depth$ & $\eqdef$ &  
    \multicolumnTWO{
      {\em Full occupancy state} $o_\depth$ ($\in \Delta(\cS \times \vTh_\depth) $), \ie,
      $Pr(s,\vth_\depth)$ for the current $\vbeta_{0:\depth-1}$,
      and thus verifies $\occ_\depth(\vth_\depth)=\sum_{s\in\cS} o_\depth(s,\vth_\depth)$. 
      Is used in the implementation to simplify computations (\eg, of $r_t$ and $\occ_{\depth+1}$ through $b$).
    }
    % for complexity reasons as it makes
    % it easier to retrieve $b(s|\vth_\depth)$ as
    % $\frac{o_\depth(s,\vth_\depth)}{\occ_\depth(\vth_\depth)}$ than to
    % compute through a Bayesian filtering it process (as in an HMM).
    \medskip \\
    \multicolumn{4}{c}{\underline{Decision rules and strategies}} % -------------------------------------------------
    \medskip \\
    $\pi^i_{0:\depth}$ & $\eqdef$ &
    \multicolumnTWO{A {\em pure strategy} for \player $i$ is a mapping $\pi^i_{0:\depth}$ from private histories in $\Theta^i_t$ ($\forall t \in \{0\twodots \depth\}$) to {\em single} private   actions in $\cA^i$. %
    By default, $\pi^i\eqdef \pi^i_{0:H-1}$.}
    \\
    $\mu^i_{0:\depth}$ & $\eqdef$ &
    \multicolumnTWO{ A {\em mixed strategy}
    $\mu^i_{0:\depth}$ for \player $i$ is a probability distribution over pure strategies.
    It is used by first sampling one of the pure strategies (at $t=0$), and then executing it until $t=\depth$.
    }
    \\
    $\mu^i_{0:\depth' \mid \occ_\depth \rangle}$ & $\eqdef$ &
    \multicolumnTWO{($\depth\leq \depth'$) is a mixed strategy {\em compatible} with some \os{} $\occ_\depth$, \ie, that could induce this \os{} at $\depth$ (assuming an appropriate complementary $\mu^{\neg i}_{0:\depth' \mid \occ_\depth \rangle}$).
    }
    \\
    $\beta^i_\depth$ & $\eqdef$ & 
    \multicolumnTWO{
      A {\em (behavioral) decision rule} (\dr) at time $\depth$ for \player $i$ is a mapping $\beta^i_\depth$ from private \aoh{}s in $\Theta^i_\depth$ to {\em distributions} over private actions.
      We note $\beta^i_\depth(\theta^i_\depth,a^i)$ the probability to pick $a^i$ when facing $\theta^i_\depth$.
    }
    \\
    $\beta^i_{\depth:\depth'}$ & $\eqdef$ &
    \multicolumnTWO{
      $(\beta^i_\depth, \dots, \beta^i_{\depth'})$ is a {\em behavioral strategy} for \player $i$ from time step $\depth$ to $\depth'$  (included).
    }
    \\
    $rw^i(\theta^i_\depth, a^i_\depth)$ & $\eqdef$ &
    \multicolumnTWO{
    $\prod_{t=0}^\depth  \beta^i_{0:}(a^i_t | a^i_0, z^i_1, a^i_1, \dots, z^i_t) $ is the {\em realization weight} (RW) of sequence $a^i_0, z^i_1, a^i_1, \dots, a^i_\depth (=\theta^i_\depth, a^1_\depth)$ under strategy $\beta^i_{0:}$.
    }
    \\
    $ rw^i( \phi^i_{\depth:\depth'} | \theta^i_\depth) $ & $\eqdef$ &
    \multicolumnTWO{
    $\prod_{t=\depth}^{\depth'}  \beta^i_{0:}(a^i_t | \theta^i_\depth, a^i_\depth, \dots, z^i_t)$
    is the RW of a {\em suffix sequence} $\phi^i_{\depth:\depth'} = a^i_\depth, \dots, a^i_{\depth'}$ ``conditioned'' on a {\em prefix sequence/\aoh{}} $\theta^i_\depth$.
    }
    \\
    % %%%%%%%%%%%%%%%%%%%%%%%%    \medskip \\
    $\vpi_{0:\depth}$ & $\eqdef$ &
    \multicolumnTWO{
       is a {\em pure strategy profile}.
    }
    \\
    $\vmu_{0:\depth}$ & $\eqdef$ &
    \multicolumnTWO{ is a {\em
    mixed strategy profile}.
    }
    \\
    $\vmu_{0:\depth' \mid \occ_\depth \rangle}$ & $\eqdef$ &
    \multicolumnTWO{
    ($\depth\leq \depth'$) is a mixed strategy profile {\em compatible} with some \os{} $\occ_\depth$, \ie, that could induce this \os{} at $\depth$.
    }
    \\
    $\vbeta_\depth$ & $\eqdef$ &
    \multicolumnTWO{
      $\langle \beta^1_\depth, \beta^2_\depth \rangle$
      ($\in \cB = \cup_{t=0}^{H-1} \cB_t$) is a {\em decision rule
        profile}.
    }
    \\
    $\vbeta_{\depth:\depth'}$ & $\eqdef$ &
    \multicolumnTWO{
      $ \langle \beta^1_{\depth:\depth'}, \beta^2_{\depth:\depth'}
      \rangle$ is a {\em behavioral strategy profile}.
    }
    \medskip \\
    %%%%%%%%%%%%%%%%%%%%%%%%
    \end{tabularx}
\end{table}

\begin{table}[htbp] 
  \begin{tabularx}{\textwidth}{r@{ }c@{ }l@{ }X} %1.17
    \multicolumn{4}{c}{\underline{Rewards and value functions}} % -------------------------------------------------
    \medskip \\
    $r_{\max}$ & $\eqdef$ & $\max_{s,\va}r(s,\va)$ & Maximum possible reward. \\
    $r_{\min}$ & $\eqdef$ & $\min_{s,\va}r(s,\va)$ & Minimum possible reward. \\
    $V_\depth(\occ_\depth,\vbeta_{\depth:})$  & $\eqdef$ & 
    % \multicolumn{2}{>{\hsize=\dimexpr2\hsize+2\tabcolsep+\arrayrulewidth\relax}X}{
    % 
    $ E[\sum_{t=\depth}^{H-1} \gamma^t R_t \mid \occ_\depth, \vbeta_{\depth:}] $,
    &
    {\em Value} of  $\vbeta_{\depth:H-1}$ in
    \os $\occ_\depth$.
    \\
    & &
    \multicolumnTWO{
      where $R_t$ is the random var. for the reward at $t$.}
    \\
    $V_\depth^*(\occ_\depth)$ & $\eqdef$ & % 
    $\max_{\beta^1_{\depth:}} \min_{\beta^2_{\depth:}} V_\depth(\occ_\depth, \vbeta_{\depth:})$ %: $V_\depth^{*} : \Delta(\vTh_\depth) \mapsto \reals$
    & \text{ \it Optimal value function} \\
    % $V_\depth(\occ_\depth, \vbeta_\depth)$ & $\eqdef$ & Value function for a strategy profile $\vbeta_\depth$.\\
    $Q_\depth^*(\occ_\depth,\vbeta_\depth)$ & $\eqdef$  %
    &
    $ r(\occ_\depth,\vbeta_\depth) + \gamma V_{\depth+1}^{*}(T(\occ_\depth,\vbeta_\depth))$
    & \text{ \it Opt. (joint) action-value fct.} \\
    $W_\depth^{i,*}(\occ_\depth,\beta_\depth^i)$ & $\eqdef$ 
    & $opt_{\beta^{\neg i}_\depth} Q_\depth^*(\occ_\depth,\vbeta_\depth)$,
    & \text{ \it Opt. (individual) action-value fct.} \\
    & &   \multicolumn{2}{>{\hsize=\dimexpr2\hsize+2\tabcolsep+\arrayrulewidth\relax}X}{
      where $opt=\max$ if $i=1$, $\min$ otherwise.
    }%
    % (new concept ``in between'' $V_\depth^{*}$ and $Q_\depth^{*}(\occ_\depth,\vbeta_\depth)$)
    \\
    $\nu^2_{[\occ_\depth^{c,1}, \beta_{\depth:}^2]}$ & $\eqdef$ & 
    \multicolumnTWO{
      Vector of values (one component per \aoh{} $\theta^1_\depth$) for $1$'s best response to $\beta_{\depth:}^2$ assuming $\occ_\depth^{c,1}$.
    % }
    % \\
    % % Support vector of the scalar product representing $V_\depth^{*}$ (see Theorem~\Cref{theo|ConvexConcaveV}).\\
    % & & 
    % \multicolumnTWO{
      This solution of a POMDP allows computing $V_\depth^{*}$ (see \Cref{theo|ConvexConcaveV}).
    }
    \medskip \\
    \multicolumn{4}{c}{\underline{Approximations}} % -------------------------------------------------
    \medskip \\
    $\upb{V}_\depth(\occ_\depth)$ & $\eqdef$ & 
    \multicolumnTWO{
      Upper bound approximation of $V^*_\depth(\occ_\depth)$;
     relies on data set $\upb{\cI}_{\depth-1}$. 
    }
    \\
    $\lob{V}_\depth(\occ_\depth)$ & $\eqdef$ & 
    \multicolumnTWO{
      Lower bound approximation of $V^*_\depth(\occ_\depth)$; relies on data set $\lob{\cI}_{\depth-1}$. 
    }
    \\
    $\upbW{\depth}{}(\occ_\depth,\beta^1_\depth)$ & $\eqdef$ & 
    \multicolumnTWO{
      Upper bound approximation of $W_\depth^{*,1}(\occ_\depth,\beta^1_\depth)$; relies on data set $\upb{\mathcal{I}}_\depth$. 
    }
    \\
    $\lobW{\depth}{}(\occ_\depth,\beta^2_\depth)$ & $\eqdef$ & 
    \multicolumnTWO{
      Lower bound approximation of $W_\depth^{*,2}(\occ_\depth,\beta^2_\depth)$; relies on data set $\lob{\mathcal{I}}_\depth$. 
    }
    \\
    $\upb\nu_{\depth}^2$ & $\eqdef$ & 
    \multicolumnTWO{
      Vector (with one component per \aoh{} $\theta^1_\depth$) used in $\upb{V}_\depth$ and $\upbW{\depth-1}{}$ (if $\depth\geq 1$).
    }
    \medskip \\
    \multicolumn{4}{c}{\underline{Miscellaneous}} % -------------------------------------------------
    \medskip \\
    $w_\depth$ & $\eqdef$ & 
    \multicolumnTWO{
      Denotes a triplet $\langle \occ^{c,1}_{\depth-1}, \beta^1_{\depth-1}, \upb\nu^2_\depth, \rangle \in \upb{\mathcal{I}}_\depth$ (or a triplet in $\lob{\mathcal{I}}_\depth$).
  }
  \\
    $\tree_{\depth}^2$ & $\eqdef$ & 
    \multicolumnTWO{
      Distribution over triplets $w_{\depth+1} \in \upb{\mathcal{I}}_{\depth+1}$ (inducing a recursively defined strategy from $\depth$ to $H-1$).
    % }
    % \\
    % & & 
    % \multicolumnTWO{
      Often denotes the strategy it induces. 
    }
    \\
    $x^{\top}$ & $\eqdef$ & 
    \multicolumnTWO{
      The transpose of a (usually column) vector $x$ of $\reals^n$.
    }
    \\ %Used to highlight a scalar product. \\
    $c[y]$ & $\eqdef$ & 
    \multicolumnTWO{
      Denotes field $c$ of object/tuple $y$.
    }
    \\
    $\supp(d)$ & $\eqdef$ & 
    \multicolumnTWO{
      Support of distribution $d$, \ie, set of its non-zero probability elements.
    }
    \\
    \bottomrule
  \end{tabularx}
%}
\end{table}

% =========================================

\section{Background}

%\subsection{From zs-POSGs to \zsomgs}
%\subsection{Introducing Occupancy Markov Games}

\temporallyHidden{
\olivier{[2022-01-26] Tentative de "plan" révisé (qui revient pas mal à ce qu'on avait au départ, avec des retouches):
\begin{enumerate}
    \item Introduire le \zsomg{} associé à un zs-POSG.\\
    Note: Il n'est à ce stade question que de décrire un processus, pas un problème à résoudre (un jeu et son concept de solution).
    \item Expliquer que, pour un \zsomg{}, on propose de chercher une solution (un équilibre de Nash) non pas en boucle fermée (décidant d'une règle de décision en fonction de l'état d'occupation atteint), mais en boucle ouverte, de sorte que ce soit aussi une solution du zs-POSG (et que la notion de fonction de valeur soit identique).
    \\
    Sauf erreur, il est évident qu'on ne fait que ré-écrire le même problème.%
    \footnote{\olivier{Note: On pourrait presque se passer d'introduire les \zsomg{}. Ce dont on a besoin c'est des états d'occupations avec dynamique markovienne.  Mais bon, ça ne coûte pas grand chose d'introduire les \zsomg{}, et ça permet de bien expliquer qu'on va les résoudre en boucle ouverte (contrairement à ce que j'ai cru).}}
    \item Introduire le concept de sous-jeu, en expliquant qu'a priori on ne sait pas si, dans notre contexte en boucle ouverte, résoudre un sous-jeu aide à résoudre le jeu de départ.
    %\item \sout{Montrer que la fonction de valeur $V_\depth(\occ_\depth,\vbeta_{\depth:})$ est linéaire en $\beta^i_t$ pour $\depth \leq t \leq H-1$.\footnote{Ce n'est pas vrai en tout $\beta^i_{t:t'}$ pour $t<t'$.}}
    \item Montrer que la valeur maximin est égale à la valeur obtenue en changeant arbitrairement l'ordre des opérateurs max et min de tous les pas de temps (avec une procédure qui me fait penser au "tri des crêpes" en ces temps de chandeleur).\footnote{\url{https://fr.wikipedia.org/wiki/Tri_de_cr\%C3\%AApes}}
    \\
    A priori je ne vois pas comment faire autrement que de passer par les stratégies mixtes suffixes (comme dans le "vieux" papier arXiv),\footnote{Pour ces stratégies mixtes on voudrait idéalement la notation $\mu$ qu'on pourrait éventuellement reprendre aux méta-stratégies qui ne semblent finalement pas très utiles.} ou plutôt par les poids de réalisation (ce qui me semble bien plus pratique).
    On a besoin de la linéarité en celles-ci pour échanger des maxs avec des mins.
    \item En particulier, on a:
    \begin{align}
        \max_{\beta^1_{\depth:}} \min_{\beta^2_{\depth:}} V_\depth(\occ_\depth,\vbeta_{\depth:})
        & = \min_{\beta^2_{\depth:}} \max_{\beta^1_{\depth:}} V_\depth(\occ_\depth,\vbeta_{\depth:}),
    \end{align}
    ce qui implique que les argmax et argmin obtenus par l'un et l'autre problème forment un équilibre de Nash !
    \item Déduire aussi de l'interchangeabilité des mins et maxs la récurrence suivante, qui justifie l'utilisation du POB:
    \begin{align}
        V^*_\depth(\occ_\depth)
        & = \max_{\beta^1_{\depth}} \min_{\beta^2_{\depth}} \max_{\beta^1_{\depth+1:}} \min_{\beta^2_{\depth+1:}} V_\depth(\occ_\depth,\vbeta_{\depth:})
        \\
        & = \max_{\beta^1_{\depth}} \min_{\beta^2_{\depth}} \max_{\beta^1_{\depth+1:}} \min_{\beta^2_{\depth+1:}} \left[ r(\occ_\depth,\vbeta_{\depth}) + V_{\depth+1}(T(\occ_\depth,\vbeta_{\depth}),\vbeta_{\depth+1:}) \right]
        \\
        & = \max_{\beta^1_{\depth}} \min_{\beta^2_{\depth}} \left[ r(\occ_\depth,\vbeta_{\depth}) + \max_{\beta^1_{\depth+1:}} \min_{\beta^2_{\depth+1:}} V_{\depth+1}(T(\occ_\depth,\vbeta_{\depth}),\vbeta_{\depth+1:}) \right]
        \\
        & = \max_{\beta^1_{\depth}} \min_{\beta^2_{\depth}} \left[ r(\occ_\depth,\vbeta_{\depth}) +  V^*_{\depth+1}(T(\occ_\depth,\vbeta_{\depth})) \right].
    \end{align}
    \item Expliquer que, avec la formule ci-dessus, on obtient une règle de décision $\beta^1_\depth$ qui fait nécessairement partie d'une stratégie solution pour 1. Cela justifie (il me semble) de suivre une telle \dr lors de la construction d'une trajectoire.
    \item Construire l'approximateur $\upbW{\depth}{}$ (+ initialisation) et son opérateur de MàJ.
    \item Expliquer que la MàJ s'accompagne naturellement du calcul d'une stratégie arborescente \uline{pour le joueur 2} de niveau de sécurité (de plus haute (/pire) valeur quelle que soit la stratégie de 1) en $\occ_\depth$ \uline{au plus}  $\upbW{\depth}{}(\occ_\depth,\beta^1_{\depth:1})$.
    \item \dots
\end{enumerate}
}
}

\subsection{Occupancy States}
\label{app|occStates}

%\label{app|fromTo}

% \aurelien{Jilles : à enlever \& mettre plutôt dans un tech-report?}

The following result shows that the occupancy state is %
(i) Markovian, \ie, its value at $\depth$ only depends on its previous
value $\occ_{\depth-1}$, the system dynamics
$\PP{s}{a^1,a^2}{s'}{z^1,z^2}$, and the last behavioral decision rules
$\beta^1_{\depth-1}$ and $\beta^2_{\depth-1}$, and %
(ii) sufficient to estimate the expected reward.
Note that it holds for general-sum POSGs with any number of agents,
and as many reward functions; %
similar results have already been established, \eg, for Dec-POMDPs
(\cf \citep[Theorem~1]{DibAmaBufCha-jair16}).

% \begin{lemma}
%   \label{lem|occSufficient}
%   $\occ_{\vbeta_{0:\depth-1}}$, together with $\vbeta_\depth$, is a sufficient
%   statistics to compute %
%   (i) the next \os, $\occ_{\vbeta_{0:\depth}}$, and %
%   (ii) the expected reward at $\depth$.
% \end{lemma}

\lemOccSufficient*

\begin{proof}
  \label{proof|lem|occSufficient}
  Let us first derive a recursive way of computing
  $ \occ_{ \vbeta_{0:\depth} }( \vth_\depth, \va_\depth,
  \vz_{\depth+1}) $:
  \begin{align*}
    & \occ_{ \vbeta_{0:\depth} }( \vth_\depth, \va_\depth, \vz_{\depth+1}) %
    \eqdef Pr( \vth_\depth, \va_\depth, \vz_{\depth+1} \mid \vbeta_{0:\depth} ) \\
    & = \sum_{s_\depth, s_{\depth+1}}
    Pr( \vth_\depth, \va_\depth, \vz_{\depth+1}, s_\depth, s_{\depth+1} \mid \vbeta_{0:\depth} )
    \\
    & = \sum_{s_\depth, s_{\depth+1}}
    Pr( \vz_{\depth+1}, s_{\depth+1} \mid \vth_\depth, \va_\depth, s_\depth, \vbeta_{0:\depth} ) 
    Pr( \va_\depth \mid  \vth_\depth, s_\depth, \vbeta_{0:\depth} ) 
    \\ & \qquad \qquad 
    Pr( s_\depth \mid \vth_\depth, \vbeta_{0:\depth} ) 
    Pr( \vth_\depth \mid \vbeta_{0:\depth} ) 
    \\
    & = \sum_{s_\depth, s_{\depth+1}}
    \underbrace{ Pr( \vz_{\depth+1}, s_{\depth+1} \mid \va_\depth, s_\depth ) }_{= \PP{s_\depth}{\va_\depth}{s_{\depth+1}}{\vz_{\depth+1}}}
    \underbrace{ Pr( \va_\depth \mid  \vth_\depth, \vbeta_{\depth} ) }_{= \vbeta(\vth_\depth, \va_\depth)}
    \underbrace{ Pr( s_\depth \mid \vth_\depth, \vbeta_{0:\depth} ) }_{= b(s_\depth \mid \vth_\depth)}
    \underbrace{ Pr( \vth_\depth \mid \vbeta_{0:\depth-1} ) }_{= \occ_{\vbeta_{0:\depth-1}}(\vth_\depth)},
    \intertext{(where $b(s \mid \vth_\depth)$ is the belief over states obtained by a usual HMM filtering process)}
    & = \sum_{s_\depth, s_{\depth+1}}
    \PP{s_\depth}{\va_\depth}{s_{\depth+1}}{\vz_{\depth+1}}
    \vbeta(\vth_\depth, \va_\depth)
    b(s_\depth \mid \vth_\depth)
    \occ_{\vbeta_{0:\depth-1}}(\vth_\depth).
  \end{align*}
  $\occ_{ \vbeta_{0:\depth} }$ can thus be computed from
  $\occ_{\vbeta_{0:\depth-1}}$ and $\vbeta_\depth$ without explicitly
  using $\vbeta_{0:\depth-1}$ or earlier occupancy states.

  Then, let us compute the expected reward at $\depth$ given
  $\vbeta_{0:\depth}$:
  \begin{align*}
    & E[r(S_\depth,A^1_\depth,A^2_\depth) \mid \vbeta_{0:\depth} ] \\ %
    & = \sum_{s_\depth, \va_\depth} r(s_\depth, \va_\depth) Pr( s_\depth, \va_\depth \mid \vbeta_{0:\depth} )
    \\
    & = \sum_{s_\depth, \va_\depth} \sum_{\vth_\depth} r(s_\depth, \va_\depth) Pr( s_\depth, \va_\depth, \vth_\depth \mid \vbeta_{0:\depth} )
    \\
    & = \sum_{s_\depth, \va_\depth} \sum_{\vth_\depth} r(s_\depth, \va_\depth)
    Pr( s_\depth, \va_\depth \mid \vth_\depth, \vbeta_{0:\depth} )
    Pr( \vth_\depth \mid \vbeta_{0:\depth} )
    \\
    & = \sum_{s_\depth, \va_\depth} \sum_{\vth_\depth} r(s_\depth, \va_\depth)
    \underbrace{ Pr( \va_\depth \mid \vth_\depth, \vbeta_{0:\depth} ) }_{ \vbeta_\depth(\vth_\depth, \va_\depth) }
    \underbrace{ Pr( s_\depth \mid \vth_\depth, \vbeta_{0:\depth} ) }_{ b(s_\depth \mid \vth_\depth) }
    \underbrace{ Pr( \vth_\depth \mid \vbeta_{0:\depth} ) }_{\occ_{\vbeta_{0:\depth-1}}( \vth_\depth ) }
    \\
    & = \sum_{s_\depth, \va_\depth} \sum_{\vth_\depth} r(s_\depth, \va_\depth)
     \vbeta_\depth(\vth_\depth, \va_\depth)
     b(s_\depth \mid \vth_\depth) 
     \occ_{\vbeta_{0:\depth-1}}( \vth_\depth ).
  \end{align*}  
  The expected reward at $\depth$ can thus be computed from
  $\occ_{\vbeta_{0:\depth-1}}$ and $\vbeta_\depth$ without explicitly
  using $\vbeta_{0:\depth-1}$ or earlier occupancy states.
\end{proof}

% \section{Theoretical Contributions}

\section{Occupancy Markov Games: Definition and Preliminary Properties}

%\subsubsection{Introducing Occupancy Markov Games}
%\label{sec|oSG}

\subsection{Properties of \texorpdfstring{$V^*$}{V*}}
\label{sec|propVstar}

Before proving the postulates implicitly used by \citep{WigOliRoi-corr16}, we need to show that we can reason with mixed strategies in subgames as is usually done on full games.

\subsubsection{\texorpdfstring{$V^*$}{V*} is not linear in \texorpdfstring{$\beta$}{behavioral strategies} (behavioral strategies)}
\label{ex|VnotBilinearInBetas}
Let us consider the following (finite-horizon, deterministic) Non-Observable MDP:
%\olivier{\scriptsize Note: Of course that's an MDP, but formalizing it as a POMDP allows talking about behavioral strategies.}
\begin{align*}
    \cS & \eqdef \{ -2, -1, 0, +1, +2\}, 
    \quad b_0(0) =1, & \text{(always start in $s=0$)} \\
    \cA & \eqdef \{-1, +1\}, & \text{(moves = add or subtract 1)} \\
    T(s,a) & \eqdef \min \{ +2, \max \{ -2, s+a \} \},
    & \text{(dép. de $+1$ ou $-1$ dans $\cS$)} \\
    \cZ & \eqdef \{ none \},
    \quad O(none) \eqdef 1, & \text{(no observation)} \\
    r(s) & \eqdef
        \begin{cases}
        +1 & \text{si }s \in \{-2,+2\} \\
        0 & \text{sinon,}
        \end{cases}
        & \text{($\abs{s}=2$ : victoire !)}
    \\
    \gamma & \eqdef 1,
    \quad H \eqdef 2.
\end{align*}

Let us then consider two particular behavioral strategies:
\begin{align*}
    \forall \theta,\ \beta^{+}(A=+1|\theta) &= 1 & \text{(always $+1$), and} \\
    \forall \theta,\ \beta^{-}(A=-1|\theta) &= 1 & \text{(always $-1$)}.
\end{align*}
These two strategies are optimal, with an expected return of $+1$, because, at $t=H=2$,
$\beta^+$ reaches $+2$ w.p. $1$, and 
$\beta^-$ reaches $-2$ w.p. $1$:
\begin{align*}
    V(\beta^+) &= V(\beta^-) = +1.
\end{align*}
%\aurelien{Tout le problème vient d'ici et du sens qu'on donne à "+". Si l'on travaille avec des stratégies comportementales, le "+" est celui énoncé ici et il n'y a pas bi-linéarité. Par contre, si on utile le "+" des stratégies mixtes, i.e $\frac{1}{2}\beta^+ + \frac{1}{2} \beta^-$: tirer avec probabilité 1/2 $\beta^+$ et avec proba 1/2 $\beta^-$, *puis s'y tenir* alors on va avoir bi-linéarité.}

Let us now consider their linear combination $\beta^\pm \eqdef \frac{1}{2}\beta^+ + \frac{1}{2} \beta^-$:
\begin{align*}
    \forall \theta,\ \beta^{\pm}(A=-1|\theta) &= 0.5, \\
    \forall \theta,\ \beta^{\pm}(A=+1|\theta) &= 0.5.
\end{align*}
Here, the probability to reach $s=-2$ or $s=+2$ at the last time step is much lower, and gives the value of that strategy:
\begin{align*}
    V(\beta^{\pm})
    & = Pr(s_2=+2 | \beta^{\pm} ) + Pr(s_2=+2 | \beta^{\pm} ) \\
    & = Pr(a_0=+1 | \beta^{\pm} ) \cdot Pr(a_1=+1 | \beta^{\pm} ) \\
    & + Pr(a_0=-1 | \beta^{\pm} ) \cdot Pr(a_1=-1 | \beta^{\pm} ) \\
    & = \underbrace{0.5\cdot 0.5}_{0.25} + \underbrace{0.5\cdot 0.5}_{0.25} 
    = 0.5.
\end{align*}

\poubelle{
\newpage

\begin{tcolorbox}[breakable,enhanced]
Intially, $\beta_{\depth:}$ is defined as a function, taking as input histories and yielding a distribution probability. Then, the canonical sense to $\lambda \beta_{\depth:}^1 + (1-\lambda) \tilde \beta_{\depth:}^1$ is the function $\psi_{\depth:} : \theta_t \mapsto \psi_{\depth:}(\theta_t) \in \Delta(\cA^1)$. For any $\theta_t,a$, $\psi_{\depth:}(\theta_t)(a) =\lambda \mathbb{P}(a \mid \beta_{\depth:}^1(\theta_t)) + (1-\lambda) \mathbb{P}(a \mid \tilde \beta_{\depth:}^1(\theta_t))$ which is still a behavioral strategy. $V$ is not linear in the space of behavioral strategies as Olivier's example shows, with this sense given to "+".

Now, let us define (as I understand, Jilles' formuation) $\mathcal{F}^1 \eqdef \Delta(\cB^1_{\depth:})$. The value of any $(\mu^1,\mu^2) \in \mathcal{F}^1 \times \mathcal{F}^2$ is $V(\occ,\mu^1,\mu^2) = \int_{\cB_{\depth:}^1 \times \cB_{\depth:}^2} V(\occ_\depth,\beta_\depth^1,\beta_\depth^2) d\mu^1(\beta_\depth^1)d\mu^2(\beta_\depth^2)$. $V$ is linear in $\mathcal{F}^1$ and $\mathcal{F}^2$ (see theorem 4.9 in \url{https://olivier.garet.xyz/cours/ip/cours_l3_s1.pdf}). $\mathcal{F}^1$ is compact, convex. 

Now, Sion's theorem shows that $\max_{\mu^1} \min_{\mu^2} V(\occ_\depth,\beta_\depth^1,\beta_\depth^2) = \min_{\mu^2} \max_{\mu^1}  V(\occ_\depth,\beta_\depth^1,\beta_\depth^2)$.

We now only need to show (but should be trivial) that 
\begin{itemize}
    \item $\mathcal{F}$ is convex
    \item $\mathcal{F} \simeq \cB_{\depth:}$ (already done in \Cref{th|eqStrategies})
\end{itemize}
and we should be fine.
\end{tcolorbox}
}

\subsubsection{Back to Mixed Strategies}
\label{app|backToMixedStrategies}

% \Olivier{Warning: I am now first viewing/defining mixed strategies at
%   $\depth$ in a different (more natural?) manner.}

% To be able to apply von Neumann's Minimax theorem in subgames, we need
% to re-introduce mixed strategies---as a mathematical tool rather than
% to actually execute them---, which we do now, along with giving some
% preliminary results.

We now generalize mixed strategies as a mathematical tool to handle
subgames of a zs-OMG as normal-form games, and give some preliminary
results.

First, for a given $\occ_\depth$ and $\depth \leq \depth'$, let $\vmu_{0:\depth'-1|\occ_\depth\rangle}$ denote a mixed strategy profile that is defined over $0:\depth'-1$, and induces (/is {\em compatible} with) $\occ_\depth$ at time $\depth$.
%
%For a given $\occ_\depth$, let $\vmu_{0:\depth-1|\occ_\depth\rangle}$ be an arbitrarily chosen mixed strategy profile that leads to (/is compatible with) $\occ_\depth$ at time $\depth$, and here only defined over time interval $0:\depth-1$.
%
Then, to complete a given mixed {\em prefix} strategy $\vmu_{0:\depth'-1|\occ_\depth\rangle}$ (here $\depth=\depth'$), the solver should
provide each player with a different {\em suffix} strategy to execute for
each $\theta^i_\depth$ it could be facing.
We now detail how to build an equivalent set of mixed {\em full}
strategies for \player $i$.
Each of the pure {\em prefix} strategies $\pi^i_{0:\depth-1}$ used in
$\mu^i_{0:\depth-1|\occ_\depth\rangle}$ (belonging to a set denoted
$\Pi^i_{0:\depth-1|\occ_\depth\rangle}$) can be extended by appending a
different pure {\em suffix} strategy $\pi^i_{\depth:H-1}$ at each of its
leaf nodes, which leads to a large set of pure strategies
$\Pi^i_{0:H-1}(\pi^i_{0:\depth-1})$.
Then, let $M^i_{0:H-1|\occ_\depth\rangle}$ be the set of mixed {\em full}
strategies $\mu^i_{0:H-1|\occ_\depth\rangle}$ obtained by considering
the distributions over
$\bigcup_{\pi^i_{0:\depth-1}\in \Pi^i_{0:\depth-1|\occ_\depth\rangle}}
\Pi^i_{0:H-1}(\pi^i_{0:\depth-1})$
that verify, $\forall \pi^i_{0:\depth-1}$,
\begin{align}
  \sum_{\substack{\pi^i_{0:H-1} \in \\ \Pi^i_{0:H-1}(\pi^i_{0:\depth-1}) \hspace{-1cm}}} \mu^i_{0:H-1|\occ_\depth\rangle}(\pi^i_{0:H-1})
  & = \mu^i_{0:\depth-1|\occ_\depth\rangle}(\pi^i_{0:\depth-1}).
    \label{eq:compatible:mixed}
\end{align}
This is the set of mixed strategies compatible with $\occ_\depth$.

%\begin{restatable}[Proof in \ifextended{App.~\ref{proofLemEquivalence}}{\citep{icml20ext}}]{lemma}{lemEquivalence}

\begin{restatable}[Proof in App.~\extref{proofLemEquivalence}]{lemma}{lemEquivalence}
  \label{lem:equivalence}
  %\IfAppendix{{\em (originally stated on
  %  page~\pageref{lem:equivalence})}}{}
  % 
  $M^i_{0:H-1|\occ_\depth\rangle}$ is convex and equivalent to the set
  of behavioral strategies $\beta^i_{0:H-1|\occ_\depth \rangle}$, thus
  sufficient to search for a Nash equilibrium in $\occ_\depth$.
\end{restatable}

\begin{proof}
  Let $\mu^i_{0:H-1|\occ_\depth\rangle}$ and
  $\nu^i_{0:H-1|\occ_\depth\rangle}$ be two mixed strategies in
  $M^i_{0:H-1|\occ_\depth\rangle}$, \ie, which are both full and compatible
  with occupancy state $\occ_\depth$ at time step $\depth$, and
  $\alpha \in [0,1]$.
  Then, for any $\pi^i_{0:\depth-1}$,
  \begin{align*}
    & \sum_{\pi^i_{0:H-1} \in \Pi^i_{0:H-1}(\pi^i_{0:\depth-1})} \left[
      \alpha \cdot \mu^i_{0:H-1|\occ_\depth\rangle}(\pi^i_{0:H-1}) + (1-\alpha) \cdot \nu^i_{0:H-1|\occ_\depth\rangle}(\pi^i_{0:H-1})
      \right] \\
    & =
      \alpha \left[ \sum_{\pi^i_{0:H-1} \in \Pi^i_{0:H-1}(\pi^i_{0:\depth-1})}
      \mu^i_{0:H-1|\occ_\depth\rangle}(\pi^i_{0:H-1}) \right]
      \\
      & \qquad + (1-\alpha) \left[ \sum_{\pi^i_{0:H-1} \in \Pi^i_{0:H-1}(\pi^i_{0:\depth-1})}
      \nu^i_{0:H-1|\occ_\depth\rangle}(\pi^i_{0:H-1}) \right] \\
    \intertext{(because both mixed strategies are compatible with $\occ_\depth$ 
    (eq.~\ref{eq:compatible:mixed}, p.~\pageref{eq:compatible:mixed}):)}
    & =
      \alpha \cdot \mu^i_{0:\depth-1|\occ_\depth\rangle}(\pi^i_{0:\depth-1})
      + (1-\alpha) \cdot \mu^i_{0:\depth-1|\occ_\depth\rangle}(\pi^i_{0:\depth-1}) \\
    & =
      \mu^i_{0:\depth-1|\occ_\depth\rangle}(\pi^i_{0:\depth-1}).
  \end{align*}
  Eq.~\ref{eq:compatible:mixed} thus also applies to
  $\alpha \cdot \mu^i_{0:H-1|\occ_\depth\rangle} + (1-\alpha) \cdot
  \nu^i_{0:H-1|\occ_\depth\rangle}$,
  proving that it belongs to $M^i_{0:H-1|\occ_\depth\rangle}$ and, as
  a consequence, that this set is convex.

  The equivalence with the set of behavioral strategies simply relies
  on the fact that all mixed strategies over $\depth:H-1$ can be
  independently generated at each action-observation history
  $\theta^i_{0:\depth-1}$.
\end{proof}

% \medskip
% \hrule
% \medskip

% {\tiny
% To be able to apply von Neumann's Minimax theorem in subgames, we need
% to re-introduce mixed strategies, which we do now, along with giving
% some preliminary results.

% First, what is a mixed strategy when in $\occ_\depth$?
% %
% The central planner should provide \player $i$ with a different
% strategy to execute for each history $\theta^i_\depth$ it could be
% facing.
% %
% Given $\theta^i_\depth$, we thus define
% %
% a {\em mixed strategy for $\theta^i_\depth$}, noted
% $\mu^i_{\depth:H-1 | \theta^i_\depth \rangle}$, as a mixture of pure
% strategies defined over $\{\depth\twodots H-1\}$, and
% %
% a {\em mixed strategy for $\occ_\depth$}, noted
% $\mu^i_{\depth:H-1 | \occ_\depth\rangle}$, as one
% $\mu^i_{\depth:H-1 | \theta^i_\depth\rangle}$ for each
% $\theta^i_\depth$ appearing in $\occ_\depth$.

% Then, by finding a mixed strategy $\mu^i_{0:\depth-1}$ compatible with
% $\occ_\depth$ (\ie, that would lead to $\occ_\depth$), any
% $\mu^i_{\depth:H-1 | \occ_\depth\rangle}$ can easily be turned into a
% mixed strategy defined over $\{0\twodots H-1\}$:
% $\mu^i_{0:H-1 | \occ_\depth\rangle}$.
% %
% Using always the same $\mu^i_{0:\depth-1}$ for a given $\occ_\depth$,
% the space of these mixed strategies
% $\mu^i_{0:H-1 | \occ_\depth\rangle}$ is {\bf WHAT?}.
% %
% Conversely, one can turn any $\mu^i_{0:H-1 | \occ_\depth\rangle}$ into
% an equivalent $\mu^i_{\depth:H-1 | \occ_\depth\rangle}$.
% }

While only future rewards are relevant when making a decision at
$\depth$, reasoning with mixed strategies defined from $t=0$ will be
convenient because $V_\depth(\occ_\depth,\cdot,\cdot)$ is linear in
$\mu^i_{0:H-1 | \occ_\depth\rangle}$, which allows coming back to a
standard normal-form game and applying known results.

In the remaining, we simply note $\mu^i$ (without index) the mixed
strategies in $M^i_{0:H-1|\occ_\depth\rangle}$, set which we now note
$M^i_{|\occ_\depth\rangle}$.
Also, since we shall work with local game
$Q^*_\depth(\occ_\depth,\vbeta_\depth)$, let us
% consider a decision rule profile $\vbeta_\depth$ and
define:
%
% $M^i_{|\occ_\depth\rangle}$ the set of $i$'s mixed strategies
% (over $\{0\twodots H-1\}$) compatible with $\occ_\depth$,
% %
% % $M^i_{0:H-1|\occ_\depth,\vbeta_\depth\rangle} \eqdef
% % M^i_{0:H-1|T(\occ_\depth,\vbeta_\depth)\rangle}$,
% %
% and
% 
$M^i_{|\occ_\depth,\beta^j_\depth\rangle}$ the set of $i$'s mixed
strategies compatible with occupancy states reachable given
$\occ_\depth$ and $\beta^j_\depth$ {\footnotesize (with either $j=i$
  or $j=-i$)}.
Then,
$M^i_{|T(\occ_\depth,\vbeta_\depth)\rangle} \subseteq
M^i_{|\occ_\depth,\beta^j_\depth\rangle} \subseteq
M^i_{|\occ_\depth\rangle}$
(inclusion due to the latter sets being less constrained in their
definition).
As a consequence, if maximizing some function $f$ over $i$'s mixed
strategies compatible with a given $\occ_\depth$: 
\begin{align*}
  \max_{\mu^i \in M^i_{| \occ_\depth \rangle}} f(\occ_\depth, \mu^i, \dots)
   & \geq \max_{\mu^i \in M^i_{| \occ_\depth, \beta^j_\depth \rangle}} f(\occ_\depth, \mu^i, \dots)
   \geq \max_{\mu^i \in M^i_{| \occ_\depth, \vbeta_\depth \rangle}} f(\occ_\depth, \mu^i, \dots).
\end{align*}

\subsubsection{Von Neumann's Minimax Theorem for Subgames and a Bellman Optimality Equation}

Using the previous results, one can show that \citeauthor{Neu-ma28}'s minimax theorem applies in any subgame, allowing to swap operators $\max$ and $\min$.

\thUnicityNev*

\begin{proof}
For any occupancy state $\occ_\depth$,
\begin{align}
    \max_{\beta_{\depth:}^1} \min_{\beta_{\depth:}^2} V(\occ_\depth,\vbeta_{\depth:}) & = \max_{\mu_{\depth:}^1} \min_{\mu_{\depth:}^2} V(\occ_\depth,\vmu_{\depth:}) \qquad \text{(\citeauthor{Kuhn-ctg50}'s theorem (generalized))}\\
    & = \min_{\mu_{\depth:}^2} \max_{\mu_{\depth:}^1}  V(\occ_\depth,\vmu_{\depth:}) \qquad \text{(\citeauthor{Neu-ma28}'s theorem)}\\
    & = \min_{\beta_{\depth:}^2} \max_{\beta_{\depth:}^1}  V(\occ_\depth,\vbeta_{\depth:}) \qquad \text{(again \citeauthor{Kuhn-ctg50}'s theorem (generalized))}
\end{align}
\end{proof}

One can also show that a Bellman optimality equation allows relating optimal values in subgames at $\depth$ and $\depth+1$, leading to a recursive expression of $V^*_0$.
%Bellman's principle of optimality applies to the computation of $\max_{\beta_{\depth:}^1} \min_{\beta_{\depth:}^2} V_\depth(\occ,\vbeta_{\depth:})$ for any given occupancy state $\occ_\depth$.

%\olivier{Ci-dessous: Il manquait la réf à la première occurrence de ce théorème (dans le corps du papier). C'est corrigé.}
%\aurelien{ok, merci}
\bellmanPpe*

\begin{proof}
  \label{proofLemAbnormalMaximinimax}
  % First, given an $\occ_\depth$, $\beta^1_\depth$ and
  % $\beta^2_\depth$, let us denote
  % $\mu^i_{\depth+1} | \occ_\depth,\beta^1_\depth,\beta^2_\depth$ the
  % fact that $\mu^i_{\depth+1}$ is compatible with
  % $T(\occ_\depth,\beta^1_\depth,\beta^2_\depth)$.
  % % 
  % One can observe that the mixed strategies of player $i$ compatible
  % with $T(\occ_\depth,\beta^1_\depth,\beta^2_\depth)$ do not depend
  % on $\beta^{-i}_\depth$.
  % % 
  % This suggests simplifying the above notation as
  % $\mu^i_{\depth+1} | \occ_\depth,\beta^i_\depth$.

  % Now,
  Focusing, without loss of generality, on player $1$, we have
  (complementary explanations follow for numbered lines in
  particular):
  %\olivier{A-t-on vraiment besoin d'introduire cette fonction $maximin(\occ_\depth)$ (dont je n'aime guère le nom) ?}
  \begin{align}
  & \max_{\beta^1_\depth} \min_{\beta^2_\depth} Q^*_\depth(\occ_\depth,\beta^1_\depth,\beta^2_\depth) \nonumber % &
     = \max_{\beta^1_\depth} \min_{\beta^2_\depth} \left[ r(\occ_\depth, \beta^1_\depth, \beta^2_\depth)
      + \gamma V^*_{\depth+1}(T(\occ_\depth,\beta^1_\depth, \beta^2_\depth)) \right] \nonumber \\
    \intertext{($V^*_{\depth+1}(T(\occ_\depth,\beta^1_\depth, \beta^2_\depth))$ being the Nash equilibrium value of normal-form game $V_{\depth+1}(T(\occ_\depth,\beta^1_\depth, \beta^2_\depth), \mu^1, \mu^2 )$:)} % using the minimax theorem
    & = \max_{\beta^1_\depth} \min_{\beta^2_\depth} \left[ r(\occ_\depth, \beta^1_\depth, \beta^2_\depth)
      + \gamma
      \max_{\mu^1\in M^1_{| \occ_\depth, \vbeta_\depth \rangle}}
      \min_{\mu^2\in M^2_{| \occ_\depth, \vbeta_\depth \rangle}}
      V_{\depth+1}(T(\occ_\depth,\beta^1_\depth, \beta^2_\depth), \mu^1, \mu^2 ) \right] \nonumber \\
    % & = \max_{\beta^1_\depth} \min_{\beta^2_\depth} \left[ r(\occ_\depth, \beta^1_\depth, \beta^2_\depth)
    %   + \gamma
    %   \max_{\mu^1_{\depth+1} | \occ_\depth, \beta^1_\depth)}
    %   \min_{\mu^2_{\depth+1} | \occ_\depth, \beta^2_\depth)}
    %   V_{\depth+1}(T(\occ_\depth,\beta^1_\depth, \beta^2_\depth), \mu^1_{\depth+1}, \mu^2_{\depth+1} ) \right] \nonumber \\
    % \intertext{(because the reward in $\occ_\depth$ does not depend on $\mu^1, \mu^2$
    % given $\beta^1_\depth, \beta^2_\depth$:)}
    & = \max_{\beta^1_\depth} \min_{\beta^2_\depth}
      \max_{\mu^1\in M^1_{| \occ_\depth, \vbeta_\depth\rangle}}
      \min_{\mu^2\in M^2_{| \occ_\depth, \vbeta_\depth\rangle}}
      \left[ r(\occ_\depth, \beta^1_\depth, \beta^2_\depth)
      + \gamma V_{\depth+1}(T(\occ_\depth,\beta^1_\depth, \beta^2_\depth), \mu^1, \mu^2 ) \right] \nonumber \\
    % & = \max_{\beta^1_\depth} \min_{\beta^2_\depth}
    %   \max_{\mu^1_{\depth+1} | \occ_\depth, \beta^1_\depth)}
    %   \min_{\mu^2_{\depth+1} | \occ_\depth, \beta^2_\depth)}
    %   \left[ r(\occ_\depth, \beta^1_\depth, \beta^2_\depth)
    %   + \gamma V_{\depth+1}(T(\occ_\depth,\beta^1_\depth, \beta^2_\depth), \mu^1_{\depth+1}, \mu^2_{\depth+1} ) \right] \nonumber \\
    \intertext{(using the equivalence between maximin and minimax values for the (constrained normal-form) game at $\depth+1$, the last two max and min operators can be swapped:)}
    & = \max_{\beta^1_\depth} \min_{\beta^2_\depth}
      \min_{\mu^2\in M^2_{| \occ_\depth, \vbeta_\depth \rangle}}
      \max_{\mu^1\in M^1_{| \occ_\depth, \vbeta_\depth \rangle}}
      \left[ r(\occ_\depth, \beta^1_\depth, \beta^2_\depth)
      + \gamma V_{\depth+1}(T(\occ_\depth,\beta^1_\depth, \beta^2_\depth), \mu^1, \mu^2) \right] \nonumber \\
    % & = \max_{\beta^1_\depth} \min_{\beta^2_\depth}
    %   \min_{\mu^2_{\depth+1} | \occ_\depth, \beta^2_\depth)}
    %   \max_{\mu^1_{\depth+1} | \occ_\depth, \beta^1_\depth)}
    %   \left[ r(\occ_\depth, \beta^1_\depth, \beta^2_\depth)
    %   + \gamma V_{\depth+1}(T(\occ_\depth,\beta^1_\depth, \beta^2_\depth), \mu^1_{\depth+1}, \mu^2_{\depth+1} ) \right] \nonumber \\
    \intertext{(merging both mins (and with explanations thereafter):)}
    % by using equivalences between behavioral and mixed strategies, with $\beta^2_\depth(\mu_2)$ the decision rule at time $\depth$ induced by $\mu_2$:)}
    \nonumber
    & = \max_{\beta^1_\depth}
      \min_{\mu^2 \in M^2_{| \occ_\depth, \beta^1_\depth \rangle}}
      \max_{\mu^1 \in M^1_{| \occ_\depth, \beta^1_\depth, \beta^2_\depth(\mu^2) \rangle}}
      \\ & 
      \left[ r(\occ_\depth, \beta^1_\depth, \beta^2_\depth(\mu_2))
      + \gamma V_{\depth+1}(T(\occ_\depth,\beta^1_\depth, \beta^2_\depth(\mu_2)), \mu^1, \mu^2 ) \right] \label{eq|twomins} \\
      % \intertext{(with some new equivalence that needs a proof:)}
    %\intertext{(with the equivalence property discussed before the lemma:)}
    \intertext{(since ignoring the opponent's decision rule does not influence the expected return:)}
    & = \max_{\beta^1_\depth}
      \min_{\mu^2\in M^2_{| \occ_\depth \rangle}}
      \max_{\mu^1 \in M^1_{| \occ_\depth, \beta^1_\depth\rangle}}
      \left[ r(\occ_\depth, \beta^1_\depth, \beta^2_\depth(\mu_2))
      + \gamma V_{\depth+1}(T(\occ_\depth,\beta^1_\depth, \beta^2_\depth(\mu_2)), \mu^1, \mu^2 ) \right] \nonumber \\
    \intertext{(using again the minimax theorem's equivalence between maximin and minimax on an appropriate game:)}
    & = \max_{\beta^1_\depth}
      \max_{\mu^1\in M^1_{| \occ_\depth, \beta^1_\depth\rangle}}
      \min_{\mu^2\in M^2_{| \occ_\depth\rangle}}
      \left[ r(\occ_\depth, \beta^1_\depth, \beta^2_\depth(\mu_2))
      + \gamma V_{\depth+1}(T(\occ_\depth,\beta^1_\depth, \beta^2_\depth(\mu_2)), \mu^1, \mu^2 ) \right] \label{eq|appropriateGame} \\
    \intertext{(merging both maxs (and with explanations thereafter):)}
    & =
      \max_{\mu^1\in M^1_{| \occ_\depth\rangle}}
      \min_{\mu^2_{\depth} | \occ_\depth)}
      \left[ r(\occ_\depth, \beta^1_\depth(\mu_1), \beta^2_\depth(\mu_2))
      + \gamma V_{\depth+1}(T(\occ_\depth,\beta^1_\depth(\mu^1), \beta^2_\depth(\mu_2)), \mu^1, \mu^2 ) \right] \label{eq|twomaxs} \\
      % \intertext{(again with some new equivalence that needs a proof:)}
    \intertext{(again with the equivalence property discussed before the lemma:)}
    & =
      \max_{\mu^1\in M^1_{| \occ_\depth\rangle}}
      \min_{\mu^2\in M^2_{| \occ_\depth\rangle}}
      V_\depth(\occ_\depth, \mu^1, \mu^2) \nonumber \\
    & =
      \max_{\beta^1_{\depth:H-1 | \occ_\depth\rangle}}
      \min_{\beta^2_{\depth:H-1 | \occ_\depth\rangle}}
      V_\depth(\occ_\depth, \beta^1_{\depth:H-1}, \beta^2_{\depth:H-1}) \nonumber \\
    & \eqdef
      V^*_\depth(\occ_\depth). \nonumber
  \end{align}

  More precisely, line \ref{eq|twomins} (and, similarly, line
  \ref{eq|twomaxs}) is obtained by observing that
  \begin{itemize}
  \item minimizing over both (i) $\beta^2_\depth$ and (ii) $\mu^2$
    constrained by $\occ_\depth$ and $\vbeta_\depth$ is equivalent to
    minimizing over $\mu^2$ constrained by $\occ_\depth$ and
    $\beta^1_\depth$; and
  \item in the reminder of the formula, decision rule $\beta^2_\depth$
    at time $\depth$ can be retrieved as a function of $\mu^2$ (noted
    $\beta^2_\depth(\mu^2)$).
  \end{itemize}
  Also, line \ref{eq|appropriateGame} results from the observation
  that, while $M^1_{| \occ_\depth, \beta^1_\depth \rangle}$ and
  $M^2_{| \occ_\depth \rangle}$ allow to actually make decisions over
  different time intervals, we are here minimizing over $\mu^2$ while
  maximizing over $\mu^1$ a function that is linear in both input
  spaces.
  This amounts to solving some 2-player zero-sum normal-form game,
  hence the applicability of von Neumann's minimax theorem.

  The above derivation tells us that the maximin value (the best
  outcome player $1$ can guarantee whatever player $2$'s strategy) in
  the one-time-step game is thus the Nash equilibrium value (NEV) for
  the complete subgame from $\depth$ onwards.
\end{proof}

\temporallyHidden{
\olivier{Il {\bf faut} utiliser l'environnement "restatable" pour ne pas avoir à maintenir plusieurs versions du même théorème/corollaire/\dots Les corrections faites dans le corps de l'article ne sont pas reprises ici.}

\olivier{{\bf La suite était écrite en supposant qu'on travaillerait avec des méta-stratégies (donc en boucle fermée). On peut encore en tirer quelques idées, mais bcp de choses deviennent hors-sujet, à commencer par les réfs biblio.} 
Je crois qu'il faut expliquer clairement (dans l'article) que:
\begin{itemize}
    \item le concept solution qui nous intéresse pour un \zsomg{} est celui d'équilibre de Nash; puis
    \item le POB s'applique aux \zsomg{} (ce qui n'a rien d'évident), permettant de construire un \nes{} en $\occ_\depth$ si on a des \nes{} pour tout \os atteignable depuis $\occ_\depth$, càd d'écrire une équation de Bellman.
\end{itemize}
Tout cela risque d'être compliqué parce qu'on manipule des méta-stratégies, donc un espace de stratégies avec un nombre infini de dimensions !}
\Aurelien{Concernant tous les soucis de maxmin et minmax, on devrait pouvoir voir un zs-oMG comme un zs-SG particulier avec espace d'état infini; on a pas déjà des propriétés dessus? Du genre le théorème de Sion (pas celui de Von-neumann qui lui se place effectivement dans le cas d'une matrice, donc finie); le théorème de sion se place juste dans un espace topologique vectoriels $X \times Y$ où la fonction de payoff est convexe/concave, donc il devrait s'appliquer à $\cM^1 \times \cM^2$. La convexe-concavité en les meta-stratégies me semble facile.}
\olivier{Les mêmes idées étaient aussi en train de me venir :) On notera que, en plus de l'espace d'états infini, l'espace d'actions (que l'on considère les règles de décision pour un jeu local, ou les méta-stratégies (comportementales) pour un sous-jeu) est aussi infini. Ces espaces sont aussi convexes, mais il n'est pas évident que l'espace des méta-stratégies (un espace de fonctions) soit compact.
\\
Note: Papiers trouvés sur les SG à états et actions continus: d'abord \citep{Sobel-jap73} dans le cas somme générale et avec facteur d'atténuation, de là, \citep{MaiPAr-jota70,MaiPAr-jota71} dans le cas somme nulle et (toujours) avec facteur d'atténuation.}
\aurelien{de ce que je cois, le papier de Sobel, 1973 donne bien la solution a notre probleme, non? les hypothese telles qu'espace metriue etc.. sont triviallement respectees}
\aurelien{ps : $\cM^i$ est bien ce qu'on appelle un "espace topologique vectoriel", puisque c'est un espace vectoriel (potentiellement de dimension infinie, peu importe) qui a une topologie.}
\aurelien{
\begin{align}
    r(\occ_\depth,\vmu_\depth) = \sum_{\boldsymbol \theta_\depth, \boldsymbol a} \vmu(\occ_\depth)(\boldsymbol a \mid \boldsymbol \theta_\depth) \sum_{s} \occ_\depth(s \mid \boldsymbol \theta) r(s,\boldsymbol a)
\end{align}
On montre après facilement par récurrence que ça s'applique aussi sur plusieurs pas de temps. Du coup, minimax = maximin et par la preuve classique d'échange de min et de max, on montre que POB s'applique bien aux zs-omg. Non?
}
\olivier{J'ai du mal à suivre.
\begin{itemize} 
    \item D'abord, si cette récurrence cherche à montrer que la récompense est linéaire dans l'espace des méta-stratégies, ça ne me semble pas évident.
    En effet, si on modifie $\mu(\occ_\depth)$, alors l'état d'occupation atteint $\occ_{\depth+1}$ n'est plus le même, et la règle de décision associée peut être complètement différente.
    \\
    Ce doit être en particulier vrai si on s'intéresse à Matching Pennies: $\mu^{*,2}_1(\occ_1)$ est une fonction discontinue, basculant d'une règle de décision déterministe à une autre en un certain $\occ_1$ limite (donc ici en un certain $\mu^{*,1}_0(b_0)$ limite).
    Ainsi, si on part de $\vmu=\vmu^*$ et qu'on fait varier $\mu^{1}_0(b_0)$ linéairement, la récompense $r(\occ_0,\vmu_0)$ va croître linéairement jusqu'en la règle de décision limite ($\mu^{*,1}_0(b_0)$), puis décroître linéairement.
    \item Ensuite, le théorème de Sion requière qu'un des espaces d'actions soit compact, ce qui ne me semble pas évident quand il s'agit de méta-stratégies (entre autres par incompétence sur le sujet). [Les autres hypothèses seraient à vérifier proprement, mais me semblent moins problématiques.]
\end{itemize}
}

\Aurelien{je trouve qu'on s'éloigne du sujet et que cette discussion  pourrait être intéressante pour la version JAIR, mais surtout en tout début d'article, elle risque de perdre le lecteur. Ca ne me choquerait pas de définir les sous-jeux "restreints" d'un zs-omg où les stratégies sont des stratégies comportementales en signalant par une remarque que ce n'est pas un sous-jeu d'un zs-omg, mais que les sous-jeux des zs-omgs ne nous intéressent pas vraiment. Peu importe ce que ça représente en tant que stratégie/jeu dans l'omg, la suele chose qui nous intéresse c'est la valeur des fonctions.}

\olivier{Hum\dots Si je comprends bien, tu proposes de ne pas aborder le \zsomg{} avec des méta-stratégies comme solutions, mais avec des stratégies comportementales. C'est peut-être faisable, mais il faut
(1) bien expliquer qu'on fait cela (qu'on veut résoudre le \zsomg{} en "boucle ouverte" et non en "boucle fermée" comme on résoudrait un SG d'habitude); et
(2) démontrer que, dans ce problème, le POB s'applique: si je sais résoudre le "sous-jeu" (trouver un \nes{}) partant de n'importe quel $\occ_{\depth+1}$, alors je sais en faire autant pour tout $\occ_\depth$. 
\\
Ceci-dit, je n'arrive pas à voir comment justifier les descentes de HSVI qui se font en concaténant des règles de décision. Où peut-on montrer que, dans le sous-jeu associé à $\occ_\depth$, la première règle de décision est solution du jeu local ?
Peut-être a-t-on déjà bcp d'éléments du puzzle, et suffit-il de les remettre à la bonne place ?}

\aurelien{
\begin{lemma}[Sion's theorem]
  for any occupancy state $\occ_\depth$, sion's theorem implies that $\max_{\mu_{\depth:}^1 \in \cM^1} \min_{\mu_{\depth:}^2 \in \cM^2} V(\occ_\depth, \vmu_{\depth:}) = \min_{\mu_{\depth:}^2 \in \cM^2} \max_{\mu_{\depth:}^1 \in \cM^1}  V(\occ_\depth, \vmu_{\depth:})$
\end{lemma}
\begin{lemma}[POB for zs-oMGs]
  \begin{align}
      & \max_{\mu_{\depth:}^1 \in \cM^1} \min_{\mu_{\depth:}^2 \in \cM^2} V(\occ_\depth, \vmu_{\depth:}) \\
      & = \max_{\mu_{\depth}^1 \in \cM^1} \max_{\mu_{\depth+1:}^1 \in \cM^1} \min_{\mu_{\depth:}^2 \in \cM^2} V(\occ_\depth, \vmu_{\depth:})  \qquad \text{$\downarrow$ th. de sion}\\
      & = \max_{\mu_{\depth}^1 \in \cM^1} \min_{\mu_{\depth}^2 \in \cM^2} r(\occ_\depth,\vmu_\depth) + \gamma \max_{\mu_{\depth+1:}^1 \in \cM^1} \min_{\mu_{\depth+1:}^2 \in \cM^2} V(T(\occ_\depth,\vmu_\depth), \vmu_{\depth+1:}) 
  \end{align}
\end{lemma}
}
\begin{theorem}
  The NEV value of a zs-oMG at $\occ_\depth$ is equal to the value of the zs-POSG subgame at $\occ_\depth$.
\end{theorem}

\begin{proof}
\label{app|temp}
We prove this property by induction.
%\olivier{Il existe la commande $\backslash$vmu pour obtenir $\vmu$ (au lieu d'avoir à écrire $\{\backslash$boldsymbol mu$\}$). Idem pour $\vbeta$, \dots}
First, let $\occ_{H-1}$ be any occupancy-state for the last time step.
Then,
\begin{align}
    V_{oMG}^* &= \max_{\mu^1_{H-1} \in M_{H-1}^1} \min_{\mu^2_{H-1} \in M_{H-1}^2} r(\occ_{H-1},\vmu) \\
    & = \max_{\mu^1 \in M_{H-1}^1} \min_{\mu^2 \in M_{H-1}^2} r(\occ_{H-1}, \vmu_{H-1}(\occ_{H-1})) \\
    & = \max_{\beta^1_{H-1} \in \cB_{H-1}^1} \min_{\beta^2_{H-1} \in \cB_{H-1}^2} r(\occ_{H-1}, \vbeta_{H-1}) \\
    & = V_{POSG}^*(\occ_{H-1}).
\end{align}

Now, assume that, for any occupancy state $\occ_{\depth+1}$, $V_{oMG}^*(\occ_{\depth+1}) = V_{POSG}^*(\occ_{\depth+1})$.
Then, for any occupancy-state $\occ_\depth$,
\begin{align}
    \intertext{\olivier{Pour pouvoir écrire la première ligne, il faut déjà avoir démontré que la valeur de l'équilibre de Nash d'un sous-zs-OMG en $\occ_\depth$ peut s'écrire via le maximin ou le minimax.  Or il me semble qu'on n'a pas encore cette propriété du zs-OMG.}}
    &\max_{\mu_{\depth:}^1} \min_{\mu_{\depth:}^2} V_{oMG}(\occ_\depth, \boldsymbol \mu_{\depth:}) \\
    & = \max_{\mu_{\depth:}^1} \min_{\mu_{\depth:}^2} r(\occ_\depth, \boldsymbol \mu_{\depth}) + \gamma V_{oMG}(T(\occ_\depth, \boldsymbol \mu_{\depth}), \boldsymbol \mu_{\depth+1:}) %\\
    \intertext {\olivier{Même observation pour le jeu local en $\occ_\depth$.}}
    & = \max_{\mu_{\depth}^1} \min_{\mu_{\depth}^2} r(\occ_\depth, \boldsymbol \mu_{\depth}) + \gamma \max_{\mu_{\depth+1:}^1} \min_{\mu_{\depth+1:}^2} V_{oMG}(T(\occ_\depth, \boldsymbol \mu_{\depth}), \boldsymbol \mu_{\depth+1:}) \\
    & = \max_{\mu_{\depth}^1} \min_{\mu_{\depth}^2} r(\occ_\depth, \boldsymbol \mu_{\depth}) + \gamma \max_{\beta_{\depth+1:}^1} \min_{\beta_{\depth+1:}^2} V_{POSG}(T(\occ_\depth, \boldsymbol \mu_{\depth}), \boldsymbol \beta_{\depth+1:}) \\
    & = \max_{\beta_{\depth}^1} \min_{\beta_{\depth}^2} r(\occ_\depth, \boldsymbol \beta_{\depth}) + \gamma \max_{\beta_{\depth+1:}^1} \min_{\beta_{\depth+1:}^2} V_{POSG}(T(\occ_\depth, \boldsymbol \mu_{\depth}), \boldsymbol \beta_{\depth+1:}) \\
    & = \max_{\beta_{\depth:}^1} \min_{\beta_{\depth:}^2} r(\occ_\depth, \boldsymbol \beta_{\depth}) + \gamma V_{POSG}(T(\occ_\depth, \boldsymbol \mu_{\depth}), \boldsymbol \beta_{\depth+1:}) \\
    & = \max_{\beta_{\depth:}^1} \min_{\beta_{\depth:}^2} V_{POSG}(\occ_\depth, \boldsymbol \beta_{\depth:}) \\
\end{align}
\end{proof}
}

\section{Solving zs-OMGs}
\label{app|zsOMGs}

%\subsection{Introducing Local Games}
%\label{app|bellman}

%
\temporallyHidden{
\aurelien{
First, using \Cref{app|lin|occ}, von Neumman's minimax theorem applies to any subgame $(\beta_\depth^1,\beta_\depth^2) \to V(\occ_\depth,\beta_\depth^1,\beta_\depth^2)$. 
Secondly, we show that Bellman's principle of Optimality applies to the search for NES on any subgame $V(\occ_\depth,\cdot,\cdot)$.
\bellmanPpe*
\begin{proof}
\temporallyHidden{
Soit $\zeta = \langle (i_1,t_1), \dots (i_{2K}, t_{2K}) \rangle$ un ordre cible (qu'on veut atteindre) des opérateurs de $\depth$ à $H-1$ (d'où $K=(H-1)-\depth + 1$ ?).
    On peut en extraire $\zeta^1 = \langle (i^1_1,t^1_1), \dots (i^1_K, t^1_K) \rangle$ (resp. $\zeta^2$) qui ne contient que les paires dont le premier terme est $1$ (resp. $2$).
    Les max (resp. min) étant interchangeable, on a:
    \begin{align}
        V^*_\depth(\occ_\depth)
        & = \max_{\beta^1_{\depth:}} \min_{\beta^2_{\depth:}} V_\depth(\occ_\depth,\vbeta_{\depth:}) \\
        & = \max_{\beta^1_0} \dots \max_{\beta^1_{H-1}}
        \min_{\beta^2_0} \dots \min_{\beta^2_{H-1}}
        V_\depth(\occ_\depth,\vbeta_{\depth:}) \\
        & = \max_{\beta^1_0} \dots \max_{\beta^1_{H-1}}
        \min_{\beta^2_0} \dots \min_{\beta^2_{H-1}}
        V_\depth(\occ_\depth,\vbeta_{\depth:}) \\
        & \dots \text{TODO}
    \end{align}
    \item En déduire la récurrence suivante, qui justifie l'utilisation du POB:
    \begin{align}
        V^*_\depth(\occ_\depth)
        & = \max_{\beta^1_{\depth}} \min_{\beta^2_{\depth}} \max_{\beta^1_{\depth+1:}} \min_{\beta^2_{\depth+1:}} V_\depth(\occ_\depth,\vbeta_{\depth:})
        \\
        & = \max_{\beta^1_{\depth}} \min_{\beta^2_{\depth}} \max_{\beta^1_{\depth+1:}} \min_{\beta^2_{\depth+1:}} \left[ r(\occ_\depth,\vbeta_{\depth}) + V_{\depth+1}(T(\occ_\depth,\vbeta_{\depth}),\vbeta_{\depth+1:}) \right]
        \\
        & = \max_{\beta^1_{\depth}} \min_{\beta^2_{\depth}} \left[ r(\occ_\depth,\vbeta_{\depth}) + \max_{\beta^1_{\depth+1:}} \min_{\beta^2_{\depth+1:}} V_{\depth+1}(T(\occ_\depth,\vbeta_{\depth}),\vbeta_{\depth+1:}) \right]
        \\
        & = \max_{\beta^1_{\depth}} \min_{\beta^2_{\depth}} \left[ r(\occ_\depth,\vbeta_{\depth}) +  V^*_{\depth+1}(T(\occ_\depth,\vbeta_{\depth})) \right].
    \end{align}
    }
    \begin{align}
    maximin(\occ_\depth) 
    & \eqdef \max_{\beta^1_\depth} \min_{\beta^2_\depth} \left[ r(\occ_\depth, \beta^1_\depth, \beta^2_\depth)
      + \gamma V^*_{\depth+1}(T(\occ_\depth,\beta^1_\depth, \beta^2_\depth)) \right] \nonumber \\
    \intertext{($V^*_{\depth+1}(T(\occ_\depth,\beta^1_\depth, \beta^2_\depth))$ being the Nash equilibrium value of normal-form game $V_{\depth+1}(T(\occ_\depth,\beta^1_\depth, \beta^2_\depth), \mu^1, \mu^2 )$:)} % using the minimax theorem
    & = \max_{\beta^1_\depth} \min_{\beta^2_\depth} \left[ r(\occ_\depth, \beta^1_\depth, \beta^2_\depth)
      + \gamma
      \max_{\beta_{\depth+1:}^1}
      \min_{\beta_{\depth+1:}^2}
      V_{\depth+1}(T(\occ_\depth,\beta^1_\depth, \beta^2_\depth), \vbeta_{\depth+1:} ) \right] \nonumber \\
    % & = \max_{\beta^1_\depth} \min_{\beta^2_\depth} \left[ r(\occ_\depth, \beta^1_\depth, \beta^2_\depth)
    %   + \gamma
    %   \max_{\mu^1_{\depth+1} | \occ_\depth, \beta^1_\depth)}
    %   \min_{\mu^2_{\depth+1} | \occ_\depth, \beta^2_\depth)}
    %   V_{\depth+1}(T(\occ_\depth,\beta^1_\depth, \beta^2_\depth), \mu^1_{\depth+1}, \mu^2_{\depth+1} ) \right] \nonumber \\
    % \intertext{(because the reward in $\occ_\depth$ does not depend on $\mu^1, \mu^2$
    % given $\beta^1_\depth, \beta^2_\depth$:)}
    & = \max_{\beta^1_\depth} \min_{\beta^2_\depth}
      \max_{\beta_{\depth+1:}^1}
      \min_{\beta_{\depth+1:}^2}
      \left[ r(\occ_\depth, \beta^1_\depth, \beta^2_\depth)
      + \gamma V_{\depth+1}(T(\occ_\depth,\beta^1_\depth, \beta^2_\depth), \mu^1, \mu^2 ) \right] \nonumber \\
    % & = \max_{\beta^1_\depth} \min_{\beta^2_\depth}
    %   \max_{\mu^1_{\depth+1} | \occ_\depth, \beta^1_\depth)}
    %   \min_{\mu^2_{\depth+1} | \occ_\depth, \beta^2_\depth)}
    %   \left[ r(\occ_\depth, \beta^1_\depth, \beta^2_\depth)
    %   + \gamma V_{\depth+1}(T(\occ_\depth,\beta^1_\depth, \beta^2_\depth), \mu^1_{\depth+1}, \mu^2_{\depth+1} ) \right] \nonumber \\
    \intertext{(using the equivalence between maximin and minimax values for the (constrained normal-form) game at $\depth+1$, the last two max and min operators can be swapped:)}
    & = \max_{\beta^1_\depth} \min_{\beta^2_\depth}
      \min_{\beta_{\depth+1:}^2}
      \max_{\beta_{\depth+1:}^1}
      \left[ r(\occ_\depth, \beta^1_\depth, \beta^2_\depth)
      + \gamma V_{\depth+1}(T(\occ_\depth,\beta^1_\depth, \beta^2_\depth), \vbeta_{\depth:}) \right] \nonumber \\
    % & = \max_{\beta^1_\depth} \min_{\beta^2_\depth}
    %   \min_{\mu^2_{\depth+1} | \occ_\depth, \beta^2_\depth)}
    %   \max_{\mu^1_{\depth+1} | \occ_\depth, \beta^1_\depth)}
    %   \left[ r(\occ_\depth, \beta^1_\depth, \beta^2_\depth)
    %   + \gamma V_{\depth+1}(T(\occ_\depth,\beta^1_\depth, \beta^2_\depth), \mu^1_{\depth+1}, \mu^2_{\depth+1} ) \right] \nonumber \\
    \intertext{(merging both mins (and with explanations thereafter):)}
    % by using equivalences between behavioral and mixed strategies, with $\beta^2_\depth(\mu_2)$ the decision rule at time $\depth$ induced by $\mu_2$:)}
    & = \max_{\beta^1_\depth}
      \min_{\beta_{\depth:}^2}
      \max_{\beta_{\depth+1:}^1}
      \left[ r(\occ_\depth, \beta^1_\depth, \beta^2_\depth(\mu_2))
      + \gamma V_{\depth+1}(T(\occ_\depth,\beta^1_\depth, \beta^2_\depth), \vbeta_{\depth:}) \right] \label{eq|twomins} \\
      % \intertext{(with some new equivalence that needs a proof:)}
    %\intertext{(with the equivalence property discussed before the lemma:)}
    \intertext{(using again the minimax theorem's equivalence between maximin and minimax on an appropriate game:)}
    & = \max_{\beta^1_\depth}
      \max_{\beta_{\depth+1:}^1}
      \min_{\beta_{\depth+1:}^2}
      \left[ r(\occ_\depth, \beta^1_\depth, \beta^2_\depth(\mu_2))
      + \gamma V_{\depth+1}(T(\occ_\depth,\beta^1_\depth, \beta^2_\depth), \vbeta_{\depth:}) \right] \label{eq|appropriateGame} \\
    \intertext{(merging both maxs (and with explanations thereafter):)}
    & =
      \max_{\beta_{\depth:}^1}
      \min_{\beta_{\depth:}^1}
      \left[ r(\occ_\depth, \vbeta_\depth)
      + \gamma V_{\depth+1}(T(\occ_\depth,\beta^1_\depth, \beta^2_\depth), \vbeta_{\depth:}) \right] \label{eq|twomaxs} \\
    & \eqdef
      V^*_\depth(\occ_\depth). \nonumber
  \end{align}
\end{proof}
}
%
%
%
%
%
%
%------------------------------------------
%
%
}

\subsection{Preliminary Properties}

\subsubsection{Properties of \texorpdfstring{$\Tm{1}$}{Tm1} and \texorpdfstring{$\Tc{1}$}{Tc1}}

The first two lemmas below present properties of $\Tm{1}$ and $\Tc{1}$ that will be useful afterwards.
% to demonstrate concavity and convexity properties of $Q^*$. 
%study and approximate $W^{*,1}_\depth$.

\begin{lemma}
  \labelT{lem|T1mlin}
  $\Tm{1}(\occ_\depth, \vbeta_\depth)$ is linear in $\occ_\depth$, $\beta^1_\depth$, and $\beta^2_\depth$.
\end{lemma}

\begin{proof}
  \label{proof|lem|T1mlin}
  \begin{align}
    & \Tm{1}(\occ_\depth,\vbeta_\depth)(\theta^1_\depth,a^1,z^1) \\ %
    & = \sum_{\theta^2_\depth, a^2, z^2} T(\occ_\depth, \vbeta_\depth)({
      ( \theta^1_\depth, a^1, z^1 ),
      ( \theta^2_\depth, a^2, z^2 )
    })
    \qquad %&
    \text{(from \Cshref{eq|transition})}
    \nonumber
    \\
    & = \sum_{s',\theta^2_\depth,a^2,z^2} \beta^1_\depth(\theta^1_\depth,a^1) \beta^2_\depth(\theta^2_\depth,a^2)  \sum_{s} P^{z^1,z^2}_{a^1,a^2}(s'|s) 
    b(s|\theta^1_\depth,\theta^2_\depth) \occ_\depth (\theta^1_\depth,\theta^2_\depth)
    \nonumber
    \\
    & = \beta^1_\depth(\theta^1_\depth,a^1) \sum_{\theta^2_\depth,a^2} \beta^2_\depth(\theta^2_\depth,a^2)  \sum_{s,s',z^2} P^{z^1,z^2}_{a^1,a^2}(s'|s) 
    b(s|\theta^1_\depth,\theta^2_\depth) \occ_\depth (\theta^1_\depth,\theta^2_\depth).
    \label{eq|occm1}
    % 
    % \qedhere <- breaks label number !
  \end{align}
\end{proof}

\begin{lemma}
  \labelT{lem|T1cindep}
  $\Tc{1}(\occ_\depth, \vbeta_\depth)$ is independent of $\beta^1_\depth$ and $\occ^{m,1}_\depth$.
\end{lemma}

\begin{proof}
\ifArxiv{
See \cite{Wiggers-msc15}, Lemma 4.2.3.
}
{
  \label{proof|lem|T1cindep}
  \begin{align*}
    % \occ_{\depth+1}^{c,1}((\theta^2_\depth,a^2,z^2)|(\theta^1_\depth, a^1, z^1))
    \hspace{3cm}
    & \hspace{-3cm} \Tc{1}(\occ_\depth, \vbeta_\depth)((\theta^2_\depth,a^2,z^2) | (\theta^1_\depth, a^1, z^1)) %
    = \frac{
      T(\occ_\depth, \vbeta_\depth)( (\theta^1_\depth, a^1, z^1), (\theta^2_\depth,a^2,z^2) )
    }{
      \sum_{\theta^2_\depth, a^2, z^2} T(\occ_\depth, \vbeta_\depth)( (\theta^1_\depth, a^1, z^1), (\theta^2_\depth,a^2,z^2) )
    } \\
    & = \frac{
      \beta^1_\depth(\theta^1_\depth,a^1) \beta^2_\depth(\theta^2_\depth,a^2)  \sum_{s,s'} P^{z^1,z^2}_{a^1,a^2}(s'|s) b(s | \theta^1_\depth, \theta^2_\depth) \occ_\depth (\theta^1_\depth, \theta^2_\depth)
    }{
      \beta^1_\depth(\theta^1_\depth,a^1) \sum_{\theta^2,a^2} \beta^2_\depth(\theta^2_\depth,a^2)  \sum_{s,s',z^2} P^{z^1,z^2}_{a^1,a^2}(s'|s) b(s | \theta^1_\depth, \theta^2_\depth) \occ_\depth (\theta^1_\depth,\theta^2_\depth)
    } \\
    & = \frac{
      \beta^2_\depth(\theta^2_\depth,a^2)  \sum_{s,s'} P^{z^1,z^2}_{a^1,a^2}(s'|s) b(s | \theta^1_\depth, \theta^2_\depth) \occ_\depth (\theta^1_\depth,\theta^2_\depth)
    }{
      \sum_{\theta^2,a^2} \beta^2_\depth(\theta^2_\depth,a^2)  \sum_{s,s',z^2} P^{z^1,z^2}_{a^1,a^2}(s'|s) b(s | \theta^1_\depth, \theta^2_\depth) \occ_\depth (\theta^1_\depth,\theta^2_\depth)
    }
    \\
    & = \frac{
      \beta^2_\depth(\theta^2_\depth,a^2)  \sum_{s,s'} P^{z^1,z^2}_{a^1,a^2}(s'|s) b(s | \theta^1_\depth, \theta^2_\depth)
      \overbrace{ \occ^{c,1}_\depth(\theta^2_\depth | \theta^1_\depth) \occ^{m,1}_\depth(\theta^1_\depth) }
    }{
      \sum_{\theta^2,a^2} \beta^2_\depth(\theta^2_\depth,a^2)  \sum_{s,s',z^2} P^{z^1,z^2}_{a^1,a^2}(s'|s) b(s | \theta^1_\depth, \theta^2_\depth)
      \underbrace{ \occ^{c,1}_\depth(\theta^2_\depth | \theta^1_\depth) \occ^{m,1}_\depth(\theta^1_\depth) }
    } \\
    & = \frac{
      \left( \beta^2_\depth(\theta^2_\depth,a^2)  \sum_{s,s'} P^{z^1,z^2}_{a^1,a^2}(s'|s) b(s | \theta^1_\depth, \theta^2_\depth)
        \occ^{c,1}_\depth(\theta^2_\depth | \theta^1_\depth) \right) \occ^{m,1}_\depth(\theta^1_\depth)
    }{
      \left( \sum_{\theta^2,a^2} \beta^2_\depth(\theta^2_\depth,a^2)  \sum_{s,s',z^2} P^{z^1,z^2}_{a^1,a^2}(s'|s) b(s | \theta^1_\depth, \theta^2_\depth)
        \occ^{c,1}_\depth(\theta^2_\depth | \theta^1_\depth)  \right) \occ^{m,1}_\depth(\theta^1_\depth)
    } \\
    & = \frac{
      \beta^2_\depth(\theta^2_\depth,a^2)  \sum_{s,s'} P^{z^1,z^2}_{a^1,a^2}(s'|s) b(s | \theta^1_\depth, \theta^2_\depth)
        \occ^{c,1}_\depth(\theta^2_\depth | \theta^1_\depth)
    }{
      \sum_{\theta^2,a^2} \beta^2_\depth(\theta^2_\depth,a^2)  \sum_{s,s',z^2} P^{z^1,z^2}_{a^1,a^2}(s'|s) b(s | \theta^1_\depth, \theta^2_\depth)
      \occ^{c,1}_\depth(\theta^2_\depth | \theta^1_\depth)
    }.
    %\label{eq|occc1} 
    % 
    % \qedhere
  \end{align*}
}
\end{proof}

\poubelle{
\begin{tcolorbox}[breakable, enhanced]
  \uline{Addendum:} The following complementary properties explain why seeking for better approximations is difficult.

  \begin{proposition}
    \labelT{lemma|NotLipschitzBeta}
    $\Tc{i}(\occ_\depth,\vbeta_\depth)$ may be non continuous (thus
    non Lipschitz-continuous) w.r.t. $\beta_\depth^{\neg i}$.
    % is not necessarely Lipschitz continuous w.r.t. $\beta_\depth^{\neg i}$.
    % 
  \end{proposition}

  \begin{proof}
    \label{proof|lemma|NotLipschitzBeta}
    First, let us define
    \begin{align}
      & \begin{array}{r@{\,}c@{\,}c@{\,}c}
          f : & S_3(1) & \to & \reals  \\
              & (x,y,z) & \mapsto & \frac{\alpha x}{\alpha x + \beta y},
        \end{array}
      \end{align}
      where $(\alpha,\beta) \in (\mathbb{R}^{+,*})^2$ and $S_k(1)$ is the $k$-dimensional probability simplex.
      % $H = \{(x,y,z) \in \mathbb{R}^3\ |\ x+y+z = 1\}$ .
      
      One can show that $f$ is not Lipschitz-Continuous.
      Indeed, the sequences 
      \begin{align} 
        (u_n)_n & = \left( f\left(\frac{1}{n}, \frac{1}{n^2}, 1-\left(\frac{1}{n} + \frac{1}{n^2}\right)\right)\right)_n
        \text{ and} \\
        (v_n)_n & = \left( f\left(\frac{1}{n^2}, \frac{1}{n}, 1-\left(\frac{1}{n} + \frac{1}{n^2}\right)\right)\right)_n
      \end{align} 
      converge towards different values (respectively 1 and 0).
      $f$ is thus not continuous around $(0,0,1)$, and %
      therefore not Lipschitz continuous.
      
      This property extends to functions of the form %
      $f(x, y_1, \dots, y_I, z_1, \dots, z_J) = %
      \frac{\alpha x}{\alpha x + \sum_{i=1}^I \beta_i y_i}$ with %
      \begin{itemize}
      \item $I,J\in \mathbb{N}^*$,
      \item $(x,y_1,\dots,y_I,z_1,\dots,z_J) \in S_{1+I+J}(1)$,
      \item positive scalars $\alpha$ and $\vbeta_i$
        ($i\in \{1, \dots, I\}$).
      \end{itemize}

      Note: In the following, we make plausible assumptions without
      providing a detailed example.
      Let us now consider \Cshref{eq|occc1} for two tuples
      $\langle \theta^1_\depth, a^1, z^1 \rangle$ and
      $\langle \theta^2_\depth, a^2, z^2 \rangle$ such that
      $\occ_\depth(\theta^1_\depth,\theta^2_\depth)\neq 0$:
      \begin{align}
        & \Tc{1}(\occ_\depth,\vbeta_\depth)(\theta^2_{\tau},a^2,z^2|\theta^1_{\tau},a^1,z^1) \\
        % & = \frac{\beta^2_{\tau}(\theta^2_{\tau},a^2)  \sum_{s,s'} P^{z^1,z^2}_{a^1,a^2}(s'|s) o_{\tau} (s,\theta^1_{\tau},\theta^2_{\tau}))}{
        % \sum_{\theta^2,a^2} \beta^2_{\tau}(\theta^2_{\tau},a^2)  \sum_{s,s',z^2} P^{z^1,z^2}_{a^1,a^2}(s'|s) o_{\tau} (s,\theta^1_{\tau},\theta^2_{\tau}))} \\
        & \qquad =  \frac{
          \beta^2_\depth(\theta^2_\depth,a^2) \left[ \sum_{s,s'} P^{z^1,z^2}_{a^1,a^2}(s'|s) b(s | \theta^1_\depth, \theta^2_\depth) \right] \occ_\depth (\theta^1_\depth,\theta^2_\depth)
        }{
          \sum_{\hat\theta^2, \hat a^2} \beta^2_\depth(\hat \theta^2_\depth, \hat a^2) \left[ \sum_{s,s',\hat z^2} P^{z^1,\hat z^2}_{a^1, \hat a^2}(s'|s) b(s | \theta^1_\depth, \hat \theta^2_\depth) \right] \occ_\depth (\theta^1_\depth, \hat\theta^2_\depth)} %
        \intertext{and assuming a simple case where $\occ^{c,1}(\theta^2_\depth|\theta^1_\depth)=1$ (\ie, all other \aoh{}s for $2$ being impossible):}
        & \qquad =  \frac{
          \beta^2_\depth(\theta^2_\depth,a^2)  \left[ \sum_{s,s'} P^{z^1,z^2}_{a^1,a^2}(s'|s) b(s | \theta^1_\depth, \theta^2_\depth) \right] \occ_\depth (\theta^1_\depth,\theta^2_\depth)
        }{
          \sum_{\hat a^2} \beta^2_\depth(\theta^2_\depth, \hat a^2) \left[ \sum_{s,s',\hat z^2} P^{z^1, \hat z^2}_{a^1,\hat a^2}(s'|s) b(s | \theta^1_\depth, \theta^2_\depth) \right] \occ_\depth (\theta^1_\depth,\theta^2_\depth)}.
      \end{align}
      Then, in cases where %
      \begin{itemize}
      \item
        $\sum_{s,s'} P^{z^1,z^2}_{a^1,\tilde a^2}(s'|s) b(s |
        \theta^1_\depth, \theta^2_\depth) >0$ for action $a^2$, and %
      \item
        $\sum_{s,s',\hat z^2} P^{z^1,\hat z^2}_{a^1,\tilde a^2}(s'|s) b(s |
        \theta^1_\depth, \theta^2_\depth) >0$
        for some, but not all, other actions $\tilde a^2$ %
        ($=0$ typically when $z^1$ and $\tilde a^2$ are incompatible),
      \end{itemize}
      we recognize the above function
      $f(x, y_1, \dots, y_I, z_1, \dots, z_J)$, which is not
      continuous.
      % For the sake of simplicity, assume that only one history $\theta^2_\depth$ is accessible. It is possible, for a certain pair $a^2$ that $ \sum_{s,s',z^2} P^{z^1,z^2}_{a^1,a^2}(s'|s) b(s | \theta^1_\depth, \theta^2_\depth) \occ_\depth (\theta^1_\depth,\theta^2_\depth)$ is equal to $0$ so that $\occ_{\tau+1}^{c,1}(\theta^2_{\tau},a^2,z^2|\theta^1_{\tau},a^1,z^1)$ has the same form as $f$ and therefore would not be Lipschitz-Continuous w.r.t $\beta^2_\depth$.
%
      
      Such situations where $\Tc{1}$ is not continuous thus may
      indeed happen.
    \end{proof}

% ------------------

    \begin{proposition}
      \labelT{lemma|NotLipschitzOccC}
      $\Tc{i}(\occ_\depth,\vbeta_\depth)$ may be non continuous
      (thus non Lipschitz-continuous) w.r.t. $\occ^{c,1}_\depth$.
    \end{proposition}

    % \begin{comment} % [OLIVIER] Merci de laisser cette fausse preuve ici :)
    % \begin{proof}
    %   \label{proof|lemma|NotLipschitzOccC}
    %   A similar proof as for \Cref{lemma|NotLipschitzBeta}
    %   applies, the variables corresponding to parameters
    %   $\occ^{c,1}_\depth(\theta^2_\depth|\theta^1_\depth)$ under fixed
    %   $\theta^1_\depth$.
    % \end{proof}
    % \end{comment}

    A similar proof as for \Cref{lemma|NotLipschitzBeta}
    applies, the variables corresponding to parameters
    $\occ^{c,1}_\depth(\theta^2_\depth|\theta^1_\depth)$ under fixed
    $\theta^1_\depth$.

\end{tcolorbox}
%
%
%
% fin de la poubelle pour la non-lipschicité de \Tc{i}(\occ_\depth,\vbeta_\depth)
}
% ----------------------------------------------------------

\subsubsection{Linearity and Lipschitz-continuity of \texorpdfstring{$T (\occ_\depth, \beta^1_\depth, \beta^2_\depth)$}{the Transition Function} }
%\label{proofLemOccLin}

% \lemOccLin*

\begin{restatable}[]{lemma}{lemOccLin}
%[Proof in \extCshref{proofLemOccLin}]{lemma}{lemOccLin}
  \labelT{lem|occ|lin}
  % \IfAppendix{{\em (originally stated on page~\pageref{lem|occ|lin})}}{}
  % 
  At depth $\depth$, $T(\occ_\depth,\vbeta_\depth)$ is linear in
  $\beta^1_\depth$, $\beta^2_\depth$, and $\occ_\depth$, where
  $\vbeta_\depth=\langle \beta^1_\depth, \beta^2_\depth\rangle$.
  It is more precisely $1$-Lipschitz-continuous ($1$-LC) in
  $\occ_\depth$ (in $1$-norm), \ie, for any $\occ_\depth$,
  $\occ'_\depth$:
  \begin{align*}
    \norm{T(\occ'_\depth,\vbeta_\depth) - T(\occ_\depth,\vbeta_\depth)}_1
    & \leq 1 \cdot \norm{\occ'_\depth - \occ_\depth}_1.
  \end{align*}
\end{restatable}

\begin{proof}
  \label{proof|lem|occ|lin}
  Let $\occ$ be an occupancy state at time $\depth$ and
  $\vbeta_\depth$ be a decision rule.
  Then, as seen in the proof of \Cref{lem|occSufficient}, the
  next occupancy state $\occ' = T(\occ,\vbeta_\depth)$ satisfies,
  for any $s'$ and $(\vth,\va,\vz)$:
  \begin{align*}
    \occ'(\vth,\va,\vz)
    & \eqdef Pr(\vth,\va,\vz | \occ, \beta^1_\depth, \beta^2_\depth) \\
    % %& = Pr(s', \vz | \vth, \va, \occ, \beta^1_\depth, \beta^2_\depth) Pr(\vth, \va |\occ, \beta^1_\depth, \beta^2_\depth) \\
    % & = \sum_{s', s \in \cS} Pr(s', s, \vth, \va, \vz | \occ, \beta^1_\depth, \beta^2_\depth) \\
    % & = \sum_{s', s\in \cS} Pr(s', \vz | s, \vth, \va, \occ, \beta^1_\depth, \beta^2_\depth) Pr(s, \vth, \va |\occ, \beta^1_\depth, \beta^2_\depth) \\
    % & = \sum_{s', s\in \cS}
    %   Pr(s', \vz | s, \va)
    %   Pr(\va |s, \vth, \occ, \beta^1_\depth, \beta^2_\depth)
    %   Pr(s, \vth |\occ, \beta^1_\depth, \beta^2_\depth) \\
    % & = \sum_{s', s\in \cS}
    %   Pr(s', \vz | s, \va )
    %   Pr(\va | \vth, \beta^1_\depth, \beta^2_\depth)
    %   Pr(s, \vth | \occ) \\
    % & = \sum_{s', s\in \cS} \PP{s}{\va}{s'}{\vz} \vbeta_\depth(\vth,\va) b(s| \vth) Pr(\vth| \occ)  \\
    & = \beta^1_\depth(\theta^1, a^1) \beta^2_\depth(\theta^2, a^2) \left[ \sum_{s', s\in \cS} \PP{s}{\va}{s'}{\vz} b(s| \vth) \right] \occ(\vth) .
  \end{align*}
% \hrule
%   \begin{align*}
%     \tilde{\occ}(\tilde{s},(\vth,\va,\vz)) 
%     & = \sum_{s\in \cS, \vth \in \vTh} \occ_\depth(s,\vth) \vbeta_\depth(\vth,\va) \PP{s}{\va}{s'}{\vz} \\
%     \intertext{(The above expression gives us the linearity in occupancy space.)}
%     & = \sum_{s\in \cS, \vth \in \vTh} \occ_\depth(s,\vth) \beta^1_\depth(\theta^1,a^1) \beta^2_\depth(\theta^2,a^2) \PP{s}{a^1,a^2}{s'}{z^1,z^2} \\
%     & = \sum_{\vth \in \vTh} \beta^1_\depth(\theta^1,a^1) \beta^2_\depth(\theta^2,a^2) \left( \sum_{s\in \cS} \occ_\depth(s,\vth) \PP{s}{a^1,a^2}{s'}{z^1,z^2} \right).
%   \end{align*}
  $b(s|\vth)$ depending only on the model (transition function and initial belief),
  the next occupancy state $\occ'$ thus evolves linearly w.r.t.
  (i) {\em private} decision rules $\beta^1_\depth$ and $\beta^2_\depth$, and
  % for a given private history, and
  % 
  (ii) the occupancy state $\occ$.

  The $1$-Lipschitz-continuity holds because each component of vector
  $\occ_\depth$ is distributed over multiple components of $\occ'$.
  Indeed, let us view two occupancy states as vectors
  $\vx,\vy \in \reals^n$, and their corresponding next states under
  $\vbeta_\depth$ as $M \vx$ and $M \vy$, where
  $M \in \reals^{m\times n}$ is the corresponding transition matrix
  (\ie, which turns $\occ$ into
  $\occ' \eqdef T(\occ_\depth,\vbeta_\depth)$).
  Then,
  \begin{align*}
    \norm{M\vx - M\vy}_{1}
    & \eqdef \sum_{j=1}^m \ \abs{ \sum_{i=1}^n M_{i,j} (x_i - y_i) } \\
    & \leq \sum_{j=1}^m  \sum_{i=1}^n \abs{M_{i,j} (x_i - y_i) }
    & \text{(convexity of $\abs{\cdot}$)} \\
    & = \sum_{j=1}^m  \sum_{i=1}^n M_{i,j} \abs{ x_i - y_i }
    & \text{($\forall {i,j},\ M_{i,j}\geq 0$)} \\
    & = \sum_{i=1}^n \underbrace{\sum_{j=1}^m M_{i,j}}_{=1} \abs{ x_i - y_i } 
    & \text{($M$ is a transition matrix)} \\
    & \eqdef \norm{\vx-\vy}_1.
    % 
    %& \qedhere
  \end{align*}
\end{proof}

\subsubsection{Lipschitz-Continuity of \texorpdfstring{$V^*$}{V*}}
\label{app|LC|V}
%\label{proofLemVLinOcc}
%\label{proofCorVLCOcc}
%\label{proofLemNuLC}

The next two results demonstrate that, in the finite horizon setting,
$V^*$ is Lipschitz-continuous (LC) in occupancy space, which allows
defining LC upper- and lower-bound approximations.

\begin{restatable}{lemma}{lemVLinOcc}
  \labelT{app|lin|occ}
  %\IfAppendix{{\em (originally stated on
  %  page~\pageref{lem|V|lin|occ})}}{}
  % 
  At depth $\depth$,
  $V_\depth(\occ_\depth,\vbeta_{\depth:})$ is linear % (as in ``linear mapping'')
  w.r.t.  $\occ_\depth$ and $\vbeta_{\depth:}$.
\end{restatable}

Note: This result in fact applies to any reward function of a
general-sum POSG with any number of agents (here $N$), \eg, to a
Dec-POMDP.
The following proof handles the general case (with
$\vbeta_\depth \eqdef \langle \beta^1_\depth, \dots, \beta^N_\depth
\rangle$, and
$\vbeta_\depth(\va|\vth) = \prod_{i=1}^N
\beta^i_\depth(a^i,\theta^1)$).

\begin{proof}
  \label{proof|lem|V|lin|occ}
  This property trivially holds for $\depth=H-1$ because
  \begin{align*}
    V_{H-1}(\occ_{H-1},\vbeta_{H-1:}) 
    & = r(\occ_{H-1},\vbeta_{H-1}) \\
    & = \sum_{s,\va} \left( \sum_\vth Pr(s,\va|\vth) \occ_{H-1}(\vth) \right) r(s,\va) \\
    & = \sum_{s,\va} \left( \sum_\vth b(s|\vth) \vbeta_\depth(\va|\vth) \occ_{H-1}(\vth) \right) r(s,\va) \\
    & = \sum_{s,\vth}  b(s | \vth) \occ_{H-1}(\vth)  \left( \sum_{\va} \vbeta_\depth(\va|\vth) r(s,\va) \right).
  \end{align*}
  Now, let us assume that the property holds for
  $\depth+1 \in \{1 \twodots H-1\}$.
  Then,
  \begin{align*}
    V_\depth(\occ_\depth,\vbeta_{\depth:}) 
    & = \sum_{s,\va} \Big( \sum_\vth b(s| \vth) \vbeta_\depth(\va|\vth) \occ_\depth(\vth)  \Big) r(s, \va)
    + \gamma V_{\depth+1}\left( T (\occ_\depth, \vbeta_\depth), \vbeta_{\depth+1:} \right) \\
    & = \sum_{s,\vth}  b(s| \vth) \occ_\depth(\vth)  \Big( \sum_{\va} \vbeta_\depth(\va|\vth) r(s, \va) \Big)
    + \gamma V_{\depth+1}\left( T (\occ_\depth, \vbeta_\depth), \vbeta_{\depth+1:} \right) .
  \end{align*}
  As
  \begin{itemize}
  \item $T(\occ_\depth,\vbeta_\depth)$ is linear in $\occ_\depth$
    {\footnotesize (\Cref{lem|occ|lin})} %(\Cref{proof|lem|occ|lin})}
    and
  \item $V_{\depth+1}(\occ_{\depth+1}, \vbeta_{\depth+1:})$ is
    linear in $\occ_{\depth+1}$ {\footnotesize (induction hypothesis)},
  \end{itemize}
  their composition,
  $V_{\depth+1} ( T (\occ_\depth,\vbeta_\depth),
  \vbeta_{\depth+1:} )$,
  is also linear in $\occ_\depth$, and so is
  $V_\depth(\occ_\depth,\vbeta_{\depth:})$. Similarly, $V_\depth(\occ_\depth,\vbeta_\depth)$ is linear in $\vbeta_\depth$ for any $\occ_\depth$.
\end{proof}

\begin{comment}
\begin{restatable}{theorem}{corVLCOcc}
  \labelT{app|V|LC|occ}
  %\IfAppendix{{\em (originally stated on
  %    page~\pageref{cor|V|LC|occ})}}{}
  % 
  Let
  $\h{H}{\depth}{\gamma} \eqdef \frac{1-\gamma^{H-\depth}}{1-\gamma}$
  % if $\gamma<1$, else $\h{H}{\depth}{\gamma} \eqdef H-\depth$ (if $\gamma=1$). %
  (or $\h{H}{\depth}{\gamma} \eqdef H-\depth$ if $\gamma=1$).
  % 
  Then $V^*_\depth(\occ_\depth)$ is $\lt{\depth}$-Lipschitz continuous in
  $\occ_\depth$ at any depth $\depth \in \{0 \twodots H-1\}$, where
  $\lt{\depth} = \frac{1}{2} \h{H}{\depth}{\gamma} %
  \left( r_{\max} - r_{\min} \right)$. 
 % \left[ \max_{s,\va}r(s,\va)
 %    - \min_{s,\va} r(s,a) \right]$.
\end{restatable}
\end{comment}

\corVLCOcc*

\begin{proof}
  \label{proof|cor|V|LC|occ}
  At depth $\depth$, the value of any behavioral strategy
  $\vbeta_{\depth:}$ is bounded, independently of $\occ_\depth$, by
  \begin{align*}
    V^{\max}_\depth & \eqdef \h{H}{\depth}{\gamma} r_{\max}, \quad \text{where } r_{\max} \eqdef \max_{s,\va}r(s,\va),
    \text{ and } \\
    V^{\min}_\depth &\eqdef \h{H}{\depth}{\gamma}  r_{\min}, \quad \text{where } r_{\min} \eqdef \min_{s,\va}r(s,\va).
  \end{align*}
  Thus, $V_{\vbeta_{\depth:}}$ being a linear function defined over a
  probability simplex ($\Occ_\depth$) (\cf \Cref{proof|lem|V|lin|occ}) and
  bounded by $[V^{\min}_\depth,V^{\max}_\depth]$, we can apply
  \citeauthor{Horak-phd19}'s PhD thesis' Lemma~3.5 (p.~33) \citeyear{Horak-phd19} to establish that it is also
  $\lt{\depth}$-LC, \ie,
  \begin{align*}
    \abs{V_{\vbeta_{\depth:}}(\occ) -V_{\vbeta_{\depth:}}(\occ')}
    & \leq \lt{\depth} \norm{\occ-\occ'}_{\p} \quad (\forall \occ, \occ'), \\
    \text{with }
    \lt{\depth}
    & = \frac{V^{\max}_{\depth}-V^{\min}_{\depth}}{2}.
  \end{align*}

  Considering now optimal solutions, this means that, at depth
  $\depth$ and for any $(\occ,\occ') \in \Occ_\depth$:
  \begin{align*}
    & V^*_\depth(\occ) - V^*_\depth(\occ')  \\ 
    & = \max_{\beta^1_{\depth:}} \min_{\beta^2_{\depth:}} V_\depth(\occ, \beta^1_{\depth:}, \beta^2_{\depth:}) 
    - \max_{\beta'^1_{\depth:}} \min_{\beta'^2_{\depth:}} V_\depth(\occ', \beta'^1_{\depth:}, \beta'^2_{\depth:}) \\
    % & \leq \max_{\beta^1_{\depth:}} \min_{\beta^2_{\depth:}} V_\depth(\occ, \beta^1_{\depth:}, \beta^2_{\depth:}) \\
    % &  - \max_{\beta'^1_{\depth:}} \min_{\beta'^2_{\depth:}} \left[ V_\depth(\occ, \beta'^1_{\depth:}, \beta'^2_{\depth:}) - \lt{\depth} \norm{\occ-\occ'}_{\p} \right]\\
    & \leq \max_{\beta^1_{\depth:}} \min_{\beta^2_{\depth:}} \left[ V_\depth(\occ', \beta^1_{\depth:}, \beta^2_{\depth:}) + \lt{\depth} \norm{\occ-\occ'}_{\p} \right] 
    - \max_{\beta'^1_{\depth:}} \min_{\beta'^2_{\depth:}} V_\depth(\occ', \beta'^1_{\depth:}, \beta'^2_{\depth:}) \\
    % & = \max_{\beta^1_{\depth:}} \min_{\beta^2_{\depth:}} V_\depth(\occ, \beta^1_{\depth:}, \beta^2_{\depth:}) \\
    % & - \max_{\beta'^1_{\depth:}} \min_{\beta'^2_{\depth:}} \left[ V_\depth(\occ, \beta'^1_{\depth:}, \beta'^2_{\depth:}) \right] + \lt{\depth} \norm{\occ-\occ'}_{\p} \\
    & = \lt{\depth} \norm{\occ-\occ'}_{\p}.
  \end{align*}
  Symmetrically,
  $V^*_\depth(\occ) - V^*_\depth(\occ') \geq -\lt{\depth}
  \norm{\occ-\occ'}_{\p}$, hence the expected result:
  \begin{align*}
    \abs{ V^*_\depth(\occ) - V^*_\depth(\occ') } 
    & \leq \lt{\depth} \norm{\occ-\occ'}_{\p}.
    % 
    %\qedhere
  \end{align*}
\end{proof}

As it will be used later, let us also present the following lemma.

\begin{lemma}
  \labelT{lem|nuLC}
  Let us consider $\depth \in \{0 \twodots H-1\}$, $\theta^1_\depth$, and
  $\tree^2_\depth$.
  Then
  $\nu^2_{[\occ^{c,1}_{\depth},
    \tree^2_{\depth}]}(\theta^1_{\depth})$ is $\l_\depth$-LC in
  $\occ^{c,1}_\depth(\cdot|\theta^1_\depth)$.
      
  Equivalently, we will also write that
  $\nu^2_{[\occ^{c,1}_{\depth}, \tree^2_{\depth}]}$ is $\l_\depth$-LC
  in $\occ^{c,1}_\depth$ in {\em vector-wise} 1-norm, \ie:
  \begin{align*}
    \vabs{
      \nu^2_{[\occ^{c,1}_{\depth}, \tree^2_{\depth}]} - \nu^2_{[\tilde\occ^{c,1}_{\depth}, \tree^2_{\depth}]}
    }_1
    & \vleq \lt\depth \vnorm{
      \occ^{c,1}_{\depth} - \tilde\occ^{c,1}_{\depth}
    }_1,
  \end{align*}
  where %
  (i) the absolute value of a vector is obtained by taking the
  absolute value of each component; and %
  (ii) the vector-wise 1-norm of a matrix is a vector made of the
  1-norm of each of its component vectors.
\end{lemma}

% \olivier{Attention! Ce lemme repose sur un terme conditionnel
%   $\occ^{c,1}_{\depth}$ complet, alors qu'on n'en récupère qu'un
%   incomplet en partant de $\occ_\depth$ (ne tenant pas compte des
%   historiques $\theta^1_\depth$ hors du support).}

\begin{proof}
  \label{proof|lem|nuLC}
  %Any fixed strategy $\beta^2_{0:}$ induces a POMDP for $1$,
  %
  For any $\theta^1_\depth$, $\occ^{c,1}_{\depth}$ and $\tree^2_{\depth}$ induce a POMDP for Player $1$
  from $\depth$ on,
  where %
  (i) the state at any $t\in \{\depth \twodots H-1\}$ corresponds to a pair
  $\langle s, \theta^2_t \rangle$, and %
  (ii) the initial belief is derived from $\occ^{c,1}_{\depth}(\cdot|\theta^1_t)$.
  The belief state at $t$ thus gives:
  \begin{align*}
    b_{\theta^1_t}(s, \theta^2_t)
    & \eqdef Pr(s, \theta^2_t | \theta^1_t) % \\ &
    = \underbrace{Pr(s | \theta^2_t, \theta^1_t)}_{b^{\hmm}_{\theta^2_t, \theta^1_t}(s)}
    \cdot \underbrace{Pr(\theta^2_t | \theta^1_t)}_{\occ^{c,1}_t(\theta^2_t|\theta^1_t)}.
  \end{align*}
  So,
  \begin{itemize}
  \item the value function of any behavioral strategy
    $\beta^1_{\depth:}$ is linear at $t$ in $b_{\theta^1_t}$, thus (in
    particular) in $\occ^{c,1}_t(\cdot|\theta^1_t)$; and
  \item the optimal value function is LC at $t$ also in
    $b_{\theta^1_t}$ (with the same depth-dependent upper-bounding
    Lipschitz constant $\l_t$ as in the proof of \mbox{\Cref{app|V|LC|occ}}),\footnote{The
      proof process is similar. The only difference lies in the space
      at hand, but without any impact on the resulting formulas.} %
    thus (in particular) in $\occ^{c,1}_t(\cdot|\theta^1_t)$.
  \end{itemize}

  Using $t=\depth$, the optimal value function is
  $\nu^2_{[\occ^{c,1}_{\depth},
    \tree^2_{\depth}]}(\theta^1_{\depth})$, which is thus
  $\l_\depth$-LC in $\occ^{c,1}_{\depth}(\cdot|\theta^1_\depth)$.
\end{proof}

% =============================

\subsection{Bounding Approximations of \texorpdfstring{$V^*$}{V*}, \texorpdfstring{$W^{1,*}$}{W1*} and \texorpdfstring{$W^{2,*}$}{W2*}} % Deriving Approximations
\label{app|derivingApproximations}

\subsubsection{\texorpdfstring{$\upb{V}_\depth$}{upbV} and \texorpdfstring{$\lob{V}_\depth$}{lobV}}
\label{app|derivingApproximationsV}

To find a form that could be appropriate for an upper bound approximation of $V^*_\depth$,
let us consider an \os{} $\occ_\depth$ and a single tuple $\langle { %
  \tilde\occ_\depth, %
  \nu^2_{[\tilde\occ_\depth^{c,1}, \beta^2_{\depth:}]} %
} \rangle$,
and define
$\zeta_\depth \eqdef \occ_\depth^{m,1}\tilde\occ_\depth^{c,1}$.
Then,
\begin{align*}
  V^*(\occ_\depth)
  & \leq V^*(\zeta_\depth) + \lt{\depth}  \norm{ \occ_\depth - \zeta_\depth }_1 
  & \text{(LC, \cf \Cref{app|V|LC|occ})} \\
  & = V^*(\occ_\depth^{m,1} \tilde\occ_\depth^{c,1})
  + \lt{\depth}  \norm{ \occ_\depth - \zeta_\depth }_1 \\
  & \label{eq|upbV}
  \leq \occ_\depth^{m,1} \cdot \nu^2_{[\tilde\occ_\depth^{c,1}, \beta^2_{\depth:}]}
  + \lt{\depth}  \norm{ \occ_\depth - \occ_\depth^{m,1}\tilde\occ_\depth^{c,1} }_1. 
  & \text{(Cvx, \cf \Cref{theo|ConvexConcaveV})} %\\
\end{align*}
Notes:
\begin{itemize}
\item $\tilde\occ^{m,1}_\depth$ does not appear in the resulting
  upper bound, thus will not need to be specified.
\item For $\depth=H-1$,
  $\nu^2_{[\tilde\occ_\depth^{c,1}, \beta_{\depth:}^2]}$ is a simple
  function of $r$, $\tilde\occ_\depth^{c,1}$, $\beta_{\depth:}^2$, and
  the dynamics of the system, as described in Eq.~(9) of
  \citet{WigOliRoi-corr16}.
\end{itemize}

From this, we can deduce the following appropriate forms of upper and
(symmetrically) lower bound function approximations for $V^*_\depth$:
\begin{align*}
  \upb{V}_\depth(\occ_\depth)
  & = \min_{   \langle \tilde\occ_\depth^{c,1}, \langle \upb\nu^2_\depth, \tree_{\depth:}^2 \rangle  \rangle  \in \upb{\cJ}_\depth } %
  \left[ \occ_\depth^{m,1} \cdot \upb\nu^2_\depth + \lt{\depth} \norm{ \occ_\depth - \occ_\depth^{m,1}\tilde\occ_\depth^{c,1} }_1 \right], \text{ and} \\
  % \intertext{\olivier{ou, mieux, pour aller dans le sens de Jilles:}}
  % & = \tcg{\min\{} {
  %   \tcg{V^{\max}_\depth ,}
  %   \min_{   \langle \tilde\occ_\depth^{c,1}, \upb\nu^2_\depth  \rangle  \in \upb{bagV} } %
  %   \left[ \occ_\depth^{m,1} \cdot \upb\nu^2_\depth + \lt{\depth} \norm{ \occ_\depth - \occ_\depth^{m,1}\tilde\occ_\depth^{c,1} }_1 \right]
  % } \tcg{\}}, \text{ and} \\
  \lob{V}_\depth(\occ_\depth) 
  & = \max_{ \langle \tilde\occ_\depth^{c,2}, \langle \lob\nu^1_\depth, \tree_{\depth:}^1 \rangle \rangle \in \lob{\cJ}_\depth } %
  \left[ \occ_\depth^{m,2} \cdot \lob\nu^1_\depth - \lt{\depth} \norm{ \occ_\depth - \occ_\depth^{m,2}\tilde\occ_\depth^{c,2} }_1 \right],
\end{align*}
% \aurelien{$- \lt{\depth} \norm{ \occ_\depth - \occ_\depth^{m,2}\tilde\occ_\depth^{c,2} }_1$ plutôt?}
% \olivier{Tout à fait. Merci bcp.}
which are respectively concave in $\occ^{m,1}_\depth$ and convex in
$\occ^{m,2}_\depth$, and which both exploit the Lipschitz continuity.

\subsubsection{\texorpdfstring{$\upbW{\depth}{}$}{upbW} and \texorpdfstring{$\lobW{\depth}{}$}{lobW}}
\label{app|derivingApproximationsW}
\label{lemma|OptimalQValues}
%\label{approxEpsilonClosely}

Note: We discuss all depths from $0$ to $H-1$, even though we do not
need these approximations at $\depth=H-1$.

%These results first lead to the following property of $W_\depth^{*,1}$.
%
Let us first see how concavity-convexity properties affect $W_\depth^{*,1}$.

\begin{lemma}
  Considering that vectors $\nu^2_{[\occ_{H}^{c,1},\beta_{H:}^2]}$
  are null vectors, we have, for all $\depth \in \{0\twodots H-1\}$:
  \begin{align*}
    W_\depth^{1,*}(\occ_\depth,\beta^1_\depth) % 
    & = \min_{\beta^2_\depth, \langle \beta_{\depth+1:}^2, \nu^2_{{[\Tc{1}(\occ_\depth, \beta^2_\depth), \beta^2_{\depth+1:}]}} \rangle} \beta^1_\depth \cdot \Big[ {
      \vr(\occ_\depth, \cdot, \beta^2_\depth) } \\
    & \qquad  + {\gamma \Tm{1}(\occ_\depth, \cdot, \beta^2_\depth) \cdot \nu^2_{{[\Tc{1}(\occ_\depth, \beta^2_\depth),\beta^2_{\depth+1:}]}}
    } \Big].
  \end{align*}
\end{lemma}

\begin{proof}
  \label{proof|lemma|OptimalQValues}
  Considering that vectors $\nu^2_{[\occ_{H}^{c,1},\beta^2_{H:}]}$
  are null vectors, we have, for all $\depth \in \{0 \twodots H-1\}$:
  \begin{align*}
     W_\depth^{1,*}(\occ_\depth,\beta^1_\depth) % 
    & = \min_{\beta^2_\depth}  Q^*_\depth(\occ_\depth, \beta^1_\depth, \beta^2_\depth) %
     = \min_{\beta^2_\depth}  \left[
      r(\occ_\depth, \vbeta_\depth)
      + \gamma V^*_{\depth+1}( T(\occ_\depth, \vbeta_\depth) )
    \right] %\\
    \intertext{(Line below exploits \Cref{theo|ConvexConcaveV} (p.~\pageref{theo|ConvexConcaveV}) %
      and $\Tc{1}$'s independence from $\beta^1_\depth$ (\Cref{lem|T1cindep}).)}
    & \hspace{-1.25cm} = \min_{\beta^2_\depth}  \left[
      r(\occ_\depth, \vbeta_\depth)
      + \gamma \min_{\langle \beta_{\depth+1:}^2, \nu^2_{{[\Tc{1}(\occ_\depth, \beta^2_\depth), \beta^2_{\depth+1:}]}} \rangle}  \left[
        \Tm{1}(\occ_\depth, \vbeta_\depth) \cdot \nu^2_{{[\Tc{1}(\occ_\depth, \beta^2_\depth), \beta^2_{\depth+1:}]}}
      \right]
    \right] \\
    & \hspace{-1.25cm} = \min_{\beta^2_\depth, \langle \beta_{\depth+1:}^2, \nu^2_{{[\Tc{1}(\occ_\depth, \beta^2_\depth), \beta^2_{\depth+1:}]}} \rangle} \left[
      r(\occ_\depth, \vbeta_\depth)
      + \gamma \Tm{1}(\occ_\depth, \vbeta_\depth) \cdot \nu^2_{{[\Tc{1}(\occ_\depth, \beta^2_\depth), \beta^2_{\depth+1:}]}}
    \right] %\\
    \intertext{(Line below exploits $r$ and $\Tm{1}$'s linearity in $\beta^1_\depth$ (\Cref{lem|T1mlin}).)}
    & = \min_{\beta^2_\depth, \langle \beta_{\depth+1:}^2, \nu^2_{{[\Tc{1}(\occ_\depth, \beta^2_\depth), \beta^2_{\depth+1:}]}} \rangle} {\beta^1_\depth}^{\top} \! \cdot \Big[ {
      \vr(\occ_\depth, \cdot, \beta^2_\depth) }
     \\ & \qquad + \gamma \Tm{1}(\occ_\depth, \cdot, \beta^2_\depth) \cdot \nu^2_{{[\Tc{1}(\occ_\depth, \beta^2_\depth),\beta^2_{\depth+1:}]}}
     \Big].
    % 
    %\qedhere
  \end{align*}
\end{proof}

Note that, since $V^*_H=0$, $\depth=H-1$ is a particular case which can be simply re-written:
\begin{align*}
   W_\depth^{1,*}(\occ_\depth,\beta^1_\depth) % 
   & = \min_{\beta^2_\depth} {\beta^1_\depth}^{\top} \! \cdot
   \vr(\occ_\depth, \cdot, \beta^2_\depth).
 \end{align*}

\poubelle{
\begin{tcolorbox}[breakable, enhanced]

  \uline{Addendum:} The following complementary property is not
  directly used in the present work, but makes for a more complete
  table of properties (\Cref{tab|PropertyTable}).

    \begin{proposition}
      \labelT{lemma|Wconcave}
      $W^{1,*}_\depth(\occ_\depth,\beta^1_\depth)$ is concave in $\beta^1_\depth$.
    \end{proposition}

    \begin{proof}
      \label{proof|lemma|Wconcave}
      Let %
      $X$ and $Y$ be two convex domains, % 
      $f: X\times Y \mapsto \reals$ be a concave-convex function, and % 
      $g(x)\eqdef \min_{y\in Y} f(x,y)$.
      Then, for any $x_1, x_2 \in X$, and any $\alpha \in [0,1]$,
      \begin{align}
        g(\alpha x_1 + (1-\alpha) x_2)
        & = min_y \underbrace{ f( \alpha x_1 + (1-\alpha) x_2, y ) }_{
          \geq \alpha f( x_1, y ) + (1-\alpha) f( x_2, y )
        }
        & \text{(concavity in $x$)}
        \\
        & \geq min_y \left[ \alpha f( x_1, y ) + (1-\alpha) f( x_2, y ) \right]
        \\
        & \geq min_y \alpha f( x_1, y ) +  min_y (1-\alpha) f( x_2, y )
        \\
        & = \alpha g( x_1 ) +  (1-\alpha) g( x_2 ).
        & \text{($\alpha\geq 0$)}
      \end{align}
      $g$ is thus concave in $x$.

      This result directly applies to the function at hand, proving its concavity in $\beta^1_\depth$.
    \end{proof}
\end{tcolorbox}
%
%
%
% fin de poubelle pour la concavity de $W_\depth^{1,*}(\occ_\depth,\beta_\depth^1) en \beta_\depth^1$.
}

To find a form that could be appropriate for an upper bound approximation of $W^{*,1}_\depth$,
let us now consider an \os{} $\occ_\depth$ and a single tuple %
$\langle { %
  \tilde\occ_\depth, %
  \tilde\beta^2_\depth, %
  \nu^2_{[\Tc{1}(\tilde\occ_\depth, \tilde\beta^2_\depth), %
    \tilde\beta_{\depth+1:}^2]} %
} \rangle$.
Then,
\begin{align}
  W_\depth^{1,*}(\occ_\depth,\beta^1_\depth) % 
  % & = \min_{\beta^2_\depth}  Q^*_\depth(\occ_\depth, \beta^1_\depth, \beta^2_\depth) \\
  & = \min_{\beta^2_\depth}  \left[
    r(\occ_\depth, \vbeta_\depth)
    + \gamma V^*_{\depth+1}( T(\occ_\depth, \vbeta_\depth) )
  \right]
  \nonumber 
  \\
  \nonumber
  & \leq   
  r(\occ_\depth, \beta^1_\depth, \tilde\beta^2_\depth)
  + \gamma  V^{BR,1}_{\depth+1}( T(\occ_\depth, \beta^1_\depth, \tilde\beta^2_\depth) | \tilde\beta_{\depth+1:}^2)
  \nonumber
  \intertext{\text{\small (Use 
    $\tilde\beta^2_\depth$ \& $\tilde\beta^2_{\depth+1:}$ instead of mins)}} %\\
  \intertext{\small (where $V^{BR,1}_{\depth+1}( T(\occ_\depth, \beta^1_\depth, \tilde\beta^2_\depth) | \tilde\beta_{\depth+1:}^2)$ is the value of 1's best response to $\tilde\beta^2_{\depth+1:}$ if in $T(\occ_\depth, \beta^1_\depth, \tilde\beta^2_\depth)$)}
  \nonumber
  & = r(\occ_\depth, \beta^1_\depth, \tilde\beta^2_\depth)
  + \gamma 
  \Tm{1}(\occ_\depth, \beta^1_\depth, \tilde\beta^2_\depth) \cdot \underbrace{ \nu^2_{[\Tc{1}(\occ^{c,1}_\depth, \tilde\beta^2_\depth), \tilde\beta_{\depth+1:}^2]} } %_{\text{($\lt{\depth+1}$-LC in $\occ^{c,1}_{\depth+1}(\cdot|\theta^1_\depth)$)}}
  \nonumber
   \intertext{\small (Lem.~3 of \citet{WigOliRoi-corr16})}
   \nonumber
   \\
  & \leq r(\occ_\depth, \beta^1_\depth, \tilde\beta^2_\depth)
  + \gamma 
  \Tm{1}(\occ_\depth, \beta^1_\depth, \tilde\beta^2_\depth) \cdot \Big( {
    \nu^2_{[\Tc{1}(\tilde\occ^{c,1}_\depth, \tilde\beta^2_\depth), \tilde\beta_{\depth+1:}^2]} }
  \\
  \nonumber
  & \qquad \qquad {
    + \lt{\depth+1} \cdot \vnorm{ \Tc{1}(\occ^{c,1}_\depth, \tilde\beta^2_\depth) - \Tc{1}(\tilde\occ^{c,1}_\depth, \tilde\beta^2_\depth) }_1
  } \Big)
  %\qquad \qquad \text{\small  of $\nu^2_{[\Tc{1}(\occ^{c,1}_\depth, \tilde\beta^2_\depth), \tilde\beta_{\depth+1:}^2]}$)}
  \nonumber 
  \\
  \intertext{\small (\Cshref{lem|nuLC}: $\lt{\depth+1}$-LC of $\nu^2_{[\Tc{1}(\occ^{c,1}_\depth, \tilde\beta^2_\depth), \tilde\beta_{\depth+1:}^2]}$)}
  & \label{eq|WsingleTerm}
  = {\beta^1_\depth}^{\top} \! \cdot \Big[
  r(\occ_\depth, \cdot, \tilde\beta^2_\depth)
  + \gamma 
  \Tm{1}(\occ_\depth, \cdot, \tilde\beta^2_\depth) \cdot \Big( {
    \nu^2_{[\Tc{1}(\tilde\occ^{c,1}_\depth, \tilde\beta^2_\depth), \tilde\beta_{\depth+1:}^2]} }
  \\
  \nonumber
  & \qquad \qquad {
    + \lt{\depth+1} \cdot \vnorm{ \Tc{1}(\occ^{c,1}_\depth, \tilde\beta^2_\depth) - \Tc{1}(\tilde\occ^{c,1}_\depth, \tilde\beta^2_\depth) }_1
  } \Big)
  \Big]
  \nonumber
    \intertext{\small (Linearity in $\beta^1_\depth$)}
    \nonumber
\\
\nonumber
  & = {\beta^1_\depth}^{\top} \! \cdot \Big[
  r(\occ_\depth, \cdot, \tilde\beta^2_\depth)
  + \gamma 
  \Tm{1}(\occ_\depth, \cdot, \tilde\beta^2_\depth) \cdot {
    \nu^2_{[\Tc{1}(\tilde\occ^{c,1}_\depth, \tilde\beta^2_\depth), \tilde\beta_{\depth+1:}^2]} }
  \\
  \nonumber
  & \qquad \qquad {
    + \gamma \lt{\depth+1} \cdot \norm{ T(\occ_\depth, \cdot, \tilde\beta^2_\depth) - 
      \Tm{1}(\occ_\depth, \cdot, \tilde\beta^2_\depth) \Tc{1}(\tilde\occ^{c,1}_\depth, \tilde\beta^2_\depth) }_1
  }
  \Big]
  \nonumber
    \intertext{\small (Alternative writing)}
    \nonumber
\end{align}

From this, we can deduce the following appropriate forms of %
(i) upper bounding approximation for $W^{1,*}_\depth$ and %
(ii) (symmetrically) of lower bound approximation for
$W^{2,*}_\depth$:
\begin{align*}
  \upbW{\depth}{}(\occ_\depth, \beta^1_\depth) 
  & = \min_{ %\substack{
      \langle \tilde\occ^{c,1}_\depth, \beta^2_\depth, \upb\nu^2_{\depth+1} \rangle %\\
      \in \upb{\mathcal{I}}_\depth
    }% }
    {\beta^1_\depth}^{\top} \cdot \Big[ {
      \vr(\occ_\depth, \cdot, \beta^2_\depth)
      + \gamma \Tm{1}(\occ_\depth, \cdot, \beta^2_\depth) \cdot \upb\nu^2_{\depth+1}
    } \\
    & \qquad 
    + \gamma \lt{\depth+1} \cdot \norm{ T(\occ_\depth, \cdot, \beta^2_\depth) - \Tm{1}(\occ_\depth, \cdot, \beta^2_\depth) \Tc{1}(\tilde\occ^{c,1}_\depth, \beta^2_\depth) }_1
    \Big],
    \text{ and} \\
    % \intertext{\olivier{ou, mieux, pour aller dans le sens de Jilles:}}   
    % & = \tcg{\min\{
    %   V^{\max}_\depth,}
    %   \min_{ %\substack{
    %     \langle \tilde\occ^{c,1}_\depth, \beta^2_\depth, \upb\nu^2_{\depth+1} \rangle %\\
    %     \in \upb{\mathcal{I}}^1_\depth
    %   }% }
    %   {\beta^1_\depth}^{\top} \cdot \Big[ {
    %     \vr(\occ_\depth, \cdot, \beta^2_\depth)
    %     + \gamma \Tm{1}(\occ_\depth, \cdot, \beta^2_\depth) \cdot \upb\nu^2_{\depth+1}
    %   } \\
    %   & \qquad \qquad
    %   + \gamma \lt{\depth+1} \cdot \norm{ T(\occ_\depth, \cdot, \beta^2_\depth) - \Tm{1}(\occ_\depth, \cdot, \beta^2_\depth) \Tc{1}(\tilde\occ^{c,1}_\depth, \beta^2_\depth) }_1
    % \Big] \tcg{\}},
    % \text{ and} \\
  \lobW{\depth}{}(\occ_\depth, \beta^2_\depth) 
  & = \max_{ %\substack{
      \langle \tilde\occ^{c,2}_\depth, \beta^1_\depth, \lob\nu^1_{\depth+1} \rangle %\\
      \in \lob{\mathcal{I}}_\depth
    }% }
    % \underbrace{
    {\beta^2_\depth}^{\top} \cdot \Big[ {
      \vr(\occ_\depth, \beta^1_\depth, \cdot)
      + \gamma \Tm{2}(\occ_\depth, \beta^1_\depth, \cdot) \cdot \lob\nu^1_{\depth+1}
    } \\
    & \qquad \qquad
    - \gamma \lt{\depth+1} \cdot \norm{ T(\occ_\depth, \beta^1_\depth, \cdot) - \Tm{2}(\occ_\depth, \beta^1_\depth, \cdot) \Tc{2}(\tilde\occ^{c,2}_\depth, \beta^1_\depth) }_1
    \Big],
\end{align*}
where $\upb\nu^2_{\depth+1}$ and $\lob\nu^1_{\depth+1}$ respectively upper and lower bound the actual vectors associated to the players' future strategies (resp. of $2$ and $1$).

Again, $\depth=H-1$ is a particular case where only the reward term is preserved.

This constitutes the proof to the following proposition

\coreThUpperBounds*

\poubelle{
\thInclusionWandV*

\begin{proof}
\labelT{app|thinclusionWandV}
First, let us notice that, at $\depth=0$, the occupancy-state space is reduced to a singleton, $\{ \occ_0= \langle 1 \rangle \}$, because of the single (empty) joint \aoh{}. The security-level vectors $\nu$ are thus one-dimensional, and here considered as scalar numbers.

Let us assume that sets $\upb{\mathcal{J}}_{0}$ and $\lob{\mathcal{J}}_{0}$ are such that
\begin{align*}
    \upb{V}_0(\occ_0) - \lob{V}_0(\occ_0)
    & \leq \epsilon,
\end{align*}
and let
$\lob{w}^* = \langle \occ_0^{c,2}, \langle \lob{\nu}_0^*, \tree_0^{1,*} \rangle \rangle$ and
$\upb{w}^* = \langle \occ_0^{c,1}, \langle \upb{\nu}_0^*, \tree_0^{2,*} \rangle \rangle$ be the tuples returned by $\argmax_{\lob{w}\in \lob{\cJ}_0} \lob{\nu}_0^2$ and $\argmin_{\upb{w}\in \upb{\cJ}_0} \upb{\nu}_0^1$.
Then, noting that $\occ_0= \langle 1 \rangle$, 
\begin{align*}
    \nu^1_{[\occ_0^{c,2},\tree_0^{1,*}]} - \nu^2_{[\occ_0^{c,1},\tree_0^{2,*})} 
    & \leq \upb{\nu}_0^* - \lob{\nu}_0^*
    \\
    & = \max_{\lob{w}\in \lob{\cJ}_0} \lob{\nu}_0^2
    - \min_{\upb{w}\in \upb{\cJ}_0} \upb{\nu}_0^1
    \\
    & = 
    \upb{V}_0(\occ_0) - \lob{V}_0(\occ_0)
    \\
    & \leq \epsilon.
\end{align*}
Thus, $\tree_{0}^1$ and $\tree_{0}^2$ are two strategies whose security levels are $\epsilon$-close, and thus form an $\epsilon$-\nes{} of the zs-POSG.
\end{proof}
}

\subsection{Related Operators}

\subsubsection{Selection Operator: Solving for \texorpdfstring{$\beta^1_\depth$}{beta1t} as an LP} %, and computing $\nu$ through its dual}
\label{sec|getLP}

\begin{proposition}
  \labelT{prop|bilinearValue}
  Using now a distribution $\tree^2_\depth$ over tuples 
  $w = \langle \tilde\occ^{c,1}_\depth, \beta^2_\depth, \upb\nu^2_{\depth+1} \rangle \in
  \upb{\mathcal{I}}^1_\depth$,
  the corresponding upper-bounding value for ``profile''
  $\langle \beta^1_\depth, \tree^2_\depth \rangle$ when in
  $\occ_\depth$ can be written as an expectancy:
  \begin{align*}
    {\beta^1_\depth}^{\top} \cdot M^{\occ_\depth} \cdot \tree^2_\depth,
  \end{align*}
  where $M^{\occ_\depth}$ is an $|\Theta^1_\depth \times \cA^1| \times |\upb{\mathcal{I}}^1_\depth|$ matrix.
\end{proposition}

\begin{proof}
  \label{proof|prop|bilinearValue}
  From the right-hand side term in (\Cref{eq|WsingleTerm}), the
  upper-bounding value associated to $\occ_\depth$, $\beta^1_\depth$ and
  a tuple
  $\langle \tilde\occ^{c,1}_\depth, \beta^2_\depth, \upb\nu^2_{\depth+1} \rangle \in
  \upb{\mathcal{I}}^1_\depth$ can be written:
  \begin{align*}
    & {\beta^1_\depth}^{\top} \! \cdot \Big[
  r(\occ_\depth, \cdot, \beta^2_\depth)
  + \gamma 
  \Tm{1}(\occ_\depth, \cdot, \beta^2_\depth) \cdot \Big( \upb\nu^2_{\depth+1}
%  \\
%  & \qquad \qquad
  {
    + \lt{\depth+1} \cdot \vnorm{ \Tc{1}(\occ^{c,1}_\depth, \tilde\beta^2_\depth) - \Tc{1}(\tilde\occ^{c,1}_\depth, \tilde\beta^2_\depth) }_1
  } \Big)
  \Big].
  \end{align*}
  Using now a distribution $\tree^2_\depth$ over tuples 
  $w = \langle \tilde\occ^{c,1}_\depth, \beta^2_\depth, \upb\nu^2_{\depth+1} \rangle \in
  \upb{\mathcal{I}}^1_\depth$,
  the corresponding upper-bounding value for ``profile''
  $\langle \beta^1_\depth, \tree^2_\depth \rangle$ when in
  $\occ_\depth$ can be written as an expectancy:
  {\scalefont{.92}
  \begin{align*}
    & \sum_{w \in \upbW{\depth}{}}
    {\beta^1_\depth}^{\top} \cdot  \Big[
    \vr(\occ_\depth, \cdot, \beta^2_\depth[w])
    + \gamma \Tm{1}(\occ_\depth, \cdot, \beta^2_\depth[w]) \cdot \Big( \upb\nu^2_{\depth+1}[w]
    % \\
    % & \qquad \qquad 
    \\ & \qquad + \lt{\depth+1} 
    \cdot \vnorm{ \Tc{1}(\occ^{c,1}_\depth, \beta^2_\depth[w]) - \Tc{1}(\tilde\occ^{c,1}_\depth[w], \beta^2_\depth[w]) }_1
    \Big)
    \Big] \cdot \tree^2_\depth(w)
    \intertext{(where $x[w]$ denotes the field $x$ of tuple $w$)}
    & = {\beta^1_\depth}^{\top} \cdot M^{\occ_\depth} \cdot \tree^2_\depth,
  \end{align*}
  }
  where $M^{\occ_\depth}$ is an $|\Theta^1_\depth \times \cA^1| \times |\upb{\mathcal{I}}^1_\depth|$ matrix.
\end{proof}

For implementation purposes, using \Cshref{eq|reward,eq|occm1} (to develop respectively $r(\cdot,\cdot,\cdot)$ and
$\Tm{1}(\cdot,\cdot,\cdot)$), we can derive the expression of a
component, \ie, the upper-bounding value if $a^1$ is applied in
$\theta^1_\depth$ while $w$ is chosen:
\begin{align*}
  M^{\occ_\depth}_{(\langle \theta^1_\depth, a^1\rangle, w)}  %
  & \eqdef 
  \vr(\occ_\depth, \cdot, \beta^2_\depth[w])
  + \gamma \Tm{1}(\occ_\depth, \cdot, \beta^2_\depth[w]) \cdot \\
  & \qquad \Big( \upb\nu^2_{\depth+1}[w] %\\ %
  %& \qquad
  +  \lt{\depth+1} \cdot
  \vnorm{  \Tc{1}(\occ^{c,1}_\depth, \beta^2_\depth[w]) - \Tc{1}(\tilde\occ^{c,1}_\depth[w], \beta^2_\depth[w]) }_1
  \Big) \\
  & =
  { 
    \sum_{s,\theta^2_\depth, a^2} \occ_\depth(\vth_\depth) b(s | \vth_\depth)
    \beta^2_\depth[w](a^2|\theta^2) r(s,\va)
  } \\
  & \qquad + \gamma \sum_{z^1} {
    \left[
      \sum_{\theta^2_\depth,a^2} \beta^2_\depth[w](a^2 | \theta^2_\depth)  \sum_{s,s',z^2} P^{\vz}_{\va}(s'|s) 
      b(s | \vth_\depth) \occ_\depth (\vth_\depth)
    \right]
  } \cdot \\   %
  & \qquad \Big( \upb\nu^2_{\depth+1}[w](\theta^1_\depth, a^1, z^1)
  \\
  & \qquad +
  \lt{\depth+1} \cdot
  \vnorm{  \Tc{1}(\occ^{c,1}_\depth, \beta^2_\depth[w]) - \Tc{1}(\tilde\occ^{c,1}_\depth[w], \beta^2_\depth[w]) }_1(\theta^1_\depth, a^1, z^1)
  \Big) \\
  & = \sum_{\theta^2_\depth} \occ_\depth(\vth_\depth)
  \sum_{a^2} \beta^2_\depth[w](a^2|\theta^2_\depth) \\
  & \qquad \cdot \Bigg(
  { 
    \sum_{s}  b(s | \vth_\depth) r(s,\va)
  } \\
  & \qquad
  + \gamma \sum_{z^1} {
    \left[
      \sum_{s,s',z^2} P^{\vz}_{\va}(s'|s) 
      b(s | \vth_\depth)
    \right]
  }
  \cdot \Big( \upb\nu^2_{\depth+1}[w](\theta^1_\depth, a^1, z^1) \\ %
  & \qquad + \lt{\depth+1} \cdot
  \vnorm{  \Tc{1}(\occ^{c,1}_\depth, \beta^2_\depth[w]) - \Tc{1}(\tilde\occ^{c,1}_\depth[w], \beta^2_\depth[w]) }_1(\theta^1_\depth, a^1, z^1)
  \Big) \Bigg).
\end{align*}

Then, solving
$\max_{\beta^1_\depth} \upbW{\depth}{}(\occ_\depth, \beta^1_\depth)$
can be rewritten as solving a zero-sum game where pure strategies are:
\begin{itemize}
\item for Player $1$, the choice of not $1$, but $|\Theta^1_\depth|$ actions (among $|\cA^1|$) and,
\item for Player $2$, the choice of $1$ element of $\upb{\mathcal{I}}^1_\depth$.
\end{itemize}
One can view it as a Bayesian game with one type per history
$\theta^1_\depth$ for $1$, and a single type for $2$.

With our upper bound, %
$\max_{\beta^1_\depth} \upbW{\depth}{}(\occ_\depth,\beta^1_\depth)$
can thus be solved as the \ifextended{linear program}{LP}:
\begin{align*}
  %\label{eq|LP}
  & \begin{array}{l@{\ }l@{\ }ll}
      \displaystyle
      \max_{\beta_\depth^1,v}
      v
      \quad \text{s.t. } %\qquad \text{($v$ est un scalaire)}
      & \text{(i)}
      & \forall w \in \upb{\mathcal{I}}^1_\depth, %
      & v \leq {\beta_\depth^1}^{\top} \! \cdot M^{\occ_\depth}_{(\cdot,w)}
      \\
      & \text{(ii)}
      & \forall \theta_\depth^1 \in \Theta_\depth^1,
      & {\displaystyle \sum_{a^1}} \beta_\depth^1(a^1|\theta_\depth^1)
        = 1,
    \end{array}
    \intertext{whose dual LP is given by}
    % DUAL
    % 
    %\label{eq|DLP}
    & \begin{array}{l@{\ }l@{\ }ll}
        \displaystyle
        \min_{\tree^2_\depth,v}
        v
        \quad \text{s.t. }
        & \text{(i)}
        & \forall (\theta^1_\depth, a^1 ) \in \Theta^1_\depth\times \cA^1, %
        & v \geq  M^{\occ_\depth}_{((\theta^1_\depth,a^1),\cdot)} \cdot \tree^2_\depth
        \\
        & \text{(ii)}
        &
        & {\displaystyle \sum_{w \in \upb{\mathcal{I}}^1_\depth}} \! \tree^2_\depth( w )
          = 1.
      \end{array}
    \end{align*}
As can be noted, $M^{\occ_\depth}$'s columns corresponding to
$0$-probability histories $\theta^1_\depth$ in $\occ^{m,1}_\depth$ are
empty (full of zeros), so that the corresponding decision rules (for these histories)
are not relevant and can be set arbitrarily.
The actual implementation thus ignores these histories, whose
corresponding decision rules also do not need to be stored.

\begin{remark}[Interpretation of $M^{\occ_\depth}$]
The content of this matrix can be interpreted by noting that, a given $w$ containing a behavioral strategy $\beta^2_{\depth:H-1}$ and an \os{} $\tilde\occ_\depth$, a pair $\langle w, \theta^1_\depth \rangle$ induces a POMDP for player~1 whose state space is made of pairs $\langle s, \theta^2_t\rangle$, and whose initial belief $b_\depth$ depends on $\tilde\occ_\depth$ and $\theta_\depth^1$.
Solving this POMDP amounts to finding a best response of player $1$ to $\beta_{\depth:H-1}^2$.
In this setting, an element $M^{\occ_\depth}_{((\theta_\depth^1,a^1),w)}$ is an upper-bound of the optimal (POMDP) $Q$-value when player~1 performs $a^1$ while facing $b_\depth$ ($Q^*_{\text{POMDP}}(b_\depth,a^1)$).
\end{remark}

%with null columns if $\occ_\depth^{m,1}(\theta_\depth^1)=0$.
%for improbable histories $\theta_\depth^1$ under $\occ_\depth^{m,1}$.

% Coming back to the particular case of $\depth=H-1$, we can solve the
% (exact, not approximate) game at this last time step for
% $\beta^1_{H-1}$ and $\beta^2_{H-1}$ at the same time through the dual
% problems:
% \begin{align*}
%   \label{eq|lastExactupb}
%   \max_{\beta^1_{H-1}} W_{H-1}^{1,*}(\occ_{H-1},\beta_{H-1}^1) % 
%   & = \max_{\beta^1_{H-1}} \min_{\beta_{H-1}^2} {\beta^1_{H-1}}^{\top} \! \cdot
%   \vr(\occ_{H-1}, \cdot, \beta^2_{H-1}), \qquad \text{and} \\
%   \label{eq|lastExactlob}
%   \min_{\beta^1_{H-1}} W_{H-1}^{2,*}(\occ_{H-1},\beta_{H-1}^2) % 
%   & = \min_{\beta^1_{H-1}} \max_{\beta_{H-1}^2} {\beta^2_{H-1}}^{\top} \! \cdot
%   \vr(\occ_{H-1}, \beta^1_{H-1}, \cdot),
% \end{align*}
% which can be done by solving a single LP (and extracting the solution
% of its dual simultaneously).

%
% \olivier{Note: Traîter le cas ``approximation d'un horizon infini''
%   quelque part. %
%   Il est simple si l'on fait une approximation constante partout. %
% }

\subsubsection{Upper Bounding \texorpdfstring{$\nu^2_{[\occ^{c,1}_\depth,\tree^2_\depth]}$}{nu2}}
\label{sec|compute|nu}

Adding a new complete tuple to $\upb{\mathcal{I}}^1_\depth$ requires a new
vector $\upb\nu^2_{\depth}$ that upper bounds the vector
$\nu^2_{[\occ^{c,1}_\depth,\tree^2_\depth]}$ associated to the
strategy induced by $\tree^2_\depth$.
We can obtain one in a recursive manner (not solving the induced
POMDP).
%
%To that end, we first need the following intermediate result.

\begin{restatable}{proposition}{propRecNu}
  \labelT{prop|rec|nu}
  % \IfAppendix{{\em (originally stated on
  % page~\pageref{prop|rec|nu})}}{}
  % 
  For each $\tree^2_\depth$ obtained as the solution of the
  aforementioned (dual) LP in $\occ_\depth$, and each
  $\theta^1_\depth$,
  $\nu^2_{[\occ^{c,1}_{\depth},
    \tree^2_{\depth}]}(\theta^1_\depth)$ is upper bounded by a value
  $\upb\nu^2_{\depth}(\theta^1_\depth)$ that depends on vectors
  $\upb\nu^2_{\depth+1}$ in the support of $\tree^2_\depth$.
  In particular, if $\theta^1_\depth \in \supp(\occ^{m,1}_\depth)$, we
  have:
  \begin{align*}
    \upb\nu^2_{\depth}(\theta^1_\depth)
    & \eqdef \frac{1}{\occ^{1}_{\depth,m}(\theta^1_\depth)} \max_{a^1 \in \cA^1} M^{\occ_\depth}_{((\theta^1_\depth, a^1), . )} \cdot \tree^2_\depth.
  \end{align*}
  %(detailed formulas within the   proof).
\end{restatable}

\begin{proof}
  \label{proof|prop|rec|nu}

For a newly derived $\tree^2_\depth$, as
  $\nu^2_{[\occ^{c,1}_\depth, \tree^2_\depth]}(\theta^1_\depth)$
  is the value of $1$'s best action ($\in \cA^1$) if $1$ %
  (i) observes $\theta^1_\depth$ while in $\occ^{c,1}_\depth$ and %
  (ii) $2$ plays $\tree^2_\depth$, we have:
  % (and assuming long-term returns
  % determined by given $\nu^\star_{\depth+1}$ vectors), we have: %
  \begin{align*}
    & \nu^2_{[\occ^{c,1}_\depth, \tree^2_\depth]}(\theta^1_\depth) %
    \eqdef V^\star_{[\occ^{c,1}_\depth, \tree^2_\depth]}(\theta^1_\depth) %
    \qquad \text{(optimal POMDP value function)}
    \nonumber \\
%    \intertext{\olivier{Add intermediate line(s), including a reminder that $\occ^{c,1}_{\depth+1}\eqdef \Tc{1}(\occ_\depth,\vbeta_\depth)=\Tc{1}(\occ^{c,1}_\depth,\beta^2_\depth)$  {\scriptsize (lem.~\Crefpage{lem|T1cindep})}.}}
    %& = \max_{\beta^1_{\depth:}} V_\depth(\beta^1_{\depth:} \mid \theta^1_\depth, \occ^{c,1}_\depth, \tree^2_\depth) \\
    & = \max_{\beta^1_{\depth:}} \E\left[
      \sum_{t=\depth}^H \gamma^{t-\depth} R_t
      \mid \beta^1_{\depth:}, \theta^1_\depth, \occ^{c,1}_\depth, \tree^2_\depth
    \right]
    \nonumber \\
    \nonumber
    & = \max_{a^1} \E\left[
      R_\depth
      +
      \gamma \max_{\beta^1_{\depth+1:}}
      \E\Bigg[
        \sum_{t=\depth+1}^H \gamma^{t-(\depth+1)} R_t
        \mid \beta^1_{\depth+1:}, \langle \theta^1_\depth, a^1, Z^1 \rangle, \occ^{c,1}_{\depth+1}, \tree^2_{\depth+1}
      \right]
      \\
      & \quad \Bigg\vert a^1, \theta^1_\depth, \occ^{c,1}_\depth, \tree^2_\depth
    \Bigg] \\
    & = \max_{a^1} \E\left[
      R_\depth
      + \gamma V^\star_{[\occ^{c,1}_{\depth+1}, \tree^2_{\depth+1}]}( \theta^1_\depth, a^1, Z^1 )
      \mid a^1, \theta^1_\depth, \occ^{c,1}_\depth, \tree^2_\depth
    \right] \\
    \nonumber
    & = \max_{a^1} \sum_{w, \theta^2_\depth, a^2, z^1}
    \underbrace{Pr(w, \theta^2_\depth, z^1, a^2 \mid a^1, \theta^1_\depth, \occ^{c,1}_\depth, \tree^2_\depth)} \\
    & \quad \cdot
    \left(
      r(\vth_\depth,\va_\depth)
      + \gamma \nu^2_{[\occ^{c,1}_{\depth+1}, \tree^2_{\depth+1}[w]]}( \theta^1_\depth, a^1, z^1 )
    \right)
     \intertext{\small (where $\occ^{c,1}_{\depth+1}=\Tc{1}(\occ^{c,1}_\depth, \beta^2_\depth[w])$
      {\scriptsize (\Crefpage{lem|T1cindep})})}
      \nonumber
    & = \max_{a_1}  
    \sum_{w, \theta^2_\depth, a^2, z^1}
    \underbrace{Pr(w | \tree^2_\depth)}
    \cdot \underbrace{Pr(\theta^2_\depth | \theta^1_\depth, \occ^{c,1}_\depth)}
    \cdot \underbrace{Pr( a^2 | \beta^2_\depth[w], \theta^2_\depth)}
    \cdot \underbrace{Pr( z^1 | \vth_\depth, \va_\depth)} \\
    & \quad \cdot
    \left(
      r(\vth_\depth, \va) 
      + \gamma \nu^2_{[\occ^{c,1}_{\depth+1}, \tree^2_{\depth+1}[w]]}(\theta^1_\depth, a^1, z^1)
    \right) \\
    % =============================
    & = \max_{a_1}  \sum_w \tree^2_\depth(w)
    \sum_{\theta^2_\depth} \occ^{c,1}_\depth(\theta^2_\depth|\theta^1_\depth)
    \sum_{a^2} \beta^2_\depth[w](a^2|\theta^2_\depth) \\
    \nonumber
    & \quad \cdot
    \left(
      r(\vth_\depth, \va) 
      + \gamma \sum_{z^1} Pr(z^1|\vth_\depth,\va) \underbrace{\nu^2_{[\occ^{c,1}_{\depth+1}, \tree^2_{\depth+1}[w]]}(\theta^1_\depth, a^1, z^1)}
    \right) \\
    \intertext{then, as $\nu^2_{[\occ^{c,1}_{\depth+1}, \tree^2_{\depth+1}[w]]}$ is $\l_{\depth+1}$-LC in (any) $\occ^{c,1}_{\depth+1}$ (\Cref{lem|nuLC}),}
    \nonumber
    & \leq \max_{a_1} \sum_w \tree^2_\depth(w)
    \sum_{\theta^2_\depth} \occ^{c,1}_\depth(\theta^2_\depth|\theta^1_\depth)
    \sum_{a^2} \beta^2_\depth[w](a^2|\theta^2_\depth) \cdot
    \Bigg(
    r(\vth_\depth, \va) + \gamma \sum_{z^1} Pr(z^1|\vth_\depth,\va)
    \\
    & \quad \cdot
      \left[
        \overbrace{
          \underbrace{\nu^2_{[\tilde\occ^{c,1}_{\depth+1}[w], \tree^2_{\depth+1}[w]]}(\theta^1_\depth, a^1, z^1)}
          + \l_{\depth+1} \vnorm{\occ^{c,1}_{\depth+1} - \tilde\occ^{c,1}_{\depth+1}[w]}_1(\theta^1_\depth, a^1, z^1)
        }
      \right]
    \Bigg) \\
    \nonumber & \leq \max_{a_1} \sum_w \tree^2_\depth(w)
    \sum_{\theta^2_\depth} \occ^{c,1}_\depth(\theta^2_\depth|\theta^1_\depth)
    \sum_{a^2} \beta^2_\depth[w](a^2|\theta^2_\depth) \cdot
    \Bigg(
    \underbrace{ r(\vth_\depth, \va) }
    + \gamma \sum_{z^1} \underbrace{Pr(z^1|\vth_\depth,\va)}
    \\
    & \quad
    \cdot \left[
      \overbrace{ \upb\nu^2_{\depth+1}[w](\theta^1_\depth, a^1, z^1) }
      +  \l_{\depth+1} \vnorm{\occ^{c,1}_{\depth+1} - \tilde\occ^{c,1}_{\depth+1}[w]}_1(\theta^1_\depth, a^1, z^1)
    \right]
    \Bigg) 
    \\
    \nonumber & = \max_{a_1} \sum_w \tree^2_\depth(w)
    \sum_{\theta^2_\depth} \occ^{c,1}_\depth(\theta^2_\depth|\theta^1_\depth)
    \sum_{a^2} \beta^2_\depth[w](a^2|\theta^2_\depth) 
    \cdot
    \Bigg(
      \overbrace{\sum_s b(s|\vth_\depth) r(s, \va)}
    \\
    \nonumber & \quad
      + \gamma \sum_{z^1} \left( \overbrace{ \sum_s b(s|\vth_\depth) \underbrace{Pr(z^1|s,\va)} } \right)
      \cdot \Big[ \upb\nu^2_{\depth+1}[w](\theta^1_\depth, a^1, z^1)
     \\
    & \quad
    + \l_{\depth+1} \underbrace{ \vnorm{\occ^{c,1}_{\depth+1} - \tilde\occ^{c,1}_{\depth+1}[w]}_1(\theta^1_\depth, a^1, z^1) }
    \Big]
    \Bigg)
    \\
    \nonumber
    & = \max_{a_1} \sum_w \tree^2_\depth(w)
    \sum_{\theta^2_\depth} \occ^{c,1}_\depth(\theta^2_\depth|\theta^1_\depth)
    \sum_{a^2} \beta^2_\depth[w](a^2|\theta^2_\depth)
    \label{eq|nuNotInSupport} 
    \\
       \nonumber
        & \quad
    \cdot
    \Bigg(
    \sum_s b(s|\vth_\depth) r(s, \va)     
    + \gamma \sum_{z^1} \left( \overbrace{ \sum_{s, s', z^2} b(s|\vth_\depth) P^{\vz}_{\va}(s'|s) } \right)
    \cdot 
    \Big[ \upb\nu^2_{\depth+1}[w](\theta^1_\depth, a^1, z^1)
     \\
    & \quad
    + \l_{\depth+1}
    \overbrace{ \vnorm{ \Tc{1}(\occ^{c,1}_\depth, \beta^2_\depth[w]) - \Tc{1}(\tilde\occ^{c,1}_\depth[w], \beta^2_\depth[w]) }_1(\theta^1_\depth, a^1, z^1) }
    \Big] \Bigg)
     \\
    & = \frac{1}{\occ^{1}_{\depth,m}(\theta^1_\depth)} \max_{a^1 \in \cA^1} M^{\occ_\depth}_{((\theta^1_\depth, a^1), . )} \cdot \tree^2_\depth.
    \nonumber
  \end{align*}
\end{proof}

\subsubsection{Strategy Conversion}
\label{App|th:eqStrategies}

Firstly, we give details regarding solutions of Dual LPs (\Cref{eq|LP}) inducing behavioral strategies. 
As suggested in \Cref{subsection:actionSelectionAndBackupOperators}, one can show by induction that for any timestep $\depth$, each $\tree_{\depth}^2$ is actually equivalent to an element of $\Delta(\cB_{\depth:}^2)$. The following lemma shows that for any timestep $\depth$, each element of $\Delta(\cB_{\depth:}^2)$ induces an element of $\cB_{\depth:}^2$.

%\aurelien{il faut discuter de ce théorème. Comme on l'a vu dans le papier, on peut montrer par récurrence que les $\tree_t$ sont bien des éléments de $\Delta(\cB)$, le lemme suivant permet de passer de $\Delta(\cB)$ à $\cB$.}

\begin{restatable}[]{lemma}{App|th:eqStrategies}
\label{th|eqStrategies}
Each $\tree_{\depth}^i \in \Delta{(\cB_{\depth:}^i)}$ induces a behavioral strategy. More precisely, we prove that (i) there is a natural injection from the set $\cB_{\depth:}^i$ to the set of distributions $\Delta(\cB_{\depth:}^i)$  and 
(ii) there is a surjection from the set $\Delta(\cB_{\depth:}^i)$ to $\cB_{\depth:}^i$.
\end{restatable}

\begin{proof}
% First, it is forward to show that every behavioral strategy is simply a degenerate tree-shaped strategy.
By induction on $\depth \in \{0,\dots,H-1\}$, we prove that $\cup_{t=\depth}^H \cB_t^i \subset \cup_{t=\depth}^H \Delta(\cB^i_t)$. 
Firstly, for $\depth = H-1$, for all $ \beta_{H-1} \in \cB_{H-1}$, one can pick the degenerate distribution $\tree_{H-1} = \beta_{H-1}$ which is in $\Delta(\cB^i_{H-1})$. 
Next, assume that $\cup_{t=\depth+1}^{H-1} \cB_{t}^i \subset \cup_{t=\depth+1}^{H-1} \Delta(\cB^i_{t})$ for some $\depth \in \{0,\dots,H-2\}$, then for all $\beta_{\depth:H-1},\ \beta_{\depth:H-1} = \beta_\depth \oplus \beta_{\depth+1:H-1}$. By the induction hypothesis, there is $\tree_{\depth+1:H-1} \in \Delta(\cB_{\depth+1:H-1})$ equal to $\beta_{\depth+1:H-1}$. Thus, we define $\tree_{\depth:H-1} = \beta_{\depth} \oplus \tree_{\depth+1:H-1} = \beta_\depth \oplus \beta_{\depth+1:H-1} = \beta_{\depth:H-1}$ which is in $\cB_{\depth:H-1}$.
From this follows a natural injection from the behavioral strategies' set to the set of distributions over behavioral strategies.

The surjection from $\Delta(\cB_{\depth:}^i)$ to $\cB_{\depth:}^i$ is given by the realization weight computation algorithm detailed in \Cref{alg|extractingBeta}.
%\aurelien{Mouais.. encore faudrait-il prouver que cet algorithme se termine bien (en temps fini) vers un élément de $\cB_{\depth:}^i$ et dire qu'une stratégie $\tree_{0}^2$ est bien sous forme de poids de réalisation}
\end{proof}
\label{app|stratExtraction}

As discussed in \Cref{app|stratExtraction}, no effort is required to
extract a solution strategy for a player from the lower bound (for
$1$) or the upper bound (for $2$), but that strategy is in an unusual
recursive form.
We will here see (in the finite horizon setting) how to derive a
(unique) equivalent behavioral strategy $\beta^i_{0:}$ using
realization weights \citep{KolMegSte-stoc94} in intermediate steps.
To that end, we first define these realization weights in the case of
a behavioral strategy (rather than for a mixed strategy
as done by \citeauthor{KolMegSte-stoc94}) and present some useful
properties.

\paragraph{About Realization Weights}

Let us denote $rw^i(a^i_0, z^i_1, a^i_1, \dots, a^i_\depth)$ the {\em
  realization weight} (RW) of sequence
$a^i_0, z^i_1, a^i_1, \dots, a^i_\depth$ under strategy
$\beta^i_{0:}$, defined as
\begin{align}
  rw^i(a^i_0, z^i_1, a^i_1, \dots, a^i_\depth)
  & \eqdef \prod_{t=0}^\depth  \beta^i_{0:}(a^i_t | a^i_0, z^i_1, a^i_1, \dots, z^i_t) \\
  & = rw^i(a^i_0, z^i_1, a^i_1, \dots, a^i_{\depth-1}) \cdot \beta^i_{0:}(a^i_\depth | \underbrace{a^i_0, z^i_1, a^i_1, \dots, z^i_\depth}_{\theta^i_\depth}).
  \intertext{This definition already leads to useful results such as:}
  \beta^i_{0:}(a^i_\depth | \theta^i_\depth)
  & = \frac{
    rw^i(\theta^i_{\depth-1}, a^i_{\depth-1}, z^i_\depth, a^i_\depth)
  }{
    rw^i(\theta^i_{\depth-1}, a^i_{\depth-1})
  },
  \label{eq|betaFromRw}
  % \\
  % %\intertext{and (as used to derive RWs recursively in the algorithm in two different ways)}
  % rw^i(\theta^i_{\depth-1}, a^i_{\depth-1})
  % & = \frac{
  %   rw^i(\theta^i_{\depth-1}, a^i_{\depth-1}, z^i_\depth, a^i_\depth)
  % }{
  %   \beta^i_{0:}(a^i_\depth | \theta^i_\depth)
  % },
  \intertext{and}
  \forall z^i_\depth, \quad
  rw^i(\theta^i_{\depth-1}, a^i_{\depth-1})
  & = rw^i(\theta^i_{\depth-1}, a^i_{\depth-1}) \cdot \underbrace{\sum_{a^i_\depth} \beta(a^i_\depth | \theta^i_{\depth-1}, a^i_{\depth-1}, z^i_\depth )}_{=1} \\
  & = \sum_{a^i_\depth} rw^i(\theta^i_{\depth-1}, a^i_{\depth-1}) \cdot \beta(a^i_\depth | \theta^i_{\depth-1}, a^i_{\depth-1}, z^i_\depth ) \\
  & = \sum_{a^i_\depth} rw^i(\theta^i_{\depth-1}, a^i_{\depth-1}, z^i_\depth, a^i_\depth).
  \label{eq|RWsFromFullLengthRWs}
\end{align}

We now extend \citeauthor{KolMegSte-stoc94}'s definition by introducing
%the case where we have only $\beta^i_{\depth:}$, which gives us
{\em conditional realization weights}, where the realization weight of a
{\em suffix sequence} is ``conditioned'' on a {\em prefix sequence}:
\begin{align}
  & rw^i(\underbrace{a^i_\depth, \dots, a^i_{\depth'}}_{\text{suffix seq.}} | \underbrace{a^i_0, \dots, z^i_\depth}_{\text{prefix seq.}})
  \eqdef \prod_{t=\depth}^{\depth'}  \beta^i_{0:}(a^i_t | a^i_0, \dots, z^i_\depth, a^i_\depth, \dots, z^i_t)
  \label{eq|rw|def} \\
  & = \beta^i_{0:}(a^i_{\depth} | a^i_0, \dots, z^i_\depth)   \cdot rw^i(a^i_{\depth+1}, \dots, a^i_{\depth'} | a^i_0, \dots, z^i_{\depth+1}).
  \label{eq|rw|rec|beta}
  % \intertext{and, in particular, for $\depth'=H-1$ (called {\em full length} (conditional) realization weights),}
  % rw^i(a^i_\depth, \dots, a^i_{H-1} | a^i_0, \dots, z^i_\depth)
  % & = \beta^i_{0:}(a^i_{\depth} | a^i_0, \dots, z^i_\depth)   \cdot rw^i(a^i_{\depth+1}, \dots, a^i_{H-1} | a^i_0, \dots, z^i_{\depth+1}).
  % \label{eq|rw|rec|beta}
 % \\
 %  \text{(thus, conversely, }
 %  \beta^i(a^i_\depth | \theta^i_\depth)
 %  & = \frac{
 %    rw^i(a^i_{\depth}, \dots, a^i_{\depth'} | a^i_0, \dots, z^i_{\depth})
 %  }{
 %    rw^i(a^i_{\depth+1}, \dots, a^i_{\depth'} | a^i_0, \dots, z^i_{\depth+1})
 %  }).
\end{align}
As can be noted, this definition only requires the knowledge of a
partial strategy $\beta^i_{\depth:}$ rather than a complete strategy
$\beta^i_{0:}$.

\paragraph{Mixing Realization Weights}

Let $\depth'\geq \depth+1$, and $rw^i[w]$ denote the realization
weights of some element $w$ at $\depth+1$.
%
% Let us assume that we have, for each element $w$ at $\depth+1$,
% conditional realization weights $rw^i[w]$ are known (meaning, all with
% prefixes/conditions from $0$ to $\depth+1$).
% 
Then, for some $\tree^i_\depth$, we have
\begin{align}
  & rw[\tree^i_\depth](a^i_{\depth+1}, \dots, a^i_{\depth'} | a^i_0, \dots, z^i_{\depth+1})
  \\
  & = \sum_w \tree^i_\depth(w) \cdot rw[w](a^i_{\depth+1}, \dots, a^i_{\depth'} | a^i_0, \dots, z^i_{\depth+1}).
  \label{eq|rw|rec|delta}
\end{align}

% If I'm not mistaken, \cref{eq|rw|rec|beta,eq|rw|rec|delta} are all
% we need to obtain, through a recursive process, realization weights
% for any recursively defined strategy in a sub-DAG.
% % 
% \aurelien{Je suis perturbé, dans notre algo, nous n'avons pas $\beta^i_{0:}(a^i_{\depth} | a^i_0, \dots, z^i_\depth)$, enfin disons qu'il s'écrit comme une combinaison sur les $\beta$ des sacs. Du coup, j'aurais plutôt fait sortir la somme $\sum_w \tree^i_\depth(w) \cdot \beta^i_{0:}(a^i_{\depth} | a^i_0, \dots, z^i_\depth) \cdot rw[w](a^i_{\depth+1}, \dots, a^i_{\depth'} | a^i_0, \dots, z^i_{\depth+1})$}
  
% Note: At $H-1$, the initialization is given by \cref{eq|rw|def} when

\paragraph{\texorpdfstring{From $w^i_0$ to $\beta^i_{0:}$}{From wi0 to betai0:}}

% Using the above results, we now describe how to derive from the
%   recursive policy $\tree^i_0$, in order,%
% \begin{enumerate}
% \item full length (conditional) realization weights,
% \item (unconditional) realization weights, and
% \item $\beta^i_{0:}$.
% \end{enumerate}

%
\SetKwFunction{FExtract}{${\text{\bf Extract}}$}
\SetKwFunction{FRecGetRWMix}{${\text{\bf RecGetRWMix}}$}
\SetKwFunction{FRecGetRWCat}{${\text{\bf RecGetRWCat}}$}

\begin{algorithm}%[H]
  \caption{Extracting $\beta^i_{0:}$ from $w^i_0$}
  \label{alg|extractingBeta}
  
  \DontPrintSemicolon
  %\SetKwFunction{FExtract}{${\text{\bf Extract}}$}
  %\SetKwFunction{FRecGetRWMix}{${\text{\bf RecGetRWMix}}$}
  %\SetKwFunction{FRecGetRWCat}{${\text{\bf RecGetRWCat}}$}
  
  \Fct{\FExtract{$w^i_0$}}{
    \tcc{Step 1., keeping only $rw(\theta^i_{0:H-1})$ for all $\theta^i_{0:H-1}$}
    $\left( rw(\theta^i_{0:H-1}) \right)_{\theta^i_{0:H-1}} \gets$ \FRecGetRWMix{$0 , w^i_0$}\;
    
    \tcc{Step 2.}
    \For{$t=H-2, \dots, 0$}{
      \ForAll{$\theta^i_{0:t}, a^i_{t}$}{
        $z^i_{t+1} \gets z^i$ s.t. $\beta_t(\cdot|\theta^i_{0:t}, a^i_t, z^i)$ is defined\;
        $rw(\theta^1_{0:t}, a^i_t) %
        \gets \sum_{a^i_{t+1}} rw(\theta^i_{0:t}, a^i_t, z^i_{t+1},a^i_{t+1} | - )$ % \tcc{BUG !}
        \label{line|RWsFromFullLengthRWs}
      }
    }
    
    \tcc{Step 3.}
    \For{$t=H-1, \dots, 0$}{
      \ForAll{$\theta^i_{0:t}, a^i_t$}{
        $\beta^i_t(a^i_t | \theta^i_{0:t}) %
        \gets  \frac{
          rw^i(\theta^i_{0:t-1}, a^i_{t-1}, z^i_t, a^i_t)
        }{
          rw^i(\theta^i_{0:t-1}, a^i_{t-1})
        }
        $
        \label{line|betaFromRw}
      }
    }
    \Return{$\beta^i_{0:}$}
  }
  
  \Fct{\FRecGetRWMix{$t , w = \langle \beta^i_t, \tree^i_t \rangle $}}{
    \For{$w'$ s.t. $\tree^i_t(w')>0$}{
      $rwCat[w'] \gets$ \FRecGetRWCat{$t,w'$} %\tcc{Use caching !}
    }
    \ForAll{$(a^i_0, \dots, a^i_{H-1})$}{
      $ rwMix[w](a^i_{t}, \dots, a^i_{H-1} | a^i_0, \dots, z^i_{t}) %
      $
      
      $\gets \sum_{w'} {
        \tree^i_t(w')
        \cdot rwCat[w'](a^i_{t+1}, \dots, a^i_{H-1} | a^i_0, \dots, z^i_{t+1})
      }
      $
      \label{line|rw|rec|beta}
    }
    \Return{$rwMix[w]$}
  }

  \Fct{\FRecGetRWCat{$t , w = \langle \beta^i_t, \tree^i_t \rangle $}}{
    
    \eIf{$t=H-1$}{
      \ForAll{$(a^i_0, \dots, a^i_{H-1})$}{
        $ rwCat[w](a^i_{H-1} | a^i_0, \dots, z^i_{H-1})
        \gets  \beta^i_t(a^i_{H-1} | a^i_0, \dots, z^i_{H-1}) $
      }
    }{
      $rwMix[w]
      \gets $ \FRecGetRWMix{$t,w$}\;   % \tcc{Use caching !}
      \ForAll{$(a^i_0, \dots, a^i_{H-1})$}{
        $ rwCat[w](a^i_{t}, \dots, a^i_{H-1} | a^i_0, \dots, z^i_{t}) %
        \gets  \beta^i_t(a^i_t|a^i_0, \dots, z^i_t)
        \cdot rwMix[w](a^i_{t+1}, \dots, a^i_{H-1} | a^i_0, \dots, z^i_{t+1})
        $
        \label{line|rw|rec|delta}
      }
    }        
    \Return{$rwCat[w]$}
  }
\end{algorithm}

Using the above results, function \FExtract in
\Cref{alg|extractingBeta} derives a behavioral strategy $\beta^i_{0:}$
equivalent to the recursive strategy induced by some tuple $w^i_0$ in
3 steps as follows:
\begin{enumerate}
\item{\bf From $w^i_0$ to $rw(\theta^i_{0:H-1}, a^i_{H-1})$
    ($\forall (\theta^i_{0:H-1}, a^i_{H-1})$) ---}
  These (classical) realization weights are obtained by recursively
  going through the directed acyclic graph describing the recursive
  strategy, computing {\em full length} (conditional) realization
  weights % (because they go from $0$ to $H-1$)
  $rw(\theta^i_{t:H-1}, a^i_{H-1} | \theta^i_{0:t})$ (for $t=H-1$ down
  to $0$).
  
  When in a leaf node, at depth $H-1$, the initialization is given by
  \Cref{eq|rw|def} when $\depth=\depth'=H-1$:
  \begin{align*}
    rw^i(a^i_{H-1} | a^i_0, \dots, z^i_{H-1})
    & \eqdef \prod_{t={H-1}}^{{H-1}}  \beta^i(a^i_t | a^i_0, \dots, z^i_t) \\
    & =  \beta^i(a^i_{H-1} | a^i_0, \dots, z^i_{H-1}).
  \end{align*}
  Then, in the backward phase, %\Cref{eq|rw|rec|beta,eq|rw|rec|delta}
  we can compute full length realization weights
  $rw(\theta^i_{t+1:H-1}, a^i_{H-1} | \theta^i_{0:t})$ with
  increasingly longer suffixes (thus shorter prefixes) using %
  (i) \Cref{eq|rw|rec|delta} (in function \FRecGetRWMix,
  \cref{line|rw|rec|delta}) to ``mix'' several strategies using the
  distribution $\tree^i_t$ attached to the current $w$, and
  (ii) \Cref{eq|rw|rec|beta}, with $\depth'=H-1$, (in function
  \FRecGetRWCat, \cref{line|rw|rec|beta}) to concatenate the
  behavioral decision rule $\beta^i_t$ attached to the current $w$ in
  front of the strategy induced by the distribution $\tree^i_t$ also
  attached to $w$.
  Note: Memoization can here be used to avoid repeating the same
  computations.
\item{\bf Retrieving (classical) realization weights $rw(\theta^i_{0:t}, a^i_t | -)$ ($\forall t$) ---}
  We can now compute realization weights
  $rw(\theta^i_{0:t}, a^i_t | -)$ for all $t$'s using
  \Cref{eq|RWsFromFullLengthRWs} (\cref{line|RWsFromFullLengthRWs}).
\item{\bf Retrieving behavioral decision rules $\beta^i_t$ ---}
  Applying \Cref{eq|betaFromRw} (\cref{line|betaFromRw}) then provides
  the expected behavioral decision rules.
\end{enumerate}

In practice, lossless compressions are used to reduce the
dimensionality of the occupancy state (\cf \Cref{sec|XP|algorithms}),
which are currently lost in the current implementation of the conversion.
Ideally, one would like to preserve compressions whenever possible or
at least retrieve them afterwards, and possibly identify further
compressions in the solution strategy.
%
% \olivier{
%   Two opportunities for further compression are that:
%   \begin{itemize}
%   \item the compression focuses on preserving the strategy, not
%     properties regarding the system dynamics; and
%   \item in some elements $w$, certain decisions may be undefined (\ie,
%     for certain histories), so that, rather than picking a value at
%     random, one could maintain intervals of feasible values until
%     selecting values that allow for compressing the strategy.
%   \end{itemize}
% }
% ------------------

%--------------------------

\section{HSVI for zs-POSGs}
\label{proofLemMaxRadius}
\label{proofLemThr}

This section presents
% a version of \zsomg-HSVI
% (\Cref{alg|zsPOSGwithLP+VWs+}) that correctly handles step
% $\depth=H-1$, plus
%
results that help
(i) tune \zsomg-HSVI's radius parameter $\rho$, ensuring that
trajectories will always stop,
and (ii) then demonstrate the finite time convergence of this
algorithm.

\subsection{Algorithm}

\subsubsection{Setting \texorpdfstring{$\radius$}{rho}}
\label{sec|settingRadius}

% \lemMaxRadius*

%\begin{restatable}[Proof in \ifextended{App.~\Cref{proofLemMaxRadius}}{\citep{icml20ext}}]{lemma}{lemMaxRadius}
\begin{restatable}[Proof in \extCshref{proofLemMaxRadius}]{proposition}{lemMaxRadius}
  \labelT{lem|MaxRadius}
%  \IfAppendix{{\em (originally stated on page~\pageref{lem|MaxRadius})}}{}
  % 
  Bounding $\lt{\depth}$ by $\l^{\infty} = \frac{1}{2} \frac{1}{1-\gamma} %
  \left[ r_{\max} - r_{\min} \right]$
  when $\gamma<1$, and noting that
  \begin{align}
    \label{eq|thr}
    \thr(\depth)
    & = 
      \gamma^{-\depth}\epsilon - 2 \radius \l^\infty \frac{\gamma^{-\depth}-1}{1-\gamma}
      \quad \text{if } \gamma<1
      \\ \nonumber
    & \text{(} = \epsilon - \radius (r_{\max}-r_{\min})  (2H + 1 - \depth) \depth 
        \quad \text{if } \gamma=1 \text{)},
  \end{align}
  one can ensure positivity of the threshold at any
  $\depth \in 1 \twodots H-1$ by enforcing %
  $0  < \radius < \frac{1-\gamma}{2\l^\infty}\epsilon$ % if $\gamma<1$,  % 
  (or $0  < \radius <\frac{\epsilon}{(r_{\max}-r_{\min}) (H + 1)H}$  if $\gamma=1$).
\end{restatable}

\begin{proof}
  \label{proof|lem|MaxRadius}
  \uline{Let us first consider the case $\gamma<1$.}\\
  We have (for $\depth \in \{ 1 \twodots H-1 \}$):
  \begin{align*}
    \thr(\depth)
    & = \gamma^{-\depth}\epsilon - \sum_{i=1}^\depth 2 \radius \l^\infty \gamma^{-i} \\
    & = \gamma^{-\depth}\epsilon - 2 \radius \l^\infty \sum_{i=1}^\depth \gamma^{-i} \\
    & = \gamma^{-\depth}\epsilon - 2 \radius \l^\infty \left( \gamma^{-1} + \gamma^{-2} + \cdots + \gamma^{-\depth} \right) \\
    & = \gamma^{-\depth}\epsilon - 2 \radius \l^\infty \gamma^{-1} \left( \gamma^{0} + \gamma^{-1} + \cdots + \gamma^{-(\depth-1)} \right) \\
    & = \gamma^{-\depth}\epsilon - 2 \radius \l^\infty \gamma^{-1} \frac{\gamma^{-\depth}-1}{\gamma^{-1}-1} \\
    & = \gamma^{-\depth}\epsilon - 2 \radius \l^\infty \frac{\gamma^{-\depth}-1}{1-\gamma}.
  \end{align*}
  Then, let us derive the following equivalent inequalities:
  \begin{align*}
    0
    & < \thr(\depth) \\ %\quad
    %& \Leftrightarrow \quad
    2 \radius \l^\infty \frac{\gamma^{-\depth}-1}{1-\gamma}
    & < \gamma^{-\depth}\epsilon \\
    %& \Leftrightarrow \quad
    \radius
    & < \frac{1}{2\l^\infty} \frac{1-\gamma}{\gamma^{-\depth}-1} \gamma^{-\depth} \epsilon \\
    %& \implies \quad
    \radius
    & < \frac{1}{2\l^\infty} \frac{1-\gamma}{1-\gamma^\depth}  \epsilon.
    % \leq \frac{\epsilon}{2\l}.
  \end{align*}
  To ensure positivity of the threshold for any $\depth \geq 1$, one
  thus just needs to set $\radius$ as a positive value smaller than
  $\frac{1-\gamma}{2\l^\infty}\epsilon$.

  \uline{Let us now consider the case $\gamma=1$.}\\
  We have (for $\depth\in \{1,\dots,H-1\}$):
  \begin{align*}
    \thr(\depth)
    & \eqdef \epsilon - \sum_{i=1}^\depth 2 \radius \l_{\depth-i} \\
    & = \epsilon - \sum_{i=1}^\depth 2 \radius (H-(\depth-i))\cdot(r_{\max}-r_{\min}) \\
    & = \epsilon - 2 \radius (r_{\max}-r_{\min}) \left[ \depth(H-\depth) + \sum_{i=1}^\depth i \right]  \\
    & = \epsilon - 2 \radius (r_{\max}-r_{\min}) \left[ \depth H - \depth^2 + \frac{1}{2}\depth (\depth+1) \right]  \\
    & = \epsilon - 2 \radius (r_{\max}-r_{\min}) \left[  (H+\frac{1}{2}) \depth - \frac{1}{2} \depth^2 \right]  \\
    & = \epsilon - \radius (r_{\max}-r_{\min}) \left[ (2H+1) \depth - \depth^2 \right]  \\
    & = \epsilon - \radius (r_{\max}-r_{\min}) \left[ (2H + 1 - \depth) \depth \right].
  \end{align*}
  Then, let us derive the following equivalent inequalities:
  \begin{align*}
    0
    & < \thr(\depth) \\ 
    \radius (r_{\max}-r_{\min}) (2H + 1 - \depth) \depth
    & < \epsilon
    \qquad \qquad \qquad \text{(holds when $\depth=0$ and $\depth=H+1$)}
    \\
    \radius 
    & < \frac{\epsilon}{(r_{\max}-r_{\min}) (2H + 1 - \depth) \depth}
    \text{ (when $\depth \in \{0 \twodots H+1\}$).}
  \end{align*}
  The function
  $f: \depth \mapsto \frac{\epsilon}{(r_{\max}-r_{\min}) (2H + 1 -
    \depth) \depth}$
  reaches its minimum (for $\depth \in (0,H+1)$) when
  $\depth=H+\frac{1}{2}$.
  To ensure positivity of the threshold for any $\depth \in \{1 \twodots H-1 \}$, one
  thus just needs to set $\radius$ as a positive value smaller than
  $\frac{\epsilon}{(r_{\max}-r_{\min}) (H + 1)H}$.
\end{proof}

% ------------------------

\subsection{Finite-Time Convergence}

\subsubsection{Convergence Proof}
\label{sec|ConvergenceProof}

Proving the finite-time convergence of \zsomg-HSVI to an error-bounded
solution requires some preliminary lemmas.

%Let us start with some lemmas.

%\olivier{$t \to \depth$ !}

\begin{lemma}
  \labelT{lemma|OMG-HSVIContraction}
  Let $(\occ_0,\dots,\occ_{\depth+1})$ be a full trajectory generated by \zsomg-HSVI and %
  $\vbeta_\depth$ the behavioral \dr profile that induced the last transition, \ie, $\occ_{\depth+1}= T(\occ_\depth, \vbeta_\depth)$.
  % 
  % \uwave{ Let $\nu_{\depth+1}^2$ be some vector returned by the HSVI algorithm at step $\depth+1$. }
  % 
 Then, after updating $\upbW{\depth}{}$ and $\lobW{\depth}{}$, %
 we have that $\upbW{\depth}{}(\occ_\depth,\beta^1_\depth) - \lobW{\depth}{}(\occ_\depth,\beta^2_\depth) \leq \gamma \thr(\depth+1)$.
\end{lemma}

\begin{proof}
  \label{proof|lemma|OMG-HSVIContraction}
  By definition,
  \begin{align*}
    \upbW{\depth}{} (\occ_\depth, \beta^1_\depth)
    & = \min_{ \substack{
        \langle \tilde\occ^{c,1}_\depth, \tilde{\beta}^2_\depth, \langle \upb\nu^2_{\depth+1}, \tree_{\depth+1:}^2 \rangle \rangle \\
        \in \upb{\mathcal{I}}^1_\depth
      }
    } 
      \beta^1_\depth \cdot \Big(
        r(\occ_\depth, \cdot, \tilde{\beta}^2_\depth)
        + \gamma \Tm{1}(\occ_\depth,\cdot, \tilde{\beta}^2_\depth)  \\
        & \qquad \cdot \Big[  \upb\nu^2_{\depth+1} 
        % \\
        % & \qquad \qquad
        +  \lt{\depth+1} \vnorm{ \Tc{1}(\occ^{c,1}_\depth, \beta^2_\depth) - \Tc{1}(\tilde\occ^{c,1}_\depth, \beta^2_\depth) }_1
        \Big]
    \Big).
  \end{align*}
  Therefore, after the update ($\beta^2_\depth$ and $\beta^1_\depth$
  being added to their respective bags ($\upb{\mathcal{I}}^1_\depth$ and
  $\lob{\mathcal{I}}^2_\depth$) along with vectors $\upb\nu^2_{\depth+1}$ and
  $\lob\nu^1_{\depth+1}$),
  \begin{align*}
    \upbW{\depth}{} (\occ_\depth, \beta^1_\depth)
    & \leq
    \beta^1_\depth \cdot \left[
      r(\occ_\depth, \cdot, \beta^2_\depth)
      + \gamma \Tm{1}(\occ_\depth,\cdot, \beta^2_\depth) \cdot \upb\nu^2_{\depth+1} 
    \right], \text{ and} \\
    \lobW{\depth}{} (\occ_\depth, \beta^2_\depth)
    & \geq
    \beta^2_\depth \cdot \left[
      r(\occ_\depth, \beta^1_\depth, \cdot)
      + \gamma \Tm{2}(\occ_\depth, \beta^1_\depth, \cdot) \cdot \lob\nu^1_{\depth+1} 
    \right].
%  \end{align*}
  \intertext{Then,}
%  \begin{align*}
    \upbW{\depth}{} (\occ_\depth, \beta^1_\depth) - \lobW{\depth}{}(\occ_\depth,\beta^2_\depth)
    & \leq \left[ \cancel{r(\occ_\depth, \beta^1_\depth, \beta^2_\depth)} + \gamma \Tm{1}(\occ_\depth, \vbeta_\depth) \cdot \upb\nu^2_{\depth+1} \right] % \\%
    % & \quad
    \\
    & \qquad - \left[ \cancel{r(\occ_\depth, \beta^1_\depth, \beta^2_\depth)} +  \gamma \Tm{2}(\occ_\depth,\vbeta_\depth) \cdot \lob\nu^1_{\depth+1} \right]\\
    & = \gamma \left[ \upb{V}(T(\occ_\depth,\vbeta_\depth)) - \lob{V}(T(\occ_\depth,\vbeta_\depth)) \right] \\
    & \leq \gamma \thr(\depth+1)
    \qquad \text{(Holds at the end of any trajectory.)}
    %\qedhere
  \end{align*}
\end{proof}

%\begin{definition}
%Les opérateurs de mise à jour $\upb{K}$ et $\lob{K}$ s'appliquant sur $V$ et $W$ sont définis dans l'algorithme 2.
%\end{definition}
\begin{lemma}[Monotonic evolution of $\upbW{\depth}{}$ and $\lobW{\depth}{}$]
  \labelT{lemma|DecreaseFunctions}
  %\uwave{The update operator implies that $K\upbW{\depth}{} \leq \upbW{\depth}{}$ and $K\lobW{\depth} \geq \lobW{\depth}$.}\\
  % 
  Let $K\upbW{\depth}{}$ and $K\lobW{\depth}{}$ be the approximations
  after an update at $\occ_\depth$ with behavioral \dr
  $\langle \upb\beta^1_\depth, \lob\beta^2_\depth \rangle$
  (respectively associated to vectors $\upb\nu^2_{\depth+1}$ and
  $\lob\nu^1_{\depth+1}$).
  Let also $K^{(n+1)}\upbW{\depth}{}$ and $K^{(n+1)}\lobW{\depth}{}$ be
  the same approximations after $n$ other updates (in various \os{}s).
  Then,
% the behavioral \dr used computed strategies, $\forall n \in \mathbb{N}^*$
  \begin{align*}
    % &\max_{\beta^1_\depth} K^{(n)} \upbW{\depth}{}(\occ_\depth,\beta^1_\depth) \leq \max_{\beta^1_\depth} K\upbW{\depth}{}(\occ_\depth,\beta^1_\depth) \leq \upbW{\depth}{}(\occ_\depth,\beta_\depth^{1,*}) \qquad \text{and} \\
    % &\min_{\beta^2_\depth} K^{(n)} \lobW{\depth}{}(\occ_\depth,\beta^2_\depth) \leq \min_{\beta^2_\depth} K\lobW{\depth}{}(\occ_\depth,\beta^2_\depth) \geq \lobW{\depth}{}(\occ_\depth,\beta_\depth^{2,*}).
    \max_{\beta^1_\depth} K^{(n+1)} \upbW{\depth}{}(\occ_\depth,\beta^1_\depth)
    & \leq \max_{\beta^1_\depth} K\upbW{\depth}{}(\occ_\depth,\beta^1_\depth)
      \leq \upbW{\depth}{}(\occ_\depth, \upb\beta^1_\depth) \quad \text{ and} \\
    \min_{\beta^2_\depth} K^{(n+1)} \lobW{\depth}{}(\occ_\depth,\beta^2_\depth)
    & \geq \min_{\beta^2_\depth} K\lobW{\depth}{}(\occ_\depth, \beta^2_\depth)
      \geq \lobW{\depth}{}(\occ_\depth, \lob\beta^2_\depth).
  \end{align*}
\end{lemma}

\begin{proof}
  \label{proof|lemma|DecreaseFunctions}
  Starting from the definition,
  \begin{align*} 
    & \max_{\beta^1_\depth} K\upbW{\depth}{}(\occ_\depth,\beta^1_\depth) \\
    & = \max_{\beta^1_\depth} \hspace{-.5cm} \min_{ \substack{ 
        \langle \tilde\occ^{c,1}_\depth, \beta^2_\depth, \langle \upb\nu^2_{\depth+1},\tree_{\depth+1:}^2 \rangle  \rangle \in \\
        \upb{\mathcal{I}}^1_{\depth} \cup \{ \langle \occ^{c,1}_\depth, \lob\beta^2_\depth, \langle \upb\nu^2_{\depth+1},\tree_{\depth+1:}^2 \rangle \rangle \}
      }}
    \beta^1_\depth \cdot \bigg[
    r(\occ_\depth,\cdot,\beta^2_\depth)
    \\
    & \qquad + \gamma \Tm{1}(\occ_\depth,\cdot,\beta^2_\depth) % \\
    % & \qquad
    \cdot \Big( 
    \upb\nu^2_{\depth+1}
    +  \lt{\depth+1} \vnorm{ \Tc{1}(\occ^{c,1}_\depth, \beta^2_\depth) - \Tc{1}(\tilde\occ^{c,1}_\depth, \beta^2_\depth) }_1
    \Big)
    \bigg] \\
    & \leq \max_{\beta^1_\depth} \min_{\langle \tilde\occ^{c,1}_\depth, \beta^2_\depth, \langle \upb\nu^2_{\depth+1},\tree_{\depth+1:}^2 \rangle \rangle \in \upb{\mathcal{I}}^1_{\depth}}
    \beta^1_\depth \cdot \bigg[
      r(\occ_\depth,\cdot,\beta^2_\depth)
      \\
      & \qquad + \gamma \Tm{1}(\occ_\depth,\cdot,\beta^2_\depth) % \\
      % & \qquad
    \cdot \Big(
    \upb\nu^2_{\depth+1}
    + \lt{\depth+1} \vnorm{ \Tc{1}(\occ^{c,1}_\depth, \beta^2_\depth) - \Tc{1}(\tilde\occ^{c,1}_\depth, \beta^2_\depth) }_1
    \Big)
    \bigg] \\
    & = \max_{\beta^1_\depth} \upbW{\depth}{}(\occ_\depth,\beta^1_\depth) \\
    & = \upbW{\depth}{}(\occ_\depth, \upb\beta^1_\depth).
  \end{align*}

  Then, this upper bound approximation can only be refined, so that,
  for any $n \in {\mathbf N}$,
  \begin{align*} 
    \forall \beta^1_\depth, \quad
    K^{(n+1)}\upbW{\depth}{}(\occ_\depth,\beta^1_\depth)
    & \leq K\upbW{\depth}{}(\occ_\depth,\beta^1_\depth), \\
    \text{thus, }
    \min_{\beta^1_\depth} K^{(n+1)}\upbW{\depth}{}(\occ_\depth,\beta^1_\depth)
    & \leq \min_{\beta^1_\depth} K\upbW{\depth}{}(\occ_\depth,\beta^1_\depth).
  \end{align*}

  The expected result thus holds for $\upbW{\depth}{}$, and symmetrically for $\lobW{\depth}{}$.
\end{proof}

\begin{lemma}
  \labelT{lemma|ComparisonVAndW}
  After updating, in order, $\upbW{\depth}{}$ and $\upb{V}_\depth$, we have %
  $$K\upb{V}_\depth(\occ_\depth) \leq \max_{\beta^1_\depth} K\upbW{\depth}{}(\occ_\depth,\beta^1_\depth).$$

  After updating, in order, $\lobW{\depth}{}$ and $\lob{V}_\depth$, we have %
  $$K\lob{V}_\depth(\occ_\depth) \geq \min_{\beta^2_\depth} K\lobW{\depth}{}(\occ_\depth,\beta^2_\depth).$$
\end{lemma}

\begin{proof}
  \label{proof|lemma|ComparisonVAndW}
  After updating $\upb{\mathcal{I}}^1_\depth$, the algorithm computes
  (\Cref{alg|zsPOSGwithLP+VWs+},
  \cref{line|computeDelta}) a new solution
  $\upb\tree^2_\depth$ of the dual LP (at $\occ^1_\depth$) and the
  associated vector $\upb\nu^2_\depth$, so that
  \begin{align*}
    \max_{\beta^1_\depth} K \upbW{\depth}{}(\occ^1_\depth, \beta^1_\depth)
    & = \occ^{m,1}_\depth \cdot \upb\nu^2_\depth.
    \intertext{This vector will feed $\upb{bagV}_\depth$ along with
      $\occ^1_\depth$, so that}
    K \upb{V}_\depth(\occ_\depth)
    & \leq \occ^{m,1}_\depth \cdot \upb\nu^2_\depth. 
    \intertext{As a consequence,}
    K \upb{V}_\depth(\occ_\depth)
    & \leq \max_{\beta^1_\depth} K \upbW{\depth}{}(\occ_\depth, \beta^1_\depth).
    \end{align*}

    The symmetric property holds for $K\lob{V}_\depth$ and $K\lobW{\depth}{}$, which concludes the proof.
\end{proof}

% \begin{theorem}
%   \zsomg-HSVI converges in finite time, \ie, $\upb{V}(b_0) - \lob{V}(b_0) \leq \epsilon$.
% \end{theorem}

\thmTermination*

% \olivier{Il manquait la définition de la propriété prouvée/retrouvée à chaque étape de la récurrence !}

\begin{proof}
  \label{proof|thm|termination}
  % Let $H$ be the horizon of the problem, \ie, the
  % maximum depth of a trajectory given the termination criterion.
  % 
  We will prove by induction from $\depth=H$ to $0$, that the
  algorithm stops expanding \os{}s at depth $\depth$ after finitely
  many iterations (/trajectories).

%  \olivier{$H$ or $H-1$ ?}
  % 
  First, by definition of horizon $H$, no \os $\occ_H$ is
  ever expanded.
  The property thus holds at $\depth=H$.

  % By induction, the time step $\depth = H-1$ is such that,
  % $\forall \occ_\depth$,
  % $\upb{V}_\depth(\occ_\depth) = \lob{V}_\depth(\occ_\depth)$ since it
  % is reduced to the bilinear function $r(\occ_\depth,\cdot,\cdot)$.
  % % 
  % Therefore, the length of any trajectory of the algorithm is, at most, $H-1$.
 
  Let us now assume that the property holds at depth $\depth+1$ after $N_{\depth+1}$ iterations.
  % 
  % Now, assume that every trajectory is, at least, of length $\depth+1$.
  % 
  % We show that the algorithm can not generate an infinite number of trajectories of length $\depth+1$.
  % 
  By contradiction, let us assume that the algorithm generates infinitely many trajectories of length $\depth+1$.
  Then, because $\Occ_\depth \times {\cal B}_\depth$ is compact, after
  some time the algorithm will have visited
  $\langle \occ_\depth, \vbeta_\depth \rangle$, then, some iterations
  later, $\langle \occ_\depth^{'}, \vbeta_\depth^{'} \rangle$, such
  that $\norm{\occ_\depth - \occ_\depth^{'} }_1 \leq \rho$.
  % $\overbrace{\norm{ \vbeta_\depth - \vbeta_\depth^{'} \norm{ \leq \rho}^\text{n'a pas d'utilité ici}$
  % 
  Let us also note the corresponding terminal \os{s} (because
  trajectories beyond iteration $N_{\depth+1}$ do not go further)
  $\occ_{\depth+1} = T(\occ_\depth, \beta_\depth)$ and
  $\occ_{\depth+1}^{'} = T(\occ_\depth^{'}, \vbeta_\depth^{'})$.
  % 
  % Without loss of generality, suppose $\occ_\depth$ was visited first.

  % By Lemma~\Cref{lemma|OMG-HSVIContraction}, %
  % $\upbW{\depth}{} (\occ_\depth, \beta^1_\depth) - \lobW{\depth} (\occ_\depth,\beta^2_\depth) \leq \gamma \thr(\depth+1)$.
  % 
  % Par Lipschitz-continuité (voir \textit{tentative de} preuve précédente), on a $\upb{W^1_\depth} (\occ_\depth^{'}, \beta_\depth^{1,'}) - \lobW{\depth}{}(\occ_\depth^{'},\beta_\depth^{2,'}) \leq \gamma \thr(\depth+1) + 2\lt{\depth} (\rho + \rho_s) = \thr(t)$.
  % 
  Now, we show that the second trajectory should not have happened, \ie, $\upb{V}(\occ_\depth^{'}) - \lob{V}(\occ_\depth^{'}) \leq \thr(\depth)$.

  % On a $\max_{\beta^1_\depth} \upb{W}(\occ_\depth^{'},\beta^1_\depth) = \upb{W}(\occ_\depth^{'},\beta_\depth^{1,'})$ et $\min_{\beta^1_\depth} \lob{W}(\occ_\depth^{'},\beta^2_\depth) = \lob{W}(\occ_\depth^{'},\beta_\depth^{2,'})$
  
  Combining the previous lemmas,
  \begin{align*}
    \upb{V}(\occ_\depth^{'})
    & \leq \upb{V}(\occ_\depth) + \lt{\depth} \norm{\occ_\depth - \occ_\depth^{'}}_1
    \qquad \qquad \text{(By Lipschitz-Continuity)}\\
    & \leq \max_{\tilde{\beta}^1_\depth} \upbW{\depth}{}(\occ_\depth,\tilde{\beta}^1_\depth) + \lt{\depth} \norm{\occ_\depth - \occ_\depth^{'}}_1
    \qquad \qquad \text{(\Cref{lemma|ComparisonVAndW})} \\
    & \leq \upbW{\depth}{}(\occ_\depth,\beta^1_\depth) + \lt{\depth} \norm{\occ_\depth - \occ_\depth^{'}}_1
    \qquad \qquad \text{(\Cref{lemma|DecreaseFunctions})} \\
    & = \upbW{\depth}{}(\occ_\depth,\beta^1_\depth) + \lt{\depth} \rho.
    \intertext{Symmetrically, we also have}
    \lob{V}(\occ_\depth^{'})
    & \geq \lobW{\depth}{}(\occ_\depth,\beta^2_\depth) - \lt{\depth} \rho.
    \intertext{Hence,}
    \upb{V}(\occ_\depth^{'}) - \lob{V}(\occ_\depth^{'})
    & \leq \left( \upbW{\depth}{} (\occ_\depth, \beta^1_\depth) + \lt{\depth} \rho \right)
    - \left( \lobW{\depth}{} (\occ_\depth,\beta^2_\depth) - \lt{\depth} \rho \right) \\
    & = \left( \upbW{\depth}{} (\occ_\depth, \beta^1_\depth) - \lobW{\depth}{} (\occ_\depth,\beta^2_\depth) \right)
    + 2 \lt{\depth} \rho  \\
    & \leq \gamma \thr(\depth+1) + 2 \lt{\depth} \rho
    \qquad \qquad \text{(\Cref{lemma|OMG-HSVIContraction})} \\
    & = \gamma \left( \gamma^{-(\depth+1)} \epsilon - \sum_{i=1}^{\depth+1} 2 \radius \l_{\depth+1-i} \gamma^{-i}  \right) %
    + 2 \lt{\depth} \rho \\
    & = \gamma^{-\depth} \epsilon - \sum_{i=1}^{\depth+1} 2 \radius \l_{\depth+1-i} \gamma^{-i+1} %
    + 2 \lt{\depth} \rho \\
    & = \gamma^{-\depth} \epsilon - \sum_{j=0}^{\depth} 2 \radius \l_{\depth-j} \gamma^{-j} %
    + 2 \lt{\depth} \rho \\
    & = \gamma^{-\depth} \epsilon - \cancel{ 2 \radius \l_{\depth-0} \gamma^{-0} } - \sum_{j=1}^{\depth} 2 \radius \l_{\depth-j} \gamma^{-j} %
    + \cancel{ 2 \lt{\depth} \rho } % \\
    % & = \gamma^{-\depth} \epsilon - \sum_{j=1}^{\depth} 2 \radius \l_{\depth-j} \gamma^{-j} % \\
    % &
    = \thr(\depth).
  \end{align*}
  Therefore, $\occ^{'}_\depth$ should not have been expanded. %the second trial should not have happened.
  This shows that the algorithm will generate only a finite number of
  trajectories of length $\depth$.
\end{proof}

% ------------------------

\subsubsection{Handling Infinite Horizons}
\label{proofLemFiniteTrials}

\lemFiniteTrials*

\begin{proof}{(detailed version)}
  \label{proof|lem|finiteTrials}
  Since $W$ is the largest possible width, any trajectory stops in the
  worst case at depth $\depth$ such that
  \begin{align*}
    %& \phantom{\Leftrightarrow} \quad
    \thr(\depth) & < \WUL \\
    %& \Leftrightarrow \quad
    \gamma^{-\depth}\epsilon - 2 \radius \l^\infty \frac{\gamma^{-\depth}-1}{1-\gamma}
    & < \WUL 
    & \text{(from \Cshref{eq|thr})} \\
    %& \Leftrightarrow \quad
    \gamma^{-\depth} \epsilon - 2 \radius \l^\infty \frac{\gamma^{-\depth}}{1-\gamma} 
    - 2 \radius \l^\infty \frac{-1}{1-\gamma}
    & < \WUL \\
    %& \Leftrightarrow \quad
    \gamma^{-\depth} \underbrace{\left(\epsilon - \frac{2 \radius \l^\infty}{1-\gamma} \right)}_{>0 \quad \text{(\Cshref{lem|MaxRadius})}}
    & < \WUL - \frac{2 \radius \l^\infty }{1-\gamma}  \\
    %& \Leftrightarrow \quad
    \gamma^{-\depth} 
    & < \frac{
      \WUL - \frac{2 \radius \l^\infty }{1-\gamma}
    }{
      \epsilon - \frac{2 \radius \l^\infty}{1-\gamma}
    }\\
    %& \Leftrightarrow \quad
    \exp(-\depth\ln(\gamma)) 
    & < \exp\left(\ln\left(\frac{
          \WUL - \frac{2 \radius \l^\infty }{1-\gamma}
        }{
          \epsilon - \frac{2 \radius \l^\infty}{1-\gamma}
        }\right)\right)\\
    %& \Leftrightarrow \quad
    -\depth\ln(\gamma)
    & < \ln\left(\frac{
          \WUL - \frac{2 \radius \l^\infty }{1-\gamma}
        }{
          \epsilon - \frac{2 \radius \l^\infty}{1-\gamma}
        }\right)\\
    %& \Leftrightarrow \quad
    \depth\ln(\gamma)
    & > \ln\left(\frac{
          \epsilon - \frac{2 \radius \l^\infty}{1-\gamma}
        }{
          \WUL - \frac{2 \radius \l^\infty }{1-\gamma}
        }\right)\\
    %& \Leftrightarrow \quad
    \depth 
    & <
    \log_{\gamma}\left(\frac{
        \epsilon - \frac{2 \radius \l^\infty}{1-\gamma}
      }{
        \WUL - \frac{2 \radius \l^\infty }{1-\gamma}
      }\right).
    % 
    %\qedhere
  \end{align*}
\end{proof}

Even if the problem horizon is infinite, trajectories will thus have bounded length.
Then, everything beyond this {\em effective} horizon will rely on the
upper- and lower-bound initializations and the corresponding
strategies.

\poubelle{
\subsection{Pruning \texorpdfstring{$\upb{V}_\depth$}{upbV}}
%\label{app|pruningV}
%\label{lem|pruningV}
%Here, I try to prove the theorem about pruning in the article : 

The following key theorem allows reusing usual POMDP $\max$-planes
pruning techniques in our setting (reverting them to handle
$\min$-planes upper bound approximations).

%\lemPruningV*
\begin{restatable}[Proof in \Cref{lem|pruningV}]{theorem}{lemPruningV}
  \labelT{lem|pruningV}
  % \IfAppendix{{\em (originally stated on page~\pageref{lem|pruningV})}}{}
  % \label{theo|Pruning}
  % 
  Let $P$ be a $\min$-planes pruning operator (inverse of $\max$-planes pruning for POMDPs), and
  $\upb\nu^2_{[\occ_\depth^{c,1}, \cdot]}$.
  If $P$ correctly identifies $\upb\nu^2_{[\occ_\depth^{c,1}, \cdot]}$ as non-dominated (or
  resp. dominated) under fixed $\occ^{c,1}_\depth$, then
  $\upb\nu^2_{[\occ_\depth^{c,1}, \cdot]}$ is non-dominated
  (or resp. dominated) in $\Occ_\depth$.
\end{restatable}

\begin{proof}
  %\label{proof|lem|pruningV}
  We will demonstrate that:
  \begin{itemize}
  \item if $P$ shows that a vector $\nu^2_\depth$ (associated to
    $\occ_\depth$) is dominated {\em under fixed $\occ_\depth^{c,1}$}
    by a $\min$-planes upper bound relying only on other vectors
    $\tilde \nu^2_\depth$, then this vector is dominated in the whole
    space $\Occ$;
  \item else, the vector $\nu^2_\depth$ is useful at least around
    $\xi_\depth = (\xi_\depth^{m,1}, \occ_\depth^{c,1})$, where
    $\xi_\depth^{m,1}$ is the domination point returned by $P$.
  \end{itemize}
Note: The following is simply showing that, if the linear part is dominated by a $\min$-planes approximation for a given conditional term $\occ_\depth^{c,1}$, then the Lipschitz generalization in the space of conditional terms is also dominated since $\l$ is constant.

Given a matrix $M=(m_{i,j})$, let $\vnorm{M}_1$ denote the column vector whose $i$th component is $\norm{m_{i,\cdot}}_1$.
Here, such matrices will correspond to conditional terms,
$\vnorm{ \occ_\depth^{c,1} - \tilde \occ_\depth^{c,1} }_1$ denoting
the vector whose component for \aoh{} $\theta^1_\depth$ is
$\norm{ \occ_\depth^{c,1}(\cdot | \theta_{\depth}^1) - \tilde
\occ_\depth^{c,1}( \cdot | \theta_{\depth}^1) }_1$ (where $\occ_\depth^{c,1}( \cdot | \theta_{\depth}^1)$ may also be denoted $\occ_\depth^{c,1}(\theta_{\depth}^1)$ for brevity).

Let us assume that the vector $\nu^2_\depth$ (associated to $\occ_\depth^{c,1}$) is dominated under $\occ_\depth^{c,1}$, \ie, $\forall \xi_\depth^{m,1}$,
\begin{align*}
  (\xi_\depth^{m,1})^\top \cdot ( \nu^2_\depth + \lt{\depth} \overbrace{\vnorm{ \occ_\depth^{c,1} - \occ_\depth^{c,1} }_1}^\text{$0$})
  & \geq \min_{\tilde \nu^2_\depth, \tilde \occ_\depth^{c,1}}  \left[ (\xi_\depth^{m,1})^{\top} \cdot ( \tilde \nu^2_\depth + \lt{\depth} \vnorm{ \occ_\depth^{c,1} - \tilde \occ_\depth^{c,1} }_1 ) \right] .
  \intertext{We will show that, $\forall \xi_\depth = (\xi_\depth^{m,1},\xi_\depth^{c,1})$,}
  (\xi_\depth^{m,1})^{\top} \cdot ( \nu^2_\depth + \lt{\depth} \vnorm{ \xi_\depth^{c,1} - \occ_\depth^{c,1} }_1 )
  & \geq \min_{\tilde \nu^2_\depth, \tilde \occ_\depth^{c,1}}  \left[ (\xi_\depth^{m,1})^{\top} \cdot ( \tilde \nu^2_\depth + \lt{\depth} \vnorm{ \xi_\depth^{c,1} - \tilde \occ_\depth^{c,1} }_1 ) \right].
  \intertext{Let $\xi_\depth$ be an occupancy state. %
    First, remark that $\exists \langle \tilde \nu^2_\depth, \tilde \occ_\depth^{c,1} \rangle$  such that}
  % \aurelien{changer les $\nu_{\occ_\depth}$ : triplets,\dots.}}
  (\xi_\depth^{m,1})^{\top} \cdot ( \nu^2_\depth + \lt{\depth} \overbrace{\vnorm{ \occ_\depth^{c,1} - \occ_\depth^{c,1} }_1 }^\text{$0$})
  & \geq (\xi_\depth^{m,1})^{\top} \cdot ( \tilde \nu^2_\depth + \lt{\depth} \vnorm{ \occ_\depth^{c,1} - \tilde \occ_\depth^{c,1} }_1 ).
\end{align*}

For the sake of clarity, let us introduce the following functions (where $\xx$, $\yy$, and $\zz$ will denote conditional terms for player $1$):
\begin{align}
  % &\text{\aurelien{$g(x)
      %       \eqdef \sum_{\theta^1_\depth} \xi_\depth^{m,1}(\theta^1_\depth) (\nu_{\yy}(\theta^1_\depth) + \lt{\depth} \norm{ \yy(\theta^1_\depth) - \xx(\theta^1_\depth)}_1)$}} \\
  g(x)
  & \eqdef \sum_{\theta^1_\depth} \xi_\depth^{m,1}(\theta^1_\depth) \cdot (\nu^2_{\yy}(\theta^1_\depth) + \lt{\depth} \norm{ \yy(\theta^1_\depth) - \xx(\theta^1_\depth)}_1)
  \nonumber \\
  & = g(\yy) + \lt{\depth} (\xi_\depth^{m,1})^{\top} \cdot \vnorm{\yy - \xx}_1, %
  \intertext{and}
  h(\xx) 
  & \eqdef \sum_{\theta^1_\depth} \xi_\depth^{m,1}(\theta^1_\depth) \cdot (\nu^2_{\zz}(\theta^1_\depth) + \lt{\depth} \norm{\zz(\theta^1_\depth) - \xx(\theta^1_\depth)}_1)
  \nonumber \\
  & = h(\zz) + \lt{\depth} (\xi_\depth^{m,1})^{\top} \cdot \vnorm{\zz - \xx}_1.
  \nonumber
  % \end{align*}
  \intertext{Let us assume that $g(\yy) \geq h(\yy)$, and show that $g \geq h$. %
    First,}
  % 
  % \begin{align*}
  g(\xx) &= g(\yy) + \lt{\depth} (\xi_\depth^{m,1})^{\top} \cdot \vnorm{\xx-\yy}_1
  \nonumber \\
  & \geq h(\yy) + \lt{\depth} (\xi_\depth^{m,1})^{\top} \cdot \vnorm{\xx-\yy}_1
  \nonumber \\
  & = h(\zz) + \lt{\depth} (\xi_\depth^{m,1})^{\top} \cdot ( \vnorm{\yy-\zz}_1 + \vnorm{\xx-\yy}_1 )
  \nonumber \\
  &\geq h(\zz) + \lt{\depth} (\xi_\depth^{m,1})^{\top} \cdot \left( \vnorm{\yy-\zz}_1 + \abs{ \vnorm{\xx-\zz}_1 - \vnorm{\zz-\yy}_1  } \right).
  \label{eq|azerty}
\end{align}
Now, $\forall \theta^1_\depth$, if $\norm{\xx(\theta^1_\depth) - \zz(\theta^1_\depth)}_1 - \norm{\zz(\theta^1_\depth) - \yy(\theta^1_\depth)}_1 \geq 0$, then
\begin{align}
  & \norm{\yy(\theta^1_\depth)  - \zz(\theta^1_\depth) }_1 + \left|\norm{\xx(\theta^1_\depth) - \zz(\theta^1_\depth)}_1 - \norm{\zz(\theta^1_\depth) - \yy(\theta^1_\depth)}_1 \right|
  \nonumber \\
  & =  \cancel{\norm{\yy(\theta^1_\depth) - \zz(\theta^1_\depth)}_1} + \norm{\xx(\theta^1_\depth) - \zz(\theta^1_\depth) }_1 - \cancel{\norm{\zz(\theta^1_\depth) - \yy(\theta^1_\depth) }_1}
  \nonumber \\
  & = \norm{\xx(\theta^1_\depth) - \zz(\theta^1_\depth)}_1,
  \label{eq|uiop}
  \intertext{else,}
  & \norm{\yy(\theta^1_\depth)  - \zz(\theta^1_\depth) }_1 + \left|\norm{\xx(\theta^1_\depth) - \zz(\theta^1_\depth)}_1 - \norm{\zz(\theta^1_\depth) - \yy(\theta^1_\depth)}_1 \right| \\
  & =  2 \norm{\yy(\theta^1_\depth) - \zz(\theta^1_\depth)}_1 - \norm{\xx(\theta^1_\depth) - \zz(\theta^1_\depth)}_1
  \nonumber \\
  & \geq  \norm{\xx(\theta^1_\depth) - \zz(\theta^1_\depth)}_1.
  \label{eq|qsdf}
  \intertext{Finally, coming back to (\Cref{eq|azerty}):} % \aurelien{les deux lignes d'en dessous sont moches, (137) vient de bien avant et (138) est un résultat \dots.}}
  g(\xx)
  & \geq h(\zz) + \lt{\depth} (\xi_\depth^{m,1})^{\top} \cdot \left( \vnorm{\yy-\zz}_1 + \abs{ \vnorm{\xx-\zz}_1 - \vnorm{\zz-\yy}_1 } \right)
  \nonumber \\
  & \geq h(\zz) + \lt{\depth} (\xi_\depth^{m,1})^{\top} \cdot \norm{\xx(\theta^1_\depth) - \zz(\theta^1_\depth)}_1
  \qquad \text{(from (\Cref{eq|uiop}+\Cref{eq|qsdf}))}
  \nonumber \\
  & \geq h(\xx).
  \nonumber
%\end{align}
%\aurelien{Plus vraiment $\dots$ là l'écriture de $g$ et $h$ est très proche d'être ce qui est écrit en dessous \dots} \\
\intertext{With $x=\xi_\depth^{c,1}$, $y=\occ_\depth^{c,1}$ and $z=\tilde\occ_\depth^{c,1}$, this gives:}
% \begin{align}
  g(\xi_\depth^{c,1}) %
  & = \sum_{\theta^1_\depth} \xi_\depth^{m,1}(\theta^1_\depth) (\nu^2_\depth(\theta^1_\depth) + \lt{\depth} \norm{ \occ_\depth^{c,1}(\theta^1_\depth) - \xi_\depth^{c,1}(\theta^1_\depth)}_1) \\
  \geq h(\xi_\depth^{c,1}) %
  & = \sum_{\theta^1_\depth} \xi_\depth^{m,1}(\theta^1_\depth) (\tilde \nu^2_\depth(\theta^1_\depth) + \lt{\depth} \norm{\tilde \occ_\depth^{c,1}(\theta^1_\depth) - \xi_\depth^{c,1}(\theta^1_\depth)}_1).
\end{align}
%
% We apply this general result to our special case where $g(\xi_\depth^{c,1}) = \sum_{\theta^1_\depth} \xi_\depth^{m,1}(\theta^1_\depth) (\nu_{\occ_\depth^{c,1}}(\theta^1_\depth) + \lt{\depth} \norm{ \occ_\depth^{c,1}(\theta^1_\depth) - \xi_\depth^{c,1}(\theta^1_\depth)}_1)$ and $h(\xi_\depth^{c,1}) = \sum_{\theta^1_\depth} \xi_\depth^{m,1}(\theta^1_\depth) (\tilde \nu_{\tilde \occ_\depth^{c,1}}(\theta^1_\depth) + \lt{\depth} \norm{\tilde \occ_\depth^{c,1}(\theta^1_\depth) - \xi_\depth^{c,1}(\theta^1_\depth)}_1)$.
%
This shows that $\nu^2_\depth$ is dominated for every $(\xi_\depth^{m,1},\xi_\depth^{c,1})$, where both  $\xi_\depth^{m,1}$ and $\xi_\depth^{c,1}$ are arbitrary.
Therefore, one can prune a vector $\nu^2_\depth$ using $P$ applied in the space where $\occ_\depth^{c,1}$ is fixed.

% Moreover, it is straightforward to show that if $P$ has some properties, then those properties are preserved as explained below.
% \begin{itemize}
% \item If $P$ misses some vectors, so does the zsPOSG pruning.
% \item If $P$ finds every vectors, so does the zsPOSG pruning.
% \item The zsPOSG pruning does not prune useful vectors if $P$ does not (no false positives).
% \end{itemize}

As a consequence, some properties of $P$ are preserved in its extension to zsPOSGs:
\begin{itemize}
\item If $P$ correctly identifies $\nu^2_\depth$ as non-dominated at
  $\occ^{c,1}_\depth$, then
  $\langle \nu^2_\depth, \occ^{c,1}_\depth \rangle$ is
  non-dominated in $\Occ_\depth$. \\
  That is, if $P$ does not induce false negatives, neither does its
  extension to zsPOSGs.
\item If $P$ correctly identifies $\nu^2_\depth$ as dominated at
  $\occ^{c,1}_\depth$, then
  $\langle \nu^2_\depth, \occ^{c,1}_\depth \rangle$ is dominated
  in $\Occ_\depth$. \\
  That is, if $P$ does not induce false positives, neither does its
  extension to zsPOSGs.
\end{itemize}
\end{proof}

}

\poubelle{
\section{Illustration of safety concerns}
\label{app|zsPOSGvszsoMG}

In this section, we detail two main types of issues when simply concatenating decision rules found by greedy selection at each time step.

\subsection{Issues due to Incomplete Strategies}

The first issue arises when some strategies have a $0$ probability for some action, which could makes some histories unreachable and as a consequence, the next decision rule would be incomplete.

\subsubsection{Problem Definition}

Here, a deterministic zero-sum problem (illustrated in \Cref{fig:Automate1}) 
%(which is, a fortiori, a zero-sum Partially Observable Stochastic Game)
is introduced to illustrate the concerns raised by using a temporal decomposition to solve a zs-POSGs. It is defined as a tuple $\langle \cS, \cA^1, \cA^2, \cZ^1, \cZ^2, P, r, H, \gamma, b_0 \rangle$ where:
\begin{itemize}
  \item $\cS = \{ s_0, s_1, s_2, s_3\}$;
  \item $\cA^1 = \cA^2 = \{a_0,a_1\}$;
  \item $\cZ^1 = \cZ^2 = \{z_{a_0},z_{a_1},none\}$;
  \item $\PP{s}{\va}{s'}{\vz}=T(s,\va,s')\cdot \mathcal{O}(\va,s',\vz)$, using the next two definitions;
  \item $T$ is deterministic and such that
        \begin{itemize}
          \item $T(s_0, a_0, a_0) = s_0$ (If P1 and P2 play $(a_0,a_0)$, the system stays in $s_0$.),
          \item $T(s_0, a_1, a_0) = s_1$ (If P1 doesn't play $a_0$, the system evolves into $s_1$.),
          \item $T(s_0, a_0, a_1) = s_2$ (If P2 doesn't play $a_0$, the system evolves into $s_2$.),
          \item $T(s_0, a_1, a_1) = s_3$ (If no one plays $a_0$, the systems evolves into $s_3$.),
          \item $T(s_1, \cdot, \cdot) = s1$ (sink state),
          \item $T(s_2, \cdot, \cdot) = s2$ (sink state),
          \item $T(s_3, \cdot, \cdot) = s3$ (sink state);
        \end{itemize}
  \item $\mathcal{O}$ is deterministic and such that
        \begin{itemize}
          \item $\forall (x,y) \in \cA^1 \times \cA^2,\ \forall s \in \{s_1,s_2\},\ \mathcal{O}(s,x,y) = (z_y,z_x)$ (both players observe their opponent's past action),
          %\item $\forall (x,y) \in \cA^1 \times \cA^2,\ \mathcal{O}(s_2,x,y) = (z_y,z_x)$ (both players observe their opponent's past action)
          \item $\forall (x,y) \in \cA^1 \times \cA^2,\ \mathcal{O}(s_0,x,y) = \mathcal{O}(s_3,x,y) = none$ (both players receive a trivial observation);
        \end{itemize}
  \item $r$ is such that ($\cdot$ is used to denote "for all")
        \begin{itemize}
          \item $r(s_0,\cdot,\cdot) = 0$,
          \item $r(s_1,\cdot,a_0) = -1$ and $r(s_1,\cdot,a_1) = +100$,
          \item $r(s_2,a_0,\cdot) = +1$ and $r(s_2,a_1,\cdot) = -100$,
          \item $r(s_3,\cdot,\cdot) = 0$;
        \end{itemize}
    \item $H=2$;
    \item $\gamma = 1$;
    \item $b_0 = 1.0 \cdot \tree_{s_0}$ is a Dirac distribution: the system surely begins at state $s_0$. 
\end{itemize}

%\olivier{Ci-dessous: C'est quoi cet horrible flottant écrit au milieu d'un paragraphe !}

%Then, a Nash Equilibrium for this game is $\beta_{0:1}^1 = \beta_{0:1}^2 = \{\emptyset = a_0 \oplus (a_0,z_0) = a_0)\}$/
%and will be found by the (naive) algorithm \ref{NaiveAlgorithm}
%where the argument maximum can be found by a Lipschitz bi-level Optimization for example.
%This strategy profile is indeed a NEs because the value of the game is $0$ (symetric game) and $V_0(b_0, \beta_{0:1}^1, a_1) = +1 > 0$ (player $1$ plays $a_0$ in $s_2$). Therefore, player $2$ has no interest in changing her strategy since player $1$ will observe the deviation (in the \zsomg) and is able to adapt her strategy.

\temporallyHidden{
\olivier{\begin{itemize}
    \item Je commencerais bien par l'explication de haut-niveau: "Si un joueur concatène simplement des règles de décision obtenues l'une après l'autre, on risque d'obtenir des stratégies incomplètes susceptibles d'être exploitées par l'adversaire."
    \item L'incomplétude des stratégies n'est qu'un problème. Celui que j'avais en tête (et décrit dans un autre papier) est qu'une telle concaténation ne permet pas de "bluffer en retard" (comme dans le matching pennies séquentialisé).
    \item Je suis embêté parce que cet exemple ne permet pas de montrer comment HSVI résoud cet autre type de problème...
\end{itemize}
}
}

%\olivier{Tentative pour expliquer plus clairement les choses:
%\begin{enumerate}
%    \item Pour chaque joueur, la seule règle de décision optimale au pas de temps 0 consiste à jouer $a_0$.
%    \item Étant donné ces règles de décision, chaque joueur va nécessairement voir $z_{a_0}$ comme première observation.
%    \item Depuis l'état d'occupation (déterministe) obtenu, la règle de décision optimale consiste, pour chaque joueur, à jouer $a_0$ quand il est face à l'historique $(a_0,z_{a_0})$, et ne précise rien pour les autres historiques (improbables).
%    \item ...
%\end{enumerate}
%}

\subsubsection{Unsafe Solution}
%\aurelien{je n'aime pas du tout le titre, je vais le changer}
The best way for both players to act at time step $0$ is always $a_0$, whatever her opponent plays.
At time step $1$, considering both players played $a_0$ before, then players should both play $a_0$. 
Let us note this (incomplete) strategy for player $i$ $\tilde \beta_{0:1}^i$.
%
% = (a_0 \oplus (a_0,z_{a_0}) \to a_0)$.
%
Player $2$ has no interest in deviating at time step $0$ from $a_0$ to $a_1$ because it would result in a payoff $+1$ since player $1$ could respond by $a_1$. However, this response for player $1$ to player $2$ deviating from $a_0$ was never considered when building $\tilde \beta_{0:1}^1$ because $2$ should not have interest in doing so. Failing at defining what to do when observing $z_{a_1}$ could be arbitrarily bad for $1$ (for example, playing randomly $a_0$ or $a_1$ with probability $0.5$ each would result in a payoff $-49$).

\begin{figure}[ht]
\begin{center}
\begin{tikzpicture}[->,>=stealth',shorten >=1pt,auto,node distance=2.8cm,
                    semithick]
  \tikzstyle{every state}=[fill=red,draw=none,text=white]

  \node[initial,state] (A)                    {$s_0$};
  \node[state]         (B) [above right of=A] {$s_1$};
  \node[state]         (D) [below right of=A] {$s_3$};
  \node[state]         (C) [below right of=B] {$s_2$};

  \path (A) edge [loop below] node[xshift=-0.5cm,yshift=-0.0cm] {$\begin{array}{c}
       (a_0,a_0)  \\
       (none,none)
  \end{array}$} (A)
            edge              node {$(a_0,a_1)$; $(z_{a_0},z_{a_1})$} (B)
            edge              node {$(a_1,a_0)$; $(z_{a_1},z_{a_0})$} (C)
            edge              node[xshift=-0.1cm,yshift=-0.5cm] {$\begin{array}{c}
       (a_1,a_1)  \\
       (none,none)
  \end{array}$} (D)
        %(B) edge [loop above] node {$(\cdot,\cdot)$} (B)
        %(C) edge [loop above] node {$(\cdot,\cdot)$} (C)
        %(D) edge [loop below] node {$(\cdot,\cdot)$} (D);
        ;
\end{tikzpicture}
\end{center}
\caption{Automaton representing a simply zs-POSG possibly inducing incomplete strategies}
\label{fig:Automate1}
\end{figure}

%Finally, we showed that a \zsomg is not equivalent (in general) to the zs-POSG it is derived from and finding a NES to the \zsomg does not necessarily help to solve the zs-POSG.
This shows that introducing a temporal decomposition of zs-POSGs comes with issues when trying to concatenate the strategies for the different time steps. Player $i$ must take into account every possible strategy of her opponent and not only those who should be interesting for $\neg i$.

\subsubsection{HSVI's Solution}

However, HSVI builds "safe" strategies, as solutions of a dual linear program, which comes with guarantees against best-responses of the opponent. At time step $0$, the dual LP is
\begin{align}
  %\label{eq|LP}
  %\label{eq|DLP}
  & \begin{array}{l@{\ }l@{\ }llll@{ }r@{\ }l}
      \min_{\tree_{0}^2,v}
      & v
      & \text{ s.t.} 
      & \text{(i) }
        \forall a^1, %
       v \geq  M^{\occ_0}_{(a^1,\cdot)} \cdot \tree^2_{0}
      \\ &&& \text{(ii) }
       {\displaystyle \sum_{w \in \upb{\cI}_0}} \! \tree^2_{0}( w )
       = 1,
    \end{array}
\end{align}
where it appears that every action $a^1$ of player $1$ is considered to constrain $\tree_{0}^2$.
%\aurelien{expliquer comment lire les stratégies qu'on écrit}
It results in the $\epsilon=0.001$-\nes{} (the strategy for player $1$ is the same at the problem is symetric):
%{0={none={a1=0,0000038147, a0=0,9999961853, }}, 1={(a1,z0)={a0=1, }, (a0,z1)={a1=1, }, (a0,none)={a0=1, }, (a1,none)={a0=1, }}}
\begin{itemize}
    \item $t=0 :$
    \begin{itemize}
    \item  $\beta^2_0(a_1 \mid \emptyset)=0.000004$
    \item $\beta^2_0(a_0 \mid \emptyset)=0.999996$ 
    \end{itemize}
    %\item $t=0 : \{\beta_0(a1 \mid \emptyset)=0.000004, \beta_0(a0 \mid \emptyset)=a0=0.999996\}$
    \item $t=1 :$
        \begin{itemize}
        \item $\beta^2_1(a_1 \mid (a_1,z_{a_0}))=1$
        \item $\beta^2_1(a_1 \mid (a_0,z_{a_1}))=1$
        \item $\beta^2_1(a_0 \mid (a_0,none))=1$
        \item $\beta^2_1(a_0 \mid (a_1,none))=1$ 
        \end{itemize}
\end{itemize}
obtained  by using (at the last iteration of HSVI) the following distribution $\tree_{0}^2$ 
%weeight : 0.5
%object : ({(si,((-),(-)))=1, } list compressed histories : {} and {},({none={a1=0, a0=1, }},Vector: {(a0,none)=0.0} -- strat: {1={(a0,none)={a1=0, a0=1, }}} realization weights : {((a0,none),(a1,-))=0.0, ((a0,none),(a0,-))=1.0} s : {(si,((a0,none),(a0,none)))=1, }list compressed histories : {} and {}))
%weeight : 0.5
%object : ({(si,((-),(-)))=1, } list compressed histories : {} and {},({none={a1=0,0000076294, a0=0,9999923706, }},Vector: {(a1,z0)=900.0, (a0,z1)=-900.0, (a0,none)=0.0, (a1,none)=0.0} -- strat: {1={(a1,z0)={a1=0, a0=1, }, (a0,z1)={a1=1, a0=0, }, (a0,none)={a1=0, a0=1, }, (a1,none)={a1=0, a0=1, }}} realization weights : {((a0,none),(a1,-))=0.0, ((a0,z1),(a1,-))=1.0, ((a0,none),(a0,-))=1.0, ((a1,none),(a1,-))=0.0, ((a0,z1),(a0,-))=0.0, ((a1,none),(a0,-))=1.0, ((a1,z0),(a1,-))=0.0, ((a1,z0),(a0,-))=1.0} s : {(s2,((a0,z1),(a1,z0)))=0,0000076293, (s1,((a1,z0),(a0,z1)))=0,0000076293, (sa1,((a1,none),(a1,none)))=0,0000000001, (si,((a0,none),(a0,none)))=0,9999847413, }list compressed histories : {} and {}))
\begin{itemize}
    \item weight : 0.5 for one vector $\upb{\nu}^2_{[b_0,\beta_{0:}^2]}$ coming along with strategy :
    \begin{itemize}
        \item $t=0 :$
        \begin{itemize}
        \item $\beta^2_0(a_1 \mid \emptyset)=0$
        \item $\beta^2_0(a_0 \mid \emptyset)=1$
        \end{itemize}
        \item $t=1 :$
        \begin{itemize}
        \item $\beta^2_1(a_0 \mid (a_0,none))=1$
        \end{itemize}
    \end{itemize}
    \item weight : 0.5 for another vector $\upb{\nu}^2_{[b_0,\beta_{0:}^2]}$ coming along with strategy :
    \begin{itemize}
        \item $t=0 :$
        \begin{itemize}
        \item $\beta^2_0(a_1 \mid \emptyset)=0.000008$
        \item $\beta^2_0(a_0 \mid \emptyset)=0.999992$
        \end{itemize}
        %\item $t=0 : \{\beta_0(a1 \mid \emptyset)=0.000008, \beta_0(a0 \mid \emptyset)=a0=0.999992\}$
        %\item $t=0 : \{a1=0.000008, a0=0.999992\}$
        \item $t=1 :$
        \begin{itemize}
        \item $\beta^2_1(a_0 \mid (a_1,z_{a_0}))=1$
        \item $\beta^2_1(a_1 \mid (a_0,z_{a_1}))=1$
        \item $\beta^2_1(a_0 \mid (a_0,none))=1$
        \item $\beta^2_1(a_0 \mid (a_1,none))=1$ \end{itemize}
    \end{itemize}
\end{itemize}
This is actually very interesting as the first vector does not specify anything to do when observing $z_{a^1}$, so that the LP solution needs another vector that does define the best way to act for this specific history.

\subsection{Safety Issues}

Apart from the incompleteness issues, concatenating decision rules does not come with any form of guarantees for the built strategy as it is obtained using past strategies' knowledge which, at execution, are unknown.

\subsubsection{Problem Definition: Non-symmetric Sequential Matching Pennies}
\label{sec|MatchingPenniesDef}

%\aurelien{pour aurélien : vérifier la dynamique (typos,...)}

Let us here introduce a slight modification of the well-known Matching Pennies game.
 It is defined as a tuple $\langle \cS, \cA^1, \cA^2, \cZ^1, \cZ^2, P, r, H, \gamma, b_0 \rangle$ where:
\begin{itemize}
  \item $\cS = \{si, sh, st\}$;
  \item $\cA^1 = \cA^2 = \{ah,at\}$;
  \item $\cZ^1 = \cZ^2 = \{none\}$;
  \item $\PP{s}{\va}{s'}{\vz}=T(s,\va,s')\cdot \mathcal{O}(\va,s',\vz)$, using the next two definitions;
  \item $T$ is deterministic and such that
        \begin{itemize}
            \item $T(si,ah,\cdot) = sh$,
            \item $T(si,at,\cdot) = st$,
            \item $\forall s \in \{st,sh\},\ T(s,\cdot,at) = st,\  T(s,\cdot,ah) = sh$;
        \end{itemize}
  \item $\mathcal{O}$ is deterministic and always return the trivial observation "$none$";
  \item $r$ is such that ($\cdot$ is used to denote "for all")
        \begin{itemize}
        \item $r(si,\cdot,\cdot) = 0$,
        \item $r(st,\cdot,at)=1$,
        \item $r(st,\cdot,ah)=-1$,
        \item $r(sh,\cdot,at)=-1$,
        \item $r(sh,\cdot,ah)=2$;
        \end{itemize}
    \item $H=2$;
    \item $\gamma = 1$;
    \item $b_0 = 1.0 \cdot \tree_{si}$ is a Dirac distribution: the systems surely begins in state $si$.
\end{itemize}

Note that, in this game, players artificially play one after another. Player $1$'s \dr{} at time step $1$ is irrelevant and so is player $2$'s one at time step $0$. The ability for player $2$ to build strategies $\beta_{0:}^2$ that are safe against every possible strategy of player $1$ at time step $0$ will be the key point to illustrate the safety issues.
Also, the symmetry of the original payoff matrix is broken, which helps validate the resulting solution.

\subsubsection{The sequentialized Matching Pennies safety issues}

In this sequentialized (non-symmetric) Matching Pennies game, a strategy profile is a NES if and only if it verifies:
\begin{itemize}
\item at $t=0$, $P1$'s \dr{} $\beta^1_0$ plays $ah$ with probability $0.4$, and $at$ with probability $0.6$; and
\item at $t=1$, $P2$'s \dr{} $\beta^2_1$ plays, whatever the \aoh{} $\theta_1^2$, $ah$ with probability $0.6$ and $at$ with probability $0.4$.
\end{itemize}
The \nev{} of this game is $0.2$.
Yet, using this $\beta^1_0$ (and any \dr{} for $P2$ at $t=0$) results in a subgame for which any (prefix) strategy profile $\tilde{\vbeta}_2$ is a (subgame) Nash equilibrium with payoff $0.2$.
The strategy profile $\vbeta_0 \oplus \tilde{\vbeta}_1$ obtained by concatenation is thus generally no NES profile, in particular with the deterministic \dr{} for $P2$ returned by our implementation at $t=1$ (see below). 
%\aurelien{For example, $\tilde{\vbeta}_1$ being equal to $ah$, $\vbeta_0 \oplus \tilde{\vbeta}_1$ would not be a \nes{} as there exists a strategy for player $1$ (playing $at$ at timestep $0$ and anything at timestep $1$) that returns a payoff $-1.0$ which is strictly inferior to the \nev{}.}
%\aurelien{donner lpus de détails ici}
Indeed, if player~2 uses $\beta^2_0 \oplus \tilde\beta^2_1$, with $\tilde\beta^2_1$ always returning action $at$, then player~1 gets a payoff $+1$ larger than the \nev{} by deterministically playing $at$ at time step $0$.
%\aurelien{je trouve dommage de le pas aller dans le détail sur cet exemple simple. J'aime bien votre explication, j'aurais simplement tendance à y rajouter les détails}
%\aurelien{pour le coup, pour l'instant j'aime moins cette conte-proposition}
%\aurelien{à discuter}
%\aurelien{si P1 sait que P2 ne joue pas un eq de nash, il va changer}
%\aurelien{ou bien revenir à la def d'un eq de nash, il y a une autre stratégie qui garantit plus au P1 que la valeur du NEV du pdv de P2}
%\aurelien{rajouter la valeur du $\nev$}

\subsubsection{HSVI's Solution}
%Strategy returned by HSVI (which is a \nes{}: 
In contrast, HSVI returns "safe" strategies once converged. These strategies are, for player $1$:
%P1: ${0={none={at=0,6, ah=0,4, }}, 1={(at,none)={ah=1, }, (ah,none)={ah=1, }}}$
\begin{itemize}
    \item $t=0:$
    \begin{itemize}
    \item  $\beta^1_0(ah \mid \emptyset)=0.4$
    \item $\beta^1_0(at \mid \emptyset)=0.6$ 
    \end{itemize}
    %\item $t=0: \{\beta_0(a1 \mid \emptyset)=0.000004, \beta_0(a0 \mid \emptyset)=a0=0.999996\}$
    \item $t=1:$
        \begin{itemize}
        %\item $\beta^2_1(at \mid (at,none)=0$
        %\item $\beta^2_1(at \mid (ah,none)=0$
        \item $\beta^1_1(ah \mid (ah,none))=1.0$
        \item $\beta^1_1(ah \mid (at,none))=1.0$ 
        \end{itemize}
\end{itemize}
%
%P2: ${0={none={at=1, }}, 1={(at,none)={at=0,6, ah=0,4, }}}$
and for player $2$:
\begin{itemize}
    \item $t=0:$
    \begin{itemize}
    %\item  $\beta^2_0(ah \mid \emptyset)=0.0$
    \item $\beta^2_0(at \mid \emptyset)=1.0$ 
    \end{itemize}
    %\item $t=0: \{\beta_0(a1 \mid \emptyset)=0.000004, \beta_0(a0 \mid \emptyset)=a0=0.999996\}$
    \item $t=1:$
        \begin{itemize}
        \item $\beta^2_1(at \mid (at,none) )=0.6$
        \item $\beta^2_1(ah \mid (ah,none) )=0.4$
        \end{itemize}
\end{itemize}

They are obtained using for player $1$ by 
%and for player $1$: 
%$1.0 \cdot ({none={at=0,6, ah=0,4, }}, {1={(at,none)={at=0, ah=1, }, (ah,none)={at=0, ah=1, }}}$.
\begin{itemize}
\item weight: 1.0 for one vector $\lob{\nu}^1_{[b_0,\beta_{0:}^1]}$ coming along with strategy:
\begin{itemize}
    \item $t=0:$
    \begin{itemize}
    \item  $\beta^1_0(ah \mid \emptyset)=0.4$
    \item $\beta^1_0(at \mid \emptyset)=0.6$ 
    \end{itemize}
    %\item $t=0: \{\beta_0(a1 \mid \emptyset)=0.000004, \beta_0(a0 \mid \emptyset)=a0=0.999996\}$
    \item $t=1:$
        \begin{itemize}
        %\item $\beta^2_1(at \mid (at,none)=0$
        %\item $\beta^2_1(at \mid (ah,none)=0$
        \item $\beta^1_1(ah \mid (ah,none))=1.0$
        \item $\beta^1_1(ah \mid (at,none))=1.0$ 
        \end{itemize}
\end{itemize}
\end{itemize}
and for player $2$, $\tree_{0}^2$ is obtained by  
\begin{itemize}
\item weight: 0.4 for one vector $\upb{\nu}^2_{[b_0,\beta_{0:}^2]}$ coming along with strategy:
\begin{itemize}
    \item $t=0:$
    \begin{itemize}
    \item  $\beta^2_0(at \mid \emptyset)=1.0$
    \end{itemize}
    %\item $t=0: \{\beta_0(a1 \mid \emptyset)=0.000004, \beta_0(a0 \mid \emptyset)=a0=0.999996\}$
    \item $t=1:$
        \begin{itemize}
        \item $\beta^2_1(ah \mid (at,none) )=1.0$
    \end{itemize}
\end{itemize}
\item weight: 0.6 for one vector $\upb{\nu}^2_{[b_0,\beta_{0:}^2]}$ coming along with strategy:
\begin{itemize}
    \item $t=0:$
    \begin{itemize}
    \item  $\beta^2_0(at \mid \emptyset)=1.0$
    \end{itemize}
    \item $t=1:$
        \begin{itemize}
        \item $\beta^2_1(at \mid (at,none) )=1.0$
    \end{itemize}
\end{itemize}
\end{itemize}

As can be observed, here HSVI obtains an optimal solution strategy for $P2$ as a mixture of deterministic strategies.

}

\end{document}